\documentclass{ecsthesis}      
\usepackage{natbib}            
\usepackage{psfrag}
\usepackage{multirow}
\hypersetup{colorlinks=true}   
\newcommand{\BibTeX}{{\rm B\kern-.05em{\sc i\kern-.025em b}\kern-.08em T\kern-.1667em\lower.7ex\hbox{E}\kern-.125emX}}

\newcommand{\F}[1]{\ensuremath{\mathrm{#1}}\xspace}

\newcommand{\tr}{\F{trace}}

\newcommand{\e}{\ensuremath{\epsilon}\xspace}

\newcounter{alg}

\def\argmin{\operatornamewithlimits{arg\,min}}
\def\argmax{\operatornamewithlimits{arg\,max}}
\newcommand{\innerp}[1]{\left\langle #1 \right\rangle}
\DeclareMathOperator{\rank}{rank}

\usepackage{amsfonts}
\newcommand{\tickYes}{\textcolor{green}{\checkmark}}
\usepackage{pifont}
\newcommand{\tickNo}{\textcolor{red}{\hspace{1pt}\ding{55}}}

\usepackage{epigraph}

\usepackage{gastex}

\newenvironment{proof2}[1][Proof]{\begin{trivlist}
\item[\hskip \labelsep {\bfseries #1}]}{\end{trivlist}}
\newenvironment{proof3}[1][]{\begin{trivlist}
\item[\hskip \labelsep {\bfseries #1}]}{\hfill$\Box$\end{trivlist}}

\definecolor{gray}{gray}{0.97}
\makeatletter\newenvironment{bc_box}{%
   \begin{lrbox}{\@tempboxa}\begin{minipage}{0.9\columnwidth}}{\end{minipage}\end{lrbox}%
   {\setlength{\fboxsep}{8pt}\colorbox{gray}{\usebox{\@tempboxa}}}
}\makeatother
\newenvironment{chapabstract}{
  \begin{small}
  \begin{center}
  \begin{bc_box}\textbf{Chapter abstract}
}{
  \end{bc_box}
  \end{center}
  \end{small}
}



\usepackage{bibentry}
\makeatletter
\let\BRatbibitem\BR@bibitem
\makeatother

\usepackage{titlesec}
\titleformat{\chapter}[display]{\Large\bfseries}{\titlerule\vspace*{-12pt}\filleft\MakeUppercase{\chaptertitlename} \Huge\thechapter\vspace*{-12pt}}{0pt}{\filleft\Large\bfseries\sffamily}[\titlerule]

\titleformat{\part}[frame]{\LARGE\bfseries}{\filcenter\MakeUppercase{\partname} \Huge\thepart}{30pt}{\filcenter\Huge\bfseries\sffamily}

\usepackage[all]{hypcap} 

\usepackage{booktabs} 

\DeclareMathOperator{\sign}{sign}
\begin{document}
\begingroup
  \makeatletter
  \let\BR@bibitem\BRatbibitem
  \nobibliography*
\endgroup
\frontmatter
\title      {Supervised Metric Learning with Generalization Guarantees} 
\titleaffiche{Supervised Metric Learning \\ with Generalization Guarantees}
\author    {Aur\'elien Bellet}
\date       {11 D\'ecembre 2012}
\subject    {PhD Thesis}

\maketitle
\begin{acknowledgements}

Je tiens tout d'abord \`a remercier Pierre Dupont, Professeur \`a l'Universit\'e Catholique de Louvain, et Jose Oncina, Professeur \`a l'Universit\'e d'Alicante, d'avoir accept\'e d'\^etre les rapporteurs de mon travail de th\`ese. Leurs remarques pertinentes m'ont permis d'am\'eliorer la qualit\'e de ce manuscrit. Plus g\'en\'eralement, je remercie l'ensemble du jury, notamment R\'emi Gilleron, Professeur \`a l'Universit\'e de Lille, et Liva Ralaivola, Professeur \`a Aix-Marseille Universit\'e, qui ont tout de suite accept\'e d'\^etre examinateurs.

Je remercie chaleureusement mon directeur et mon co-directeur de th\`ese, Marc et Amaury, avec qui j'ai d\'evelopp\'e des liens professionnels et personnels qui de toute \'evidence dureront au-del\`a de cette th\`ese. Je suis particuli\`erement reconnaissant envers Marc qui, malgr\'e son attrait pour une certaine \'equipe de football, a su me convaincre de faire cette th\`ese et m'a fait confiance en acceptant un arrangement extraordinaire (dans tous les sens du terme) pour que je puisse passer ma premi\`ere ann\'ee \`a \'Edimbourg.

Je veux \'egalement saluer les coll\`egues du Laboratoire Hubert Curien et du d\'epartement d'informatique de l'UJM. En premier lieu, mon voisin de bureau et ami JP, qui fut aussi un excellent co-\'equipier Warlight. Je pense aussi aux autres doctorants (anciens et actuels) que sont Laurent, Christophe, \'Emilie, David, Fabien, Tung, Chahrazed et Mattias. Enfin, je veux mentionner les personnes rencontr\'ees dans le cadre des projets PASCAL2 et LAMPADA, notamment Emilie et Pierre du LIF de Marseille avec qui j'esp\`ere avoir l'occasion de travailler et de collaborer encore dans le futur.

D'un point de vue plus personnel, je salue \'evidemment les amis, qui sont trop nombreux pour \^etre cit\'es mais qui se reconna\^itront. La pr\'esence de certains \`a ma soutenance me fait \'enorm\'ement plaisir.
Ces remerciements ne seraient pas complets sans un mot pour Marion, qui m'a beaucoup soutenu et encourag\'e pendant ces (presque) trois ann\'ees. Elle a m\^eme essay\'e de s'int\'eresser \`a la classification lin\'eaire parcimonieuse, r\'eussissant \`a faire illusion lors d'une r\'eception \`a ECML!
Enfin, \emph{last but not least}, je d\'edie tout simplement cette th\`ese \`a mes parents, mes grands-parents et mon petit fr\`ere.

\end{acknowledgements}

\tableofcontents
\listoffigures
\listoftables
\thispagestyle{plain}

\null\vfill
``There is nothing more practical than a good theory.''

\begin{flushright}
--- James C. Maxwell
\end{flushright}

\vspace*{1cm}

``There is a theory which states that if ever anyone discovers exactly what the Universe is for and why it is here, it will instantly disappear and be replaced by something even more bizarre and inexplicable.

There is another theory which states that this has already happened.''

\begin{flushright}
--- Douglas Adams
\end{flushright}

\vfill\vfill\vfill\vfill\vfill\vfill\null
\clearpage

\mainmatter
\addtocontents{toc}{\protect\vspace{14pt}}

\chapter{Introduction}
\label{chap:intro}

The goal of machine learning is to automatically figure out how to perform tasks by generalizing from examples. A machine learning algorithm takes a data sample as input and infers a model that captures the underlying mechanism (usually assumed to be some unknown probability distribution) which generated the data. Data can consist of features vectors (e.g., the age, body mass index, blood pressure, ... of a patient) or can be structured, such as strings (e.g., text documents) or trees (e.g., XML documents). A classic setting is \emph{supervised learning}, where the algorithm has access to a set of training examples along with their labels and must learn a model that is able to accurately predict the label of future (unseen) examples. Supervised learning encompasses classification problems, where the label set is finite (for instance, predicting the label of a character in a handwriting recognition system) and regression problems, where the label set is continuous (for example, the temperature in weather forecasting). On the other hand, an \emph{unsupervised learning} algorithm has no access to the labels of the training data. A classic example is clustering, where we aim at assigning data into similar groups. The generalization ability of the learned model (i.e., its performance on unseen examples) can sometimes be guaranteed using arguments from statistical learning theory. 

Relying on the saying ``birds of a feather flock together'', many supervised and unsupervised machine learning algorithms are based on a notion of \emph{metric} (similarity or distance function) between examples, such as $k$-nearest neighbors or support vector machines in the supervised setting and $K$-Means clustering in unsupervised learning. The performance of these algorithms critically depends on the relevance of the metric to the problem at hand --- for instance, we hope that it identifies as similar the examples that share the same underlying label and as dissimilar those of different labels. Unfortunately, standard metrics (such as the Euclidean distance between feature vectors or the edit distance between strings) are often not appropriate because they fail to capture the specific nature of the problem of interest.

For this reason, a lot of effort has gone into \emph{metric learning}, the research topic devoted to automatically learning metrics from data. In this thesis, we focus on supervised metric learning, where we try to adapt the metric to the problem at hand using the information brought by a sample of labeled examples. Many of these methods aim to find the parameters of a metric so that it best satisfies a set of local constraints over the training sample, requiring for instance that pairs of examples of the same class should be similar and that those of different class should be dissimilar according to the learned metric. A large body of work has been devoted to supervised metric learning from feature vectors, in particular Mahalanobis distance learning, which essentially learns a linear projection of the data into a new space where the local constraints are better satisfied. While early methods were costly and could not be applied to medium-sized problems, recent methods offer better scalability and interesting features such as sparsity. Supervised metric learning from structured data has received less attention because it requires more complex procedures. Most of the work has focused on learning metrics based on the edit distance. Roughly speaking, the edit distance between two objects corresponds to the cheapest sequence of edit operations (insertion, deletion and substitution of subparts) turning one object into the other, where operations are assigned specific costs gathered in a matrix. Edit distance learning consists in optimizing the cost matrix and usually relies on maximizing the likelihood of pairs of similar examples in a probabilistic model.

Overall, we identify two main limitations of the current supervised metric learning methods. First, metrics are optimized based on \emph{local} constraints and used in \emph{local} algorithms, in particular $k$-nearest neighbors. However, it is unclear whether the same procedures can be used to obtain good metrics for use in \emph{global} algorithms such as linear separators, which are simple yet powerful classifiers that often require less memory and provide greater prediction speed than $k$-nearest neighbors. In this context, one may want to optimize the metrics according to a \emph{global} criterion but, to the best of our knowledge, this has never been addressed. Second, and perhaps more importantly, there is a substantial lack of theoretical understanding of generalization in metric learning. It is worth noting that in this context, the question of generalization is two-fold, as illustrated in \fref{fig:genml}. First, one may be interested in the generalization ability of the metric itself, i.e., its consistency not only on the training sample but also on unseen data coming from the same distribution. Very little work has been done on this matter, and existing frameworks lack generality. Second, one may also be interested in the generalization ability of the learning algorithm that uses the learned metric, i.e., can we derive generalization guarantees for the learned model in terms of the quality of the learned metric? In practice, the learned metric is plugged into a learning algorithm and one can only hope that it yields good results. Although some approaches optimize the metric based on the decision rule of classification algorithms such as $k$-nearest neighbors, this question has never been investigated in a formal way. As we will see later in this document, the recently-proposed theory of $(\epsilon,\gamma,\tau)$-good similarity function \citep{Balcan2008,Balcan2008a} has been the first attempt to bridge the gap between the properties of a similarity function and its performance in linear classification, but has not been used so far in the context of metric learning. This theory plays a central role in two of our contributions.

\begin{figure}[t]
\begin{center}
\includegraphics[width=0.8\textwidth]{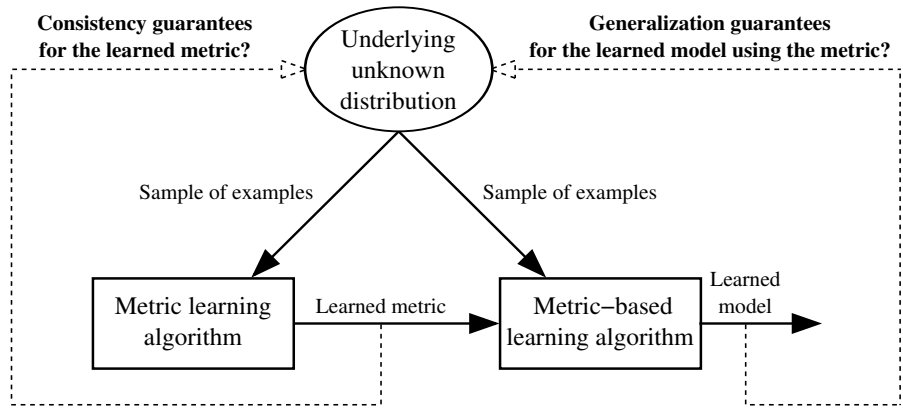}
\caption[The two-fold problem of generalization in metric learning]{The two-fold problem of generalization in metric learning. We are interested in the generalization ability of the learned metric itself: can we say anything about its consistency on unseen data drawn from the same distribution? Furthermore, we are interested in the generalization ability of the learned model using that metric: can we relate its performance on unseen data to the quality of the learned metric?}
\label{fig:genml}
\end{center}
\end{figure}

The limitations described above constitute the main motivation for this thesis, and our contributions address them in several ways. First, we introduce a string kernel that allows the use of learned edit distances in kernel-based methods such as support vector machines. This provides a way to use these learned metrics in global classifiers. Second, we propose two metric learning approaches based on $(\epsilon,\gamma,\tau)$-goodness, for which generalization guarantees can be derived both for the learned metric itself and for a linear classifier built from that metric. In the first approach (which deals with structured data), the metric is optimized with respect to local pairs to ensure the optimality of the solution. In the second approach, dealing with feature vectors allows us to optimize a global criterion that is more appropriate to linear classification. Lastly, we introduce a general framework that can be used to derive generalization guarantees for many existing metric learning methods based on local constraints.

\paragraph{Context of this work} This thesis was carried out in the machine learning team of Laboratoire Hubert Curien UMR CNRS 5516, part of University of Saint-\'Etienne and University of Lyon. The contributions presented in this thesis were developed in the context of the ANR project Lampada\footnote{\url{http://lampada.gforge.inria.fr/}} (ANR-09-EMER-007), which deals with scaling learning algorithms to handle large sets of structured data, with focuses on metric learning and sparse learning, and PASCAL2\footnote{\url{http://pascallin2.ecs.soton.ac.uk/}}, a European Network of Excellence supporting research in machine learning, statistics and optimization.

\paragraph{Outline of the thesis}

This dissertation is organized as follows.
\pref{part:background} reviews the background work relevant to this thesis:
\begin{itemize}
\item \cref{chap:preliminaries} formally introduces the scientific context: supervised learning, analytical frameworks for deriving generalization guarantees, and various types of metrics.
\item \cref{chap:metriclearning} is a large survey of supervised metric learning from feature vectors and structured data, with a focus on the relative merits and limitations of the methods of the literature.
\end{itemize}
\pref{part:struct} gathers our contributions on metric learning from structured data:
\begin{itemize}
\item \cref{chap:pr} introduces a new string kernel based on learned edit probabilities. Unlike other string edit kernels, it is parameter-free and guaranteed to be valid. Its naive form requires the computation of an infinite sum over all finite strings that can be built from the alphabet. We show how to get round this problem by using intersection of probabilistic automata and algebraic manipulation. Experiments highlight the performance of our kernel against state-of-the-art string kernels of the literature.
\item \cref{chap:ecml} builds upon the theory of $(\epsilon,\gamma,\tau)$-good similarity function. We first show that we can use edit similarities directly in this framework and achieve competitive performance. The main contribution of this chapter is a novel method for learning string and tree edit similarities called GESL (for Good Edit Similarity Learning) that relies on a relaxed version of $(\epsilon,\gamma,\tau)$-goodness. The proposed approach, which is more flexible than previous methods, learn an edit similarity from \emph{local} pairs and is then used to build a \emph{global} linear classifier. Using uniform stability arguments, we are able to derive generalization guarantees for the learned similarity that actually give an upper bound on the generalization error of the linear classifier. We conduct extensive experiments that show the usefulness of our approach and the performance and sparsity of the resulting linear classifiers.
\end{itemize}
\pref{part:vect} gathers our contributions on metric learning from feature vectors:
\begin{itemize}
\item \cref{chap:icml} presents a new bilinear similarity learning method for linear classification, called SLLC (for Similarity Learning for Linear Classification). Unlike GESL, SLLC directly optimizes the empirical $(\epsilon,\gamma,\tau)$-goodness criterion, which makes the approach entirely \emph{global}: the similarity is optimized with respect to a global criterion (instead of local pairs) and plugged in a global linear classifier. SLLC is formulated as a convex minimization problem that can be efficiently solved in a batch or online way. We also kernelize our approach, thus learning a linear similarity in a nonlinear feature space induced by a kernel. Using similar arguments as for GESL, we derive generalization guarantees for SLLC highlighting that our method actually minimizes a tighter bound on the generalization error of the classifier than GESL. Experiments on several standard datasets show that SLLC leads to competitive classifiers that have the additional advantage of being very sparse, thus speeding up prediction.
\item \cref{chap:nips} addresses the lack of general framework for establishing generalization guarantees for metric learning. It is based on a simple adaptation of algorithmic robustness to the case where training data is made of pairs of examples. We show that a robust metric learning algorithm has generalization guarantees, and furthermore that a weak notion of robustness is actually necessary and sufficient for a metric learning algorithm to generalize. We illustrate the usefulness of our approach by showing that a large class of metric learning algorithms are robust. In particular, we are able to deal with sparsity-inducing regularizers, which was not possible with previous frameworks.
\end{itemize}

\paragraph{Notation}

Throughout this document, $\mathbb{N}$ denotes the set of natural numbers while $\mathbb{R}$ and $\mathbb{R}_+$ respectively denote the sets of real numbers and nonnegative real numbers. Arbitrary sets are denoted by calligraphic letters such as $\mathcal{S}$, and $|\mathcal{S}|$ stands for the number of elements in $\mathcal{S}$. A set of $m$ elements from $\mathcal{S}$ is denoted by $\mathcal{S}^m$.

We denote vectors by bold lower case letters. For a vector $\mathbf{x}\in\mathbb{R}^d$ and $i\in[d] = \{1,\dots,d\}$, $x_i$ denotes the $i^{th}$ component of $\mathbf{x}$. The inner product between two vectors is denoted by $\innerp{\cdot,\cdot}$.
We denote matrices by bold upper case letters. For a $c\times d$ real-valued matrix $\mathbf{M}\in\mathbb{R}^{c\times d}$ and a pair of integers $(i,j)\in[c]\times[d]$, $M_{i,j}$ denotes the entry at row $i$ and column $j$ of the matrix $\mathbf{M}$. The identity matrix is denoted by $\mathbf{I}$ and the cone of symmetric positive semi-definite (PSD) $d\times d$ real-valued matrices by $\mathbb{S}^{d}_+$. $\|\cdot\|$ denotes an arbitrary (vector or matrix) norm and $\|\cdot\|_p$ the $L_p$ norm.
Strings are denoted by sans serif letters such as $\mathsf{x}$. We use $|\mathsf{x}|$ to denote the length of $\mathsf{x}$ and $\mathsf{x_i}$ to refer to its $i^{th}$ symbol.

In the context of learning problems, we use $\mathcal{X}$ and $\mathcal{Y}$ to denote the input space (or instance space) and the output space (or label space) respectively. We use $\mathcal{Z}=\mathcal{X}\times\mathcal{Y}$ to denote the joint space, and an arbitrary labeled instance is denoted by $z=(x,y)\in\mathcal{Z}$. The hinge function $[\cdot]_+:\mathbb{R}\to\mathbb{R}_+$ is defined as $[c]_+=\max(0,c)$. $\Pr[A]$ denotes the probability of the event $A$, $\mathbb{E}[X]$ the expectation of the random variable $X$ and $x\sim P$ indicates that $x$ is drawn according to the probability distribution $P$.

A summary of the notations is given in \tref{tab:notations}.

\begin{table}[t]
\begin{center}
\begin{footnotesize}
\begin{tabular}{lp{1cm}l}
\toprule
\textbf{Notation} && \textbf{Description}\\
\midrule
$\mathbb{R}$ && Set of real numbers\\
$\mathbb{R}_+$ && Set of nonnegative real numbers\\
$\mathbb{R}^d$ && Set of $d$-dimensional real-valued vectors\\
$\mathbb{R}^{c\times d}$ && Set of $c\times d$ real-valued matrices\\
$\mathbb{N}$ && Set of natural numbers, i.e., $\{0,1,\dots\}$\\
$\mathbb{S}^{d}_+$ && Cone of symmetric PSD $d\times d$ real-valued matrices\\
$[k]$ && The set $\{1,2,\dots,k\}$\\
$\mathcal{S}$ && An arbitrary set\\
$|\mathcal{S}|$ && Number of elements in $\mathcal{S}$\\
$\mathcal{S}^m$ && A set of $m$ elements from $\mathcal{S}$\\
$\mathcal{X}$ && Input space\\
$\mathcal{Y}$ && Output space\\
$z=(x,y)\in\mathcal{X}\times\mathcal{Y}$ && An arbitrary labeled instance\\
$\mathbf{x}$ && An arbitrary vector\\
$x_j$, $x_{i,j}$ && The $j^{th}$ component of $\mathbf{x}$ and $\mathbf{x_i}$\\
$\innerp{\cdot,\cdot}$ && Inner product between vectors\\
$[\cdot]_+$ && Hinge function\\
$\mathbf{M}$ && An arbitrary matrix\\
$\mathbf{I}$ && The identity matrix\\
$M_{i,j}$ && Entry at row $i$ and column $j$ of matrix $\mathbf{M}$\\
$\|\cdot\|$ && An arbitrary norm\\
$\|\cdot\|_p$ && $L_p$ norm\\
$\mathsf{x}$ && An arbitrary string\\
$|\mathsf{x}|$ && Length of string $\mathsf{x}$\\
$\mathsf{x_i},\mathsf{x_{i,j}}$ && $j^{th}$ symbol of $\mathsf{x}$ and $\mathsf{x_i}$\\
$x\sim P$ && $x$ is drawn i.i.d. from probability distribution $P$\\
$\Pr[\cdot]$ && Probability of event\\
$\mathbb{E}[\cdot]$ && Expectation of random variable\\
\bottomrule
\end{tabular}
\end{footnotesize}
\caption[Summary of notation]{Summary of notation.}
\label{tab:notations}
\end{center}
\end{table}

\part{Background}\label{part:background}
\chapter{Preliminaries}
\label{chap:preliminaries}

\begin{chapabstract}
In this chapter, we introduce the scientific context of this thesis as well as relevant background work.
We first introduce formally the supervised learning setting and describe the main ideas of statistical learning theory, with a focus on binary classification. We then present three analytical frameworks (uniform convergence, uniform stability and algorithmic robustness) for establishing that a learning algorithm has generalization guarantees. Lastly, we recall the definition of several types of metrics and give examples of such functions for feature vectors and structured data.
\end{chapabstract}

\section{Supervised Learning}
\label{sec:suplearn}

The goal of supervised learning\footnote{Note that there exist other learning paradigms, such as unsupervised learning \citep{Ghahramani2003}, semi-supervised learning \citep{Chapelle2006}, transfer learning \citep{Pan2010}, reinforcement learning \citep{Sutton1998}, etc.} is to automatically infer a model (hypothesis) from a set of labeled examples that is able to make predictions given new unlabeled data.
In the following, we review basic notions of statistical learning theory, a very popular framework pioneered by \citet{Vapnik1971}. The interested reader can refer to \citet{Vapnik1998} and \citet{Bousquet2003} for a more thorough description.

\subsection{Typical Setting}

In supervised learning, we learn a hypothesis from a set of labeled examples. This notion of training sample is formalized below.

\begin{definition}[Training sample]
A training sample of size $n$ is a set $\mathcal{T} = \{z_i=(x_i,y_i)\}_{i=1}^n$ of $n$ observations independently and identically distributed (i.i.d.) according to an unknown joint distribution $P$ over the space $\mathcal{Z} = \mathcal{X}\times\mathcal{Y}$, where $\mathcal{X}$ is the input space and $\mathcal{Y}$ the output space. For a given observation $z_i$, $x_i\in\mathcal{X}$ is the instance (or example) and $y_i\in\mathcal{Y}$ its label. When $\mathcal{Y}$ is discrete, we are dealing with a classification task, and $y_i$ is called the class of $x_i$. When $\mathcal{Y}$ is continuous, this is a regression task. In this thesis, we mainly focus on binary classification tasks, where we assume $\mathcal{Y} = \{-1,1\}$.
\end{definition}

We will mostly deal with feature vectors and strings. For feature vectors, we generally assume that $\mathcal{X} \subseteq \mathbb{R}^d$. For strings, we need the following definition.

\begin{definition}[Alphabet and string]
An alphabet $\Sigma$ is a finite nonempty set of symbols. A string $\mathsf{x}$ is a finite sequence of symbols from $\Sigma$. The empty string/symbol is denoted by $\$$ and $\Sigma^*$ is the set of all finite strings (including $\$$) that can be generated from $\Sigma$. Finally, the length of a string $\mathsf{x}$ is denoted by $|\mathsf{x}|$.
\end{definition}

We can now formally define what we mean by supervised learning.

\begin{definition}[Supervised learning]
Supervised learning is the task of inferring a function (often referred to as a hypothesis or a model) $h_\mathcal{T}:\mathcal{X}\to\mathcal{L}$ belonging to some hypothesis class $\mathcal{H}$ from a training sample $\mathcal{T}$, which ``best'' predicts $y$ from $x$ for any $(x,y)$ drawn from $P$. Note that the decision space $\mathcal{L}$ may or may not be equal to $\mathcal{Y}$.
\end{definition}

In order to choose $h_\mathcal{T}$, we need a criterion to assess the quality of an arbitrary hypothesis $h$. Given a nonnegative loss function $\ell:\mathcal{H}\times\mathcal{Z}\to\mathbb{R}^+$ measuring the degree of agreement between $h(x)$ and $y$, we define the notion of true risk.

\begin{definition}[True risk]
The true risk (also called generalization error) $R^\ell(h)$ of a hypothesis $h$ with respect to a loss function $\ell$ is the expected loss suffered by $h$ over the distribution $P$:
\begin{eqnarray*}
R^\ell(h) & = & \mathbb{E}_{z\sim P}\left[\ell(h,z)\right].
\end{eqnarray*}
\end{definition}

The most natural loss function for binary classification is the 0/1 loss (also called classification error):
$$\ell_{0/1}(h,z) = \left\{ \begin{array}{rcl}
1 && \text{if}~yh(x) < 0\\
0 && \text{otherwise}.
\end{array} \right.$$
$R^{\ell_{0/1}}(h)$ then corresponds to the proportion of time $h(x)$ and $y$ agree in sign, and in particular to the proportion of correct predictions when $\mathcal{L} = \mathcal{Y}$.

The goal of supervised learning is then to find a hypothesis that achieves the smallest true risk. Unfortunately, in general we cannot compute the true risk of a hypothesis since the distribution $P$ is unknown. We can only measure it empirically on the training sample. This is called the empirical risk.

\begin{definition}[Empirical risk]
Let $\mathcal{T} = \{z_i=(x_i,y_i)\}_{i=1}^n$ be a training sample. The empirical risk (also called empirical error) $R_\mathcal{T}^\ell(h)$ of a hypothesis $h$ over $\mathcal{T}$ with respect to a loss function $\ell$ is the average loss suffered by $h$ on the instances in $\mathcal{T}$:
\begin{eqnarray*}
R^\ell_\mathcal{T}(h) & = & \frac{1}{n}\displaystyle\sum_{i=1}^n\ell(h,z_i).
\end{eqnarray*}
\end{definition}

Under some restrictions, using the empirical risk to select the best hypothesis is a good strategy, as discussed in the next section.

\subsection{Finding a Good Hypothesis}

This section focuses on classic strategies for finding a good hypothesis in the true risk sense. The derivation of guarantees on the true risk of the selected hypothesis will be studied in \sref{sec:gengua}.

Simply minimizing the empirical risk over all possible hypotheses would obviously be a good strategy if infinitely many training instances were available. Unfortunately, in realistic scenarios, training data is limited and there always exists a hypothesis $h$, however complex, that perfectly predicts the training sample, i.e., $R_\mathcal{T}^\ell(h)=0$, but generalizes poorly, i.e., $h$ has a nonzero (potentially large) true risk. This situation where the true risk of a hypothesis is much larger than its empirical risk is called \emph{overfitting}. The intuitive idea behind it is that learning the training sample ``by heart'' does not provide good generalization to unseen data. 

There is therefore a trade-off between minimizing the empirical risk and the complexity of the considered hypotheses, known as the bias-variance trade-off. There essentially exist two ways to deal with it and avoid overfitting: (i) restrict the hypothesis space, and (ii) favor simple hypotheses over complex ones. In the following, we briefly present three classic strategies for finding a hypothesis with small true risk.

\paragraph{Empirical Risk Minimization}

The idea of the Empirical Risk Minimization (ERM) principle is to pick a restricted hypothesis space $\mathcal{H}\subset \mathcal{L}^\mathcal{X}$ (for instance, linear classifiers, decision trees, etc.) and select a hypothesis $h_\mathcal{T}\in\mathcal{H}$ that minimizes the empirical risk:
\begin{equation*}
\begin{aligned}
h_\mathcal{T} & = & \argmin_{h\in\mathcal{H}} &&& R^\ell_\mathcal{T}(h).
\end{aligned}
\end{equation*}
This may work well in practice but depends on the choice of hypothesis space. Essentially, we want $\mathcal{H}$ large enough to include hypotheses with small risk, but $\mathcal{H}$ small enough to avoid overfitting. Without background knowledge on the task, picking an appropriate $\mathcal{H}$ is difficult.

\paragraph{Structural Risk Minimization}

In Structural Risk Minimization (SRM), we use an infinite sequence of hypothesis classes $\mathcal{H}_1\subset\mathcal{H}_2\subset\dots$ of increasing size and select the hypothesis that minimizes a penalized version of the empirical risk that favors ``simple'' classes:
\begin{equation*}
\begin{aligned}
h_\mathcal{T} & = & \argmin_{h\in\mathcal{H}_c,c\in\mathbb{N}} &&& R^\ell_\mathcal{T}(h) + pen(\mathcal{H}_c).
\end{aligned}
\end{equation*}
This implements the Occam's razor principle according to which one should choose the simplest explanation consistent with the training data.

\paragraph{Regularized Risk Minimization}

Regularized Risk Minimization (RRM) also builds upon the Occam's razor principle but is easier to implement: one picks a single, large hypothesis space $\mathcal{H}$ and a regularizer (usually some norm $\|h\|$) and selects a hypothesis that achieves the best trade-off between empirical risk minimization and regularization:
\begin{equation}
\label{eq:rrm}
\begin{aligned}
h_\mathcal{T} & = & \argmin_{h\in\mathcal{H}} &&& R^\ell_\mathcal{T}(h) + \lambda\|h\|,
\end{aligned}
\end{equation}
where $\lambda$ is the trade-off parameter (in practice, it is set using validation data). The role of regularization is to penalize ``complex'' hypotheses. Note that it also provides a built-in way to break the tie between hypotheses that have the same empirical risk. 

\begin{table}[t]
\begin{center}
\begin{footnotesize}
\begin{tabular}{llll}
\toprule
\textbf{Name} & \textbf{Formula} & \textbf{Pros} & \textbf{Cons}\\
\midrule
$L_0$ norm $\|\mathbf{x}\|_0$ & Number of nonzero components & SP & NCO, NSM\\
$L_1$ norm $\|\mathbf{x}\|_1$ & $\sum |x_i|$ & CO, SP & NSM\\
(Squared) $L_2$ norm $\|\mathbf{x}\|_2^2$ & $\sum x_i^2$ & CO, SM &\\
$L_{2,1}$ norm $\|\mathbf{x}\|_{2,1}$ & Sum of $L_2$ norms of grouped variables & CO, GSP & NSM\\
\bottomrule
\end{tabular}
\end{footnotesize}
\caption[Common regularizers on vectors]{Common regularizers on vectors. CO/NCO stand for convex/nonconvex, SM/NSM for smooth/nonsmooth and SP/GSP for sparsity/group sparsity.}
\label{tab:regvec}
\end{center}
\end{table}

\begin{table}[t]
\begin{center}
\begin{footnotesize}
\begin{tabular}{llll}
\toprule
\textbf{Name} & \textbf{Formula} & \textbf{Pros} & \textbf{Cons}\\
\midrule
$L_0$ norm $\|\mathbf{M}\|_0$ & Number of nonzero components & SP & NCO, NSM\\
$L_1$ norm $\|\mathbf{M}\|_1$ & $\sum |M_{i,j}|$ & CO, SP & NSM\\
(Squared) Frobenius norm $\|\mathbf{M}\|_{\mathcal{F}}^2$ & $\sum M_{i,j}^2$ & CO, SM &\\
$L_{2,1}$ norm  $\|\mathbf{M}\|_{2,1}$ & Sum of $L_2$ norms of rows/columns & CO, GSP & NSM\\
Trace (nuclear) norm $\|\mathbf{M}\|_*$ & Sum of singular values & CO, LO & NSM\\
\bottomrule
\end{tabular}
\end{footnotesize}
\caption[Common regularizers on matrices]{Common regularizers on matrices. Abbreviations are the same as in \tref{tab:regvec}, with LO standing for low-rank.}
\label{tab:regmat}
\end{center}
\end{table}

The choice of regularizer is important and depends on the considered task and the desired effect. Common regularizers for vector and matrix models are given in \tref{tab:regvec} and \tref{tab:regmat} respectively. Some regularizers are easy to optimize because they are convex and smooth (for instance, the squared $L_2$ norm) while others do not have these convenient properties and are thus harder to deal with (see \fref{fig:reg-balls} for a graphical insight into some of these regularizers). However, the latter may bring some potentially interesting effects such as sparsity: they tend to set some parameters of the hypothesis to zero. \fref{fig:reg-L1L2} illustrates this on $L_2$ and $L_1$ constraints --- this also holds for regularization.\footnote{In fact, regularized and constrained problems are equivalent in the sense that for any value of the parameter $\beta$ of a feasible constrained problem, there exists a value of the parameter $\lambda$ of the corresponding regularized problem such that both problems have the same set of solutions, and vice versa. In practice, regularized problems are more convenient to use because they are always feasible.}

\begin{figure}[t]
\begin{center}
\psfrag{Ridge}[][][0.8]{$L_2$ norm}
\psfrag{Lasso}[][][0.8]{$L_1$ norm}
\psfrag{Group lasso}[][][0.8]{$L_{2,1}$ norm}
\includegraphics[width=0.6\textwidth]{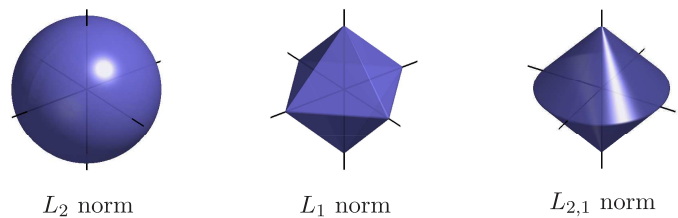}
\caption[3D unit balls of the $L_1$, $L_2$ and $L_{2,1}$ norms]{3D unit balls of the $L_1$, $L_2$ and $L_{2,1}$ norms \citep[taken from][]{Grandvalet2011}. The $L_2$ norm is convex, smooth and does not induce sparsity. The $L_1$ norm is convex, nonsmooth and induces sparsity at the coordinate level. The $L_{2,1}$ norm is convex, nonsmooth and induces sparsity at the group level (simultaneous sparsity of coordinates belonging to the same predefined group).}
\label{fig:reg-balls}
\end{center}
\end{figure}

\begin{figure}[t]
\begin{center}
\psfrag{L1 constraint}[][][0.8]{$L_1$ constraint}
\psfrag{L2 constraint}[][][0.8]{$L_2$ constraint}
\includegraphics[width=0.6\textwidth]{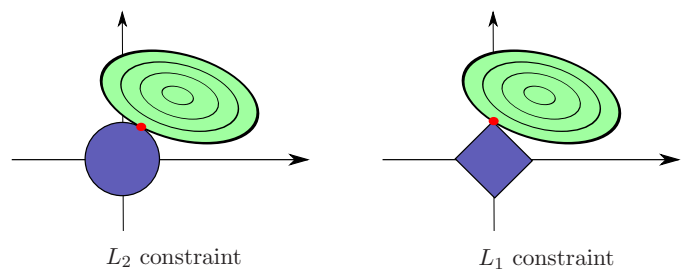}
\caption[Geometric interpretation of $L_2$ and $L_1$ constraints]{Geometric interpretation of $L_2$ and $L_1$ constraints in 2D. Suppose that we are looking for a hypothesis $\mathbf{h}\in\mathbb{R}^2$ with a constraint $\|\mathbf{h}\|\leq \beta$ (represented in dark blue) that minimizes the empirical risk (represented by the light green contour line). Unlike the $L_2$ norm, the $L_1$ norm tends to zero out coordinates, thus reducing dimensionality.}
\label{fig:reg-L1L2}
\end{center}
\end{figure}

Regularization is used in many successful learning methods and, as we will see in \sref{sec:gengua}, may help deriving generalization guarantees.

\subsection{Surrogate Loss Functions}

The methods described above all rely on minimizing the empirical risk. However, due to the nonconvexity of the 0/1 loss, minimizing (or approximately minimizing) $R^{\ell_{0/1}}$ is known to be NP-hard even for simple hypothesis classes \citep{Ben-David2003}. For this reason, surrogate convex loss functions (that can be more efficiently handled) are often used. The most prominent choices in the context of binary classification are:
\begin{itemize}
\item the hinge loss: $\ell_{hinge}(h,z) = [1-yh(x)]_+=\max(0,1-yh(x))$, used for instance in support vector machines \citep{Cortes1995}.
\item the exponential loss: $\ell_{exp}(h,z) = e^{-yh(x)}$, used in Adaboost \citep{Freund1995}.
\item the logistic loss: $\ell_{log}(h,z) = \log(1+\e^{-yh(x)})$, used in Logitboost \citep{Friedman2000}.
\end{itemize}
These loss functions are plotted in \fref{fig:loss} along with the nonconvex 0/1 loss.

\begin{figure}[t]
\begin{center}
\psfrag{Classification error (0/1 loss)}[][][0.7]{Classification error (0/1 loss)}
\psfrag{Hinge loss}[][][0.7]{Hinge loss}
\psfrag{Logistic loss}[][][0.7]{Logistic loss}
\psfrag{Exponential loss}[][][0.7]{Exponential loss}
\psfrag{h(x)*ell}[][][0.7]{$yh(x)$}
\psfrag{ell(h(x),ell)}[][][0.7]{$\ell(h,z)$}
\includegraphics[width=0.6\textwidth]{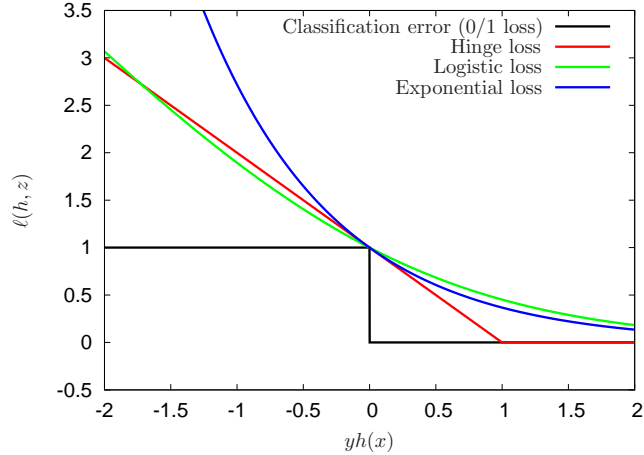}
\caption[Plot of several loss functions for binary classification]{Plot of several loss functions for binary classification.}
\label{fig:loss}
\end{center}
\end{figure}

Choosing an appropriate loss function is not an easy task and strongly depends on the problem, but there exist general results on the relative merits of different loss functions. For instance, \citet{Rosasco2004} studied statistical properties of several convex loss functions in a general classification setting and concluded that the hinge loss has a better convergence rate than other loss functions. \citet{Ben-David2012} have further shown that in the context of linear classification, the hinge loss offers the best guarantees in terms of classification error.

In the following section, we present analytical frameworks that allow the derivation of generalization guarantees, i.e., relating the empirical risk of $h_\mathcal{T}$ to its true risk.

\section{Deriving Generalization Guarantees}
\label{sec:gengua}

In the previous section, we described a few generic methods for learning a hypothesis $h_\mathcal{T}$ from a training sample $\mathcal{T}$ based on minimizing the (penalized) empirical risk. However, learning a hypothesis with small \emph{true} risk is what we are really interested in.
Typically, the empirical risk can be seen as an optimistically biased estimation of the true risk (especially when the training sample is small), and a considerable amount of research has gone into deriving generalization guarantees for learning algorithms, i.e., bounding the deviation of the true risk of the learned hypothesis from its empirical measurement. These bounds are often referred to as PAC (Probably Approximately Correct) bounds \citep{Valiant1984} and have the following form:
$$\Pr[|R^\ell(h)-R^\ell_\mathcal{T}(h)| > \epsilon] \leq \delta,$$
where $\epsilon \geq 0$ and $\delta\in[0,1]$. In other words, it bounds the probability to observe a large gap between the true risk and the empirical risk of an hypothesis.

The key instruments for deriving PAC bounds are concentration inequalities. They essentially assess the deviation of some functions of independent random variables from their expectation. Different concentration inequalities tackle different functions of the variables. The most commonly used in machine learning are Chebyshev (only one variable is considered), Hoeffding (sums of variables) and McDiarmid (that can accommodate any sufficiently regular function of the variables). For more details about concentration inequalities, see for instance the survey of \citet{Boucheron2004}.

In this section, we present three theoretical frameworks for establishing generalization bounds: uniform convergence, uniform stability and algorithmic robustness \citep[for a more general overview, please refer to the tutorial by][]{Langford2005}. Note that our contributions in \cref{chap:ecml}, \cref{chap:icml} and \cref{chap:nips} make use of these frameworks. 

\subsection{Uniform Convergence}

The theory of uniform convergence of empirical quantities to their mean \citep{Vapnik1971,Vapnik1982} is one of the most prominent tools for deriving generalization bounds. It provides guarantees that hold for any hypothesis $h\in\mathcal{H}$ (including $h_\mathcal{T}$) and essentially bounds (with some probability $1-\delta$) the true risk of $h$ by its empirical risk plus a penalty term that depends on the number of training examples $n$, the size (or complexity) of the hypothesis space $\mathcal{H}$ and the value of $\delta$. Intuitively, large $n$ brings high confidence (since as $n\to\infty$ the empirical risk converges to the true risk by the law of large numbers), complex $\mathcal{H}$ brings low confidence (since overfitting is more likely), and $\delta$ accounts for the probability of drawing an ``unlucky'' training sample (i.e., not representative of the underlying distribution $P$).

When the hypothesis space is finite, we get the following PAC bound in $O(1/\sqrt{n})$.
\begin{theorem}[Uniform convergence bound for the finite case] Let $\mathcal{T}$ be a training sample of size $n$ drawn i.i.d. from some distribution $P$, $\mathcal{H}$ a finite hypothesis space and $\delta>0$. For any $h\in\mathcal{H}$, with probability $1-\delta$ over the random sample $\mathcal{T}$, we have:
$$R^\ell(h) \leq R_\mathcal{T}^\ell(h) + \sqrt{\frac{\ln |\mathcal{H}|+\ln(1/\delta)}{2n}}.$$
\end{theorem}

When $\mathcal{H}$ is continuous (for instance, if $\mathcal{H}$ is the space of linear classifiers), we need a measure of the complexity of $\mathcal{H}$ such as the VC dimension \citep{Vapnik1971}, the fat-shattering dimension \citep{Alon1997} or the Rademacher complexity \citep{Koltchinskii2001,Bartlett2002}. For instance, using the VC dimension, we get the following bound.
\begin{theorem}[Uniform convergence bound with VC dimension] Let $\mathcal{T}$ be a training sample of size $n$ drawn i.i.d. from some distribution $P$, $\mathcal{H}$ a continuous hypothesis space with VC dimension $VC(\mathcal{H})$ and $\delta>0$. For any $h\in\mathcal{H}$, with probability $1-\delta$ over the random sample $\mathcal{T}$, we have:
$$R^\ell(h) \leq R_\mathcal{T}^\ell(h) + \sqrt{\frac{VC(\mathcal{H})\left(\ln\frac{2n}{VC(\mathcal{H})}+1\right)+\ln(4/\delta)}{n}}.$$
\end{theorem}

A drawback of uniform convergence analysis is that it is only based on the size of the training sample and the complexity of the hypothesis space, and completely ignores the learning algorithm, i.e., how the hypothesis $h_\mathcal{T}$ is selected.\footnote{In fact, the Rademacher complexity can sometimes implicitly take into account the regularization term of the algorithm.} In the following, we present two analytical frameworks that explicitly take into account the algorithm and can be used to derive generalization guarantees for $h_\mathcal{T}$ specifically, in particular in the regularized risk minimization setting \eqref{eq:rrm}.

\subsection{Uniform Stability}
\label{sec:stability}

Building on previous work on algorithmic stability, \citet{Bousquet2001,Bousquet2002} introduced new definitions that allow the derivation of generalization bounds for a large class of algorithms. Intuitively, an algorithm is said stable if it is robust to small changes in its input (in our case, the training sample), i.e., the variation in its output is small. Formally, we focus on uniform stability, a version of stability that allows the derivation of rather tight bounds.

\begin{definition}[Uniform stability]
\label{def:stability}
An algorithm $\mathcal{A}$ has uniform stability $\kappa/n$ with respect to a loss function $\ell$ if the following holds:
$$\forall \mathcal{T}, |\mathcal{T}|=n,\forall i \in [n] : \displaystyle \sup_{z} |\ell(h_\mathcal{T},z)-\ell(h_{T^i},z)|\leq \frac{\kappa}{n},$$
where $\kappa$ is a positive constant, $\mathcal{T}^i$ is obtained from the training sample $\mathcal{T}$ by replacing the $i^{th}$ example $z_i\in \mathcal{T}$ by another example $z_i'$ drawn i.i.d. from $P$, $h_\mathcal{T}$ and $h_{\mathcal{T}^i}$ are the hypotheses learned by $\mathcal{A}$ from $\mathcal{T}$ and $\mathcal{T}^i$ respectively.\footnote{\defref{def:stability} corresponds to the case where the training sample is altered through the replacement of an instance by another. \citet{Bousquet2001,Bousquet2002} also give a definition of uniform stability based on the removal of an instance from the training sample, which implies \defref{def:stability}. We will use \defref{def:stability} throughout this thesis: we find it more convenient to deal with since replacement preserves the size of the training sample.}
\end{definition}

\citet{Bousquet2001,Bousquet2002} have shown that a large class of regularized risk minimization algorithms satisfies this definition. The constant $\kappa$ typically depends on the form of the loss function, the regularizer and the regularization parameter $\lambda$. Making a good use of McDiarmid's inequality, they show that when \defref{def:stability} is fulfilled, the following bound in $O(1/\sqrt{n})$ holds.

\begin{theorem}[Uniform stability bound] Let $\mathcal{T}$ be a training sample of size $n$ drawn i.i.d. from some distribution $P$ and $\delta>0$. For any algorithm $\mathcal{A}$ with uniform stability $\kappa/n$ with respect to a loss function $\ell$ upper-bounded by some constant $B$,\footnote{Note that many loss functions are unbounded if their domain is assumed to be unbounded (see \fref{fig:loss}), but in practice they have bounded domain due for example to the common assumption that the norm of any instance is bounded.} with probability $1-\delta$ over the random sample $\mathcal{T}$, we have:
$$R^\ell(h_\mathcal{T}) \leq R_\mathcal{T}^\ell(h_\mathcal{T}) + \frac{\kappa}{n} + (2\kappa+B) \sqrt{\frac{\ln (1/\delta)}{2n}},$$
where $h_T$ is the hypothesis learned by $\mathcal{A}$ from $\mathcal{T}$.
\end{theorem}

The main difference between uniform convergence and uniform stability is that the latter incorporates regularization (through $\kappa$ and $h_T$) and does not require any hypothesis space complexity argument. In particular, uniform stability can be used to derive generalization guarantees for hypothesis classes that are difficult to analyze with classic complexity arguments, such as $k$-nearest neighbors or support vector machines that have infinite VC dimension. It can also be adapted to non-i.i.d. settings \citep{Mohri2007,Mohri2010}. We will use uniform stability in the contributions presented in \cref{chap:ecml} and \cref{chap:icml}.

On the other hand, \citet{Xu2012} have shown that algorithms with sparsity-inducing regularization are not stable.\footnote{Sparsity is seen here as the ability to identify redundant features.} Algorithmic robustness, presented in the next section, is able to deal with such algorithms. We will make use of this framework in \cref{chap:nips}.

\subsection{Algorithmic Robustness}
\label{sec:robustness}

Algorithmic robustness \citep{Xu2010,Xu2012a} is the ability of an algorithm to perform ``similarly'' on a training example and on a test example that are ``close''.
It relies on a partitioning of the space $\mathcal{Z}$ to characterize closeness: two examples are close to each other if they lie in the same partition of the space. The partition itself is based on the notion of covering number \citep{Kolmogorov1961}.
\begin{definition}[Covering number]
For a metric space $(\mathcal{S},\rho)$ and $\mathcal{V}\subset \mathcal{S}$, we say that $\hat{\mathcal{V}}\subset \mathcal{V}$ is a $\gamma$-cover of $\mathcal{V}$ if $\forall t\in\mathcal{V}$, $\exists \hat{t}\in\hat{\mathcal{V}}$ such that $\rho(t,\hat{t})\leq \gamma$. The $\gamma$-covering number of $\mathcal{V}$ is
$$\mathcal{N}(\gamma,\mathcal{V},\rho) = \min\left\{|\hat{\mathcal{V}}| : \hat{\mathcal{V}}\text{ is a }\gamma\text{-cover of }\mathcal{V}\right\}.$$
\end{definition}
In particular, when $\mathcal{X}$ is compact, $\mathcal{N}(\gamma,\mathcal{X},\rho)$ is finite, leading to a finite cover.
Then, $\mathcal{Z}$ can be partitioned into $|\mathcal{Y}|\mathcal{N}(\gamma,\mathcal{X},\rho)$ subsets such that if two examples $z=(x,y)$ and $z'=(x',y')$ belong to the same subset, then $y=y'$ and $\rho(x,x')\leq\gamma$.

We can now formally define the notion of robustness.
\begin{definition}[Algorithmic robustness]
Algorithm $\mathcal{A}$ is $(K,\epsilon(\cdot))$-robust, for $K\in\mathbb{N}$ and $\epsilon(\cdot):\mathcal{Z}^n\to\mathbb{R}$, if $\mathcal{Z}$ can be partitioned into $K$ disjoint sets, denoted by $\{C_i\}_{i=1}^K$, such that the following holds for all $\mathcal{T}\in\mathcal{Z}^n$:
$$\forall z\in \mathcal{T}, \forall z'\in\mathcal{Z},\forall i\in[K] : \text{if } z,z'\in C_i, \text{then } |\ell(h_\mathcal{T},z)-\ell(h_\mathcal{T},z')|\leq \epsilon(\mathcal{T}),$$
where $h_\mathcal{T}$ is the hypothesis learned by $\mathcal{A}$ from $\mathcal{T}$.
\end{definition}

Briefly speaking, an algorithm is robust if for any example $z'$ falling in the same subset as a training example $z$, then the gap between the losses associated with $z$ and $z'$ is bounded (by a quantity that may depend on the training sample $\mathcal{T}$). The existence of the partition itself is guaranteed by the definition of covering number. Note that both uniform stability and algorithmic robustness properties involve a bound on deviations between losses. The key difference is that uniform stability studies the variation of the loss associated with any example $z$ under small changes in the training sample (implying that the learned hypothesis itself does not vary much), while algorithmic robustness considers the deviation between the losses associated with two examples that are close (implying that the learned hypothesis is locally consistent).

\citet{Xu2010,Xu2012a} have shown that a robust algorithm has generalization guarantees. This is formalized by the following theorem.
\begin{theorem}[Robustness bound]
Let $\ell$ be a loss function upper-bounded by some constant $B$, and $\delta > 0$. If an algorithm $\mathcal{A}$ is $(K,\epsilon(\cdot))$-robust, then with probability $1-\delta$, we have:
$$R^\ell(h_\mathcal{T}) \leq R_\mathcal{T}^\ell(h_\mathcal{T}) + \epsilon(\mathcal{T}) + B\sqrt{\frac{2K\ln 2+2\ln(1/\delta)}{n}},$$
where $h_T$ is the hypothesis learned by $\mathcal{A}$ from $\mathcal{T}$.
\end{theorem}
Note that there is a tradeoff between the size $K$ of the partition and $\epsilon(\mathcal{T})$: the latter can essentially be made as small as possible by using a finer-grained cover.

PAC bounds based on robustness are generally not tight since they rely on unspecified (potentially large) covering numbers. On the other hand, a great advantage of robustness is that it can deal with a larger class of regularizers than stability (in particular, sparsity-inducing norms can be considered), and its geometric interpretation makes adaptations to non-standard settings (such as non-i.i.d. data) possible. Our contribution in \cref{chap:nips} adapts robustness to the case of metric learning, when training data consist of non-i.i.d. pairs of examples. Finally, note that \citet{Xu2010,Xu2012a} established that a weak notion of robustness is necessary and sufficient for an algorithm to generalize asymptotically, making robustness a key property for the generalization of learning algorithms.

After having presented the supervised learning setting and analytical frameworks for deriving generalization guarantees, we now turn to the topic of metrics, which has a great place in this thesis.

\section{Metrics}
\label{sec:metrics}

The notion of metric (used here as a generic term for distance, similarity or dissimilarity function) plays an important role in many machine learning problems such as classification, regression, clustering, or ranking. Successful examples include:
\begin{itemize}
\item $k$-Nearest Neighbors ($k$-NN) classification \citep{Cover1967}, where the predicted class of an instance $x$ corresponds to the majority class among the $k$-nearest neighbors of $x$ in the training sample, according to some distance or similarity.
\item Kernel methods \citep{Scholkopf2001}, where a specific type of similarity function called kernel (see \defref{def:kernel}) is used to implicitly project data into a new high-dimensional feature space. The most prominent example is Support Vector Machines (SVM) classification \citep{Cortes1995}, where a large-margin linear classifier is learned in that space.
\item $K$-Means \citep{Lloyd1982}, a clustering algorithm which aims at finding the $K$ clusters that minimize the within-cluster distance on the training sample according to some metric.
\item Information retrieval, where a similarity function is often used to retrieve documents (webpages, images, etc.) that are similar to a query or to another document \citep{Salton1975,Baeza-Yates1999,Sivic2009}.
\item Data visualization, where visualization of interesting patterns in high-dimensional data is sometimes achieved by means of a metric \citep{Venna2010,Bertini2011}.
\end{itemize}
It should be noted that metrics are especially important when dealing with structured data (such as strings, trees, or graphs) because they are often a convenient proxy to manipulate these complex objects: if a metric is available, then any metric-based algorithm (such as those presented in the above list) can be used.

In this section, we first give the definitions of distance, similarity and kernel functions (\ref{sec:metricdef}), and then give some examples (by no means an exhaustive list) of such metrics between feature vectors (\ref{sec:metricvect}) and between structured data (\ref{sec:metricstruct}).

\subsection{Definitions}
\label{sec:metricdef}

We start by introducing the definition of a distance function.
\begin{definition}[Distance function]
A distance over a set $\mathcal{X}$ is a pairwise function $d:\mathcal{X}\times\mathcal{X}\to \mathbb{R}$ which satisfies the following properties $\forall x,x',x''\in\mathcal{X}$:
\begin{enumerate}
\item $d(x,x') \geq 0$ (nonnegativity),
\item $d(x,x') = 0$ if and only if $x=x'$ (identity of indiscernibles),
\item $d(x,x') = d(x',x)$ (symmetry),
\item $d(x,x'') \leq d(x,x')+d(x',x'')$ (triangle inequality).
\end{enumerate} 
\end{definition}
A \emph{pseudo-distance} satisfies the properties of a metric, except that instead of property 2, only $d(x,x)=0$ is required. Note that the property of triangle inequality can be used to speedup learning algorithms such as $k$-NN \citep[e.g.,][]{Mico1994,Lai2007,Wang2011a} or $K$-Means \citep{Elkan2003}.

While a distance function is a well-defined mathematical concept, there is no general agreement on the definition of a (dis)similarity function, which can essentially be any pairwise function. Throughout this thesis, we will use the following definition.

\begin{definition}[Similarity function]
A (dis)similarity function is a pairwise function $K:\mathcal{X}\times\mathcal{X}\to [-1,1]$. We say that $K$ is a symmetric similarity function if $\forall x,x'\in\mathcal{X}$, $K(x,x')=K(x',x)$.
\end{definition}
A similarity function should return a high score for similar inputs and a low score for dissimilar ones (the other way around for a dissimilarity function). Note that (normalized) distance functions are dissimilarity functions.

Finally, a kernel is a special type of similarity function, as formalized by the following definition.

\begin{definition}[Kernel function]
\label{def:kernel}
A symmetric similarity function $K$ is a kernel if there exists a (possibly implicit) mapping function $\phi:\mathcal{X}\to\mathbb{H}$ from the instance space $\mathcal{X}$ to a Hilbert space $\mathbb{H}$ such that $K$ can be written as an inner product in $\mathbb{H}$:
$$K(x,x') = \innerp{\phi(x),\phi(x')}.$$
Equivalently, $K$ is a kernel if it is positive semi-definite (PSD), i.e.,
$$\displaystyle\sum_{i=1}^n\sum_{j=1}^nc_ic_jK(x_i,x_j) \geq 0$$
for all finite sequences of $x_1,\dots,x_n\in\mathcal{X}$ and $c_1,\dots,c_n\in\mathbb{R}$.
\end{definition}

Kernel functions are a key component of kernel methods such as SVM, because they can implicitly allow cheap inner product computations in very high-dimensional spaces (this is known as the ``kernel trick'') and bring an elegant theory based on Reproducing Kernel Hilbert Spaces (RKHS). Note that these advantages disappear when using an arbitrary non-PSD similarity function instead of a kernel, and the convergence of the kernel-based algorithm may not even be guaranteed in this case.\footnote{Some research has gone into training SVM with indefinite kernels, mostly based on building a PSD kernel from the indefinite one while learning the SVM classifier. The interested reader may refer to the work of \citet{Ong2004,Luss2007,Chen2008,Chen2009a} and references therein.}

\subsection{Some Metrics between Feature Vectors}
\label{sec:metricvect}

\paragraph{Minkowski distances}
Minkowski distances are a family of distances induced by $L_p$ norms. For $p\geq 1$,
\begin{equation}
\label{eq:mink}
d_p(\mathbf{x},\mathbf{x'}) = \|\mathbf{x}-\mathbf{x'}\|_p = \left(\displaystyle\sum_{i=1}^d |x_i-x'_i|^p\right)^{1/p}.
\end{equation}
From \eqref{eq:mink} we can recover three widely used distances:
\begin{itemize}
\item When $p=1$, we get the Manhattan distance:
$$d_{man}(\mathbf{x},\mathbf{x'}) = \|\mathbf{x}-\mathbf{x'}\|_1 = \displaystyle\sum_{i=1}^d |x_i-x'_i|.$$
\item When $p=2$, we get the ``ordinary'' Euclidean distance:
$$d_{euc}(\mathbf{x},\mathbf{x'}) = \|\mathbf{x}-\mathbf{x'}\|_2 = \left(\displaystyle\sum_{i=1}^d |x_i-x'_i|^2\right)^{1/2} = \sqrt{(\mathbf{x}-\mathbf{x'})^T(\mathbf{x}-\mathbf{x'})}.$$
\item When $p\to\infty$, we get the Chebyshev distance:
$$d_{che}(\mathbf{x},\mathbf{x'}) = \|\mathbf{x}-\mathbf{x'}\|_\infty = \displaystyle\max_i |x_i-x'_i|.$$
\end{itemize}
Note that when $0<p<1$, $d_p$ is not a proper distance (it violates the triangle inequality) and the corresponding (pseudo) norm is nonconvex. \fref{fig:mink} shows the corresponding unit circles for several values of $p$.

\begin{figure}[t]
\begin{center}
\psfrag{p=0}[][][0.8]{$p\to0$}
\psfrag{p=0.3}[][][0.8]{$p=0.3$}
\psfrag{p=0.5}[][][0.8]{$p=0.5$}
\psfrag{p=1}[][][0.8]{$p=1$}
\psfrag{p=1.5}[][][0.8]{$p=1.5$}
\psfrag{p=2}[][][0.8]{$p=2$}
\psfrag{p=infty}[][][0.8]{$p\to\infty$}
\includegraphics[width=0.98\textwidth]{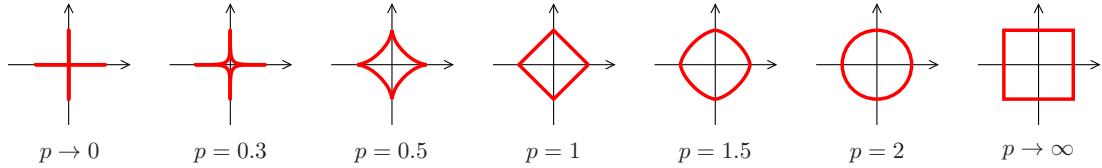}
\caption[Minkowski distances: unit circles for various values of $p$]{Minkowski distances: unit circles for various values of $p$.}
\label{fig:mink}
\end{center}
\end{figure}

\paragraph{Mahalanobis distances} The Mahalanobis distance, which incorporates knowledge about the correlation between features, is defined by
$$d_{\boldsymbol{\Sigma}^{-1}}(\mathbf{x},\mathbf{x'}) = \sqrt{(\mathbf{x}-\mathbf{x'})^T\boldsymbol{\Sigma}^{-1}(\mathbf{x}-\mathbf{x'})},$$
where $\mathbf{x}$ and $\mathbf{x'}$ are random vectors from the same distribution with covariance matrix $\boldsymbol{\Sigma}$. The term Mahalanobis distance is also used to refer to the following generalization of the original definition, sometimes referred to as generalized quadratic distances \citep{Nielsen2009}:
$$d_{\mathbf{M}}(\mathbf{x},\mathbf{x'}) = \sqrt{(\mathbf{x}-\mathbf{x'})^T\mathbf{M}(\mathbf{x}-\mathbf{x'})},$$
where $\mathbf{M}\in \mathbb{S}^{d}_+$. $\mathbb{S}^{d}_+$ denotes the cone of symmetric PSD $d\times d$ real-valued matrices. $\mathbf{M}\in \mathbb{S}^{d}_+$ ensures that $d_{\mathbf{M}}$ is a pseudo-distance. When $\mathbf{M}$ is the identity matrix, we recover the Euclidean distance. Otherwise, using Cholesky decomposition, one can rewrite $\mathbf{M}$ as $\mathbf{L}^T\mathbf{L}$, where $\mathbf{L}\in \mathbb{R}^{k\times d}$, where $k$ is the rank of $\mathbf{M}$. Hence:
\begin{eqnarray*}
d_\mathbf{M}(\mathbf{x},\mathbf{x'}) & = & \sqrt{(\mathbf{x}-\mathbf{x'})^T\mathbf{M}(\mathbf{x}-\mathbf{x'})}\\
& = & \sqrt{(\mathbf{x}-\mathbf{x'})^T\mathbf{L}^T\mathbf{L}(\mathbf{x}-\mathbf{x'})}\\
& = & \sqrt{(\mathbf{L}\mathbf{x}-\mathbf{L}\mathbf{x'})^T(\mathbf{L}\mathbf{x}-\mathbf{L}\mathbf{x'})}.
\end{eqnarray*}
Thus, a Mahalanobis distance implicitly corresponds to computing the Euclidean distance after the linear projection of the data defined by $\mathbf{L}$. Note that if $\mathbf{M}$ is low-rank, i.e., $\rank(\mathbf{M}) = r < d$, then it induces a linear projection of the data into a space of lower dimension $r$. It thus allows a more compact representation of the data and cheaper distance computations, especially when the original feature space is high-dimensional.

Because of these nice properties, learning Mahalanobis distance has attracted a lot of interest and is a major component of metric learning (see \sref{sec:mahalearning}).

\paragraph{Cosine similarity}
The cosine similarity measures the cosine of the angle between two instances, and can be computed as
$$K_{cos}(\mathbf{x},\mathbf{x'}) = \frac{\mathbf{x}^T\mathbf{x'}}{\|\mathbf{x}\|_2\|\mathbf{x'}\|_2}.$$
The cosine similarity is widely used in data mining, in particular in text retrieval \citep{Baeza-Yates1999} and more recently in image retrieval \citep[see for instance][]{Sivic2009} when data are represented as term vectors \citep{Salton1975}.

\paragraph{Bilinear similarity}
The bilinear similarity is related to the cosine similarity but does not include normalization by the norms of the inputs and is parameterized by a matrix $\mathbf{M}$:
$$K_{\mathbf{M}}(\mathbf{x},\mathbf{x'}) = \mathbf{x}^T\mathbf{M}\mathbf{x'},$$
where $\mathbf{M} \in \mathbb{R}^{d\times d}$ is not required to be PSD nor symmetric. The bilinear similarity has been used for instance in image retrieval \citep{Deng2011}.
When $\mathbf{M}$ is the identity matrix, $K_\mathbf{M}$ amounts to an unnormalized cosine similarity.
The bilinear similarity has two advantages. First, it is efficiently computable for sparse inputs: if $\mathbf{x}$ and $\mathbf{x'}$ have $k_1$ and $k_2$ nonzero features, $K_\mathbf{M}(\mathbf{x},\mathbf{x'})$ can be computed in $O(k_1k_2)$ time. Second, unlike Minkowski distance, Mahalanobis distances and the cosine similarity, it can be easily used as a similarity measure between instances of different dimension (for example, a document and a query) by choosing a nonsquare matrix $\mathbf{M}$. A major contribution of this thesis is to propose a novel method for learning a bilinear similarity (\cref{chap:icml}).

\paragraph{Linear kernel} The linear kernel is simply the inner product in the original space $\mathcal{X}$:
$$K_{lin}(\mathbf{x},\mathbf{x'}) = \innerp{\mathbf{x},\mathbf{x'}} = \mathbf{x}^T\mathbf{x'}.$$
In other words, the corresponding $\phi$ is an identity map: $\forall \mathbf{x}\in\mathcal{X}, \phi(\mathbf{x})=\mathbf{x}$. Note that $K_{lin}$ corresponds to the bilinear similarity with $\mathbf{M}=\mathbf{I}$.

\paragraph{Polynomial kernels} Polynomial kernels are defined as:
$$K_{deg}(\mathbf{x},\mathbf{x'}) = (\innerp{\mathbf{x},\mathbf{x'}}+1)^{deg},$$
where $deg\in\mathbb{N}$. It can be shown that $K_{deg}$ implicitly projects an instance into the nonlinear space $\mathbb{H}$ of all monomials of degree up to $deg$.

\paragraph{Gaussian kernel} The Gaussian kernel, also known as the RBF kernel, is a widely used kernel defined by
$$K_{gaus}(\mathbf{x},\mathbf{x'}) = \exp\left(-\frac{\|\mathbf{x}-\mathbf{x'}\|_2^2}{2\sigma^2}\right),$$
where $\sigma^2>0$ is a width parameter. For this kernel, it can be shown that the corresponding implicit nonlinear projection space $\mathbb{H}$ is infinite-dimensional.


\subsection{Some Metrics between Structured Data}
\label{sec:metricstruct}

\paragraph{Hamming distance}
The Hamming distance is a distance between strings of identical length and is equal to the number of positions at which the symbols differ. It has been used mostly for binary strings and is defined by
$$d_{ham}(\mathsf{x},\mathsf{x'}) = |\{i : \mathsf{x_i}\neq \mathsf{x'_i}\}|.$$

\paragraph{String edit distance}

The string edit distance \citep{Levenshtein1966} is a distance between strings of possibly different length built from an alphabet $\Sigma$. It is based on three elementary edit operations: insertion, deletion and substitution of a symbol. In the more general version, each operation has a specific cost, gathered in a nonnegative $(|\Sigma|+1)\times (|\Sigma|+1)$ matrix $\mathbf{C}$ (the additional row and column account for insertion and deletion costs respectively). A sequence of operations transforming a string $\mathsf{x}$ into a string $\mathsf{x'}$ is called an edit script. The edit distance between $\mathsf{x}$ and $\mathsf{x'}$ is defined as the cost of the cheapest edit script that turns $\mathsf{x}$ into $\mathsf{x'}$ and can be computed in $O(|\mathsf{x}|\cdot|\mathsf{x'}|)$ time by dynamic programming.\footnote{Note that in the case of strings of equal length, the edit distance is upper bounded by the Hamming distance.}

The classic edit distance, known as the Levenshtein distance, uses a unit cost matrix and thus corresponds to the minimum number of operations turning one string into another. For instance, the Levenshtein distance between \texttt{abb} and \texttt{aa} is equal to 2, since turning \texttt{abb} into \texttt{aa} requires at least 2 operations (e.g., substitution of \texttt{b} with \texttt{a} and deletion of \texttt{b}). On the other hand, using the cost matrix given in \tref{tab:costmatrix}, the edit distance between \texttt{abb} and \texttt{aa} is equal to 10 (deletion of \texttt{a} and two substitutions of \texttt{b} with \texttt{a} is the cheapest edit script).

\begin{table}[t]
\begin{center}
\begin{small}
\begin{tabular}{|c|ccc|}
\hline
$\mathbf{C}$ & $\$$ & \texttt{a} & \texttt{b}\\
\hline
$\$$ & 0 & 2 & 10\\
\texttt{a} & 2 & 0 & 4\\
\texttt{b} & 10 & 4 & 0\\
\hline
\end{tabular}
\end{small}
\caption[Example of an edit cost matrix]{Example of edit cost matrix $\mathbf{C}$. Here, $\Sigma=\{\mathtt{a},\mathtt{b}\}$.}
\label{tab:costmatrix}
\end{center}
\end{table}

Using task-specific costs is a key ingredient to the success of the edit distance in many applications. For some problems such as handwritten character recognition \citep{Mico1998} or protein alignment \citep{Dayhoff1978,Henikoff1992}, relevant cost matrices may be available. But a more general solution consists in automatically learning the cost matrix from data, as we shall see in \sref{sec:stringeditlearn}. One of the contributions of this thesis is to propose a new edit cost learning method (\cref{chap:ecml}).

\paragraph{Sequence alignment} Sequence alignment is a way of computing the similarity between two strings, mostly used in bioinformatics to identify regions of similarity in DNA or protein sequences \citep{Mount2004}. It corresponds to the score of the best alignment. The score of an alignment is based on the same elementary operations as the edit distance and on a score matrix for substitutions, but uses a (linear or affine) gap penalty function instead of insertion and deletion costs. The most prominent sequence alignment measures are the Needleman-Wunsch score \citep{Needleman1970} for global alignments and the Smith-Waterman score \citep{Smith1981} for local alignments. They can be computed by dynamic programming.

\paragraph{Tree edit distance}

Because of the growing interest in applications that naturally involve tree-structured data (such as the secondary structure of RNA in biology, XML documents on the web or parse trees in natural language processing), several works have extended the string edit distance to trees, resorting to the same elementary edit operations \citep[see][for a survey on the matter]{Bille2005}.
There exist two main variants of the tree edit distance that differ in the way the deletion of a node is handled. In \citet{Zhang1989}, when a node is deleted all its children are connected to its father. The best algorithms for computing this distance have an $O(n^3)$ worst-case complexity, where $n$ is the number of nodes of the largest tree \citep[see][for an empirical evaluation of several algorithms]{Pawlik2011}.
Another variant is due to \citet{Selkow1977}, where insertions and deletions are restricted to the leaves of the tree. Such a distance is relevant to specific applications. For instance, deleting a {\tt <UL>} tag (i.e., a nonleaf node) of an unordered list in an HTML document would require the iterative deletion of the {\tt <LI>} items (i.e., the subtree) first, which is a sensible thing to do in this context (see \fref{fig:tree}). This version can be computed in quadratic time.
Note that tree edit distance computations can be made significantly faster (especially for large trees) by exploiting lower bounds on the distance between two trees that are cheap to obtain \citep[see for instance][]{Yang2005}. A study on the expressiveness of similarities and distances on trees was proposed by \citet{Emms2012a}.

Like in the string case, there exists a few methods for learning the cost matrix of the tree edit distance (see \sref{sec:treeeditlearn}). Note that our edit similarity learning method, presented in \cref{chap:ecml}, can be used for both strings and trees.

\begin{figure}[t]
\begin{center}
\subfigure[]{\label{fig:tree_a}
\includegraphics[scale=0.8]{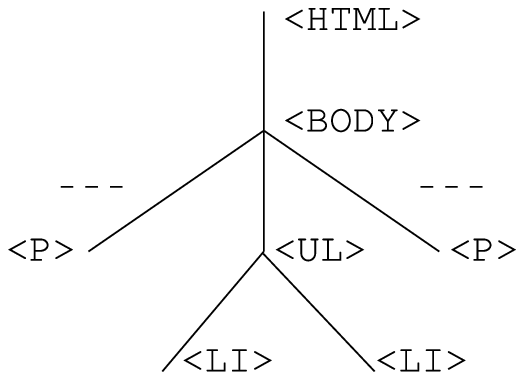}
}
\subfigure[]{\label{fig:tree_b}
\includegraphics[scale=0.8]{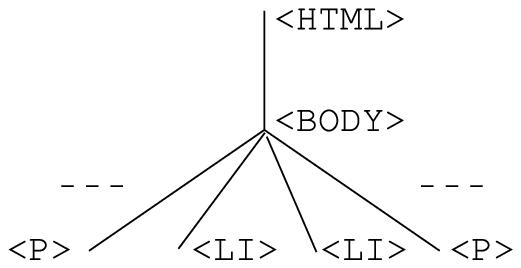}
}
\subfigure[]{\label{fig:tree_c}
\includegraphics[scale=0.8]{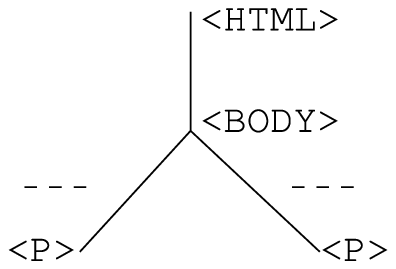}
}\\
\caption[Strategies to delete a node within a tree]{Strategies to delete a node within a tree: (a) original tree, (b) after deletion of the {\tt <UL>} node as defined by Zhang \& Shasha, and (c) after deletion of the {\tt <UL>} node as defined by Selkow.}
\label{fig:tree}
\end{center}
\end{figure}

\paragraph{Graph edit distance} Note that there also exist extensions of the edit distance to general graphs \citep{Gao2010}, but like many problems on graphs, computing a graph edit distance is NP-hard, making it impractical for real-world tasks.

\paragraph{Spectrum, subsequence and mismatch kernels} These string kernels represent strings by fixed-length feature vectors and rely on explicit mapping functions $\phi$. The spectrum kernel \citep{Leslie2002} maps each string to a vector of frequencies of all contiguous subsequences of length $p$ and computes the inner product between these vectors. The subsequence kernel \citep{Lodhi2002} and the mismatch kernel \citep{Leslie2002a} extend the spectrum kernel to inexact subsequence matching: the former considers all (possibly noncontiguous) subsequences of length $p$ while the latter allows a number of mismatches in the subsequences.


\paragraph{String edit kernels} String edit kernels are derived from the string edit distance (or related measures). The classic edit kernel \citep{Li2004} has the following form:
$$K_{L\&J}(\mathsf{x},\mathsf{x'}) = e^{-t\cdot d_{lev}(\mathsf{x},\mathsf{x'})},$$
where $d_{lev}$ is the Levenshtein distance and $t>0$ is a parameter. However, \citet{Cortes2004} have shown that this function is not PSD (and thus is not a valid kernel) in the general case for nontrivial alphabets. Thus, one has to tune $t$, hoping to make $K$ PSD. Moreover, it suffers from the so-called ``diagonal dominance'' problem (i.e., the kernel value decreases exponentially fast with the distance), and SVM is known not to perform well in this case \citep{Scholkopf2002}. A different string edit kernel was proposed by \citet{Neuhaus2006} and is defined as follows:
$$K_{N\&B}(\mathsf{x},\mathsf{x'})=\frac{1}{2} (d_{lev}(\mathsf{x},\mathsf{x_0})^2+d_{lev}(\mathsf{x_0},\mathsf{x'})^2-d_{lev}(\mathsf{x},\mathsf{x'})^2),$$
where $\mathsf{x_0}$ is called the ``zero string'' and must be picked by hand. They also propose combinations of such kernels with different zero strings. However, the validity of such kernels is not guaranteed either. \citet{Saigo2004} build a kernel from the sum of scores over all possible Smith-Waterman local alignments between two strings instead of the alignment of highest score only. They show that if the score matrix is PSD, then the kernel is valid in general. However, like $K_{L\&J}$, it suffers from the diagonal dominance problem. In practice, the authors take the logarithm of the kernel and add a sufficiently large diagonal term to ensure the validity of the kernel.

\paragraph{Convolution kernels} The framework of convolution kernels \citep{Haussler1999} can be used to derive many kernels for structured data. Roughly speaking, if structured instances can be seen as a collection of subparts, then Haussler's convolution kernel between two instances is defined as the sum of the return values of a predefined kernel over all possible pairs of subparts, and is guaranteed to be PSD. Mapping kernels \citep{Shin2008} are a generalization of convolution kernels as they allow the sum to be computed only over a predefined subset of the subpart pairs. These frameworks have been used to design several kernels between structured data \citep{Collins2001,Shin2008,Shin2011}. However, building such kernels is often not straightforward since they suppose the existence of a kernel between subparts of the structured instances.

\paragraph{Marginalized kernels} When one has access to a probabilistic model encoding for instance the probability that a string (or a tree) is turned into another one, marginalized kernels \citep{Tsuda2002,Kashima2003}, of which the Fisher kernel \citep{Jaakkola1998} is a special case, are a way of building a kernel from the output of such models. Since our string kernel proposed in \cref{chap:pr} belongs to this family, we postpone the details of the framework to \sref{sec:oureditkernel}.

\section{Conclusion}

In this chapter, we introduced the setting of supervised learning, presented analytical frameworks that allow the derivation of generalization bounds for learning algorithms, and reviewed different forms of metrics.

The contributions of this thesis can be cast as supervised metric learning methods, i.e., learning the parameters of a metric from labeled data. Because the performance of many learning algorithms using metrics critically depends on the relevance of the metric to the problem at hand, supervised metric learning has attracted a lot of interest in recent years. \cref{chap:metriclearning} is a large review of the literature on the subject.

\chapter{A Review of Supervised Metric Learning}
\label{chap:metriclearning}

\begin{chapabstract}
In this chapter, we review the literature on supervised metric learning. We start by introducing the main concepts of this research topic. Then, we cover metric learning from feature vectors (in particular, Mahalanobis distance learning) as well as metric learning from structured data such as strings and trees, with an emphasis on the pros and cons of each method. Finally, we conclude by discussing the general limitations of the current literature that motivate our work.
\end{chapabstract}

\section{Introduction}

As discussed in \sref{sec:metrics}, using an appropriate metric is key to the performance of many learning algorithms. Since manually tuning metrics (when they allow some parameterization) for a given real-world problem is often difficult and tedious, a lot of work has gone into automatically learning them from labeled data, leading to the emergence of metric learning. This chapter is devoted to a large survey of supervised metric learning techniques.

Generally speaking, supervised metric learning approaches rely on the reasonable intuition that a good similarity function should assign a large (resp. small) score to pairs of points of the same class (resp. different class), and conversely for a distance function. Following this idea, they aim at finding the parameters (usually a matrix) of the metric such that it best satisfies local constraints built from the training sample $\mathcal{T}$. They are typically pair or triplet-based constraints of the following form:
\begin{eqnarray*}
\mathcal{S} & = & \{(z_i,z_j)\in \mathcal{T}\times \mathcal{T} : x_i\text{ and }x_j\text{ should be similar}\},\\
\mathcal{D} & = & \{(z_i,z_j)\in \mathcal{T}\times \mathcal{T} : x_i\text{ and }x_j\text{ should be dissimilar}\},\\
\mathcal{R} & = & \{(z_i,z_j,z_k)\in \mathcal{T}\times \mathcal{T}\times \mathcal{T} : x_i\text{ should be more similar to }x_j\text{ than to }x_k\},
\end{eqnarray*}
where $\mathcal{S}$ and $\mathcal{D}$ are often referred to as the positive and negative training pairs respectively, and $\mathcal{R}$ as the training triplets.
These constraints are usually derived from the labels of the training instances. One may consider for instance all possible pairs/triplets or use only a subset of these, for instance based on random selection or a notion of neighborhood.

Metric learning often has a geometric interpretation: it can be seen as finding a new feature space for the data where the local constraints are better satisfied (see \fref{fig:ml} for an example). Learned metrics are typically used to improve the performance of learning algorithms based on local neighborhoods such as $k$-NN. 

\begin{figure}[t]
\begin{center}
\includegraphics[width=0.9\textwidth]{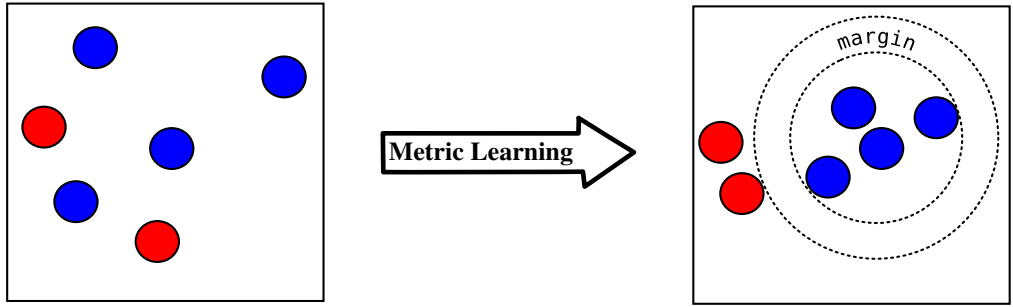}
\caption[Intuition behind metric learning]{Intuition behind metric learning. Before learning (left pane), red and blue points are not well-separated. After learning (right pane), red and blues points are separated by a certain margin.}
\label{fig:ml}
\end{center}
\end{figure}

The rest of this chapter is organized as follows. \sref{sec:mlvect} reviews metric learning approaches where data consist of feature vectors while \sref{sec:mlstruct} deals with metric learning from structured data. We conclude with a summary of the main features of the studied approaches and a discussion on some of their limitations in \sref{sec:mlchapconclu}.

\section{Metric Learning from Feature Vectors}
\label{sec:mlvect}

In this section, we focus on metric learning methods for data lying in some feature space $\mathcal{X}\subseteq\mathbb{R}^d$. In \sref{sec:mahalearning}, we review Mahalanobis distance learning, which has attracted most of the interest, as well as similarity learning in \sref{sec:simlearning} and nonlinear metric learning in \sref{sec:nonlinearml}. Finally, we list a few approaches designed for other settings in \sref{sec:otherml}.

\subsection{Mahalanobis Distance Learning}
\label{sec:mahalearning}

A great deal of work has focused on learning a (squared) Mahalanobis distance $d_\mathbf{M}^2$ parameterized by $\mathbf{M}\in \mathbb{S}^{d}_+$.
Maintaining $\mathbf{M}\in \mathbb{S}^{d}_+$ in an efficient way during the optimization process is a key challenge in Mahalanobis distance learning. Indeed, general Semi-Definite Programming (SDP) techniques \citep{Vandenberghe1996}, i.e., optimization over the PSD cone, consists in repeatedly performing a gradient step on the objective function followed by a projection step onto the PSD cone (which is done by setting the negative eigenvalues to zero). This is slow in practice because it requires eigenvalue decomposition, which scales in $O(d^3)$. Another interesting challenge is to learn a low-rank matrix (which implies a low-dimensional projection space, as noted earlier) instead of a full-rank one, since optimizing $\mathbf{M}$ subject to a rank constraint or regularization is NP-hard and thus cannot be carried out efficiently.

In this section, we review the main supervised Mahalanobis distance learning methods of the literature. We first present two early approaches that deal with the PSD constraint in a rudimentary way (\sref{sec:earlyml}). We then discuss approaches that are specific to $k$-nearest neighbors (\sref{sec:knnml}), inspired from information theory (\sref{sec:itml}), online learning methods (\sref{sec:onlineml}), approaches with generalization guarantees (\sref{sec:genml}) and a few more that do not fit any of the previous categories (\sref{sec:othermaha}).

\subsubsection{Early Approaches}
\label{sec:earlyml}

\paragraph{MMC (Xing et al.)}

The pioneering work of \citet{Xing2002} is the first Mahalanobis distance learning method. It relies on a convex SDP formulation with no regularization, which aims at maximizing the sum of distances between dissimilar points while keeping the sum of distances between similar examples small:
\begin{equation}
\label{eq:xing}
\begin{aligned}
\max_{\mathbf{M}\in \mathbb{S}^{d}_+} &&& \displaystyle\sum_{(z_i,z_j)\in \mathcal{D}}d_\mathbf{M}(\mathbf{x_i},\mathbf{x_j})\\
\text{s.t.} &&& \displaystyle\sum_{(z_i,z_j)\in \mathcal{S}}d_\mathbf{M}^2(\mathbf{x_i},\mathbf{x_j}) \leq 1.
\end{aligned}
\end{equation}
The algorithm for solving \eqref{eq:xing} is a basic SDP approach based on eigenvalue decomposition. This makes it intractable for medium and high-dimensional problems.

\paragraph{Schultz \& Joachims}

The method proposed by \citet{Schultz2003} relies on the assumption that $\mathbf{M} = \mathbf{A}^T\mathbf{W}\mathbf{A}$, where $\mathbf{A}$ is fixed and known and $\mathbf{W}$ is diagonal. We get:
$$d_\mathbf{M}^2(\mathbf{x_i},\mathbf{x_j}) = (\mathbf{A}\mathbf{x_i}-\mathbf{A}\mathbf{x_j})^T\mathbf{W}(\mathbf{A}\mathbf{x_i}-\mathbf{A}\mathbf{x_j}).$$
By definition, $\mathbf{M}$ is PSD and thus one can optimize over the diagonal matrix $\mathbf{W}$ and avoid the need for SDP. They propose a formulation based on triplet constraints:
\begin{equation}
\label{eq:schultz}
\begin{aligned}
 \min_{\mathbf{W}} &&& \|\mathbf{M}\|_{\mathcal{F}}^2\\
 \text{s.t.} &&& d_\mathbf{M}^2(\mathbf{x_i},\mathbf{x_k}) - d_\mathbf{M}^2(\mathbf{x_i},\mathbf{x_j}) \geq 1 && \forall (z_i,z_j,z_k)\in \mathcal{R},
\end{aligned}
\end{equation}
where $\|\cdot\|_{\mathcal{F}}^2$ is the squared Frobenius norm. Slack variables are introduced to allow soft constraints. Problem \eqref{eq:schultz} is convex and can be solved efficiently. The main drawback of this approach is that it is less general than full Mahalanobis distance learning: one only learns a weighting $\mathbf{W}$ of the features. Furthermore, $\mathbf{A}$ must be chosen manually.

\subsubsection{Approaches driven by Nearest Neighbors}
\label{sec:knnml}

The objective functions of the methods presented in this section are related to a nearest neighbor prediction rule.

\paragraph{NCA (Goldberger et al.)}

The idea of Neighborhood Component Analysis (NCA), introduced by \citet{Goldberger2004}, is to optimize the expected leave-one-out error of a stochastic nearest neighbor classifier in the projection space induced by $d_\mathbf{M}$. They use the decomposition $\mathbf{M} = \mathbf{L}^T\mathbf{L}$ and they define the probability that $\mathbf{x_i}$ is the neighbor of $\mathbf{x_j}$ by
\begin{eqnarray*}
p_{ij} = \frac{\exp(-\|\mathbf{L}\mathbf{x_i}-\mathbf{L}\mathbf{x_j}\|^2)}{\sum_{l\neq i}\exp(-\|\mathbf{L}\mathbf{x_i}-\mathbf{L}\mathbf{x_l}\|^2)}, & p_{ii}=0.
\end{eqnarray*}
Then, the probability that $\mathbf{x_i}$ is correctly classified is:
$$p_i = \displaystyle\sum_{j : y_j = y_i}p_{ij}.$$
They learn the distance by solving:
\begin{equation}
\label{eq:nca}
\displaystyle\max_L \sum_i p_i.
\end{equation}
Note that the matrix $\mathbf{L}$ can be chosen nonsquare, inducing a low-rank $\mathbf{M}$.
The main limitation of \eqref{eq:nca} is that it is nonconvex and thus subject to local maxima.

\paragraph{MCML (Globerson \& Roweis)}

Later on, \citet{Globerson2005} proposed an alternative convex formulation based on minimizing a KL divergence between $p_{ij}$ and an ideal distribution. Unlike NCA, this is done with respect to the matrix $\mathbf{M}$. However, like MMC, MCML requires costly projections onto the PSD cone.

\paragraph{LMNN (Weinberger et al.)} 

Large Margin Nearest Neighbors (LMNN), introduced by Weinberger et al. \citeyearpar{Weinberger2005,Weinberger2008a,Weinberger2009}, is one of the most popular Mahalanobis distance learning methods. The idea is to learn the distance such that the $k$ nearest neighbors belong to the correct class while keeping away instances of other classes. The Euclidean distance is used to determine these ``target neighbors''. Formally, the constraints are defined in the following way:
\begin{eqnarray*}
\mathcal{S} & = & \{(z_i,z_j)\in \mathcal{T}\times \mathcal{T} : \ell_i=\ell_j\text{ and }\mathbf{x_j}\text{ belongs to the $k$-neighborhood of }\mathbf{x_i}\}, \\
\mathcal{R} & = & \{(z_i,z_j,z_k)\in \mathcal{T}\times \mathcal{T}\times \mathcal{T} : (z_i,z_j)\in \mathcal{S}, \ell_i\neq\ell_k\}.
\end{eqnarray*}

The distance is learned using the following convex program:
\begin{equation}
\label{eq:lmnn}
\begin{aligned}
\min_{\mathbf{M}\in \mathbb{S}^{d}_+} &&& \displaystyle\sum_{(z_i,z_j)\in \mathcal{S}}d_\mathbf{M}^2(\mathbf{x_i},\mathbf{x_j})\\
\text{s.t.} &&& d_\mathbf{M}^2(\mathbf{x_i},\mathbf{x_k}) - d_\mathbf{M}^2(\mathbf{x_i},\mathbf{x_j}) \geq 1 && \forall (z_i,z_j,z_k)\in \mathcal{R}.
\end{aligned}
\end{equation}
Slack variables are added to get soft constraints. The authors developed a special-purpose solver (based on subgradient descent and careful book-keeping) that is able to deal with billions of constraints. In practice, LMNN is one of the best performing methods, although it is sometimes prone to overfitting due to the absence of regularization, as we will see in \cref{chap:icml}. Note that \citet{Park2011} developed an alternative algorithm for solving \eqref{eq:lmnn} based on column generation while \citet{Do2012} highlighted a relation between LMNN and Support Vector Machines.

\subsubsection{Information-Theoretic Approaches}
\label{sec:itml}

\paragraph{ITML (Davis et al.)}

Information-Theoretical Metric Learning (ITML), proposed by \citet{Davis2007}, is an important work because it introduces LogDet divergence regularization that will later be used in several other Mahalanobis distance learning methods \citep[e.g.,][]{Jain2008,Qi2009}. This Bregman divergence on PSD matrices is defined as:
$$D_{ld}(\mathbf{M},\mathbf{M_0})=\tr(\mathbf{M}\mathbf{M_0}^{-1})-\log\det(\mathbf{M}\mathbf{M_0}^{-1})-d,$$
where $d$ is the dimension of the input space and $\mathbf{M_0}$ is some PSD matrix we want to remain close to. In practice, $\mathbf{M_0}$ is often set to $\mathbf{I}$ (the identity matrix) and thus the regularization aims at keeping the learned distance close to the Euclidean distance. The key feature of the LogDet divergence is that it is finite if and only if $\mathbf{M}$ is PSD. Therefore, minimizing $D_{ld}(\mathbf{M},\mathbf{M_0})$ provides an automatic and cheap way of preserving the positive semi-definiteness of $\mathbf{M}$. The LogDet divergence is also rank-preserving: if the initial matrix $\mathbf{M_0}$ has rank $r$, the learned matrix will also have rank $r$.

ITML is formulated as follows:
\begin{equation}
\label{eq:itml}
\begin{aligned}
\min_{\mathbf{M}\in \mathbb{S}^{d}_+} &&& D_{ld}(\mathbf{M},\mathbf{M_0})\\
\text{s.t.} &&& d_\mathbf{M}^2(\mathbf{x_i},\mathbf{x_j}) \leq u && \forall (z_i,z_j)\in \mathcal{S}\\
 &&& d_\mathbf{M}^2(\mathbf{x_i},\mathbf{x_j}) \geq v && \forall (z_i,z_j)\in \mathcal{D},
\end{aligned}
\end{equation}
where $u,v\in\mathbb{R}$ are threshold parameters (as usual, slack variables are added to get soft constraints). ITML thus aims at satisfying the similarity and dissimilarity constraints while staying as close as possible to the Euclidean distance (if $\mathbf{M_0} = \mathbf{I}$). More precisely, the information-theoretic interpretation behind minimizing $D_{ld}(\mathbf{M},\mathbf{M_0})$ is that it is equivalent to minimizing the KL divergence between two multivariate Gaussian distributions parameterized by $\mathbf{M}$ and $\mathbf{M_0}$.
The algorithm proposed to solve \eqref{eq:itml} is efficient, converges to the global minimum and the resulting distance performs well in practice. A limitation of ITML is that $\mathbf{M_0}$, that must be picked by hand, can have an important influence on the quality of the learned distance.

\paragraph{SDML (Qi et al.)}

With Sparse Distance Metric Learning (SDML), \citet{Qi2009} specifically deal with the case of high-dimensional data together with few training samples, i.e., $n \ll d$. To avoid overfitting, they use a double regularization: the LogDet divergence (using $\mathbf{M_0}=\mathbf{I}$ or $\mathbf{M_0}=\boldsymbol{\Sigma}^{-1}$) and $L_1$-regularization on the off-diagonal elements of $\mathbf{M}$.
The justification for using this $L_1$-regularization is two-fold: (i) a practical one is that in high-dimensional spaces, the off-diagonal elements of $\boldsymbol{\Sigma}^{-1}$ are often very small, and (ii) a theoretical one suggested by a consistency result from a previous work in covariance matrix estimation that applies to SDML.
They use a fast algorithm based on block-coordinate descent (the optimization is done over each row of $\mathbf{M^{-1}}$) and obtain very good performance for the specific case $n \ll d$.

\subsubsection{Online Approaches}
\label{sec:onlineml}

In online learning \citep{Littlestone1988}, the algorithm receives training instances one at a time and updates at each step the current hypothesis. Although the performance of online algorithms is typically inferior to batch algorithms, they are very useful to tackle large-scale problems that batch methods fail to address due to complexity and memory issues.
Online learning methods often come with guarantees in the form of regret bounds, stating that the accumulated loss suffered along the way is not much worse than that of the best hypothesis chosen in hindsight.\footnote{A regret bound has the following general form: $\sum_{t=1}^T \ell(h,z_t) - \sum_{t=1}^T \ell(h^*,z_t) \leq O(T)$, where $T$ is the number of steps and $h^*$ is the best batch hypothesis.} However these results assume that the training pairs/triplets are generated i.i.d. (which is hardly the case in metric learning, as we will discuss later) and do not say anything about the generalization to unseen data.

\paragraph{POLA (Shalev-Shwartz et al.)}

POLA \citep{Shalev-Shwartz2004} is the first online Mahalanobis distance learning approach and learns the matrix $\mathbf{M}$ as well as a threshold $b\geq 1$. At each step, when receiving the pair $(\mathbf{z_i},\mathbf{z_j})$, POLA performs two successive orthogonal projections:
\begin{enumerate}
\item Projection of the current solution $(\mathbf{M^{i-1}},b^{i-1})$ onto $C_1 = \{(\mathbf{M},b)\in \mathbb{R}^{d^2+1}:\mathbf{M},b=[y_iy_j(d_\mathbf{M}^2(\mathbf{x_i},\mathbf{x_j})-b)+1]_+=0\}$, which is done efficiently (closed-form solution). The constraint basically requires that the distance between two instances of same (resp. different) labels be below (resp. above) the threshold $b$ with a margin 1. We get an intermediate solution $(\mathbf{M^{i-\frac{1}{2}}},b^{i-\frac{1}{2}})$ that satisfies this constraint while staying as close as possible to the previous solution.
\item Projection of $(\mathbf{M^{i-\frac{1}{2}}},b^{i-\frac{1}{2}})$ onto $C_2 = \{(\mathbf{M},b)\in \mathbb{R}^{d^2+1}:\mathbf{M}\in \mathbb{S}^{d}_+,b\geq1\}$, which is done rather efficiently (in the worst case, only needs to compute the minimal eigenvalue). This projects the matrix back onto the PSD cone. We thus get a new solution $(\mathbf{M^{i}},b^{i})$ that yields a valid Mahalanobis distance.
\end{enumerate}

A regret bound for the algorithm is provided. However, POLA relies on the unrealistic assumption that there exists $(\mathbf{M^*},b^*)$ such that $[y_iy_j(d_\mathbf{M^*}^2(\mathbf{x_i},\mathbf{x_j})-b^*)+1]_+=0$ for all training pairs (i.e., there exists a matrix and a threshold value that perfectly separate them with margin 1), and is not competitive in practice.

\paragraph{LEGO (Jain et al.)}

LEGO, developed by \citet{Jain2008}, is an improved version of POLA based on LogDet divergence regularization. It features tighter regret bounds, more efficient updates and better practical performance.

\paragraph{ITML (David et al.)} ITML, presented in \sref{sec:itml}, also has an online version with bounded regret. At each step, the algorithm minimizes a tradeoff between LogDet regularization with respect to the previous matrix and a square loss. The resulting distance generally performs slightly worse than the batch version but the algorithm can be faster.

\paragraph{MDML (Kunapuli \& Shavlik)}

The work of \citet{Kunapuli2012} is an attempt of proposing a general framework for online Mahalanobis distance learning. It is based on composite mirror descent \citep{Duchi2010a}, which allows online optimization of many regularized problems. It can accommodate a large class of loss functions and regularizers for which efficient updates are derived, and the algorithm comes with a regret bound. In the experiments, they focus on trace norm regularization, which is the best convex relaxation of the rank and thus induces low-rank matrices.
In practice, the approach has performance comparable to LMNN and ITML, is fast and sometimes induces low-rank solutions, but surprisingly the algorithm was not evaluated on large-scale datasets.

\subsubsection{Metric Learning with Generalization Guarantees}
\label{sec:genml}

As in the classic supervised learning setting (where training data consist of individual labeled instances), generalization guarantees may be derived for supervised metric learning (where training data consist of pairs or triplets). Indeed, most of supervised metric learning methods can be seen as minimizing a (regularized) loss function $\ell$ based on the training pairs/triplets. In this context, the pair-based true risk can be defined as
$$R^\ell(d^2_\mathbf{M}) = \mathbb{E}_{z,z'\sim P}\left[\ell(d^2_\mathbf{M},z,z')\right],$$
the pair-based empirical risk as
$$R^\ell_{\mathcal{S},\mathcal{D}}(d^2_\mathbf{M}) = \frac{1}{|\mathcal{S}|+|\mathcal{D}|}\displaystyle\sum_{(z_i,z_j)\in\mathcal{S}\cup\mathcal{D}}\ell(d^2_\mathbf{M},z_i,z_j),$$
and likewise for the triplet-based setting.

However, although individual training instances are assumed to be drawn i.i.d. from $P$, one cannot make the same assumption regarding the pairs or triplets themselves since they are built from the training sample. For this reason, establishing generalization guarantees for the learned metric is challenging and has so far received very little attention. To the best of our knowledge, only two approaches have tried to address this question explicitly.

\paragraph{Jin et al.}

In their paper, \citet{Jin2009} study the following general Mahalanobis distance learning formulation:
\begin{eqnarray}
\label{eq:jinform}
\displaystyle\min_{\mathbf{M} \succeq 0} & \frac{1}{n^2}\displaystyle\sum_{(z_i,z_j)\in \mathcal{T}\times \mathcal{T}}\ell(d_\mathbf{M}^2,z_i,z_j) \quad+\quad C\|\mathbf{M}\|_\mathcal{F},
\end{eqnarray}
where $C>0$ is the regularization parameter. The loss function $\ell$ is assumed to be of the form
$$\ell(d_\mathbf{M}^2,z_i,z_j) = g(y_iy_j[1-d_\mathbf{M}^2(\mathbf{x_i},\mathbf{x_j})]),$$
where $g$ is convex and Lipschitz continuous.

Relying on a definition of uniform stability adapted to the case of distance learning (where training data is made of pairs), they show that one can derive generalization bounds for the learned distance. Unfortunately, their framework is limited to Frobenius norm regularization: in particular, since it is based on uniform stability, it cannot accommodate sparsity-inducing regularizers. Note that they also propose an online algorithm that is efficient and competitive in practice.

The work of Jin et al. is related to the contributions of this thesis in two ways. First, in \cref{chap:ecml}, we make use of the same uniform stability arguments to derive learning guarantees for a learned edit similarity function, but we go a step further by deriving guarantees in terms of the error of the classifier built from this similarity. Second, in \cref{chap:nips}, we propose an alternative framework for deriving learning guarantees for metric learning based on algorithmic robustness, and we show that this framework can tackle a wider variety of problems.

\paragraph{Bian \& Tao}

The work of Bian \& Tao \citeyearpar{Bian2011,Bian2012} is another attempt of developing metric learning algorithms with generalization guarantees. They consider a class of loss functions similar to that of \citet{Jin2009}:
$$\ell(d_\mathbf{M}^2,z_i,z_j) = g(y_{ij}[c-d_\mathbf{M}^2(\mathbf{x_i},\mathbf{x_j})]),$$
where $c>0$ is a decision threshold variable and $g$ is convex and Lipschitz continuous.

The formulation they study is the following:
\begin{eqnarray}
\displaystyle\min_{(\mathbf{M},c)\in\mathcal{Q}} & \frac{1}{n^2}\displaystyle\sum_{(z_i,z_j)\in \mathcal{T}\times \mathcal{T}}\ell(d_\mathbf{M}^2,z_i,z_j),
\end{eqnarray}
where $\mathcal{Q} = \{(\mathbf{M},c) : 0\preceq \mathbf{M} \preceq \alpha\mathbf{I}, 0\leq c\leq\alpha\}$, with $\alpha$ a positive constant. This ensures that the learned metric and decision threshold are bounded.

They use a statistical analysis to derive risk bounds as well as consistency bounds (the learned distance asymptotically converges to the optimal distance). However, they rely on strong assumptions on the distribution of the examples and cannot accommodate any regularization.

\subsubsection{Other approaches}
\label{sec:othermaha}

In this section, we describe a few approaches that are outside the scope of the previous categories.

\paragraph{Rosales \& Fung} The method of \citet{Rosales2006} aims at learning matrices with entire columns/rows set to zero, thus making $\mathbf{M}$ low-rank. For this purpose, they use $L_1$ norm regularization and, restricting their framework to diagonal dominant matrices, they are able to formulate the problem as a linear program that can be solved efficiently. However, $L_1$ norm regularization favors sparsity at the entry level only, not specifically at the row/column level, even though in practice the learned matrix is sometimes low-rank. Furthermore, the approach is less general than Mahalanobis distances due to the restriction to diagonal dominant matrices.

\paragraph{SML (Ying et al.)}

SML \citep{Ying2009} is a Mahalanobis distance learning approach that regularizes $\mathbf{M}$ with the $L_{2,1}$ norm, which tends to zero out entire rows of $\mathbf{M}$ (as opposed to the $L_1$ norm used in the previous method). They essentially want to solve the following problem:
\begin{equation*}
\begin{aligned}
 \min_{\mathbf{M}\in \mathbb{S}^{d}_+} &&& \|\mathbf{M}\|_{2,1}\\
 \text{s.t.} &&& d_\mathbf{M}^2(\mathbf{x_i},\mathbf{x_k}) - d_\mathbf{M}^2(\mathbf{x_i},\mathbf{x_j}) \geq 1 && \forall (z_i,z_j,z_k)\in \mathcal{R},
\end{aligned}
\end{equation*}
where slack variables are added to get soft constraints. However, $L_{2,1}$ norm regularization is typically difficult to optimize. Using smoothing techniques the authors manage to derive an algorithm that scales in $O(d^3)$ per iteration. The method performs well in practice while inducing a lower-dimensional projection space than full-rank methods and the method of \citet{Rosales2006}. However, it cannot be applied to high-dimensional problems due to the complexity of the algorithm.

\paragraph{BoostMetric (Shen et al.)} BoostMetric \citep{Shen2009,Shen2012} adapts to Mahalanobis distance learning the ideas of boosting, where a good hypothesis is obtained through a weighted combination of so-called ``weak learners'' \citep[see the recent book on this matter by]{Schapire2012}. The method is based on the property that any PSD matrix can be decomposed into a positive linear combination of trace-one rank-one matrices. This kind of matrices is thus used as weak learner and the authors adapt the popular boosting algorithm Adaboost \citep{Freund1995} to this setting. The resulting algorithm is quite efficient since it does not require full eigenvalue decomposition but only the computation of the largest eigenvalue. In practice, BoostMetric achieves competitive performance but hardly scales to large-scale or high-dimensional datasets.

\paragraph{DML (Ying et al.)} The work of \citet{Ying2012} revisit MMC, the original approach of \citet{Xing2002}, by investigating the following formulation, called DML-eig:
\begin{equation}
\label{eq:dml-eig}
\begin{aligned}
\max_{\mathbf{M}\in \mathbb{S}^{d}_+} &&& \displaystyle\min_{(z_i,z_j)\in \mathcal{D}}d^2_\mathbf{M}(\mathbf{x_i},\mathbf{x_j})\\
\text{s.t.} &&& \displaystyle\sum_{(z_i,z_j)\in \mathcal{S}}d_\mathbf{M}^2(\mathbf{x_i},\mathbf{x_j}) \leq 1.
\end{aligned}
\end{equation}
The slight difference is that DML-eig \eqref{eq:dml-eig} maximizes the \emph{minimum} (square) distance between negative pairs while MMC \eqref{eq:xing} maximizes the sum of distances. Ying \& Li avoid the costly full eigen-decomposition used by Xing et al. by showing that \eqref{eq:dml-eig} can be cast as a well-known eigenvalue optimization problem called ``minimizing the maximal eigenvalue of a symmetric matrix''. They further show that it can be solved efficiently using a first-order algorithm that only requires the computation of the largest eigenvalue at each iteration, and that LMNN can also be cast as a similar problem. Experiments show competitive results and low computational complexity, although it might be subject to overfitting due to the absence of regularization.

\citet{Cao2012} generalize \eqref{eq:dml-eig} by studying the following formulation, called DML-$p$:
\begin{equation}
\label{eq:dml-p}
\begin{aligned}
\max_{\mathbf{M}\in \mathbb{S}^{d}_+} &&& \left(\frac{1}{|\mathcal{D}|}\displaystyle\sum_{(z_i,z_j)\in \mathcal{D}}[d_\mathbf{M}(\mathbf{x_i},\mathbf{x_j})]^{2p}\right)^{1/p}\\
\text{s.t.} &&& \displaystyle\sum_{(z_i,z_j)\in \mathcal{S}}d_\mathbf{M}^2(\mathbf{x_i},\mathbf{x_j}) \leq 1.
\end{aligned}
\end{equation}
They show that for $p\in(-\infty,1)$, \eqref{eq:dml-p} is convex and can be solved efficiently in an analogous manner as DML-eig. For $p=0.5$ we recover MMC \eqref{eq:xing} and for $p\to-\infty$ we recover DML-eig \eqref{eq:dml-eig}. Experiments show that tuning $p$ can lead to better performance than MMC or DML-eig.

\paragraph{LNML (Wang et al.)} The idea of LNML \citep{Wang2012} is to enhance metric learning methods by also learning the neighborhood (i.e., the pairs or triplets) according to which the metric is optimized. They propose an iterative approach that alternates between a neighbor assignment step (where the current metric is used to determine the neighbors according to some quality measure) and a metric learning step (where the metric is optimized with respect to the current neighborhood). Experiments are conducted on MCML and LMNN and show that more accurate metrics can be learned using their framework. Of course, this is achieved at the expense of higher computational complexity, since the metric learning algorithms must be run several times (5-10 times in their experiments).

\subsection{Similarity Learning}
\label{sec:simlearning}

Although most of the work in metric learning has focused on the Mahalanobis distance, learning similarity functions has also attracted some interest, motivated by the perspective of more scalable algorithms due to the absence of PSD constraint.

\paragraph{SiLA (Qamar et al.)}

SiLA \citep{Qamar2008} is an approach for learning similarity functions of the following form:
$$\frac{\mathbf{x}^T\mathbf{M}\mathbf{x'}}{N(\mathbf{x},\mathbf{x'})},$$
where $\mathbf{M}\in\mathbb{R}^{d\times d}$ and $N(\mathbf{x},\mathbf{x'})$ is a normalization term which depends on $\mathbf{x}$ and $\mathbf{x'}$. This similarity function can be seen as a generalization of the cosine and the bilinear similarities. The authors build on the same idea of ``target neighbors'' that was introduced in LMNN, but optimize the similarity in an online manner with an algorithm based on voted perceptron. At each step, the algorithm goes through the training set, updating the matrix when an example does not satisfy a criterion of separation. The authors present theoretical results that follow from the voted perceptron theory in the form of regret bounds for the separable and nonseparable cases. SiLA is compared to Mahalanobis metric learning approaches on three datasets. It seems to perform fine but has a rather slow convergence rate and may suffer of its lack of regularization. In subsequent work, \citet{Qamar2012} study the relationship between SiLA and RELIEF, an online feature reweighting algorithm.

\paragraph{gCosLA (Qamar \& Gaussier)}

gCosLA \citep{Qamar2009} learns generalized cosine similarities of the form
$$\frac{\mathbf{x}^T\mathbf{M}\mathbf{x'}}{\sqrt{\mathbf{x}^T\mathbf{M}\mathbf{x}}\sqrt{\mathbf{x'}^T\mathbf{M}\mathbf{x'}}},$$
where $\mathbf{M}\in\mathbb{S}_+^d$. It corresponds to a cosine similarity in the projection space implied by $\mathbf{M}$. The algorithm itself, an online procedure, is very similar to that of POLA (presented in \sref{sec:onlineml}). Indeed, they essentially use the same loss function and also have a two-step approach: a projection onto the set of arbitrary matrices that achieve zero loss on the current example pair, followed by a projection back onto the PSD cone. The first projection is different from POLA (since the generalized cosine has a normalization factor that depends on $\mathbf{M}$) but the authors manage to derive a closed-form solution. The second projection is based on a full eigenvalue decomposition of $\mathbf{M}$, making the approach costly as dimensionality grows. A regret bound for the algorithm is provided and it is shown experimentally that gCosLA converges in fewer iterations than SiLA and is generally more accurate. Its performance seems competitive with LMNN and ITML.

\paragraph{OASIS (Chechik et al.)}

The similarity learning method OASIS \citep{Chechik2009,Chechik2010} learns a bilinear similarity $K_\mathbf{M}$ (see \sref{sec:metricvect}) for large-scale problems. Since $\mathbf{M}\in\mathbb{R}^{d\times d}$ is not required to be PSD, they can optimize the similarity in an online manner using a simple and efficient algorithm, which belongs to the family of Passive-Aggressive algorithms \citep{Crammer2006}. The initialization is $\mathbf{M}=\mathbf{I}$, then at each step $t$, the algorithm draws a triplet $(z_i,z_j,z_k)\in \mathcal{R}$ and solves the following convex problem:
\begin{equation}
\label{eq:oasis}
\begin{aligned}
\mathbf{M^{\boldsymbol{t}}} & = & \displaystyle\argmin_{\mathbf{M},\xi} &&& \frac{1}{2}\|\mathbf{M}-\mathbf{M^{\boldsymbol{t}-1}}\|_{\mathcal{F}}^2+C\xi\\
&& \text{s.t.} &&& 1-d^2_\mathbf{M}(\mathbf{x_i},\mathbf{x_j})+d^2_\mathbf{M}(\mathbf{x_i},\mathbf{x_k})\leq \xi\\
&& &&& \xi \geq 0,
\end{aligned}
\end{equation}
where $C$ is the trade-off parameter between minimizing the loss and staying close from the matrix obtained at the previous step, and $\xi$ is a slack variable. Clearly, if $1-d^2_\mathbf{M}(\mathbf{x_i},\mathbf{x_j})+d^2_\mathbf{M}(\mathbf{x_i},\mathbf{x_k} \leq 0$, then $\mathbf{M^{\boldsymbol{t}}}=\mathbf{M^{\boldsymbol{t}-1}}$ is the solution of \eqref{eq:oasis}. Otherwise, the solution is obtained from a simple closed-form update. In practice, OASIS achieves competitive results on medium-scale problems and unlike most other methods, is scalable to problems with millions of training instances. However, it cannot incorporate complex regularizers and does not have generalization guarantees.

Note that the same authors derived two more algorithms for learning bilinear similarities as applications of more general frameworks. The first one is based on online learning in the manifold of low-rank matrices \citep{Shalit2010,Shalit2012} and the second one on adaptive regularization of weight matrices \citep{Crammer2012}.

\subsection{Nonlinear Metric Learning}
\label{sec:nonlinearml}

We have seen that the work in supervised metric learning from feature vectors has focused on linear metrics because they are more convenient to optimize (in particular, it is easier to derive convex formulations with the guarantee of finding the global optimum) and less prone to overfitting.
However, a drawback of linear metric learning is that it will fail to capture nonlinear patterns in the data.

An example of nonlinear metric learning is kernel learning, but existing approaches are very expensive and/or subject to local minima \citep[see for instance][]{Ong2002,Ong2005,Xu2012b}, cannot be applied to unseen data \citep{Lanckriet2002,Lanckriet2004,Tsuda2005,Kulis2006,Kulis2009} or limited to learning a combination of existing kernels such as in Multiple Kernel Learning \citep[see][for a recent survey]{Gonen2011}.

So far, the most satisfactory solution to the problem of nonlinear metric learning is probably the kernelization of linear metric learning methods, in the spirit of what is done in SVM, i.e., learn a linear metric in the nonlinear feature space induced by a kernel function and thereby combine the best of both worlds. Some metric learning approaches have been shown to be kernelizable \citep[for instance][]{Schultz2003,Shalev-Shwartz2004,Davis2007} using specific arguments, but in general kernelizing a particular metric algorithm is not trivial: a new formulation of the problem has to be derived, where interface to the data is limited to inner products, and sometimes a different implementation is necessary. Moreover, when kernelization is possible, one must learn a $n_\mathcal{T}\times n_\mathcal{T}$ matrix. As $n_\mathcal{T}$ gets large, the problem becomes intractable unless dimensionality reduction is applied. Recently though, several authors \citep{Chatpatanasiri2010,Zhang2010a} have proposed general kernelization methods based on Kernel Principal Component Analysis \citep{Scholkopf1998}. They can be used to kernelize nearly any metric learning algorithm and perform dimensionality reduction simultaneously in a very simple manner, referred to as the ``KPCA trick''. Since our bilinear similarity learning approach introduced in \cref{chap:icml} is kernelized using this trick, we postpone the details to \sref{sec:nonlinear}.

Note that kernelizing a metric learning algorithm may drastically improve the quality of the learned metric on highly nonlinear problems, but may also favor overfitting (because local pair or triplet-based constraints become much easier to satisfy in a nonlinear, high-dimensional kernel space), leading to poor generalization ability.

\subsection{Approaches for Other Settings}
\label{sec:otherml}

In this review, we discussed metric learning approaches for the general supervised learning setting. Note that there also exist methods for the semi-supervised setting \citep{Zha2009,Baghshah2009,Liu2010,Dai2012}, domain adaptation \citep{Cao2011,Geng2011,Kulis2011} and multi-task/view learning \citep{Parameswaran2010,Wang2011,Yang2012}. There also exists specific literature on metric learning for computer vision tasks such as object recognition \citep{Frome2007,Verma2012}, face recognition \citep{Guillaumin2009} or tracking \citep{Li2012}.

\section{Metric Learning from Structured data}
\label{sec:mlstruct}

As pointed out earlier, metrics have a special importance in the context of structured data: they can be used as a proxy to access data without having to manipulate these complex objects. As a consequence, given an appropriate structured metric, one can use $k$-NN, SVM, $K$-Means or any other metric-based algorithm as if the data consisted of feature vectors.

Unfortunately, for the same reasons, metric learning from structured data is challenging because most of structured metrics are combinatorial by nature, which explains why it has received less attention than metric learning from feature vectors. Most of the available literature on the matter focuses on learning metrics based on the edit distance. Clearly, for the edit distance to be meaningful, one needs costs that reflect the reality of the considered task. To take a simple example, in typographical error correction, the probability that a user hits the Q key instead of W on a QWERTY keyboard is much higher than the probability that he hits Q instead of Y. For some applications, such as protein alignment or handwritten digit recognition, well-tailored cost matrices may be available \citep{Dayhoff1978,Henikoff1992,Mico1998}. Otherwise, there is a need for automatically learning a nonnegative $(|\Sigma|+1)\times (|\Sigma|+1)$ cost matrix $\mathbf{C}$ for the task at hand.

What makes the cost matrix difficult to optimize is the fact that the edit distance is based on an optimal script which depends on the edit costs themselves. Most general-purpose approaches get round this problem by considering a stochastic variant of the edit distance, where the cost matrix defines a probability distribution over the edit operations. One can then define an edit similarity equal to the posterior probability $p_e(\mathsf{x'}|\mathsf{x})$ that an input string $\mathsf{x}$ is turned into an output string $\mathsf{x'}$. This corresponds to summing over all possible edit scripts that turn $\mathsf{x}$ into $\mathsf{x'}$ instead of only considering the optimal script. Such a stochastic edit process can be represented as a probabilistic model and one can estimate the parameters (i.e., the cost matrix) of the model that maximize the expected log-likelihood of positive pairs. This is done via an iterative Expectation-Maximization (EM) algorithm \citep{Dempster1977}, a procedure that alternates between two steps: an Expectation step (which essentially computes the function of the expected log-likelihood of the pairs with respect to the current parameters of the model) and a Maximization step (computing the updated edit costs that maximize this expected log-likelihood). Note that unlike the classic edit distance, the obtained edit similarity does not usually satisfy the properties of a distance (in fact, it is often not symmetric).

In the following, we review methods for learning string edit metrics (\sref{sec:stringeditlearn}) and tree edit metrics (\sref{sec:treeeditlearn}).

\subsection{String Edit Metric Learning}
\label{sec:stringeditlearn}

\paragraph{Generative models} The first method for learning a string edit metric was proposed by \citet{Ristad1998}. They use a memoryless stochastic transducer which models the joint probability of a pair $p_e(\mathsf{x},\mathsf{x'})$ from which $p_e(\mathsf{x'}|\mathsf{x})$ can be estimated. Parameter estimation is performed with EM and the learned edit probability is applied to the problem of learning word pronunciation in conversational speech. \citet{Bilenko2003} extended this approach to the Needleman-Wunsch Score with affine gap penalty and applied it to duplicate detection. To deal with the tendency of Maximum Likelihood estimators to overfit when the number of parameters is large (in this case, when the alphabet size is large), \citet{Takasu2009} proposes a Bayesian parameter estimation of pair-HMM providing a way to smooth the estimation. Experiments are conducted on approximate text searching in a digital library of Japanese and English documents.

\paragraph{Discriminative models} The work of \citet{Oncina2006} describes three levels of bias induced by the use of generative models: (i) dependence between edit operations, (ii) dependence between the costs and the prior distribution of strings $p_e(\mathsf{x})$, and (iii) the fact that to obtain the posterior probability one must divide by the empirical estimate of $p_e(\mathsf{x})$. These biases are highlighted by empirical experiments conducted with the method of \citet{Ristad1998}. To address these limitations, they propose the use of a conditional transducer that directly models the posterior probability $p_e(\mathsf{x'}|\mathsf{x})$ that an input string $\mathsf{x}$ is turned into an output string $\mathsf{x'}$ using edit operations. Parameter estimation is also done with EM and the paper features an application to handwritten digit recognition, where digits are represented as sequences of Freeman codes \citep{Freeman1974}. In order to allow the use of negative pairs, \citet{McCallum2005} consider another discriminative model, conditional random fields, that can deal with positive and negative pairs in specific states, still using EM for parameter estimation.

\paragraph{Methods based on gradient descent} The use of EM has two main drawbacks: (i) it may converge to a local optimum, and (ii) parameter estimation and distance calculations must be done at each iteration, which can be very costly if the size of the alphabet and/or the length of the strings are large.

\citet{Saigo2006} manage to avoid the need for an iterative procedure like EM in the context of detecting remote homology in protein sequences. They learn the parameters of the Smith-Waterman score which is plugged in their local alignment kernel \citep{Saigo2004}. Unlike the Smith-Waterman score, the local alignment kernel, which is based on the sum over all possible alignments, is differentiable and can be optimized by a gradient descent procedure. The objective function that they optimize is meant to favor the discrimination between positive and negative examples, but this is done by only using positive pairs of distant homologs. The approach has two additional drawbacks: (i) the objective function is nonconvex and it thus subject to local minima, and (ii) the kernel's validity is not guaranteed in general and is subject to the value of a parameter that must be tuned. Therefore, the authors use this learned function as a similarity measure and not as a kernel.

\subsection{Tree Edit Metric Learning}
\label{sec:treeeditlearn}

\paragraph{Bernard et al.} Extending the work of \citet{Ristad1998} and \citet{Oncina2006} on string edit similarity learning, \citet{Bernard2006,Bernard2008} propose both a generative and a discriminative model for learning tree edit costs. They rely on the tree edit distance by \citet{Selkow1977} --- which is cheaper to compute than that of \citet{Zhang1989} --- and adapt the updates of EM to this case. An application to handwritten digit recognition is proposed, where digits are represented by trees of Freeman codes.

\paragraph{Boyer et al.} The work of \citet{Boyer2007} tackles the more complex variant of the tree edit distance \citep{Zhang1989}, which allows the insertion and deletion of single nodes instead of entire subtrees only. Parameter estimation in the generative model is also based on EM, and the usefulness of the approach is illustrated on an image recognition task.

\paragraph{Neuhaus \& Bunke} In their paper, \citet{Neuhaus2007} learn a (more general) graph edit similarity, where each edit operation is modeled by a Gaussian mixture density. Parameter estimation is done using an EM-like algorithm. Unfortunately, the approach is intractable: the complexity of the EM procedure is exponential in the number of nodes (and so is the computation of the distance).

\paragraph{Dalvi et al.} The work of \citet{Dalvi2009} points out a limitation of the approach of \citet{Bernard2006,Bernard2008}: they model a distribution over tree edit scripts rather than over the trees themselves, and unlike the case of strings, there is no bijection between the edit scripts and the trees. Recovering the correct conditional probability with respect to trees requires a careful and costly procedure. They propose a more complex conditional transducer that models the conditional probability over trees and use EM for parameter estimation. They apply their method to the problem of creating robust wrappers for webpages.

\paragraph{Emms} The work of \citet{Emms2012} points out a theoretical limitation of the approach of \citet{Boyer2007}: the authors use a factorization that turns out to be incorrect in some cases. Emms shows that a correct factorization exists when only considering the edit script of highest probability instead of all possible scripts, and derives the corresponding EM updates. An obvious drawback is that the output of the model is not the probability $p_e(\mathsf{x'}|\mathsf{x})$. Moreover, experiments on a question answering task highlight that the approach is prone to overfitting, and requires smoothing and other heuristics (such as a final step of zeroing-out the diagonal of the cost matrix).

\section{Conclusion}
\label{sec:mlchapconclu}

In this chapter, we reviewed a large body of work in supervised metric learning. \tref{tab:mlvectsum} and \tref{tab:mlstructsum}  summarize the main features of the studied approaches for feature vectors and structured data respectively. 

\begin{table}[t]
\begin{center}
\begin{scriptsize}
\begin{tabular}{rccccccc}
\toprule
\textbf{Method} & \textbf{Convex} & \textbf{Scalable} & \textbf{Competitive} & \textbf{Reg.} & \textbf{Low-rank} & \textbf{Online} & \textbf{Gen.}\\
\midrule
MMC & \tickYes & \tickNo & \tickNo & \tickNo & \tickNo & \tickNo & \tickNo\\
Schultz \& Joachims & \tickYes & \tickYes & \tickNo & \tickYes & \tickNo & \tickNo & \tickNo\\
NCA & \tickNo & \tickYes & \tickNo & \tickNo & \tickYes & \tickNo & \tickNo\\
MCML & \tickYes & \tickNo &\tickNo & \tickNo & \tickNo & \tickNo & \tickNo\\
LMNN & \tickYes & \tickYes\tickYes & \tickYes & \tickNo & \tickNo & \tickNo & \tickNo\\
ITML & \tickYes & \tickYes\tickYes & \tickYes & \tickYes & \tickYes & \tickYes & \tickNo\\
SDML & \tickYes & \tickYes\tickYes & \tickYes & \tickYes & \tickNo & \tickNo & \tickNo\\
POLA & \tickYes & \tickYes & \tickNo & \tickNo & \tickNo & \tickYes & \tickNo\\
LEGO & \tickYes & \tickYes\tickYes & \tickYes & \tickYes & \tickYes & \tickYes & \tickNo\\
MDML & \tickYes & \tickYes\tickYes & \tickYes & \tickYes & \tickYes & \tickYes & \tickNo\\
Jin et al. & \tickYes & \tickYes\tickYes & \tickYes & \tickYes & \tickNo & \tickYes & \tickYes\\
Bian \& Tao & \tickYes & \tickYes & \tickYes & \tickNo & \tickNo & \tickNo & \tickYes\\
Rosales \& Fung & \tickYes & \tickYes & \tickNo & \tickYes & \tickYes & \tickNo & \tickNo\\
SML & \tickYes & \tickNo & \tickYes & \tickYes & \tickYes & \tickNo & \tickNo\\
BoostMetric & \tickYes & \tickYes & \tickYes & \tickYes & \tickNo & \tickNo & \tickNo\\
DML & \tickYes & \tickYes\tickYes & \tickYes & \tickNo & \tickNo & \tickNo & \tickNo\\
SiLA & --- & \tickYes & ? & \tickNo & \tickNo & \tickYes & \tickNo\\
gCosLA & \tickYes & \tickYes & \tickYes & \tickNo & \tickNo & \tickYes & \tickNo\\
OASIS & \tickYes & \tickYes\tickYes\tickYes & \tickYes & \tickYes & \tickNo & \tickYes & \tickNo\\
\bottomrule
\end{tabular}
\end{scriptsize}
\caption[Metric learning from feature vectors: main features of the methods]{Summary of the main features of the reviewed approaches (``Reg.'' and ``Gen.'' respectively stand for ``Regularized'' and ``Generalization guarantees'').}
\label{tab:mlvectsum}
\end{center}
\end{table}

\begin{table}[t]
\begin{center}
\begin{scriptsize}
\begin{tabular}{rccccccc}
\toprule
\textbf{Method} & \textbf{Data} & \textbf{Model} & \textbf{Scripts} & \textbf{Opt.} & \textbf{Global sol.} & \textbf{Neg. pairs} & \textbf{Gen.}\\
\midrule
Ristad \& Yianilos & Strings & Generative & All & EM & \tickNo & \tickNo & \tickNo\\
Bilenko \& Mooney & Strings & Generative & All & EM & \tickNo & \tickNo & \tickNo\\
Takasu & Strings & Generative & All & EM & \tickNo & \tickNo & \tickNo\\
Oncina \& Sebban & Strings & Discriminative & All & EM & \tickNo & \tickNo & \tickNo\\
McCallum et al. & Strings & Discriminative & All & EM & \tickNo & \tickYes & \tickNo\\
Saigo et al. & Strings & --- & All & GD & \tickNo & \tickNo  & \tickNo\\
Bernard et al. & Trees & Both & All & EM & \tickNo & \tickNo & \tickNo\\
Boyer et al. & Trees & Generative & All & EM & \tickNo & \tickNo & \tickNo\\
Neuhaus \& Bunke & Graphs & Generative & All & EM & \tickNo & \tickNo & \tickNo\\
Dalvi et al. & Trees & Discriminative & All & EM & \tickNo & \tickNo & \tickNo\\
Emms & Trees & Discriminative & Optimal & EM & \tickNo & \tickNo & \tickNo\\
\bottomrule
\end{tabular}
\end{scriptsize}
\caption[Metric learning from structured data: main features of the methods]{Summary of the main features of the reviewed approaches (``Opt.'', ``Global sol.'', ``Neg. pairs'' and ``Gen.'' respectively stand for ``Optimization'', ``Global solution'', ``Negative pairs'' and ``Generalization guarantees'').}
\label{tab:mlstructsum}
\end{center}
\end{table}

This review raises three observations:
\begin{enumerate}
\item Research efforts on metric learning from feature vectors have been mainly oriented towards deriving tractable formulations and algorithms. Boosted by some advances in batch and online numerical optimization, these efforts have been successful: recent methods are scalable and can even accommodate complex regularizers in an efficient way. However, there is an obvious lack of theoretical understanding of metric learning. First, few frameworks capable of establishing the consistency of the learned metric on unseen data have been proposed, and existing ones lack generality. Second, using a learned metric often improves the empirical performance of metric-based algorithms, but this has never been studied from a theoretical standpoint. In particular, can we relate the empirical risk of the learned metric to the true risk of the classifier that uses it?
\item There is a relatively small body of work on metric learning from structured data, presumably due to the higher complexity of the learning procedures. Almost all existing methods are based on probabilistic models: they are trained using an expensive iterative algorithm and cannot accommodate negative pairs. Furthermore, no approach is guaranteed to converge to the global optimum of the optimized quantity and again, there is a lack of theoretical study.
\item The use of learned metrics is typically restricted to algorithms based on local neighborhoods, in particular $k$-NN classifiers. Since the learned metrics are typically optimized over local constraints, it seems unclear whether they can be successfully used in more global classifiers such as SVM and other linear separators, or if new metric learning algorithms should be designed for this global setting. Furthermore, building a PSD kernel from the learned metrics is often difficult, especially for structured data (e.g., string edit kernels).
\end{enumerate}

The contributions of this thesis address these limitations. \pref{part:struct} is devoted to metric learning from structured data and consists of two main contributions. In \cref{chap:pr}, we introduce a new string kernel built from learned edit probabilities. Unlike other string edit kernels, it is guaranteed to be PSD and parameter-free. In \cref{chap:ecml}, we propose a novel string and tree edit similarity learning method based on numerical optimization, that can handle positive and negative pairs and is guaranteed to converge to the optimal solution. We are able to derive a generalization bound for our method, and this bound can be related to the generalization error of a linear classifier built from the learned similarity. \pref{part:vect} is devoted to metric learning from feature vectors and consists of two contributions. In \cref{chap:icml}, we propose a bilinear similarity learning method tailored to linear classification. The similarity is not optimized over local pair or triplet-based constraints: it directly minimizes a global quantity that upper bounds the true risk of the linear classifier built from the learned similarity. Lastly, in \cref{chap:nips}, we adapt the notion of algorithmic robustness (\sref{sec:robustness}) to the metric learning setting, which allows us to derive generalization guarantees for a large class of metric learning problems with various loss functions and regularizers.

\part{Contributions in Metric Learning from Structured Data}\label{part:struct}
\chapter{A String Kernel Based on Learned Edit Similarities}
\label{chap:pr}

\begin{chapabstract}
With the success of kernel methods, there is a growing interest in designing powerful kernels between sequences. In this chapter, we propose a new string kernel based on edit probabilities learned with a conditional transducer. Unlike other string edit kernels, it is parameter-free and guaranteed to be valid since it corresponds to a dot product in an infinite-dimensional space. While the naive computation of the kernel involves an intractable sum over an infinite number of strings, we show that it can actually be computed exactly using the intersection of probabilistic automata and a matrix inversion. Experimental results on a handwritten character recognition task show that our new kernel outperforms state-of-the-art string kernels as well as standard and learned edit distance used in a $k$-NN classifier.\\

The material of this chapter is based on the following international publication:\\

\bibentry{Bellet2010}.
\end{chapabstract}

\section{Introduction}

In recent years, with the emergence of kernel-based learning, a lot of research has gone into designing powerful kernels for structured data such as strings. A natural way of building string kernels consists in representing each sequence by a fixed-length feature vector. Many of the early string kernels, such as the spectrum, the subsequence or the mismatch kernels (presented in \sref{sec:metricstruct}), belong to this family. They sometimes perform well, but they are not very flexible and imply a significant loss of structural information.

On the other hand, measures based on (or related to) the string edit distance can capture more structural distortions and are adaptable by nature, since they are based on a cost matrix that can be used to incorporate background knowledge on the problem of interest. When a domain expertise is not available, one may learn these costs automatically from data (we have reviewed these methods in \sref{sec:mlstruct}). Unfortunately, their use is mostly restricted to $k$-NN classifiers since efforts to design string kernels from the edit distance (the so-called edit kernels) have not been satisfactory: their validity (i.e., positive semi-definiteness) is subject to the value of a parameter (that must be tuned) and/or they suffer from the ``diagonal dominance'' problem \citep{Li2004,Cortes2004,Saigo2004,Neuhaus2006}. Another drawback of these approaches is that they use the standard version of the edit distance. Adapting them to make use of learned edit similarities (that are not proper distances and sometimes not even symmetric) is often not straightforward.

In this work, we propose a new string edit kernel that makes use of conditional edit probabilities learned in the form of probabilistic models. Our kernel belongs to the family of marginalized kernels \citep{Tsuda2002,Kashima2003}, is parameter-free and guaranteed to be PSD. It also has the unusual feature of being based on a sum over an infinite number of strings. This sum may seem intractable at first glance, but drawing our inspiration from rational kernels \citep{Cortes2004}, we show that it can be computed exactly by means of the intersection of two probabilistic automata and a matrix inversion. We conduct experiments on a handwritten digit recognition task that show that our kernel outperforms state-of-the-art string kernels, as well as $k$-NN classifiers based on standard and learned edit distance measures.

The rest of this chapter is organized as follows. \sref{sec:oureditkernel} introduces our new string edit kernel. \sref{sec:editkernelcomp} is devoted to the computation of the kernel based on intersection of probabilistic automata and matrix inversion. Experimental results are presented in \sref{sec:prexpes} and we conclude in \sref{sec:prconclu}.

\section{A New Marginalized String Edit Kernel}
\label{sec:oureditkernel}

Our new string edit kernel belongs to the family of marginalized kernels \citep{Tsuda2002,Kashima2003}. Let $p(\mathsf{x},\mathsf{x'},v)$ be the probability of observing jointly a hidden variable $v \in {\cal V}$ and two observable strings $\mathsf{x},\mathsf{x'} \in \Sigma^*$. The probability $p(\mathsf{x},\mathsf{x'})$ can be obtained by marginalizing, {\it i.e.} summing over all variables $v \in {\cal V}$, the probability $p(\mathsf{x},\mathsf{x}',v)$, such that:
$$p(\mathsf{x},\mathsf{x'})=\sum_{v \in {\cal V}}p(\mathsf{x},\mathsf{x'},v)=\sum_{v \in {\cal V}}p(\mathsf{x},\mathsf{x'}|v) \cdot p(v).$$
A marginalized kernel computes this probability making the assumption that $\mathsf{x}$ and $\mathsf{x'}$ are conditionally independent given $v$, i.e.,
\begin{equation}
\label{eq:marginkernel1}
K(\mathsf{x},\mathsf{x'})=\sum_{v \in {\cal V}}p(\mathsf{x}|v) \cdot p(\mathsf{x'}|v) \cdot p(v).
\end{equation}
Note that the computation of this kernel is possible since it is assumed that ${\cal V}$ is a {\it finite set}. Let us now suppose that $p(v|\mathsf{x})$ is known instead of $p(\mathsf{x}|v)$. Then, as described in \citep{Tsuda2002}, we can use the following marginalized kernel:
\begin{equation}
\label{eq:marginkernel2}
K(\mathsf{x},\mathsf{x'})=\sum_{v \in  {\cal V}}p(v|\mathsf{x}) \cdot p(v|\mathsf{x'}) \cdot K_c(c,c'),
\end{equation}
where $K_c(c,c')$ is the joint kernel depending on combined variables $c=(\mathsf{x},v)$ and $c'=(\mathsf{x'},v)$.
An interesting way to exploit the kernel in \eref{eq:marginkernel2} as a string edit kernel is to do the following:
\begin{itemize}
\item replace the finite set ${\cal V}$ of variables $v$ by the \emph{infinite} set of strings $\mathsf{s} \in \Sigma^*$,
\item use $p(\mathsf{s}|\mathsf{x}) = p_e(\mathsf{s}|\mathsf{x})$, which is the conditional probability that a string $\mathsf{x}$ is turned into a string $\mathsf{s}$ through edit operations,
\item and take $K_c(c,c')$ to be the constant kernel that returns $1$ for all $c,c'$.
\end{itemize}
We then obtain the following new string edit kernel:
\begin{equation}
\label{eq:oureditkernel}
K_e(\mathsf{x},\mathsf{x'})=\sum_{\mathsf{s} \in  \Sigma^*}p_e(\mathsf{s}|\mathsf{x}) \cdot p_e(\mathsf{s}|\mathsf{x'}),
\end{equation}
which is PSD as it corresponds to the inner product in the Hilbert space defined by the mapping $\phi(\mathsf{x})= [p_e(\mathsf{s}|\mathsf{x})]_{\mathsf{s} \in {\Sigma^*}}$. Like the popular Gaussian kernel for feature vectors, $K_e$ projects the data into an infinite-dimensional space. Intuitively, $K_e(\mathsf{x},\mathsf{x'})$ is large when $\mathsf{x}$ and $\mathsf{x'}$ have a high probability to be turned into the same strings $\mathsf{s} \in \Sigma^*$ using edit operations.

We have already seen in \sref{sec:stringeditlearn} that there exist methods in the literature for learning $p_e(\mathsf{x'}|\mathsf{x})$ for all $\mathsf{x},\mathsf{x'}$. However, our new kernel is intractable in its current form since it involves the computation of an infinite sum over $\Sigma^*$. In the next section, we present a way of computing this infinite sum exactly and in an efficient way.

\section{Computing the Edit Kernel}
\label{sec:editkernelcomp}

While the original marginalized kernel \eqref{eq:marginkernel1} assumes that ${\cal V}$ is a finite set of variables \citep{Tsuda2002}, our string edit kernel includes an infinite sum over $\Sigma^*$. In this section, we show that (i) given two strings $\mathsf{x}$ and $\mathsf{x'}$, $p_e(\mathsf{s}|\mathsf{x})$ and $p_e(\mathsf{s}|\mathsf{x'})$ can be represented in the form of two probabilistic automata, (ii) the product $p_e(\mathsf{s}|\mathsf{x}) \cdot p_e(\mathsf{s}|\mathsf{x'})$ can be performed by intersecting the languages represented by those automata and (iii) the infinite sum over $\Sigma^*$ can then be computed by algebraic methods.

\subsection{Definitions and Notations}

We first introduce some definitions and notations regarding probabilistic transducers.

\begin{definition}
A weighted finite-state transducer (WFT) is an 8-tuple $T=(\Sigma,\Delta, \mathcal{Q},$ $\mathcal{I}, \mathcal{F}, w, \tau, \rho)$ where $\Sigma$ is the input alphabet, $\Delta$ the output alphabet, $\mathcal{Q}$ a finite set of states, $\mathcal{I} \subseteq \mathcal{Q}$ the set of initial states, $\mathcal{F} \subseteq \mathcal{Q}$ the set of final states, $w : \mathcal{Q} \times \mathcal{Q} \times (\Sigma \cup \{ \$\}) \times (\Delta \cup \{ \$\}) \to \mathbb{R}$ the transition weight function, $\tau:\mathcal{I} \to \mathbb{R}$ the initial weight function, and $\rho: \mathcal{F} \to \mathbb{R}$ the final weight function. For notational convenience, we denote $w(q_1, q_2, \mathsf{a}, \mathsf{b})$ by $w_{q_1 \to q_2}(\mathsf{a}, \mathsf{b})$ for any $q_1,q_2\in\mathcal{Q}$, $\mathsf{a}\in\Sigma$ and $\mathsf{b}\in\Delta$.
\end{definition}

\begin{definition}
A joint probabilistic finite-state transducer (jPFT) is a WFT $J=(\Sigma, \Delta, \mathcal{Q}, \mathcal{S}, \mathcal{F}, w, \tau, \rho)$ which defines a joint probability distribution over pairs of strings $\{(\mathsf{x},\mathsf{x'}) \in \Sigma^*\times \Delta^* \}$. A jPFT must satisfy the following four constraints:
\begin{enumerate}
\item The initial, final and transition weights have nonnegative values.
\item $\sum_{i \in \mathcal{S}}\tau (i) = 1$,
\item $\sum_{f \in \mathcal{F}}\rho (f)= 1$,
\item $\forall q_1 \in \mathcal{Q} : \displaystyle\sum_{q_2 \in \mathcal{Q}, \mathsf{a} \in \Sigma \cup \{\$ \}, \mathsf{b} \in \Delta \cup \{ \$ \}} w_{q_1 \to q_2}(\mathsf{a}, \mathsf{b}) = 1.$
\end{enumerate}
\end{definition}

\begin{definition}
A conditional probabilistic finite-state transducer (cPFT) is a WFT $C=(\Sigma, \Delta, Q, S, F, w, \tau, \rho)$ which defines a conditional probability distribution over the output strings $\mathsf{x'} \in \Delta^*$ given an input string $\mathsf{x}\in \Sigma^*$. For $q_1,q_2\in\mathcal{Q}$, $\mathsf{a}\in\Sigma$ and $\mathsf{b}\in\Delta$, we denote the transition $w_{q_1 \to q_2}(a,b)$ in the conditional form $w_{q_1 \to q_2}(b|a)$.
A cPFT must satisfy the same first two constraints as those of a jPFT and the following third constraint \citep[see][for a proof]{Oncina2006}:
$$\forall q_1 \in \mathcal{Q}, \forall \mathsf{a} \in \Sigma : \sum_{q_2 \in \mathcal{Q}, \mathsf{b} \in \Delta \cup \{ \$ \}} w_{q_1 \to q_2}(\mathsf{b}|\mathsf{a}) + w_{q_1 \to q_2}(\mathsf{b}|\$) = 1.$$
\end{definition}

An example of memoryless cPFT (i.e., with only one state) is shown in \fref{fig:memoryless}, where $\Sigma=\Delta=\{a,b\}$ and $\mathcal{Q}$ is composed of only one state labeled by $0$. Initial states are designated by an inward arrow that has no source state, while final states are denoted by a double circle. In \fref{fig:memoryless}, state $0$ is both initial and final.

In the following, since our string edit kernel is based on conditional edit probabilities, we will assume that a cPFT has already been learned by one of the previously mentioned methods, for instance that of \citet{Oncina2006} that we describe in more details in \aref{app:pr} for the sake of completeness. Note that the cPFT is learned only once and is then used to compute our edit kernel for any pair of strings. If a generative model is used to learn the edit parameters \citep[e.g., that of][]{Ristad1998}, the resulting jPFT can be renormalized into a cPFT a posteriori.

\begin{figure}[t]
  \begin{center}
  \includegraphics[width=0.4\textwidth]{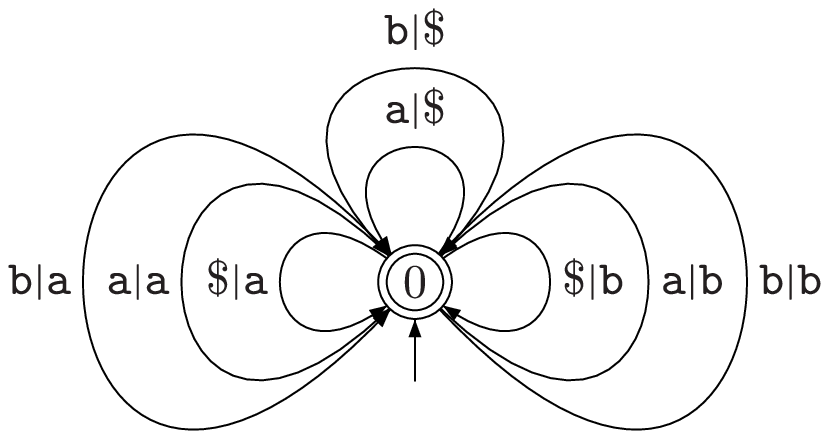}
  \caption[An example of memoryless cPFT]{A memoryless cPFT that can be used to compute the edit conditional probability of any pair of strings. Edit probabilities assigned to each transition are not shown here for the sake of readability.}
  \label{fig:memoryless}
  \end{center}
\end{figure}

\subsection{Modeling $p_e(\mathsf{s}|\mathsf{x})$ and $p_e(\mathsf{s}|\mathsf{x'})$ with Probabilistic Automata}

Since our edit kernel $K_e(\mathsf{x},\mathsf{x'})$ depends on two observable strings $\mathsf{x}$ and $\mathsf{x'}$, it is possible to represent the  distributions $p_e(\mathsf{s}|\mathsf{x})$ and $p_e(\mathsf{s}|\mathsf{x'})$ in the form of probabilistic state machines, where only $\mathsf{s}$ is a hidden variable. Given a cPFT $T$ modeling the edit probabilities and a string $\mathsf{x}$, we can define a new cPFT driven by $\mathsf{x}$, denoted by $T|\mathsf{x}$, that models $p_e(\mathsf{s}|\mathsf{x})$.

\begin{definition}
  Let $T=(\Sigma, \Delta, \mathcal{Q}, \mathcal{S}, \mathcal{F}, w, \tau, \rho)$ be a cPFT that models $p_e(\mathsf{x'}|\mathsf{x})$, {\bf $\forall \mathsf{x'} \in \Delta^*, \forall \mathsf{x} \in \Sigma^*$}. We define $T|\mathsf{x}$ as a cPFT that models $p_e(\mathsf{s}|\mathsf{x}), \forall \mathsf{s} \in \Delta^*$ and a specific observable $\mathsf{x} \in \Sigma^*$. $T|\mathsf{x}=(\Sigma, \Delta, \mathcal{Q}', \mathcal{S}', \mathcal{F}', w', \tau', \rho')$ with:
  \begin{itemize}
  \item $\mathcal{Q}' =  \{ [\mathsf{x}]_i \} \times \mathcal{Q}$ where $[\mathsf{x}]_i$ is the prefix of length $i$ of $\mathsf{x}$ (note that $[\mathsf{x}]_0 = \$$). In other words, $\mathcal{Q}'$ is a finite set of states labeled by the current prefix of $\mathsf{x}$ and its corresponding state during its parsing in $T$.\footnote{This specific notation is required to deal with nonmemoryless cPFT.}
  \item $\mathcal{S}' = \{ (\$, q) \}$ where $q \in \mathcal{S}$;
  \item $\forall q \in \mathcal{S}, \tau'((\$, q)) = \tau(q)$;
  \item $\mathcal{F}' = \{ (\mathsf{x}, q) \}$ where $q \in \mathcal{F}$;
  \item $\forall q \in \mathcal{F}, \rho'((\mathsf{x}, q)) = \rho(q)$;
  \item the  following two rules are used to define the transition weight function:
    \begin{itemize}
    \item $\forall \mathsf{b} \in \Delta \cup \{\$\}, \forall q_1,q_2 \in \mathcal{Q},  w'_{([\mathsf{x}]_i, q_1) \to ([\mathsf{x}]_{i+1}, q_2)} (\mathsf{b}|\mathsf{x_{i+1}}) = w_{q_1 \to q_2}(\mathsf{b}|\mathsf{x_{i+1}})$,
    \item $\forall \mathsf{b} \in \Delta, \forall q_1,q_2 \in \mathcal{Q}, w'_{([\mathsf{x}]_i, q_1) \to ([\mathsf{x}]_i, q_2)} (b|\$) = w_{q_1 \to q_2}(\mathsf{b}|\$)$.
    \end{itemize}
  \end{itemize}
\end{definition}

As an example, given two strings $\mathsf{x}=\mathtt{a}$ and $\mathsf{x'}=\mathtt{ab}$, \fref{fig:PDFA} shows the cPFT $T|\mathtt{a}$ and $T|\mathtt{ab}$ constructed from the memoryless transducer $T$ given in \fref{fig:memoryless}. Roughly speaking, $T|\mathtt{a}$ and $T|\mathtt{ab}$ model the output languages that can be generated through edit operations from $\mathsf{x}$ and $\mathsf{x'}$ respectively. Therefore, from these state machines, we can generate output strings and compute the conditional edit probabilities $p_e(\mathsf{s}|\mathsf{x})$ and $p_e(\mathsf{s}|\mathsf{x'})$ for any string $\mathsf{s} \in \Delta^*$. Note that the cycles outgoing from each state model the possible insertions before and after reading an input symbol.

Since the construction of $T|\mathsf{x}$ and $T|\mathsf{x'}$ is driven by the parsing of $\mathsf{x}$ and $\mathsf{x'}$ in $T$, we can omit the input alphabet $\Sigma$. 
Therefore, a transducer $T|\mathsf{x}=(\Sigma, \Delta, \mathcal{Q}, \mathcal{S}, \mathcal{F}, w, \tau, \rho)$ can be reduced to a finite-state automaton $A|\mathsf{x}=(\Delta, \mathcal{Q}, \mathcal{S}, \mathcal{F}, w', \tau, \rho)$. The transitions of $A|\mathsf{x}$ are derived from $w$ in the following way: $w'_{q_1 \to q_2}(\mathsf{b}) = w_{q_1 \to q_2}(\mathsf{b}|\mathsf{a}), \forall \mathsf{b} \in \Delta \cup \{\$\}, \forall \mathsf{a} \in \Sigma \cup \{\$\}, \forall q_1,q_2 \in \mathcal{Q}$. For example, \fref{fig:from_t_to_a} shows the resulting automata deduced from the cPFT $T|\mathtt{a}$ and $T|\mathtt{ab}$ depicted in \fref{fig:PDFA}.



\begin{figure}[t]
  \begin{center}
  \includegraphics[width=0.7\textwidth]{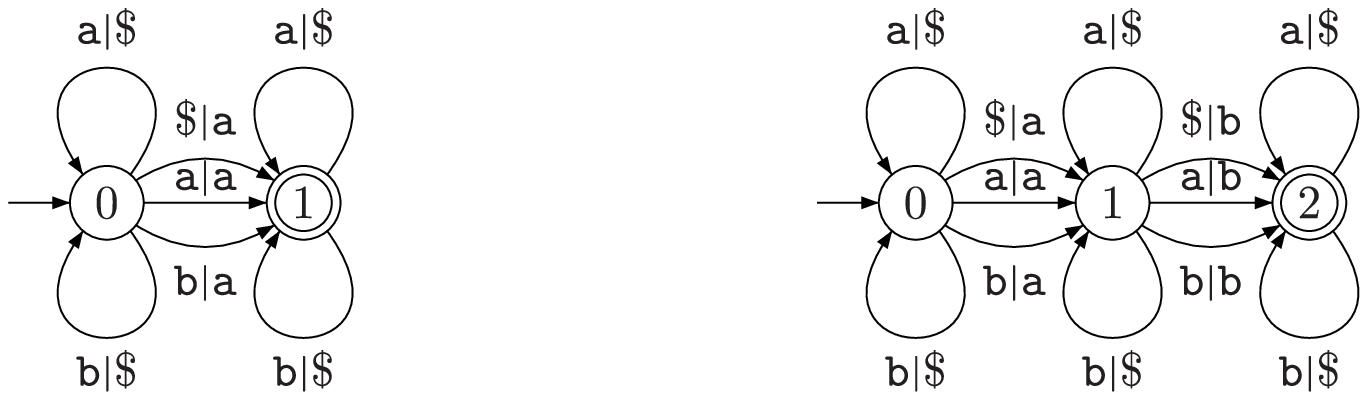}
  \caption[Two cPFT $T|\mathtt{a}$ and $T|\mathtt{ab}$ modeling $p_e(\mathsf{s}|\mathtt{a})$ and $p_e(\mathsf{s}|\mathtt{ab})$]{On the left: a cPFT $T|\mathtt{a}$ that models the output distribution conditionally to an input string $\mathsf{x}=\mathtt{a}$. For the sake of readability, state $0$ stands for $(\$,0)$, and state $1$ for $(\mathtt{a},0)$. On the right: a cPFT $T|\mathtt{ab}$ that models the output distribution given $\mathsf{x'}=\mathtt{ab}$. Again, $0$ stands for $(\$,0)$, $1$ for $(\mathtt{a},0)$, and $2$ for $(\mathtt{ab},0)$.}
  \label{fig:PDFA}
  \end{center}
\end{figure}

\begin{figure}[t]
  \begin{center}
  \includegraphics[width=0.7\textwidth]{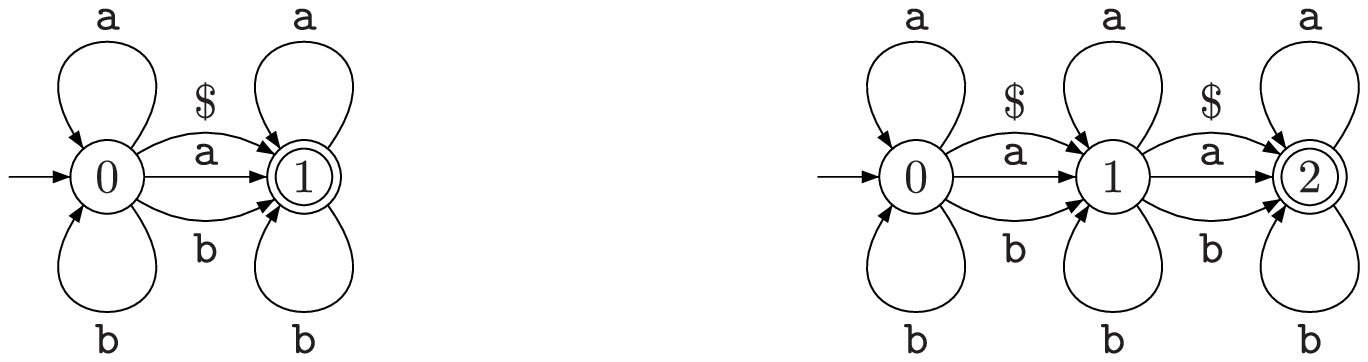}
  \caption[The cPFT $T|\mathtt{a}$ and $T|\mathtt{ab}$ represented in the form of automata]{The cPFT $T|\mathtt{a}$ and $T|\mathtt{ab}$ of \fref{fig:PDFA} represented in the form of automata.}
  \label{fig:from_t_to_a}
  \end{center}
\end{figure}

\subsection{Computing the Product $p_e(\mathsf{s}|\mathsf{x}) \cdot p_e(\mathsf{s}|\mathsf{x'})$}

The next step for computing our kernel $K_e(\mathsf{x},\mathsf{x'})$ is to compute the product $p_e(\mathsf{s}|\mathsf{x}) \cdot p_e(\mathsf{s}|\mathsf{x'})$. This can be performed by modeling the language that describes the intersection of the automata corresponding to $p_e(\mathsf{s}|\mathsf{x})$ and $p_e(\mathsf{s}|\mathsf{x'})$. This intersection can be obtained by performing a composition of transducers \citep{Cortes2004}. As mentioned by the authors, composition is a fundamental operation on weighted transducers that can be used to create complex weighted transducers from simpler ones. In this context, note that the intersection of two probabilistic automata (such as those of \fref{fig:from_t_to_a}) is a special case of composition where the input and output transition labels are identical. This intersection takes the form of a probabilistic automaton as defined below.

\begin{definition}
Let $T$ be a cPFT modeling conditional edit probabilities. Let $\mathsf{x}$ and $\mathsf{x'}$ be two strings of $\Sigma^*$. Let $A|\mathsf{x}=(\Delta, \mathcal{Q}, \mathcal{S}, \mathcal{F}, w, \tau, \rho)$ and $A|\mathsf{x'}=(\Delta, \mathcal{Q}', \mathcal{S}', \mathcal{F}', w', \tau', \rho')$ be the automata deduced from $T$ given the observable strings $\mathsf{x}$ and $\mathsf{x'}$. We define the intersection of $A|\mathsf{x}$ and $A|\mathsf{x'}$ as the automaton $A|\mathsf{x},\mathsf{x'}=(\Delta, \mathcal{Q}^A, \mathcal{S}^A, \mathcal{F}^A, w^A, \tau^A, \rho^A)$ such that:
\begin{itemize}
\item $\mathcal{Q}^A=\mathcal{Q} \times \mathcal{Q}'$,
\item $\mathcal{S}^A=\{(q,q')\}$ with $q \in \mathcal{S}$ and $q' \in \mathcal{S}'$,
\item $\mathcal{F}^A=\{(q,q')\}$ with $q \in \mathcal{F}$ and $q' \in \mathcal{F}'$,
\item $w^A_{(q_1,q'_1) \to (q_2,q'_2)}(\mathsf{b}) = w_{q_1 \to q_2}(\mathsf{b}) \cdot w'_{q'_1 \to q'_2}(\mathsf{b})$,
\item $\tau^A((q,q')) = \tau(q) \cdot \tau(q')$,
\item $\rho^A((q,q')) = \rho(q) \cdot \rho(q')$.
\end{itemize}
\end{definition}

\fref{fig:interPDFA} shows the intersection automaton of the two automata from \fref{fig:from_t_to_a}.
Let us now describe how this intersection automaton can be used to compute the infinite sum over $\Sigma^*$.

\begin{figure}[t]
  \begin{center}
  \includegraphics[width=0.6\textwidth]{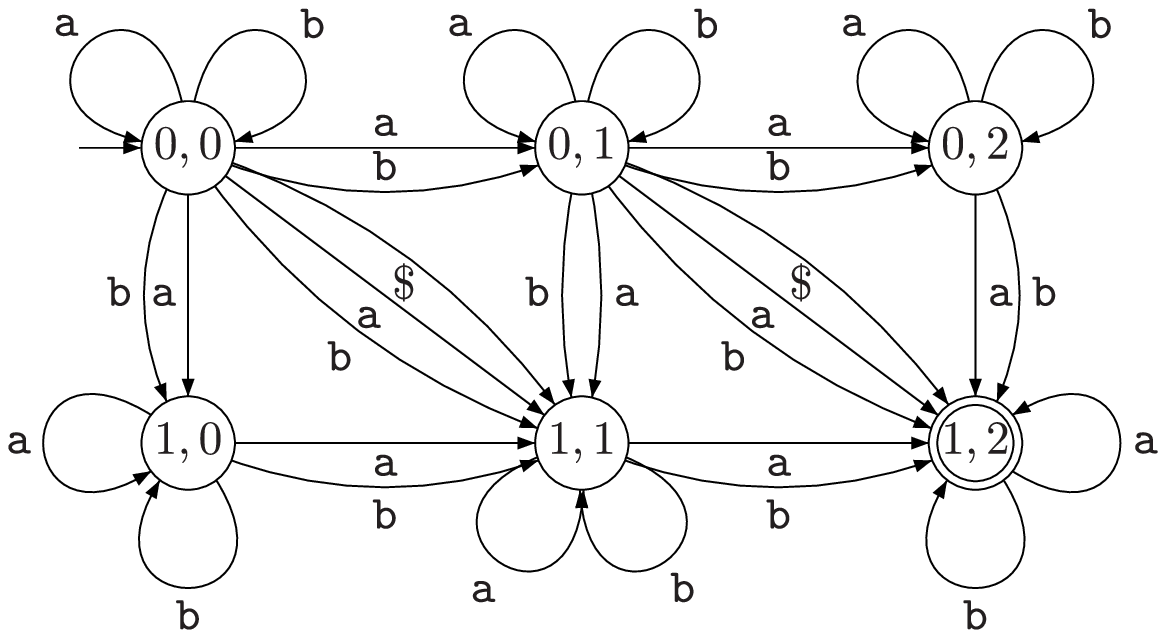}
  \caption[Automaton modeling the intersection of the automata of Figure~\ref*{fig:from_t_to_a}]{Automaton modeling the intersection of the automata of \fref{fig:from_t_to_a}.}
  \label{fig:interPDFA}
  \end{center}
\end{figure}

\subsection{Computing the Sum over $\Sigma^*$}

To simplify the notations, let $p(\mathsf{s})=p_e(\mathsf{s}|\mathsf{x}) \cdot p_e(\mathsf{s}|\mathsf{x'})$ be the probability that a string $\mathsf{s}$ is generated by an intersection automaton $A=\{\Sigma, \mathcal{Q}, \mathcal{S}, \mathcal{F}, w,\tau, \rho\}$ and $\Sigma=\{\mathsf{a_1},\dots,\mathsf{a}_{|\Sigma|}\}$ be the alphabet.

For each $\mathsf{a_k} \in \Sigma$, let $\mathbf{M_{\mathsf{a_k}}}$ be the $|\mathcal{Q}| \times |\mathcal{Q}|$ matrix gathering the probabilities $M_{\mathsf{a_k},q_i,q_j}=w_{q_i \to q_j}(\mathsf{a_k})$ that the transition going from state $q_i$ to state $q_j$ in $A$ outputs the symbol $\mathsf{a_k}$. For notational convenience, we denote this probability by $M_{\mathsf{a_k}}(q_i,q_j)$.

Now, given a string $\mathsf{s}=\mathsf{s_1}\dots\mathsf{s_t}$, $p(\mathsf{s})$ can be rewritten as follows:
\begin{equation}
\label{eq:pz}
p(\mathsf{s})=p(\mathsf{s_1}\dots\mathsf{s_t})=\boldsymbol{\tau}^T \mathbf{M_{\mathsf{s_1}}}\cdots\mathbf{M_{\mathsf{s_t}}}\boldsymbol{\rho}=\boldsymbol{\tau}^T\mathbf{M_\mathsf{s}}\boldsymbol{\rho},
\end{equation}
where $\boldsymbol{\tau}$ and $\boldsymbol{\rho}$ are two vectors of dimension $|\mathcal{Q}|$ whose components are the values returned by the weight function $\tau$ ($\forall q \in S$) and $\rho$ ($\forall q \in F$) respectively, and $\mathbf{M_\mathsf{s}}=\mathbf{M_{\mathsf{s_1}}}\cdots\mathbf{M_{\mathsf{s_t}}}$.

From \eref{eq:pz}, we get:
\begin{equation}
\label{eq:spz}
\sum_{\mathsf{s} \in \Sigma^*}p(\mathsf{s})=\sum_{\mathsf{s} \in \Sigma^*}\boldsymbol{\tau}^T\mathbf{M_\mathsf{s}}\boldsymbol{\rho}.
\end{equation}

To take into account all possible strings $\mathsf{s} \in \Sigma^*$, \eref{eq:spz} can be rewritten according to the size of the string $\mathsf{s}$:
\begin{equation}
\label{eq:eq1}
\sum_{\mathsf{s} \in \Sigma^*}p(\mathsf{s})=\sum_{i=0}^{\infty}\boldsymbol{\tau}^T(\mathbf{M_{\mathsf{a_1}}}+\mathbf{M_{\mathsf{a_2}}}+\dots+\mathbf{M_{\mathsf{a}_{|\Sigma|}}})^i\boldsymbol{\rho}=\boldsymbol{\tau}^T\sum_{i=0}^{\infty}\mathbf{M^i}\boldsymbol{\rho},
\end{equation}
where $\mathbf{M}=\mathbf{M_{\mathsf{a_1}}}+\mathbf{M_{\mathsf{a_2}}}+\dots+\mathbf{M_{\mathsf{a}_{|\Sigma|}}}$.
Denoting $\sum_{i=0}^{\infty}\mathbf{M^i}$ by $\mathbf{B}$, note that
\begin{equation}
\label{eq:eqA}
\mathbf{B}=\mathbf{I}+\mathbf{M}+\mathbf{M^2}+\mathbf{M^3}+\dots,
\end{equation}
where $\mathbf{I}$ is the identity matrix. Multiplying $\mathbf{B}$ by $\mathbf{M}$ we get
\begin{equation}
\label{eq:eqMA}
\mathbf{MB}=\mathbf{M}+\mathbf{M^2}+\mathbf{M^3}+\dots,
\end{equation}
and subtracting \eref{eq:eqA} from \eref{eq:eqMA}, we get:
\begin{equation}
\label{eq:eqMAA}
\mathbf{B}-\mathbf{MB}=\mathbf{I} \Leftrightarrow \mathbf{B}=(\mathbf{I}-\mathbf{M})^{-1}.
\end{equation}
Finally, plugging \eref{eq:eqMAA} in \eref{eq:eq1}, we get a tractable expression for our kernel:
\begin{equation}
\label{eq:eq2}
K_e(\mathsf{x},\mathsf{x'})=\sum_{\mathsf{s} \in \Sigma^*}p_e(\mathsf{s}|\mathsf{x}) \cdot p_e(\mathsf{s}|\mathsf{x'})=\boldsymbol{\tau}^T\sum_{i=0}^{\infty}\mathbf{M^i}\boldsymbol{\rho}=\boldsymbol{\tau}^T(\mathbf{I}-\mathbf{M})^{-1}\boldsymbol{\rho}.
\end{equation}

\subsection{Tractability}

In this section, we investigate the complexity of computing $K_e(\mathsf{x},\mathsf{x'})$ using \eref{eq:eq2} given two strings $\mathsf{x}$ and $\mathsf{x'}$. As we have seen in the previous section, this is essentially done by inverting a matrix. 
Let $T$ be the cPFT modeling the edit probabilities and $t$ its number of states.

The weighted automaton $T|\mathsf{x}$ describing $p_e(\mathsf{s}|\mathsf{x})$ has $t\cdot (|\mathsf{x}| + 1)$ states, and $T|\mathsf{x'}$ describing $p_e(\mathsf{s}|\mathsf{x'})$ has $t\cdot (|\mathsf{x'}| + 1)$ states (see \fref{fig:from_t_to_a} for an example). Thus, the matrix $(\mathbf{I}-\mathbf{M})$ has dimension $t^2\cdot (|\mathsf{x}| +1)\cdot (|\mathsf{x'}|+1)$.
The computational cost of each element of this matrix linearly depends on the alphabet size $|\Sigma|$. Therefore, the complexity of computing the entire matrix is $O(t^4\cdot|\mathsf{x}|^2\cdot|\mathsf{x'}|^2\cdot|\Sigma|)$. Since $\mathbf{M}$ is triangular by construction of $A|\mathsf{x},\mathsf{x'}$ (the probability of going back to a previous state is zero), the matrix inversion $(\mathbf{I}-\mathbf{M})^{-1}$ can be performed by back substitution, avoiding the complications of general Gaussian elimination. The cost of the inversion is in order of the square of the matrix dimension, that is $O(t^4\cdot|\mathsf{x}|^2\cdot|\mathsf{x'}|^2)$. This leads to an overall cost of
$$O(t^4\cdot |\mathsf{x}|^2\cdot |\mathsf{x'}|^2\cdot|\Sigma|).$$

Recall that $t$ stands for the size of the model $T$. In the case of memoryless models such as that of \citet{Oncina2006} used in the experiments, $t=1$ and thus the complexity is reduced to
$$O(|\mathsf{x}|^2\cdot |\mathsf{x'}|^2\cdot |\Sigma|).$$
Therefore, in the case of a memoryless transducer, and for small alphabet sizes, the computational cost of our edit kernel is ``only'' the square of that of the standard edit distance.

Despite the fact that $\mathbf{M}$ is triangular, the algorithmic complexity remains high when strings are long and/or when the alphabet size is large. In this case, we may approximate our kernel $K_e(\mathsf{x},\mathsf{x'})$ by computing a finite sum over the training strings $\mathsf{s}\in\mathcal{T}$. Therefore, we get
$$\hat{K_e}(\mathsf{x},\mathsf{x'})=\sum_{\mathsf{s} \in \mathcal{T}}p_e(\mathsf{s}|\mathsf{x}) \cdot p_e(\mathsf{s}|\mathsf{x'}).$$
Since the computational complexity of each probability $p_e(\mathsf{s}|\mathsf{x})$ scales in $|\mathsf{x}| + |\mathsf{s}|$, the average cost of a kernel evaluation is
$$(|\mathsf{x}| + |\mathsf{x'}| + \overline{|\mathsf{s}|})\cdot|\mathcal{T}|,$$
where $\overline{|\mathsf{s}|}$ is the average length of the training strings.

In conclusion, even if our kernel is rather costly from a complexity point of view, it can be derived from any transducer modeling edit probabilities, and may be approximated if needed. In the next section, we provide experimental evidence that our kernel outperforms standard and learned edit distances plugged in $k$-NN as well as standard string kernels.

\section{Experimental Validation}
\label{sec:prexpes}

\subsection{Setup}

To assess the relevance of our string edit kernel, we carry out experiments on the well-known NIST Special Database 3 of the National Institute of Standards and Technology, which is a handwritten character dataset.

We focus on the set of 10,000 handwritten digits given as $128 \times 128$ bitmap images. We use a training sample of about $8,000$ instances and a test sample of $2,000$ instances. Each instance is represented by a string of Freeman codes \citep{Freeman1974}. To encode a digit, the algorithm scans the bitmap from left to right, starting from the top until reaching the first pixel of the digit. It then follows the contour of the digit until it returns to the starting pixel. The string coding the digit is the sequence of Freeman codes representing the successive directions of the contour. \fref{fig:freeman} shows an example of this encoding procedure.

\begin{figure}[t]
\begin{center}
\includegraphics[width=0.7\textwidth]{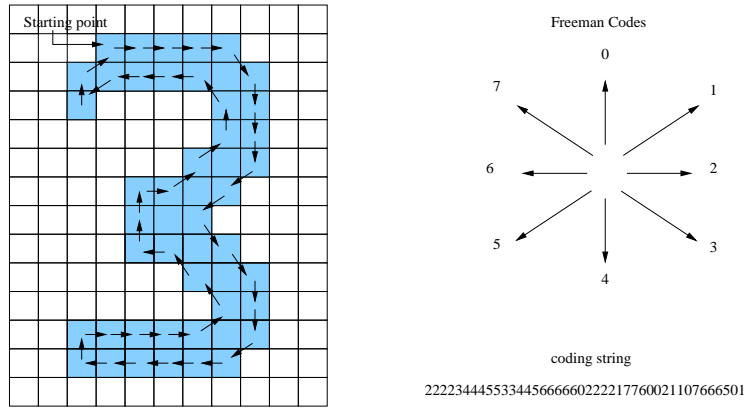}
\caption[A handwritten digit and its string representation]{A handwritten digit and its string representation.}
\label{fig:freeman}
\end{center}
\end{figure}

We use SVM-Light\footnote{\url{http://svmlight.joachims.org/}} as the SVM implementation to compare our approach with other string kernels and adopt a one-versus-all approach to deal with the multi-class setting. This consists in learning a model $M_i$ for each class, where $M_i$ is learned from a positive class made of digits labeled $i$ and a negative class made of differently labeled digits.
Then, the class of a test instance $\mathsf{x}$ is determined as follows: we compute the margin $M_i(\mathsf{x})$ for each model $M_i$. A high positive value of $M_i(\mathsf{x})$ represents a high probability for $\mathsf{x}$ to be of class $i$. The predicted class of $\mathsf{x}$ is given by $\argmax_i M_i(\mathsf{x})$.

\subsection{Comparison with Edit Distances}

As done by \citet{Neuhaus2006}, our first objective is to compare our edit kernel $K_e$ with edit distances used in a $k$-NN algorithm. We use two edit distances: (i) the standard Levenshtein edit distance $d_{lev}$ with all costs set to 1, and (ii) a stochastic edit dissimilarity $d_e(\mathsf{x},\mathsf{x'})=-\log p_e(\mathsf{x'}|\mathsf{x})$ learned with {\sc SEDiL} \citep{Boyer2008}, a software that implements (among others) the method of \citet{Oncina2006}.

We assess the performance of a 1-nearest neighbor algorithm using $d_{lev}$ and $d_e$, and compare them with our string edit kernel plugged in a SVM classifier. Note that the conditional edit probabilities $p_e(\mathsf{x'}|\mathsf{x})$ used in our edit kernel are the same as those used in $d_e(\mathsf{x},\mathsf{x'})$. Results are shown in \fref{fig:expe_1NN} with respect to an increasing number of training instances (from $100$ to $8,000$).

\begin{figure}[t]
\begin{center}
\psfrag{Size of training sample}[][][0.8]{\textbf{Size of training sample}}
\psfrag{Accuracy}[][][0.8]{\textbf{Accuracy on 2,000 test strings}}
\psfrag{DEDEDEDEDE}[][][0.8]{1-NN with $d_{lev}$}
\psfrag{PEPEPEPEPE}[][][0.8]{1-NN with $d_e$}
\psfrag{KEKEKEKEKE}[][][0.8]{SVM with $K_e$}
\includegraphics[width=0.7\textwidth]{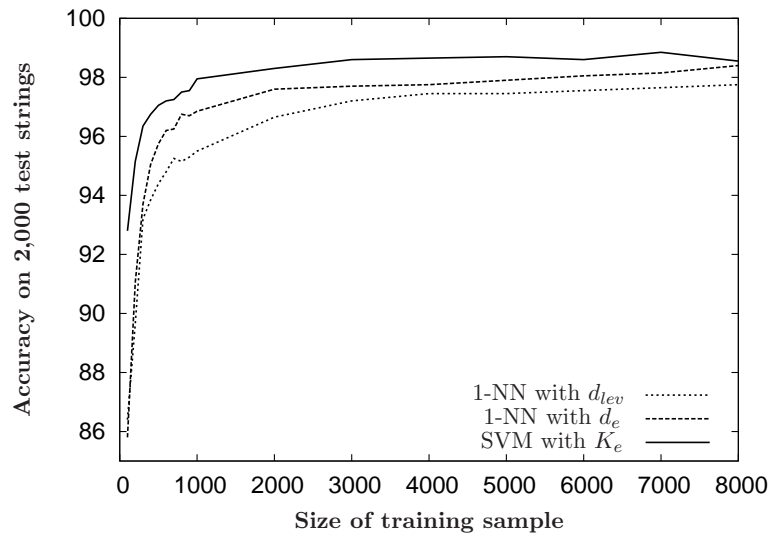}
\caption[Comparison of our edit kernel with edit distances]{Comparison of our edit kernel with edit distances on a handwritten digit recognition task.}
\label{fig:expe_1NN}
\end{center}
\end{figure}

We can make the following remarks:
\begin{itemize}
\item First, learning an edit distance $d_e$ on this classification task leads to better results than using the standard edit distance $d_{lev}$. Indeed, the accuracy of $d_e$ is always higher than that of $d_{lev}$ regardless of the size of the training sample.
\item Second, $K_e$ outperforms both the standard edit distance $d_{lev}$ and the learned edit distance $d_e$ for all training sample sizes. This highlights the usefulness of our kernel.
\end{itemize}

We estimate the statistical significance of these results using a Student's paired $t$-test. \tref{tab:test} contains the $p$-values obtained when 
comparing our kernel with $d_{lev}$ and $d_e$. Using a risk of 5\%, the difference is almost always significant in favor of our kernel (shown in boldface in the table).

\begin{table}[t]
\begin{center}
\begin{scriptsize}
\begin{tabular}{rcccccccc}
\toprule
  \textbf{Training sample size} & \textbf{1,000} & \textbf{2,000} & \textbf{3,000} & \textbf{4,000} & \textbf{5,000} & \textbf{6,000} & \textbf{7,000} & \textbf{8,000}\\
\midrule
  $K_e$ vs $d_{lev}$ & {\bf 6E-06} & {\bf 4E-04} & {\bf 1E-03} & {\bf 3E-03} & {\bf 2E-03} & {\bf 8E-03} & {\bf 2E-03} & {\bf 3E-02}\\
  $K_e$ vs $d_e$ & {\bf 1E-02} & 6E-02 & {\bf 2E-02} & {\bf 2E-02} & {\bf 2E-02} & 9E-02 & {\bf 3E-02} & 3E-01\\
\bottomrule
\end{tabular}
\end{scriptsize}
\caption[Statistical comparison of our edit kernel with edit distances]{Statistical comparison of our edit kernel with standard and learned edit distances ($p$-values of a Student's paired $t$-test). Boldface indicates that the difference is significant in favor of our kernel using a risk of 5\%.}
\label{tab:test}
\end{center}
\end{table}

These results are positive but not quite fair since our edit kernel is plugged into a SVM classifier while the edit distances are plugged into a $k$-NN classifier. In the next section, we compare $K_e$ with other string kernels of the literature.

\subsection{Comparison with Other String Kernels}

In this second series of experiments, we compare $K_e$ with:
\begin{itemize}
\item two classic string kernels, the spectrum kernel \citep{Leslie2002} and the subsequence kernel \citep{Lodhi2002},
\item a variant of the edit kernel of \citet{Li2004} based on learned edit probabilities:\footnote{$K_{L\&J}$ is made symmetric by adding $p_e(\mathsf{x'}|\mathsf{x})$ and $p_e(\mathsf{x}|\mathsf{x'})$.}
$$K_{L\&J}(\mathsf{x'},\mathsf{x}) = e^{\frac{1}{2}t(\log p_e(\mathsf{x'}|\mathsf{x}) + \log p_e(\mathsf{x}|\mathsf{x'}))},$$
\item and the edit kernel $K_{N\&B}$ \citep{Neuhaus2006} in its original version since it cannot accommodate $p_e$ in a straightforward way.
\end{itemize}
Recall that these kernels were presented in \sref{sec:metricstruct}. We did not include the local alignment kernel \citep{Saigo2004} in this experimental study since it is based on local alignments and is specific to finding remote homologies in protein sequences.

The parameter $p$ specifying the length of the considered subsequences in the spectrum and the subsequence kernel was set to 2. The subsequence kernel also has a parameter $\lambda$ which is used to give less importance to subsequences with large gaps. We set $\lambda$ to 2. The parameter $t$ of $K_{L\&J}$ was set to $0.02$. These parameter values give the best results on the dataset.

\fref{fig:expe_all} shows the results we obtain with the considered kernels. We first note that the best results are obtained with edit kernels. As in the previous experiment, \tref{tab:test2} gives the $p$-values of the Student's $t$-test. Our edit kernel significantly outperforms all other string kernels except $K_{L\&J}$: both kernels perform comparably, and the difference for a given training sample size is not significant. However, $K_e$ gives slightly better results for most training sample sizes (12 times out of 17): if a sign test is used, this yields a $p$-value of 0.07, indicating that the difference is significant with a risk of 7\%.

It is also important to keep in mind that $K_{L\&J}$ is not guaranteed to be a valid kernel and thus the parameter $t$ must be tuned with care. \fref{fig:tuning} demonstrates that this kernel can perform poorly if $t$ is not tuned properly. Unlike $K_{L\&J}$, our edit kernel is guaranteed to be valid and is parameter-free.

\begin{figure}[t]
\begin{center}
\psfrag{Size of training sample}[][][0.8]{\textbf{Size of training sample}}
\psfrag{Accuracy}[][][0.8]{\textbf{Accuracy on 2,000 test strings}}
\psfrag{SPECTRUMMM}[][][0.8]{Spectrum\hspace*{-1.1cm}}
\psfrag{SUBSEQUENC}[][][0.8]{Subsequence\hspace*{-0.5cm}}
\psfrag{LIJIANGLIJ}[][][0.8]{$K_{L\&J}$\hspace*{-0.9cm}}
\psfrag{NEUHAUSBUN}[][][0.8]{$K_{N\&B}$\hspace*{-1.6cm}}
\psfrag{KEKEKEKEKE}[][][0.8]{$K_e$\hspace*{-2cm}}
\includegraphics[width=0.7\textwidth]{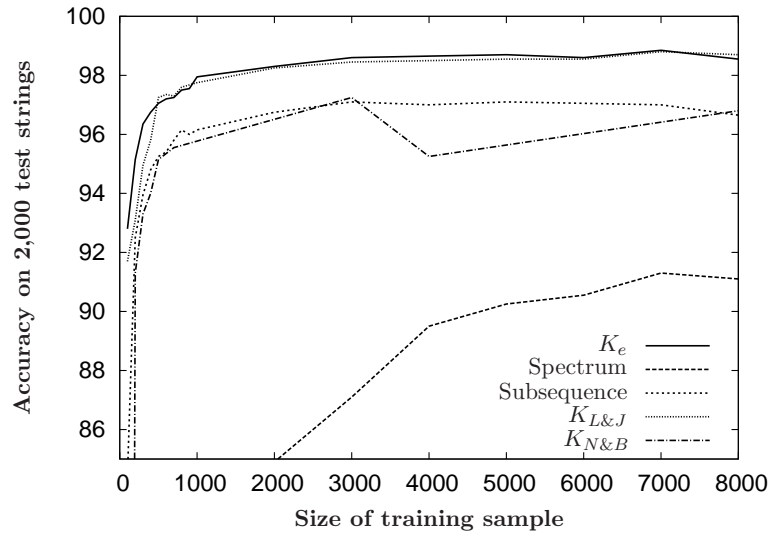}
\caption[Comparison of our edit kernel with other string kernels]{Comparison of our edit kernel with other string kernels on a handwritten digit recognition task.}
\label{fig:expe_all}
\end{center}
\end{figure}

\begin{table}[t]
\begin{center}
\begin{scriptsize}
\begin{tabular}{rcccccccc}
\toprule
  \textbf{Training sample size} & \textbf{1,000} & \textbf{2,000} & \textbf{3,000} & \textbf{4,000} & \textbf{5,000} & \textbf{6,000} & \textbf{7,000} & \textbf{8,000}\\
\midrule
  $K_e$ vs spectrum & {\bf  0} & {\bf  0} & {\bf  0} & {\bf  0} & {\bf  0} & {\bf  0} & {\bf  0}  & {\bf  0}\\
  $K_e$ vs subsequence & {\bf 4E-04} & {\bf 8E-04} & {\bf 5E-04} & {\bf 2E-04} & {\bf 2E-04} & {\bf 4E-04} & {\bf 2E-05} & {\bf 4E-05}\\
  $K_e$ vs $K_{L\&J}$ & 3E-01 & 4E-01 & 3E-01 & 3E-01 & 3E-01 & 4E-01 & 4E-01 & 7E-01\\
  $K_e$ vs $K_{N\&B}$ & {\bf 6E-06} & {\bf 6E-06} & {\bf 1E-03} & {\bf 2E-10} & {\bf 6E-09} & {\bf 4E-08} & {\bf 4E-07} & {\bf 1E-04}\\
\bottomrule
\end{tabular}
\end{scriptsize}
\caption[Statistical comparison of our edit kernel with other string kernels]{Statistical comparison of our edit kernel with other string kernels ($p$-values of a Student's paired $t$-test). Boldface indicates that the difference is significant in favor of our kernel using a risk of 5\%.}
\label{tab:test2}
\end{center}
\end{table}

\begin{figure}[t]
\begin{center}
\psfrag{Value of t}[][][0.8]{\textbf{Value of $t$}}
\psfrag{Accuracy}[][][0.8]{\textbf{Accuracy on 2,000 test strings}}
\psfrag{LIJIANGLIJ}[][][0.8]{$K_{L\&J}$\hspace*{-0.9cm}}
\includegraphics[width=0.7\textwidth]{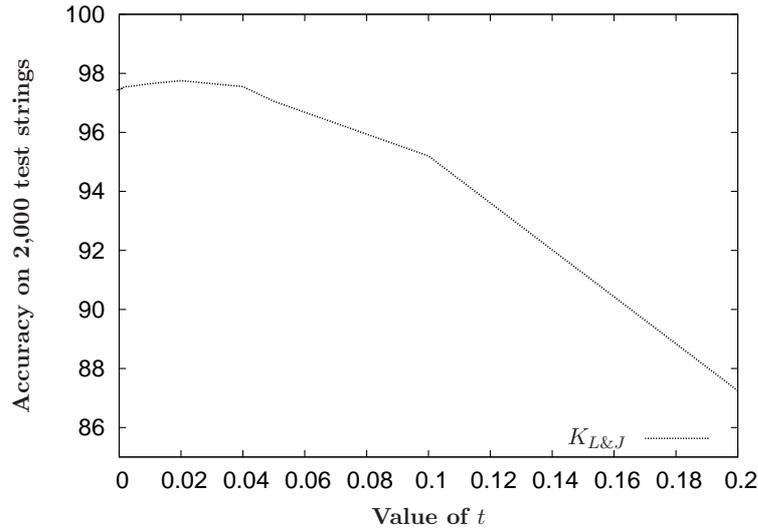}
\caption[Influence of the parameter $t$ of $K_{L\&J}$]{Influence of the parameter $t$ of $K_{L\&J}$ (1,000 training strings).}
\label{fig:tuning}
\end{center}
\end{figure}

\section{Conclusion}
\label{sec:prconclu}

In this chapter, we designed a new string edit kernel that can make use of edit probabilities learned with generative or discriminative probabilistic models while enjoying the classification performance brought by SVM. We showed that although it involves an infinite sum over an entire language, our kernel can be computed exactly through the intersection of probabilistic automata built from the edit probability model and a matrix inversion. Experiments on a handwritten digit recognition task have shown that our edit kernel outperforms standard and learned edit distance within a $k$-NN framework as well as state-of-the-art string kernels.

An interesting perspective is to improve the algorithmic complexity of our kernel. The main bottleneck in its calculation is the size of the intersection automaton that allows the computation of $p_e(\mathsf{s}|\mathsf{x})p_e(\mathsf{s}|\mathsf{x'})$. A way of reducing its size could consist in simplifying the conditional transducers $T|\mathsf{x}$ and $T|\mathsf{x'}$ from which it is built by only considering the most likely transitions and states. A simplification of these automata would have a direct impact on the dimension of the matrix that has to be inverted, and thus on the evaluation cost of the kernel.

A second perspective is to extend this work to the design of tree edit kernels. Indeed, as seen in \sref{sec:treeeditlearn}, generative and discriminative models for learning tree edit probabilities have been proposed and could be used to derive powerful tree edit kernels, based on the same ideas as in the string case.

While one of the advantages of the presented approach is to incorporate a lot of structural information by comparing inputs strings to an infinite number of strings, it makes it difficult to establish generalization guarantees. In the next chapter, we overcome this limitation by proposing a novel edit similarity learning approach that is not subject to many classic limitations of previous edit metric learning methods (in particular, those based on probabilistic models that our kernel uses) and for which we can derive a generalization bound. The idea is to relax the structural constraint on edit scripts to get an edit similarity that has a simpler form and can thus be learned through numerical optimization. The resulting (potentially non-PSD) similarity can then be used directly to build a linear classifier (that has bounded true risk), avoiding the computational cost of transforming it into a kernel. Furthermore, the linear classifiers are sparser than SVM models, speeding up prediction.

\chapter{Learning Good Edit Similarities from Local Constraints}
\label{chap:ecml}

\begin{chapabstract}
Metrics based on the edit distance are widely used to tackle problems involving string or tree-structured data. Unfortunately, as seen in \cref{chap:pr}, using them in kernel methods is often difficult and/or costly. On the other hand, the recently-proposed theory of $(\epsilon,\gamma,\tau)$-good similarity functions bridges the gap between the properties of a non-PSD similarity function and its performance in linear classification.
In this chapter, we show that this framework is well-suited to edit similarities. Furthermore, we make use of a relaxation of $(\epsilon,\gamma,\tau)$-goodness to propose a novel edit similarity learning method, GESL, that avoids the classic drawbacks of previous approaches. Using uniform stability, we derive generalization bounds that hold for a large class of loss functions and show that they can be related to the error of a linear classifier built from the similarity. We also provide experimental results on two real-world datasets highlighting that edit similarities learned with GESL induce more accurate and sparser classifiers than other (standard or learned) edit similarities.\\

The material of this chapter is based on the following international publications:\\

\bibentry{Bellet2011a}.\\

\bibentry{Bellet2011}.\\

\bibentry{Bellet2012}.
\end{chapabstract}

\section{Introduction}

As mentioned in the previous chapter, metrics based on the edit distance are widely used by practitioners when dealing with string or tree-structured data. Although they involve complex procedures, there exist a few methods (reviewed in \sref{sec:mlstruct}) for learning edit metrics for a given task. These edit metrics are typically used in a $k$-NN setting. As we have seen in \cref{chap:pr}, using them in kernel methods such as SVM requires the design of a positive semi-definite edit kernel. However, existing edit kernels are either not guaranteed to be PSD, or involve rather costly procedures \citep{Li2004,Neuhaus2006,Bellet2010}. Furthermore, there is a lack of theoretical understanding of how arbitrary similarity functions can be used to learn accurate linear classifiers.

Recently, Balcan et al. \citeyearpar{Balcan2006,Balcan2008,Balcan2008a} introduced a theory of learning with so-called $(\epsilon,\gamma,\tau)$-good similarity functions that gives intuitive, sufficient conditions for a similarity function to allow one to learn well. Essentially, a similarity function $K$ is $(\epsilon,\gamma,\tau)$-good if a $1-\epsilon$ proportion of examples are on average more similar to \emph{reasonable} examples of the same class than to \emph{reasonable} examples of the opposite class by a margin $\gamma$, where a $\tau$ proportion of examples must be \emph{reasonable}. $K$ does not have to be a metric nor positive semi-definite (PSD). They show that if $K$ is $(\epsilon,\gamma,\tau)$-good, then it can be used to build a linear separator in an explicit projection space that has margin $\gamma$ and error arbitrarily close to $\epsilon$. This separator can be learned efficiently using a linear program and tends to be sparse thanks to $L_1$ norm regularization.

The first contribution of this work is to experimentally show that this theory is well-suited to edit similarity functions and is competitive with SVM in terms of accuracy, while inducing sparser models.
Furthermore, we show that we can make use of this framework to propose a new approach to learning string and tree edit similarities which addresses the classic drawbacks of other methods in the literature, i.e., lack of generalization guarantees, high computational cost, convergence to suboptimal solution and inability to use the information brought by negative pairs.
Our approach (GESL, for Good Edit Similarity Learning) is driven by the idea of $(\epsilon,\gamma,\tau)$-goodness: we learn the edit costs so as to optimize a relaxation of the goodness of the resulting similarity function. It is based on regularized risk minimization (formulated as an efficient convex program) over some positive and negative training pairs: the similarity is thus optimized with respect to \emph{local} constraints but plugged in a \emph{global} linear classifier.
We provide an extensive theoretical study of the properties of GESL based on a notion of uniform stability adapted to metric learning \citep{Jin2009}, leading to the derivation of a generalization bound that holds for a large class of loss functions. This bound can be related to the generalization error of the linear classifier built from the similarity and is independent of the size of the alphabet, making GESL suitable for handling problems with large alphabet. To the best of our knowledge, this is the first edit metric learning method with generalization guarantees, and the first attempt to establish a theoretical relationship between a learned metric and the risk of a classifier using it. We show in a comparative experimental study that GESL has fast convergence and leads to more accurate and sparser classifiers than other (standard or learned) edit similarities.

The rest of this chapter is organized as follows. In \sref{sec:balcan}, we introduce the theory of $(\epsilon,\gamma,\tau)$-goodness. \sref{sec:ictai} features a preliminary study that provides experimental evidence that this theory is well-suited to edit similarity functions and leads to classifiers that are competitive with SVM classifiers. \sref{sec:mljlearncosts} presents GESL, our approach to learning $(\epsilon,\gamma,\tau)$-good edit similarities. We show that it is a suitable way to deal not only with strings but also with tree-structured data. We propose in \sref{sec:mljanalysis} a theoretical analysis of GESL based on uniform stability, leading to the derivation of a generalization bound. We also provide a discussion on that bound and its implications, as well as a way of deriving a bound for the case where instances have unbounded size. A wide experimental evaluation of our approach on two real-world string datasets from the natural language processing and image classification domains is provided in \sref{sec:mljexperiments}. Finally, we conclude this work in \sref{sec:mljconclu}.

\section{The Theory of $(\epsilon,\gamma,\tau)$-Good Similarity Functions}
\label{sec:balcan}

In recent work, Balcan et al. \citeyearpar{Balcan2006,Balcan2008,Balcan2008a} introduced a new theory of learning with good similarity functions. Their motivation was to overcome two major limitations of kernel theory. First, a good kernel is essentially a good similarity function, but the theory talks in terms of margin in an implicit, possibly unknown projection space, which can be a problem for intuition and design. Second, the PSD and symmetry requirement often rules out natural similarity functions for the problem at hand. As a consequence, \citet{Balcan2008a} proposed the following definition of good similarity function. 

\begin{definition}[\citeauthor{Balcan2008a}, \citeyear{Balcan2008a}]
A similarity function $K$ is an $(\epsilon,\gamma,\tau)$-good similarity function for a learning problem $P$ if there exists a (random) indicator function $R(x)$ defining a (probabilistic) set of ``reasonable points'' such that the following conditions hold:
\begin{enumerate}
\item A $1-\epsilon$ probability mass of examples $(x,y)$ satisfy
\begin{equation}
\label{eq:goodnessprop}
\mathbb{E}_{(x',y')\sim P}[yy'K(x,x')|R(x')]\geq \gamma,
\end{equation}
\item $\mathrm{Pr}_{x'}[R(x')] \geq \tau$.
\end{enumerate}
\label{def:defgoodsim}
\end{definition}

The first condition is essentially requiring that \emph{a $1-\epsilon$ proportion of examples $x$ are on average more similar to reasonable examples of the same class than to reasonable examples of the opposite class by a margin $\gamma$} and the second condition that \emph{at least a $\tau$ proportion of the examples are reasonable}.\footnote{For now, we assume that the set of reasonable points is given. The question of finding such a set is addressed later in this section.} \fref{fig:figdef} illustrates the definition on a toy example.
\begin{figure}[t]
\begin{center}
\includegraphics[width=0.6\columnwidth]{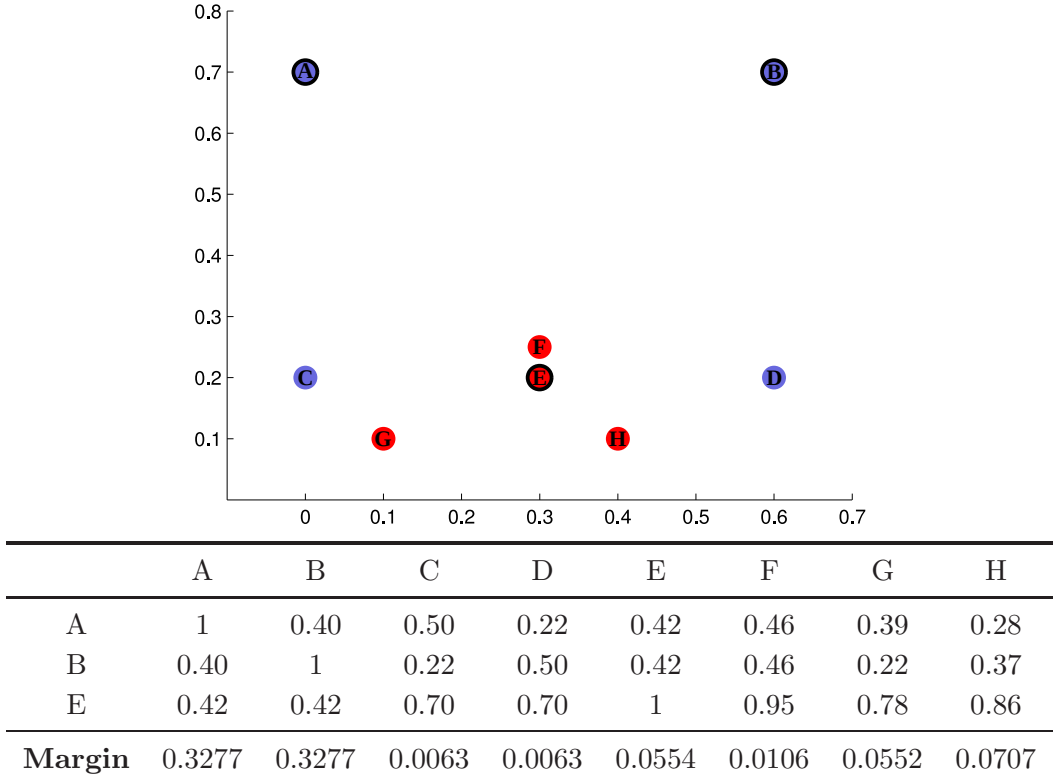}\\
\vspace{0.2cm}
\begin{tabular}{ccccccccc}
\toprule
& A & B & C & D & E & F & G & H\\
\midrule
A & 1 & 0.40 & 0.50 & 0.22 & 0.42 & 0.46 & 0.39 & 0.28\\
B & 0.40 & 1 & 0.22 & 0.50 & 0.42 & 0.46 & 0.22 & 0.37\\
E & 0.42 & 0.42 & 0.70 & 0.70 & 1 & 0.95 & 0.78 & 0.86\\
\midrule
\textbf{Margin} & 0.3277 & 0.3277 & 0.0063 & 0.0063 & 0.0554 & 0.0106 & 0.0552 & 0.0707\\
\bottomrule
\end{tabular}
\caption[A graphical insight into $(\epsilon,\gamma,\tau)$-goodness]{A graphical insight into \defref{def:defgoodsim}. Let us consider 8 points as shown above (blue represents the positive class, red the negative class) and use the similarity function $K(\mathbf{x},\mathbf{x'}) = 1-\|\mathbf{x}-\mathbf{x'}\|_2$. We picked 3 reasonable points (A, B and E, circled in black), thus we can set $\tau=3/8$. Similarity scores to the reasonable points as well as the margin achieved by each point (as given by \eref{eq:goodnessprop}) are shown in the array. There exists an infinite number of valid instantiations of $\epsilon$ and $\gamma$ since there is a trade-off between the margin $\gamma$ and proportion of margin violations $\epsilon$.
For example, $K$ is $(0,0.006,3/8)$-good because all points ($\epsilon=0$) are on average more similar to reasonable examples of the same class than to reasonable examples of the other class by a margin $\gamma=0.006$. One can also say that $K$ is $(2/8,0.01,3/8)$-good ($\epsilon=2/8$ because examples C and D violate the margin $\gamma=0.01$).}
\label{fig:figdef}
\end{center}
\end{figure}
Note that other definitions are possible, like those proposed by \citet{Wang2007,Wang2009} for unbounded dissimilarity functions. Yet \defref{def:defgoodsim} is very interesting in three respects. First, it is a strict generalization of the notion of good kernel \citep{Balcan2008a} but does not impose positive semi-definiteness nor symmetry. Second, as opposed to pair and triplet-based criteria used in metric learning, \defref{def:defgoodsim} is based on an average over some points. In other words, it relaxes the notion of local constraints, opening the door to metric learning for global algorithms. Third, these conditions are sufficient to learn well, i.e., to induce a classifier with low true risk, as we show in the following.

Let $K$ be an $(\epsilon,\gamma,\tau)$-good similarity function. If the set of reasonable points $R = \{(x_1',y_1'),(x_2',y_2'),\dots,(x_{|R|}',y_{|R|}')\}$ is known, it follows directly  from \eref{eq:goodnessprop} that the following classifier achieves true risk at most $\epsilon$ at margin $\gamma$:
$$h(x) = \sign\left[\frac{1}{|R|}\sum_{i=1}^{|R|}y_i'K(x,x_i')\right].$$
Note that $h$ is a linear classifier in the space of the similarity scores to the reasonable points. In other words, $K$ is used to project the data into a new space using the mapping $\phi:\mathcal{X}\rightarrow \mathbb{R}^{|R|}$ defined as:
$$\phi_i(x) = K(x,x_i'),\quad i\in \{1,\dots,|R|\}.$$
The projection and linear classifier corresponding to the toy example of \fref{fig:figdef} is shown in \fref{fig:phispace}.

\begin{figure}[t]
\begin{center}
\psfrag{K(x,A)}[][][0.7]{$K(\mathbf{x},\text{A})$}
\psfrag{K(x,B)}[][][0.7]{$K(\mathbf{x},\text{B})$}
\psfrag{K(x,E)}[][][0.7]{$K(\mathbf{x},\text{E})$}
\includegraphics[width=0.6\columnwidth]{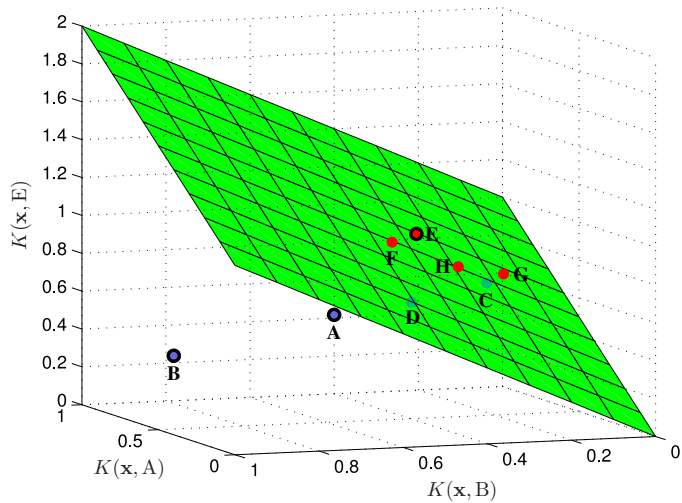}
\caption[Projection space implied by the toy example of Figure~\ref*{fig:figdef}]{Projection space ($\phi$-space) implied by the toy example of \fref{fig:figdef}: similarity scores to the reasonable points (A, B and E) are used as new features. Since $K$ is $(0,\gamma,3/8)$ for some $\gamma>0$, the linear separator of equation $K(\mathbf{x},\text{A})+K(\mathbf{x},\text{B})-K(\mathbf{x},\text{E})=0$ (shown as a green grid) achieves perfect classification, although the data were not linearly separable in the original space.}
\label{fig:phispace}
\end{center}
\end{figure}

However, in practice the set of reasonable points is unknown. We can get around this problem by sampling points (called \emph{landmarks}) and use them to project the data into a new space (using the same strategy as before).\footnote{Note that the landmark points need not to be labeled, although we do not make use of this feature in our contributions.} If we sample enough landmarks (this depends in particular on $\tau$, which defines how likely it is to draw a reasonable point), then with high probability there exists a linear classifier in that space that achieves true risk close to $\epsilon$. This is formalized in \thref{thm:thmsiml}.

\begin{theorem}[\citeauthor{Balcan2008a}, \citeyear{Balcan2008a}]
Let $K$ be an $(\epsilon,\gamma,\tau)$-good similarity function for a learning problem $P$. Let $\mathcal{L} = \{x_1',x_2',\dots,x_{n_\mathcal{L}}'\}$ be a sample of $n_\mathcal{L}=\frac{2}{\tau}\left(\log(2/\delta)+8\frac{\log(2/\delta)}{\gamma^2}\right)$ landmarks drawn from $P$. Consider the mapping $\phi^\mathcal{L}:\mathcal{X}\rightarrow \mathbb{R}^{n_\mathcal{L}}$ defined as follows: $\phi^\mathcal{L}_i(x) = K(x,x_i')$, $i\in \{1,\dots,n_\mathcal{L}\}$. Then, with probability at least $1-\delta$ over the random sample $\mathcal{L}$, the induced distribution $\phi^\mathcal{L}(P)$ in $\mathbb{R}^{n_\mathcal{L}}$ has a linear separator of error at most $\epsilon+\delta$ relative to $L_1$ margin at least $\gamma/2$.
\label{thm:thmsiml}
\end{theorem}

Unfortunately, finding this separator is NP-hard (even to approximate) because minimizing the number of $L_1$ margin violations is NP-hard. To overcome this limitation, the authors considered the hinge loss as a surrogate for the 0/1 loss (which counts the number of margin violations) in the following reformulation of \defref{def:defgoodsim}.

\begin{definition}[\citeauthor{Balcan2008a}, \citeyear{Balcan2008a}]
A similarity function $K$ is an \emph{$(\epsilon,\gamma,\tau)$-good similarity function in hinge loss} for a learning problem $P$ if there exists a (random) indicator function $R(x)$ defining a (probabilistic) set of ``reasonable points'' such that the following conditions hold:
\begin{enumerate}
\item $\mathbb{E}_{(x,y)\sim P}\left[[1-y g(x)/\gamma]_+\right]\leq \epsilon$, 
  where $g(x)=\mathbb{E}_{(x',y')\sim P}[y'K(x,x')|R(x')]$,
\item $\mathrm{Pr}_{x'}[R(x')] \geq \tau$.
\end{enumerate}
\label{def:defgoodsim2}
\end{definition}

This leads to the following theorem, similar to \thref{thm:thmsiml}.

\begin{theorem} [\citeauthor{Balcan2008a}, \citeyear{Balcan2008a}]
Let $K$ be an $(\epsilon,\gamma,\tau)$-good similarity function in hinge loss for a learning problem $P$. For any $\epsilon_1>0$ and $0\leq \delta\leq \gamma\epsilon_1/4$, let $\mathcal{L} = \{x_1',x_2',\dots,x_{n_\mathcal{L}}'\}$ be a sample of $n_\mathcal{L}=\frac{2}{\tau}\left(\log(2/\delta)+16\frac{\log(2/\delta)}{\epsilon_1\gamma^2}\right)$ landmarks drawn from $P$. Consider the mapping $\phi^\mathcal{L}:X\rightarrow \mathbb{R}^{n_\mathcal{L}}$ defined as follows: $\phi^\mathcal{L}_i(x) = K(x,x_i')$, $i\in \{1,\dots,n_\mathcal{L}\}$. Then, with probability at least $1-\delta$ over the random sample $\mathcal{L}$, the induced distribution $\phi^\mathcal{L}(P)$ in $\mathbb{R}^{n_\mathcal{L}}$ has a linear separator of error at most $\epsilon+\epsilon_1$ at margin $\gamma$.
\label{thm:thmsim}
\end{theorem}

The objective is now to find a linear separator $\boldsymbol{\alpha}\in \mathbb{R}^{n_\mathcal{L}}$ that has low true risk based on the expected hinge loss relative to $L_1$ margin $\gamma$:
$$
\mathbb{E}_{(x,y)\sim P}\left[\left[1 - y \innerp{
\boldsymbol{\alpha},\phi^\mathcal{L}(x)}/\gamma\right]_+\right]. 
$$
Using a landmark sample $\mathcal{L} = \{x_1',x_2',\dots,x_{n_\mathcal{L}}'\}$ and a training sample $\mathcal{T} = \{(x_1,y_1),$ $(x_2,y_2),\dots,(x_n,y_n)\}$, one can find this separator $\boldsymbol{\alpha}$ efficiently by solving the following linear program (LP):\footnote{The original formulation \citep{Balcan2008a} was actually $L_1$-constrained. We provide here an equivalent, more practical $L_1$-regularized form.}
\begin{equation}
\displaystyle\min_{\boldsymbol{\alpha}}  \displaystyle\sum_{i=1}^{n}\left[1-\sum_{j=1}^{n_\mathcal{L}}\alpha_jy_iK(x_i,x'_j)\right]_+ + \lambda \|\boldsymbol{\alpha}\|_1. 
\label{eq:lp1}
\end{equation}
In practice, we simply use the training examples as landmarks. In this case, learning rule \eqref{eq:lp1} --- referred to as ``Balcan's learning rule'' in the rest of this document --- is reminiscent of the standard SVM formulation, with three important differences. First, recall that $K$ is not required to be PSD nor symmetric. Second, the linear classifier lies in an explicit projection space built from $K$ (called an empirical similarity map) rather than in a possibly implicit Hilbert Space induced by a kernel. Third, it uses $L_1$ regularization, inducing sparsity in $\boldsymbol{\alpha}$ and thus reducing the number of landmarks the classifier is based on, which speeds up prediction.\footnote{Note that $L_1$ regularization has also been used in the context of standard SVM formulations, leading to the $1$-norm SVM \citep{Zhu2003}. While these classifiers may work well in practice, most of the handy SVM theory fall apart in this case. Conversely, the use of \eqref{eq:lp1} is justified by the theory presented in this section.} This regularization can be interpreted as a way to select (or approximate) the set of reasonable points among the landmarks: in a sense, $R$ is automatically worked out while learning $\boldsymbol{\alpha}$.\footnote{The problem of finding the reasonable points is not as simple if we first want to \emph{learn} the similarity function, as we will see later in this chapter.}
Note that we can control the degree of sparsity of the linear classifier: the larger $\lambda$, the sparser $\boldsymbol{\alpha}$.

To sum up, the performance of the linear classifier theoretically depends on how well the similarity function satisfies \defref{def:defgoodsim}. In this chapter, we first conduct a preliminary experimental study to investigate the level of $(\epsilon,\gamma,\tau)$-goodness of some edit similarities and their performance in classification when used in Balcan's learning rule (\sref{sec:ictai}). The rest of the chapter is the main contribution and is devoted to learning $(\epsilon,\gamma,\tau)$-good edit similarities from data.

\section{Preliminary Experimental Study}
\label{sec:ictai}

In this section, we experimentally show that the framework of $(\epsilon,\gamma,\tau)$-goodness is well-suited to edit similarities. We first investigate the goodness of edit similarities on the previously-studied handwritten digit recognition task (\sref{sec:goodnesseval}). Then, in \sref{sec:expeictai}, we compare the performance of linear classifiers learned with Balcan's rule using edit similarities with the performance of SVM using standard edit kernels.

\subsection{Are Edit Similarities Really $(\epsilon,\gamma,\tau)$-Good?}
\label{sec:goodnesseval}

In this experimental evaluation of the $(\epsilon,\gamma,\tau)$-goodness of edit similarities, we will consider the standard Levenshtein distance $d_{lev}$ and edit probabilities $p_e$ learned with the method of \citet{Oncina2006}. We actually use $-d_{lev}$ so that both similarities express a measure of closeness, making the comparison easier. We also normalized them so that they lie in $[-1,1]$.\footnote{We normalized them to zero mean and unit variance, then brought back to $1$ and $-1$ the values greater than 1 and smaller than -1 respectively. We are aware that there may be better normalizations but this is outside the scope of this work.} In the following, they are referred to as $\tilde{d}_{lev}$ and $\tilde{p}_e$. Looking at \defref{def:defgoodsim}, we can easily estimate $\epsilon$, $\gamma$ and $\tau$ using a randomly selected set of points. We illustrate this on the NIST Special Database 3, the handwritten digit recognition task already used in \cref{chap:pr}, where the digits are represented as strings of Freeman codes.

Since we do not know the set of reasonable points before learning the linear classifier, we fix $\tau = 1$ (i.e., all points are considered reasonable) and plot $\epsilon$ as a function of $\gamma$.
In order to analyze the results in different contexts, we randomly selected 500 instances of each class and estimated the goodness of the similarities for each binary problem. For brevity, we only discuss the goodness curves for two representative problems: ``0 vs. 1''  and ``0 vs. 8'', shown in \fref{fig:goodness}.
The interpretation (given by \defref{def:defgoodsim}) is that a margin $\gamma$ leads to an $\epsilon$ proportion of examples violating the margin.
For the ``0 vs. 1'' problem, shown in \fref{fig:good:01}, both similarities achieve good margin while keeping the number of violations small. The learned similarity $\tilde{p}_e$ behaves slightly better.
The ``0 vs. 8'' problem is a harder task, since the representation of an eight is often similar to that of a zero (because Freeman codes only encode the contour of the digits). \fref{fig:good:08} reflects the difficulty of the task, since margin violations are almost always higher for a given $\gamma$ than in the ``0 vs. 1'' case. For the ``0 vs. 8'' task, the learned similarity provides an important improvement over the standard edit distance: for small margin values, it achieves few margin violations.

\begin{figure}[t]
\begin{center}
\subfigure[0 vs. 1]{\label{fig:good:01}
\psfrag{gamma}[][][0.7]{\textbf{Value of $\gamma$}}
\psfrag{epsilon}[][][0.7]{\textbf{Value of $\epsilon$}}
\psfrag{lev}[][][0.7]{$\tilde{d}_{lev}$\hspace{0.25cm}}
\psfrag{prob}[][][0.7]{\hspace{0.2cm}$\tilde{p}_e$}
\includegraphics[width=0.47\textwidth]{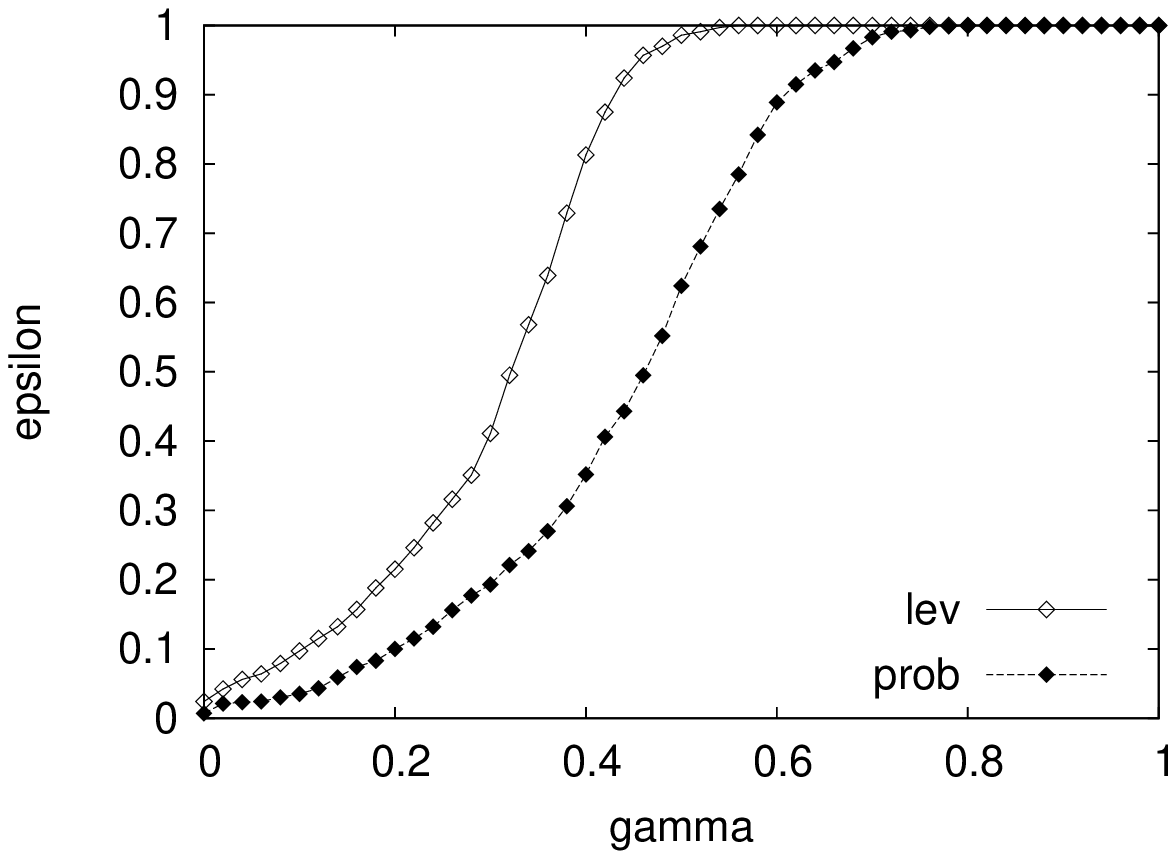}
}
\subfigure[0 vs. 8]{\label{fig:good:08}
\psfrag{gamma}[][][0.7]{\textbf{Value of $\gamma$}}
\psfrag{epsilon}[][][0.7]{\textbf{Value of $\epsilon$}}
\psfrag{lev}[][][0.7]{$\tilde{d}_{lev}$\hspace{0.25cm}}
\psfrag{prob}[][][0.7]{\hspace{0.2cm}$\tilde{p}_e$}
\includegraphics[width=0.47\textwidth]{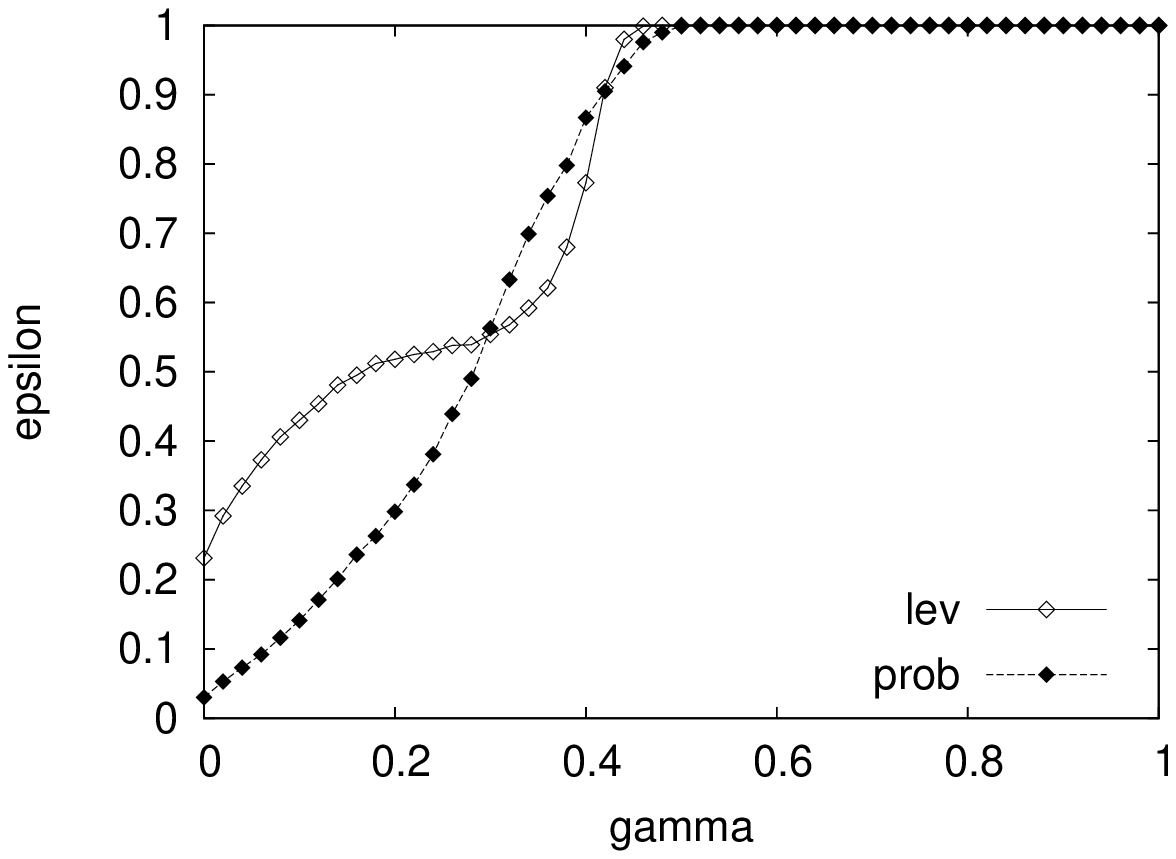}
}
\caption[Estimation of the goodness of edit similarities]{Estimation of $\epsilon$ as a function of $\gamma$ for $\tilde{d}_{lev}$ and $\tilde{p}_e$ on two handwritten digits binary classification tasks.}
\label{fig:goodness}
\end{center}
\end{figure}

To sum up, we see that decent values for $\gamma$ and $\epsilon$ are achieved even without selecting an appropriate subset $\mathcal{R}$ of reasonable points. Note that we observe a similar behavior for all binary problems in the dataset. Therefore, edit similarities satisfy \defref{def:defgoodsim} rather well, thus \thref{thm:thmsim} is meaningful and we can expect good accuracy in linear classification on this dataset. Moreover, $\tilde{p}_e$ seems to be ``$(\epsilon,\gamma,\tau)$-better'' than $\tilde{d}_{lev}$, which suggests that it could achieve better generalization performance. We will see that it is indeed the case in the next section.

\subsection{Experiments}
\label{sec:expeictai}

In this section, we provide experimental evidence that learning with Balcan's learning rule using edit similarities outperforms a $k$-NN approach and is competitive with a standard SVM approach, while inducing much sparser models. As noted earlier, standard SVM and Balcan's learning rule are similar but use different regularizers ($L_2$ norm and $L_1$ norm respectively). The comparative performance of $L_2$ and $L_1$ regularized learning rules has been the subject of previous experimental studies \citep[see for instance][]{Zhu2003} but, to the best of our knowledge, never in the context of edit similarities. Furthermore, $(\epsilon,\gamma,\tau)$-goodness provides a theoretical justification of Balcan's learning rule and casts an interesting light on this comparison.\footnote{Note that we did not include 1-norm SVM in this experimental study because the learning rule itself is very similar to Balcan's learning rule while having no grounds in SVM theory.}

We compare the following approaches: (i) Balcan's learning rule \eqref{eq:lp1} using $K(\mathsf{x},\mathsf{x'}) = \tilde{d}_{lev}(\mathsf{x},\mathsf{x'})$, (ii) Balcan's learning rule using $K(\mathsf{x},\mathsf{x'}) = \tilde{p}_e(\mathsf{x'}|\mathsf{x})$, (iii) SVM learning using $K(\mathsf{x},\mathsf{x'})=e^{-t\cdot d_{lev}(\mathsf{x},\mathsf{x'})}$, the kernel of \citet{Li2004} based on $d_{lev}$, (iv) SVM learning using $K(\mathsf{x},\mathsf{x'})=e^{\frac{1}{2}t(\log p_e(\mathsf{x'}|\mathsf{x}) + \log p_e(\mathsf{x}|\mathsf{x'}))}$, the kernel of \citet{Li2004} based on $p_e$, (v) 1-NN using $d_{lev}(\mathsf{x},\mathsf{x'})$, and (vi) 1-NN using $-p_e(\mathsf{x'}|\mathsf{x})$.
We choose \textsc{Libsvm}\footnote{\url{http://www.csie.ntu.edu.tw/~cjlin/libsvm/}} as the SVM implementation, which takes a one-versus-one approach for multi-class classification. We thus use the same strategy for multi-class classification with Balcan's learning rule.
Note that we take the training examples to be the landmarks. Therefore, all learning algorithms have access to strictly the same information (that is, similarity measurements between training examples), allowing a fair comparison.

In the following, we present results on the multi-class handwritten digit classification task and on a dataset of English and French words.

\subsubsection{Handwritten digit classification}
\label{sec:expeictaidigits}

Using the handwritten digit classification dataset, we first aim at evaluating the performance of the models obtained with different methods. We use 40 to 6,000 training examples, reporting the results under 5-fold cross-validation. The parameters of the models, such as $\lambda$ for approaches (i-ii) or $C$ and $t$ for approaches (iii-iv), are tuned by cross-validation on an independent set of examples, always selecting the value that offers the best classification accuracy.

\paragraph{Accuracy and sparsity}
Classification accuracy is reported in \fref{fig:res:acc}.
All methods perform essentially the same, except for 1-NN that is somewhat weaker.
Note that the methods based on the learned edit probabilities are, as expected, more accurate than those based on the standard edit distance.
\fref{fig:res:sparsity} shows the average size of a binary model for approaches (i-iv), i.e., the number of training examples (reasonable points or support vectors) involved in the classification of new examples. Approaches (i-ii) are 5 to 6 times sparser than (iii-iv), which confirms that learning with Balcan leads to much sparser models than standard SVM learning.

\begin{figure}[t]
\begin{center}
\subfigure[0 vs. 1]{\label{fig:res:acc}
\psfrag{Classification accuracy}[][][0.65]{\textbf{Classification accuracy}}
\psfrag{Size of training set}[][][0.65]{\textbf{Size of training sample}}
\psfrag{LP with eL}[][][0.53]{LP with $\tilde{d}_{lev}$\hspace{0.4cm}}
\psfrag{LP with pe}[][][0.53]{LP with $\tilde{p}_e$\hspace{0.2cm}}
\psfrag{SVM with eL}[][][0.53]{SVM with $d_{lev}$\hspace{0.4cm}}
\psfrag{SVM with pe}[][][0.53]{SVM with $p_e$\hspace{0.2cm}}
\psfrag{1-NN with eL}[][][0.53]{1-NN with $d_{lev}$\hspace{0.4cm}}
\psfrag{1-NN with pe}[][][0.53]{1-NN with $p_e$\hspace{0.2cm}}
\includegraphics[width=0.47\textwidth]{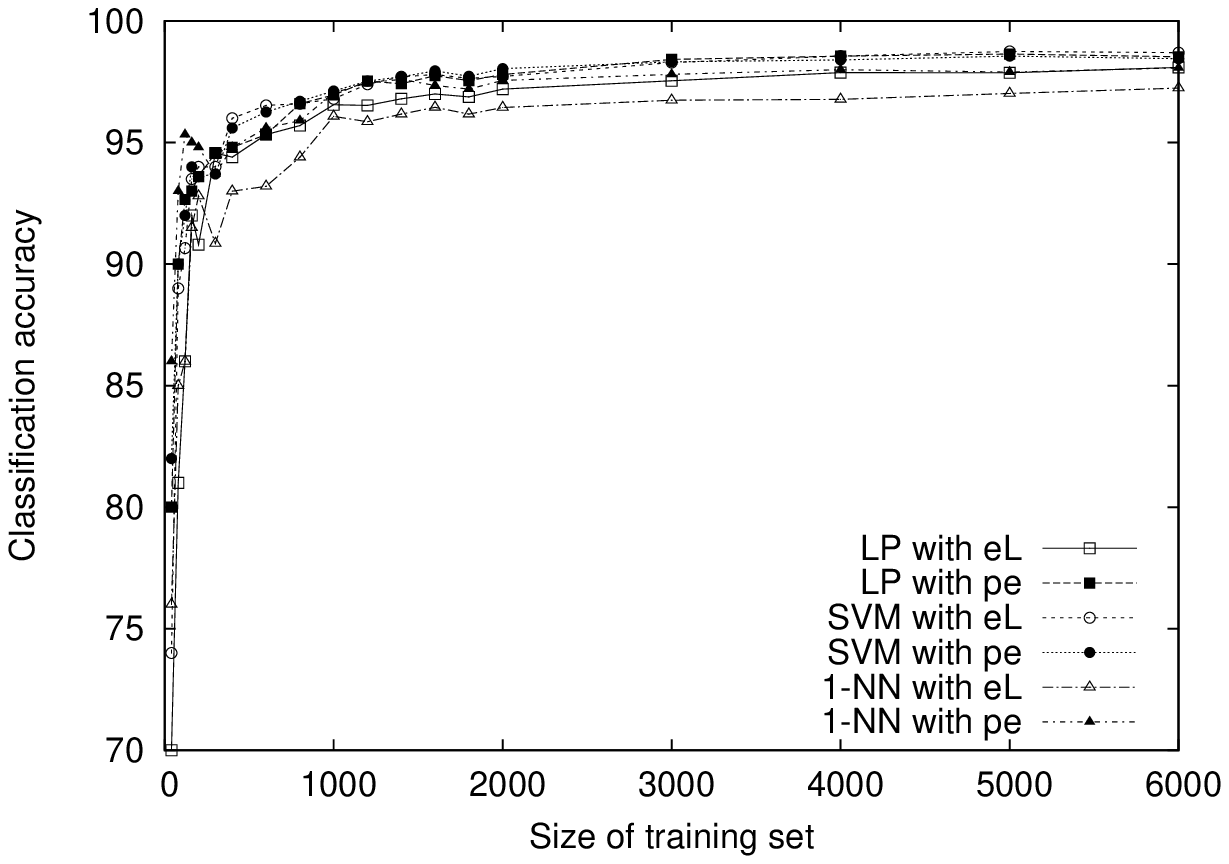}
}
\subfigure[0 vs. 8]{\label{fig:res:sparsity}
\psfrag{Size of training set}[][][0.65]{\textbf{Size of training sample}}
\psfrag{Average size of binary model}[][][0.65]{\textbf{Average size of binary model}}
\psfrag{LP with eL}[][][0.53]{LP with $\tilde{d}_{lev}$\hspace{0.4cm}}
\psfrag{LP with pe}[][][0.53]{LP with $\tilde{p}_e$\hspace{0.2cm}}
\psfrag{SVM with eL}[][][0.53]{SVM with $d_{lev}$\hspace{0.4cm}}
\psfrag{SVM with pe}[][][0.53]{SVM with $p_e$\hspace{0.2cm}}
\includegraphics[width=0.47\textwidth]{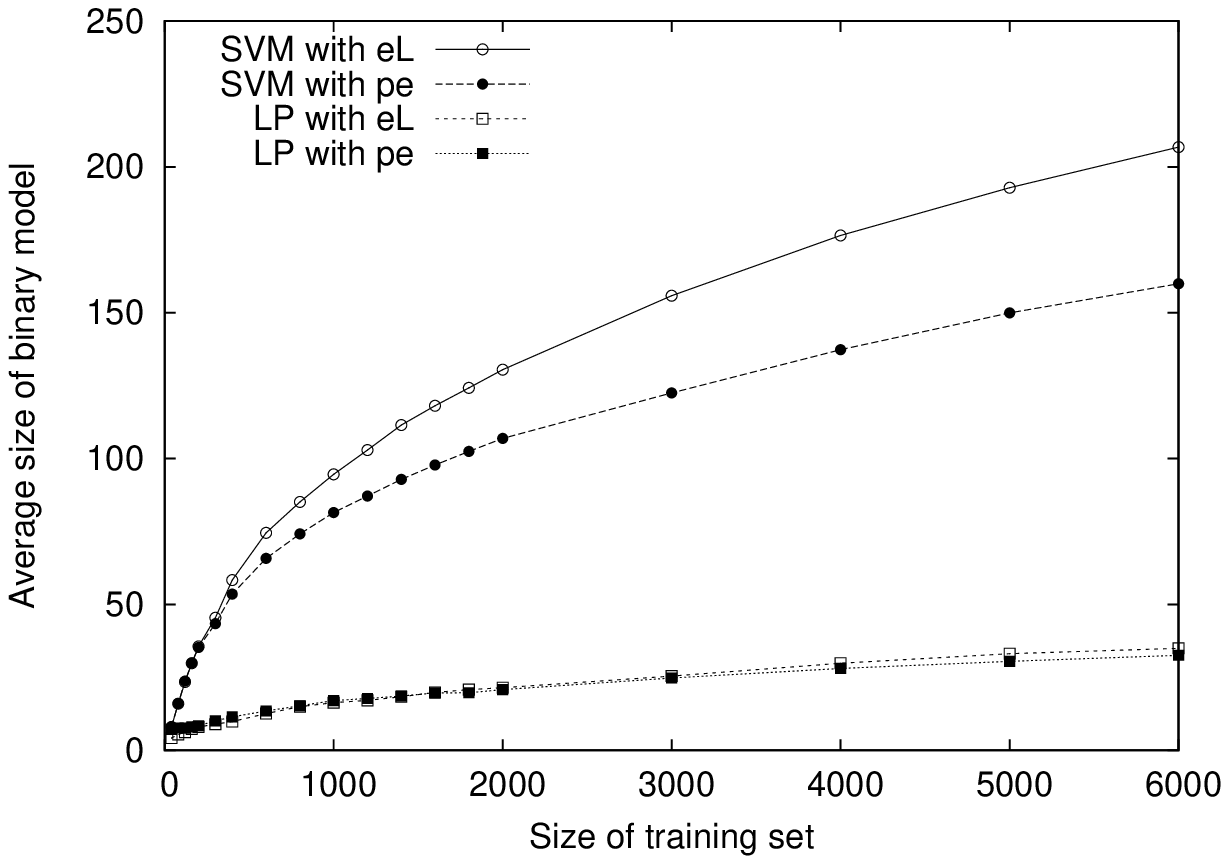}
}
\caption[Classification accuracy and sparsity (Digit dataset)]{Classification accuracy and sparsity results for methods (i-vi) over a range of training set sizes (Digit dataset).}
\label{fig:res}
\end{center}
\end{figure}

\paragraph{Influence of the parameters}
We now study the influence of parameters on the accuracy and sparsity of the models. Results are obtained on 4,000 training examples.
The influence of $\lambda$ on the models learned with Balcan is shown in \fref{fig:lambda}. The results confirm that $\lambda$ can be conveniently used to control the sparsity of the models thanks to $L_1$ regularization. It is worth noting that while the best accuracy is obtained with relatively small values ($\lambda \in [1;10]$), one can get even sparser but still very accurate models with larger values ($\lambda \in [10;200]$). This is especially true when using $\tilde{p}_e$. Therefore, one can learn a model with Balcan that is just slightly less accurate than the corresponding SVM model while being 10 to 18 times sparser. This can be a useful feature, in particular in applications where data storage is limited and/or high classification speed is required.

\begin{figure}[t]
\begin{center}
\psfrag{Accuracy}[][][0.55]{\textbf{Classification accuracy}}
\psfrag{Model size}[][][0.55]{\textbf{Average size of binary model}}
\psfrag{Value of lambda}[][][0.60]{\textbf{Value of} $\lambda$}
\psfrag{LP lev}[][][0.53]{LP with $\tilde{d}_{lev}$\hspace{0.6cm}}
\psfrag{LP prob}[][][0.53]{LP with $\tilde{p}_e$}
\includegraphics[width=0.6\textwidth]{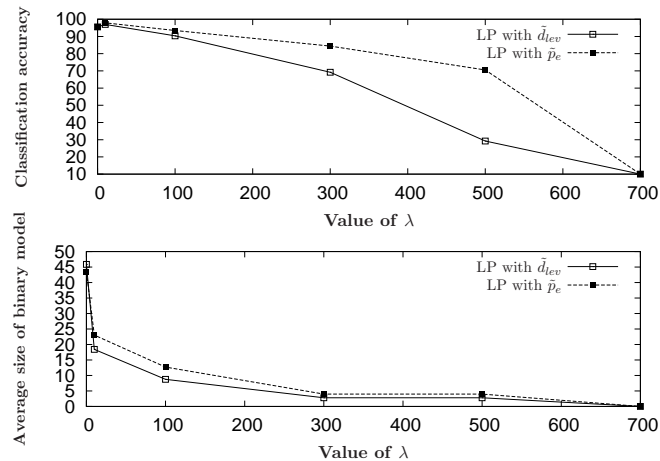}
\caption[Classification accuracy and sparsity with respect to $\lambda$]{Classification accuracy and sparsity results with respect to the value of $\lambda$ (Digit dataset).}
\label{fig:lambda}
\end{center}
\end{figure}

We also investigate the influence of parameter $t$ on the performance of the SVM models. Results are shown in \fref{fig:t} (a log-scale is used to allow a better appreciation of the variations). Both the accuracy and the sparsity of the SVM models are heavily dependent on $t$: only a narrow range of $t$ values (probably those achieving positive semi-definiteness) allows for accurate and acceptably-sized models. Furthermore, this range appears to be specific to the edit similarity used. Therefore, $t$ must be tuned very carefully, which represents a waste of time and data.

\begin{figure}[t]
\begin{center}
\psfrag{Accuracy}[][][0.55]{\textbf{Classification accuracy}}
\psfrag{Model size}[][][0.55]{\textbf{Average size of binary model}}
\psfrag{Value of t}[][][0.60]{\textbf{Value of} $t$ \textbf{(log-scale)}}
\psfrag{SVM lev}[][][0.53]{SVM with $d_{lev}$\hspace{0.6cm}}
\psfrag{SVM prob}[][][0.53]{SVM with $p_e$}
\includegraphics[width=0.6\columnwidth]{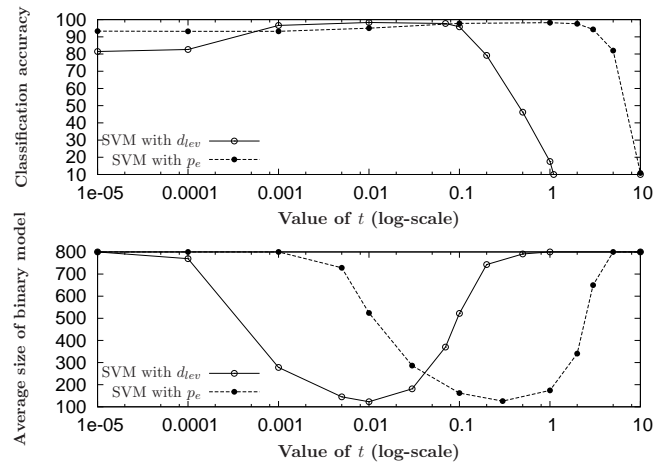}
\caption[Classification accuracy and sparsity with respect to $t$]{Classification accuracy and sparsity results with respect to the value of $t$ in log-scale (Digit dataset).}
\label{fig:t}
\end{center}
\end{figure}

Lastly, one might wonder whether the SVM parameter $C$ can also be used to improve the sparsity of the models in the same way as $\lambda$.
In order to assess this, we try a wide range of $C$ values and record the average sparsity of the models. SVM could not match the sparsity of the models learned with Balcan. The best average size for a binary model was greater than 100, i.e., more than 2 times bigger than the worst model size obtained with Balcan. This results from the tendency of $L_2$ regularization to select models that put small weights on many coordinates.

\subsubsection{English and French words classification}
\label{expeictaimots}

In this second series of experiments, we choose a different and harder task: classifying words as either English or French. We use the 2,000 top words lists from Wiktionary.\footnote{\url{http://en.wiktionary.org/wiki/Wiktionary:Frequency_lists}} We only consider unique words (i.e., not appearing in both lists) of length at least 4, and we also get rid of accent and punctuation marks. We end up with about 2,600 words. We keep 600 words aside for cross-validation of parameters, 400 words to test the models and use the remaining words to learn the models.

Classification accuracy is reported in \fref{fig:resmots:acc}.
Note that this binary task is significantly harder than the one presented in the previous section, and that once again, models based on $p_e$ perform better than those based on $d_{lev}$.
Models learned with Balcan clearly outperform $k$-NN, while SVM models are the most accurate.
Sparsity results are shown in \fref{fig:resmots:sparsity}. The gap in sparsity between models learned with Balcan and SVM models is even greater on this dataset: the number of support vectors grows almost linearly with the number of training examples. This is consistent with the theoretical rate established by \citet{Steinwart2003}.

\begin{figure}[t]
\begin{center}
\subfigure[Accuracy]{\label{fig:resmots:acc}
\psfrag{Classification accuracy}[][][0.65]{\textbf{Classification accuracy}}
\psfrag{Size of training set}[][][0.65]{\textbf{Size of training sample}}
\psfrag{LP with eL}[][][0.53]{LP with $\tilde{d}_{lev}$\hspace{0.4cm}}
\psfrag{LP with pe}[][][0.53]{LP with $\tilde{p}_e$\hspace{0.2cm}}
\psfrag{SVM with eL}[][][0.53]{SVM with $d_{lev}$\hspace{0.4cm}}
\psfrag{SVM with pe}[][][0.53]{SVM with $p_e$\hspace{0.2cm}}
\psfrag{1-NN with eL}[][][0.53]{1-NN with $d_{lev}$\hspace{0.4cm}}
\psfrag{1-NN with pe}[][][0.53]{1-NN with $p_e$\hspace{0.2cm}}
\includegraphics[width=0.47\textwidth]{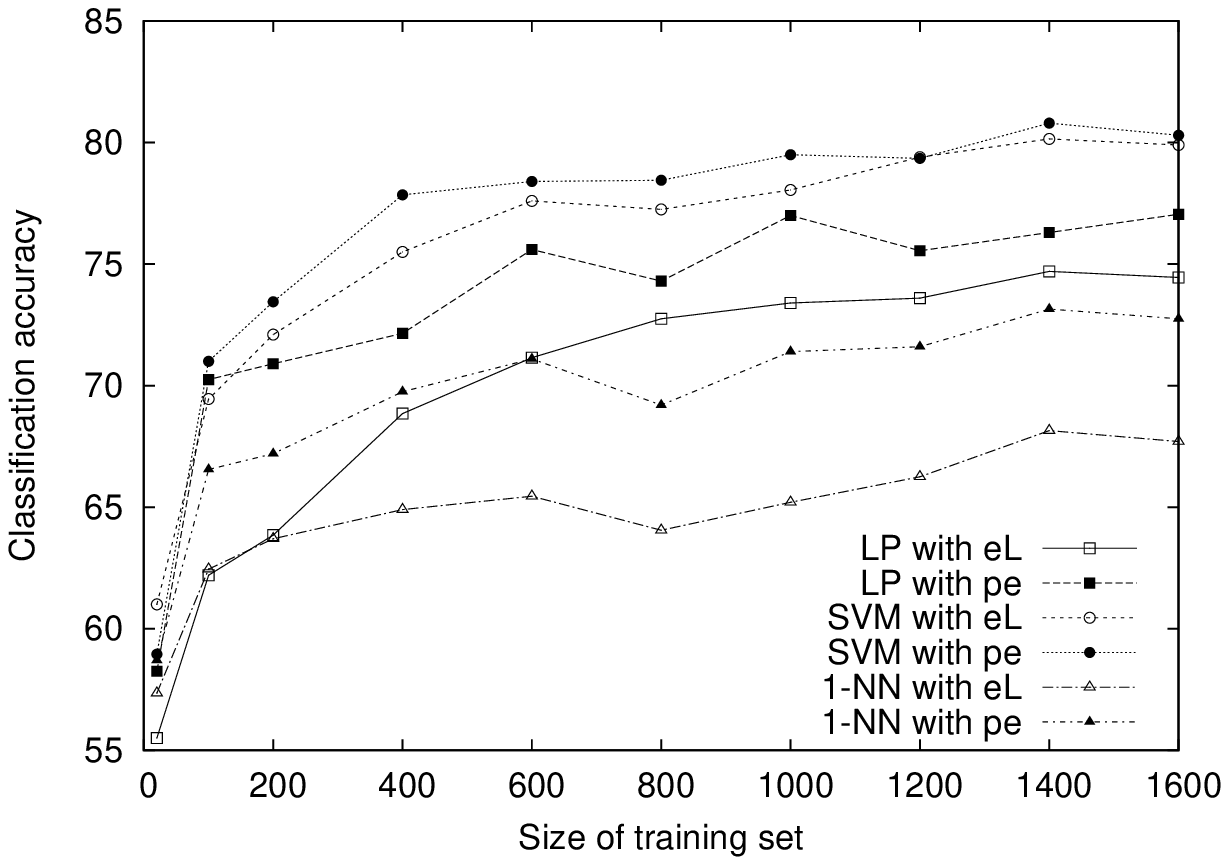}
}
\subfigure[Sparsity]{\label{fig:resmots:sparsity}
\psfrag{Size of training set}[][][0.65]{\textbf{Size of training sample}}
\psfrag{Model size}[][][0.65]{\textbf{Average size of binary model}}
\psfrag{LP with eL}[][][0.53]{LP with $\tilde{d}_{lev}$\hspace{0.4cm}}
\psfrag{LP with pe}[][][0.53]{LP with $\tilde{p}_e$\hspace{0.2cm}}
\psfrag{SVM with eL}[][][0.53]{SVM with $d_{lev}$\hspace{0.4cm}}
\psfrag{SVM with pe}[][][0.53]{SVM with $p_e$\hspace{0.2cm}}
\includegraphics[width=0.47\textwidth]{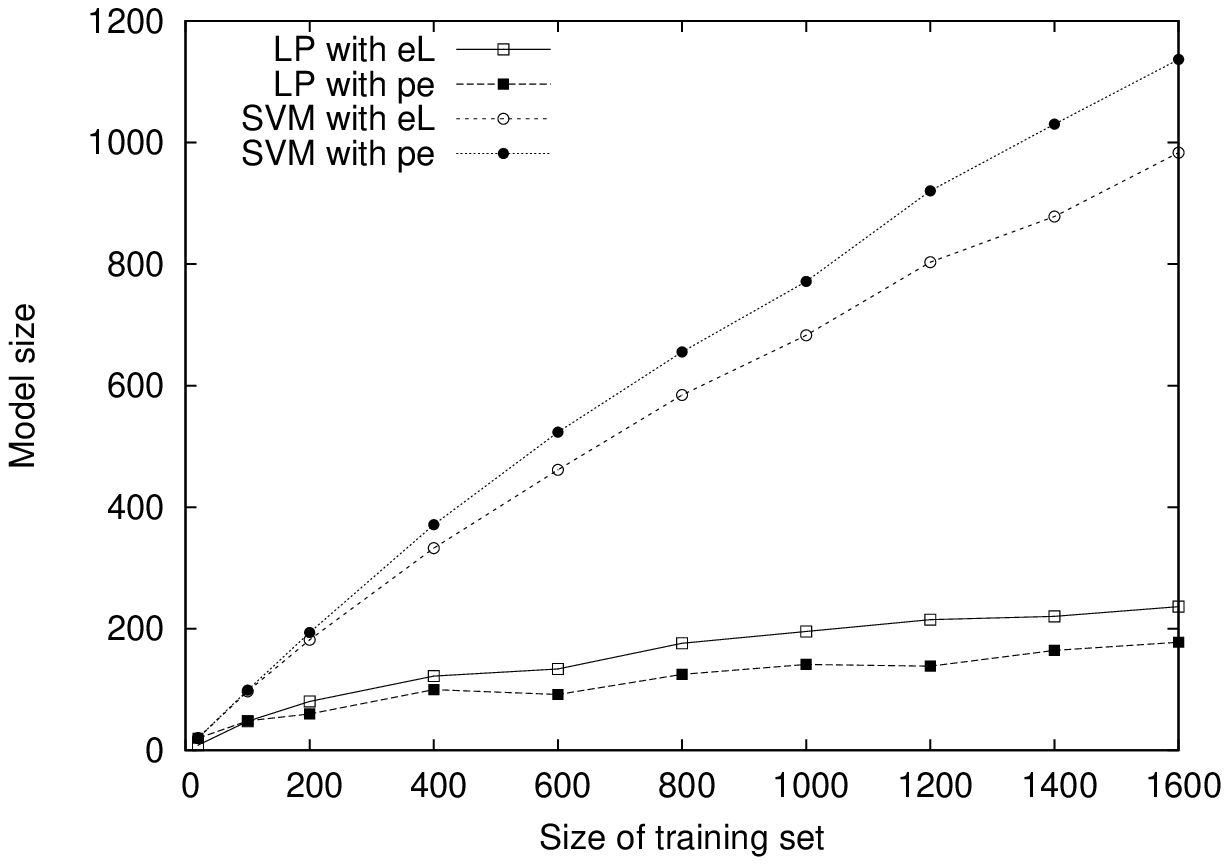}
}
\caption[Classification accuracy and sparsity (Word dataset)]{Word dataset: classification accuracy and sparsity results for methods (i-vi) over a range of training set sizes.}
\label{fig:resmots}
\end{center}
\end{figure}

\subsection{Conclusion}

In this section, we have shown that edit similarities fit the framework of $(\epsilon,\gamma,\tau)$-goodness and that the performance is competitive with standard SVM, with the additional advantages that arbitrary (in particular, non-PSD) similarities can be used and that the classifiers are sparser. We have also seen that this series of experiments confirms the theoretical dependence between the $(\epsilon,\gamma,\tau)$-goodness of the edit similarity function and its performance in classification.

However, for some tasks, standard edit similarities may satisfy the definition of $(\epsilon,\gamma,\tau)$-goodness poorly. Furthermore, existing methods for learning edit similarities rely on maximum likelihood and may not lead to an improved similarity function from an $(\epsilon,\gamma,\tau)$-goodness point of view. \citet{Kar2011} propose to automatically adapt the goodness criterion to the problem at hand. In the rest of this chapter, we take a different approach: we see the $(\epsilon,\gamma,\tau)$-goodness as a novel, theoretically well-founded criterion to optimize an edit similarity.

\section{Learning $(\epsilon,\gamma,\tau)$-Good Edit Similarity Functions}
\label{sec:mljlearncosts}


In this section, we propose a novel convex programming approach based on the theory of \citet{Balcan2008a} to learn $(\epsilon,\gamma,\tau)$-good edit similarity functions from both positive and negative pairs without requiring a costly iterative procedure. We will see in \sref{sec:mljanalysis} that this framework allows us to derive generalization bounds establishing the consistency of our method and a relationship between the learned similarities and the generalization error of the linear classifier using it.

We begin this section by introducing an exponential-based edit similarity function that can be optimized in a direct way. Then, we present our convex
programming approach to the problem of learning $(\epsilon,\gamma,\tau)$-good edit similarity functions, followed by a discussion on building relevant training pairs in this context. Finally, we end this section by showing that our approach can be straightforwardly adapted to tree edit similarity learning.

\subsection{An Exponential-based Edit Similarity Function}
\label{sec:eG}

In order to avoid the drawbacks of using iterative approaches such as EM for edit similarity learning, we propose to define an edit similarity for which the edit script does not depend on the edit costs.
 
Let $\mathbf{C}\in\mathbb{R}_+^{(|\Sigma|+1)\times(|\Sigma|+1)}$ be the edit cost matrix and for any $\mathsf{x},\mathsf{x'}\in\Sigma^*$, let $\boldsymbol{\#}(\mathsf{x},\mathsf{x'})$ be a $(|\Sigma|+1)\times (|\Sigma|+1)$ matrix whose elements $\#_{i,j}(\mathsf{x},\mathsf{x'})$ correspond to the number of times each edit operation $i\to j$ is used to turn $\mathsf{x}$ into $\mathsf{x'}$ in the Levenshtein script, $0 \leq i,j \leq |\Sigma|$. We define the following edit function: 
$$e_\mathbf{C}(\mathsf{x},\mathsf{x'}) = \displaystyle\sum_{0 \leq i,j \leq |\Sigma|}C_{i,j}\#_{i,j}(\mathsf{x},\mathsf{x'}).$$
To compute $e_\mathbf{C}$, we do not extract the optimal script with respect to $\mathbf{C}$: we use the Levenshtein script\footnote{In practice, one could use another type of script. We picked the Levenshtein script because it is a ``reasonable'' edit script, since it corresponds to a shortest script transforming $\mathsf{x}$ into $\mathsf{x'}$.} and apply custom costs $\mathbf{C}$ to it.
Therefore, since the edit script defined by $\boldsymbol{\#}(\mathsf{x},\mathsf{x'})$ is fixed, $e_\mathbf{C}(\mathsf{x},\mathsf{x'})$ is nothing more than a closed-form linear function of the edit costs and can be optimized directly.

Recall that a similarity function is assumed to be in $[-1,1]$.
To respect this requirement, we define our similarity function to be: 
$$K_\mathbf{C}(\mathsf{x},\mathsf{x'}) = 2e^{-e_\mathbf{C}(\mathsf{x},\mathsf{x'})}-1.$$
Beyond this normalization requirement, the motivation for this exponential form is related to the one for using exponential kernels in SVM classifiers: it can be seen as a way to introduce nonlinearity to further separate examples of opposite class while moving closer those of the same class.
Note that $K_\mathbf{C}$ may not be PSD nor symmetric. However, as we have seen earlier and unlike kernel theory, the theory of \citet{Balcan2008a} does not require these properties. This allows us to consider a broader type of edit similarity functions.



\subsection{Learning the Edit Costs}
\label{sec:lec}

We aim at learning the edit cost matrix $\mathbf{C}$ so as to optimize the $(\epsilon,\gamma,\tau)$-goodness of $K_\mathbf{C}$.
We first focus on optimizing the goodness based on a relaxation of \defref{def:defgoodsim2}, leading to a formulation based on the hinge loss (GESL$_{HL}$).
Then, we introduce a more general version that can accommodate other loss functions (GESL$_L$).

\subsubsection{Hinge Loss Formulation}

Here, we want to learn $K_\mathbf{C}$ so that its hinge loss-based goodness (\defref{def:defgoodsim2}) is optimized. More precisely, given a set of reasonable points and a margin $\gamma$, we want to optimize the amount of margin violation $\epsilon$.
Ideally, we would like to directly optimize \defref{def:defgoodsim2}. Unfortunately, this would result in a nonconvex formulation (summing and subtracting up exponential terms) subject to local minima.  Instead, we propose to optimize the following criterion:
\begin{equation}\label{eq:newgoodsim}
\mathbb{E}_{(\mathsf{x},y)}\left[\mathbb{E}_{(\mathsf{x'},y')}\left[\left[1-yy'K_\mathbf{C}(\mathsf{x},\mathsf{x'})/\gamma\right]_+|R(\mathsf{x'})\right]\right]\leq \epsilon'.
\end{equation}
Criterion \eqref{eq:newgoodsim} bounds that of \defref{def:defgoodsim2} due to the convexity of the hinge loss: clearly, if $K_\mathbf{C}$ satisfies \eqref{eq:newgoodsim}, then it is $(\epsilon,\gamma,\tau)$-good in hinge loss with $\epsilon \leq \epsilon'$. Indeed, it is harder to satisfy since the ``goodness'' is required with respect to each reasonable point instead of considering the average similarity to these points. Therefore, optimizing $K_\mathbf{C}$ according to \eqref{eq:newgoodsim} implies the use of pair-based constraints.

Let us now consider a training sample $\mathcal{T} = \{z_i=(\mathsf{x_i},y_i)\}_{i=1}^{n_\mathcal{T}}$ of $n_\mathcal{T}$ labeled instances. Recall that we do not know the set of reasonable points at this stage: they are inferred while learning the separator, that is, after the similarity is learned. For this reason, as in most metric learning methods, we will suppose that we are given pairs of examples. Formally, we suppose the existence of an indicator pairing function $f_{land} : \mathcal{T}\times \mathcal{T} \rightarrow \{0,1\}$ which takes as input two training examples in $\mathcal{T}$ and returns 1 if they are paired and 0 otherwise. We assume that  $f_{land}$ associates to each element $z \in \mathcal{T}$ exactly $n_\mathcal{L}$ examples (called, with a slight abuse of language, the landmarks for $z$), leading to a total of $n_\mathcal{T} n_\mathcal{L}$ pairs.
We discuss this matter further in \sref{sec:mljmatching}.

Our formulation aims at fulfilling \eqref{eq:newgoodsim} for each $(z_i,z_j)$ such that $f_{land}(z_i,z_j) = 1$. Therefore, we want $\left[1-y_iy_j K_\mathbf{C}(\mathsf{x_i},\mathsf{x_j})/\gamma\right]_+ = 0$, hence $y_iy_j K_\mathbf{C}(\mathsf{x_i},\mathsf{x_j}) \geq \gamma$.
A benefit from using this constraint is that it can easily be turned into an equivalent linear one, considering the following two cases.
\begin{enumerate}
\item If $y_i \neq y_j$, we get:
$$-K_\mathbf{C}(\mathsf{x_i},\mathsf{x_j}) \geq \gamma \Longleftrightarrow e^{-e_\mathbf{C}(\mathsf{x_i},\mathsf{x_j})} \leq \frac{1-\gamma}{2} \Longleftrightarrow e_\mathbf{C}(\mathsf{x_i},\mathsf{x_j}) \geq -\log(\frac{1-\gamma}{2}).$$
We can use a variable $B_1 \geq 0$ and write the constraint as $e_\mathbf{C}(\mathsf{x_i},\mathsf{x_j}) \geq B_1$, with the interpretation that $B_1 = -\log(\frac{1-\gamma}{2})$. In fact, $B_1 \geq -\log(\frac{1}{2})$.
\item Likewise, if $y_i = y_j$, we get $e_\mathbf{C}(\mathsf{x_i},\mathsf{x_j}) \leq -\log(\frac{1+\gamma}{2})$.
We can use a variable $B_2 \geq 0$ and write the constraint as $e_\mathbf{C}(\mathsf{x_i},\mathsf{x_j}) \leq B_2$, with the interpretation that $B_2 = -\log(\frac{1+\gamma}{2})$. In fact, $B_2 \in \left[0,-\log(\frac{1}{2})\right]$.
\end{enumerate}
The optimization problem GESL$_{HL}$ can then be expressed as follows:

\begin{equation*}
\begin{array}{lrl}
(\text{GESL}_{HL})\quad& \displaystyle\min_{\mathbf{C},B_1,B_2} & \frac{1}{n_\mathcal{T} n_\mathcal{L}}\displaystyle\sum_{\substack{1\leq
    i\leq n_\mathcal{T},\\j : f_{land}(z_i,z_j)=1}}\ell_{HL}(\mathbf{C},z_i,z_j)+\beta\|\mathbf{C}\|^2_{\cal{F}} \label{GESL_hinge}\\
& \text{s.t.} & B_1 \geq -\log(\frac{1}{2}),\quad 0 \leq B_2 \leq -\log(\frac{1}{2}),\quad B_1 - B_2 = \eta_\gamma\nonumber\\
& & C_{i,j} \geq 0,\quad 0 \leq i,j \leq |\Sigma|,\nonumber
\end{array}
\end{equation*}
where $\beta \geq 0$ is a regularization parameter on edit costs, $\eta_\gamma \geq 0$ a parameter corresponding to the desired ``margin'' and
$$\ell_{HL}(\mathbf{C},z_i,z_j)=
\left\{\begin{array}{l}
{[} B1 - e_\mathbf{C}(\mathsf{x_i},\mathsf{x_j}) {]}_{+} \textrm{ if } y_i\neq y_j\\
{[} e_\mathbf{C}(\mathsf{x_i},\mathsf{x_j}) - B2 {]}_{+} \textrm{ if } y_i=y_j 
\end{array}\right..$$
The relationship between the margin $\gamma$ and $\eta_\gamma$ is given by $\gamma=\frac{e^{\eta_\gamma}-1}{e^{\eta_\gamma}+1}$. We chose Frobenius norm regularization because (i) it is simple, smooth and thus easier to optimize, and (ii) it allows us to derive generalization guarantees using uniform stability, as we will see in \sref{sec:mljanalysis}.

GESL$_{HL}$ is a convex program, thus one can efficiently find its global optimum. Using $n_\mathcal{T} n_\mathcal{L}$ slack variables to express each hinge loss, it has $O(n_\mathcal{T} n_\mathcal{L}+|\Sigma|^2)$ variables and $O(n_\mathcal{T} n_\mathcal{L})$ constraints. Note that GESL$_{HL}$ is a sparse convex program: each constraint involves at most one string pair and a limited number of edit cost variables, making the problem faster to solve.
It is also worth noting that our approach is very flexible. First, it is general enough to be used with any definition of $e_\mathbf{C}$ that is based on an edit script (or even a convex combination of edit scripts). Second, one can incorporate additional convex constraints, for instance to include background knowledge or desired requirements on $\mathbf{C}$ (e.g., symmetry). Third, it can be easily adapted to the multi-class case. Finally, it can be generalized to a larger class of loss functions, as we show in the following section.

\subsubsection{General Formulation}


In the previous section, we made use of the hinge loss-based \defref{def:defgoodsim2} to propose GESL$_{HL}$.
Yet, other reformulations of \defref{def:defgoodsim} are possible using any convex loss function that can be used to efficiently penalize the amount of violation $\epsilon$ with respect to margin $\gamma$. For instance, the logistic loss or the exponential loss could be used. This would also allow the derivation of learning guarantees (similar to \thref{thm:thmsim}) and an efficient learning rule.

Therefore, it is useful to be able to optimize a definition of $(\epsilon,\gamma,\tau)$-goodness based on a loss other than the hinge.
Let $\ell(\mathbf{C},z,z')$ be a convex loss function with respect to an edit cost matrix $\mathbf{C}$ and a pair of examples $(z,z')$.
Our optimization problem  can then be expressed in a more general form  as follows:

\begin{equation*}
\begin{array}{lcl}
(\text{GESL}_L)\quad& \displaystyle\min_{\mathbf{C}} & \frac{1}{n_\mathcal{T} n_\mathcal{L}}\displaystyle\sum_{\substack{1\leq
    i\leq n_\mathcal{T},\\j : f_{land}(z_i,z_j)=1}}\ell(\mathbf{C},z_i,z_j)+\beta\|\mathbf{C}\|^2_{\cal{F}}.\label{GESL_general}
\end{array}
\end{equation*}

In the rest of the paper, we will use GESL to refer to our approach in general, GESL$_L$ when using an arbitrary loss function $\ell$ and GESL$_{HL}$ for the specific case of the hinge loss.

\subsection{Pairing strategy}
\label{sec:mljmatching}

The question of how one should define the pairing function $f_{land}$  relates to the open question of building training pairs in many metric learning problems. In some applications, the answer may be trivial: for instance, a misspelled word and its correction. Otherwise, popular choices are to pair each example with its nearest neighbor, random pairing or simply to consider all possible pairs.

On the other hand, the $(\epsilon,\gamma,\tau)$-goodness of the similarity should be improved with respect to the reasonable points, a subset of examples of probability $\tau$ that allows low error and large margin. However, this set depends on the similarity function itself and is thus unknown beforehand. Yet, a relevant strategy in the context of $(\epsilon,\gamma,\tau)$-goodness may be to improve the similarity with respect to carefully selected examples rather than considering all possible pairs. Consequently, we consider two pairing strategies that will be compared in our experiments (\sref{sec:mljexperiments}):
\begin{enumerate}
\item \emph{Levenshtein pairing}: we pair each $z \in \mathcal{T}$ with its $n_\mathcal{T}$ nearest neighbors of the same class and its $n_\mathcal{T}$ farthest neighbors of the opposite class, using the Levenshtein distance. This pairing strategy is meant to capture the essence of \defref{def:defgoodsim} and in particular the idea that reasonable points ``represent'' the data well. Essentially, we pair $z$ with a few points that are already good representatives of $z$ and optimize the edit costs so that they become even better representatives. Note that the choice of the Levenshtein distance to pair examples is consistent with our choice to define $e_\mathbf{C}$ according to the Levenshtein script.
\item \emph{Random pairing}: we pair each $z \in \mathcal{T}$ with a number $n_\mathcal{T}$ of randomly chosen examples of the same class and $n_\mathcal{T}$ randomly chosen examples of the opposite class.
\end{enumerate}
In either case, we have $n_\mathcal{L}=2N =\alpha n_\mathcal{T}$ with $0<\alpha\leq 1$.
Taking $\alpha=1$ corresponds to considering all possible pairs. In a sense, $\alpha$ can be seen as playing the role of $\tau$ (which gives the proportion of points that are reasonable in the definition of goodness) at the pair level, even though no direct relation can be made between the two.


\subsection{Adaptation to trees}
\label{sec:trees}

So far, we have implicitly considered that the data are strings. In this section, before presenting a theoretical analysis of GESL, we show that it may be used in a simple and efficient way to learn tree edit similarities.
As mentioned in \sref{sec:eG}, our edit function, defined as
$$e_\mathbf{C}(\mathsf{x},\mathsf{x'}) = \sum_{0 \leq i,j \leq |\Sigma|}C_{i,j}\#_{i,j}(\mathsf{x},\mathsf{x'}),$$
is nothing more than a linear combination of the edit costs, where $\#_{i,j}(\mathsf{x},\mathsf{x'})$ is the number of times the edit operation $i\to j$ occurs in the Levenshstein script turning $\mathsf{x}$ into $\mathsf{x'}$. This opens the door to a straightforward generalization of GESL to tree edit distance: instead of a string edit script, we can use a tree edit script according to either variant of the tree edit distance \citep{Zhang1989,Selkow1977} and solve the (otherwise unchanged) optimization problem presented in \sref{sec:lec}. This allows us, once again, to avoid using a costly iterative procedure. We only have to compute the edit script between two trees once, which dramatically reduces the algorithmic complexity of the learning algorithm.
Moreover, we will see that the theoretical analysis of GESL presented in the following section holds for tree edit similarity learning.


\section{Theoretical Analysis}
\label{sec:mljanalysis}

This section presents a theoretical analysis of GESL.
In \sref{sec:mljguarantees}, we derive a generalization bound guaranteeing its consistency and relating to the $(\epsilon,\gamma,\tau)$-goodness in generalization of the learned similarity function, and thus to the true risk of the linear classifier. This theoretical study is performed for a large class of loss functions. In \sref{sec:spec_hinge}, we instantiate this generalization bound for the specific case of the hinge loss (GESL$_{HL}$). Finally, \sref{sec:stoch_lang} is devoted to a discussion about the main features of the bounds, and to the presentation of a way to get rid of the assumption that the length of the strings (or the size of the trees) is bounded.

\subsection{Generalization Bound for General Loss Functions}
\label{sec:mljguarantees}

As pointed out in \cref{chap:metriclearning}, the training pairs used in metric learning are not i.i.d. and therefore the classic results of statistical learning theory do not directly hold. To derive a generalization bound for GESL$_L$, we build upon the adaptation of uniform stability to the metric learning case \citep{Jin2009} and extend it to edit similarity learning. We first prove that GESL$_L$ has a uniform stability: this is established in \thref{thm:stability}, using \lref{lem:convexN2} and the assumption of $k$-lipschitzness (\defref{def:k-lipsC}). The stability property allows us to derive our generalization bound (\thref{thm:bound}) using the McDiarmid inequality (\thref{thm:McDiarmid}) and the assumption of $(\sigma,m)$-admissibility (\defref{def:s-m-adm}).

We denote the objective function of GESL$_L$ by:
$$
F_\mathcal{T}(\mathbf{C})= \frac{1}{n_\mathcal{T}}\sum_{k=1}^{n_\mathcal{T}} \frac{1}{n_\mathcal{L}}\sum_{j=1}^{n_\mathcal{L}} \ell(\mathbf{C},z_k,z'_{k_j})+\beta \|\mathbf{C}\|^2_{\cal{F}}, 
$$
where $z'_{k_j}$ denotes the $j^{th}$ landmark associated to 
$z_k$ and $\ell(\mathbf{C},z_k,z'_{k_j})$ the loss for a pair of examples with respect to an edit cost matrix $\mathbf{C}$.

The first term of $F_\mathcal{T}(\mathbf{C})$ is the empirical risk $R^\ell_\mathcal{T}(\mathbf{C})$ over the training sample $\mathcal{T}$. 
The true risk $R^\ell(\mathbf{C})$ is given by:
$$R^\ell(\mathbf{C})=\mathbb{E}_{(z,z')\sim P}[\ell(\mathbf{C},z,z')].$$
Recall that our empirical risk is not defined over all possible training pairs, unlike most metric learning algorithms, but according to some particular landmark examples. On the other hand, the true risk is defined over any pair of instances.
For notational convenience, we also introduce the estimation error $D_\mathcal{T}$, which is the deviation between the true risk and the empirical risk:
$$D_\mathcal{T}=R^\ell(\mathbf{C_\mathcal{T}})-R^\ell_\mathcal{T}(\mathbf{C_\mathcal{T}}),$$
where $\mathbf{C_\mathcal{T}}$ denotes the edit cost matrix learned by GESL$_L$ from $\mathcal{T}$. 

In this section, we propose an analysis that holds for a large class of loss functions. We consider loss functions $\ell$ that fulfill the $k$-lipschitz property with respect to the first argument $\mathbf{C}$ (\defref{def:k-lipsC}) and the definition of $(\sigma,m)$-admissibility (\defref{def:s-m-adm}).

\begin{definition}\label{def:k-lipsC}
A loss function $\ell(\mathbf{C},z_1,z_2)$ is $k$-lipschitz with respect to its first argument if for any matrices $\mathbf{C}, \mathbf{C'}$ and any pair of labeled examples $(z_1, z_2)$:
$$
 |\ell(\mathbf{C},z_1,z_2)-\ell(\mathbf{C'},z_1,z_2)|\leq k\|\mathbf{C}-\mathbf{C'}\|_{\cal{F}}.
$$
\end{definition}

\begin{definition}\label{def:s-m-adm}
A loss function $\ell(\mathbf{C},z_1,z_2)$ is \textit{$(\sigma,m)$-admissible}, with respect to $\mathbf{C}$, if (i) it is convex with respect to its first argument and (ii) the following condition holds:
$$
 \forall z_1, z_2, z_3, z_4, |\ell(\mathbf{C},z_1,z_2)-\ell(\mathbf{C},z_3,z_4)|\leq \sigma |y_1y_2-y_3y_4|+m
$$
with $z_i=(x_i,y_i)$, for $i=1, 2, 3, 4$, are labeled examples.
\end{definition}

\defref{def:s-m-adm} requires the deviation of the losses between two pairs of examples to be bounded by a value that depends only on the labels and on some constants independent from the examples and the cost matrix $\mathbf{C}$. It follows that the labels must be bounded, which is not a strong assumption in the classification setting we are interesting in. In our case, we have binary labels ($y_i \in \{-1,1\}$), which implies that the quantity $|y_1y_2-y_3y_4|$ is either $0$ or $2$. We will see in \sref{sec:spec_hinge} that the hinge loss of GESL$_{HL}$ satisfies \defref{def:k-lipsC} and \defref{def:s-m-adm}. This can also be shown for other popular loss functions, such as the logistic loss or the exponential loss.\footnote{To satisfy \defref{def:k-lipsC}, their domain must be bounded \citep{Rosasco2004}.}

Note that from the convexity of $\ell$ with respect to its first argument, it follows that $R^\ell$, $R^\ell_\mathcal{T}$ and $F_\mathcal{T}$ are convex functions.

Our objective is to derive an upper bound on the true risk $R^\ell(C_\mathcal{T})$ with respect to the empirical risk $R^\ell_\mathcal{T}(C_\mathcal{T})$ using uniform stability (\defref{def:stability}) adapted to the case where training data consist of pairs \citep{Jin2009}.
\begin{definition}[\citeauthor{Jin2009}, \citeyear{Jin2009}]
A learning algorithm has a uniform stability in $\frac{\kappa}{n_\mathcal{T}}$, where $\kappa$ is a positive constant, if
$$
\forall(\mathcal{T},z),  \forall i, \displaystyle \sup_{z_1,z_2} |\ell(\mathbf{C_\mathcal{T}},z_1,z_2)-\ell(\mathbf{C}_{\mathcal{T}^{i,z}},z_1,z_2)|\leq 
 \frac{\kappa}{n_\mathcal{T}},
$$
where $\mathcal{T}^{i,z}$ is the new set obtained by replacing $z_i\in \mathcal{T}$ by a new example $z$. 
\end{definition}
To prove that GESL$_L$ has the property of uniform stability, we need the following lemma and the $k$-lipschitz property of $\ell$.

\begin{lemma}
\label{lem:convexN2}
Let $F_\mathcal{T}$ and $F_{\mathcal{T}^{i,z}}$ be the functions to optimize, $\mathbf{C}_\mathcal{T}$ and $\mathbf{C}_{\mathcal{T}^{i,z}}$ their corresponding minimizers, and $\beta$ the regularization parameter used in GESL$_L$. Let $\Delta \mathbf{C}=(\mathbf{C}_\mathcal{T}-\mathbf{C}_{\mathcal{T}^{i,z}})$. For any $t\in[0,1]$:
$$
 \|\mathbf{C}_\mathcal{T}\|^2_{\cal{F}}-\|\mathbf{C}_\mathcal{T} -t\Delta \mathbf{C}\|^2_{\cal{F}}+\|\mathbf{C}_{\mathcal{T}^{i,z}}\|^2_{\cal{F}}-\|\mathbf{C}_{\mathcal{T}^{i,z}} +t\Delta \mathbf{C}\|^2_{\cal{F}}  \leq \frac{(2n_\mathcal{T}+n_\mathcal{L})t2k}{\beta n_\mathcal{T}n_\mathcal{L}}\|\Delta \mathbf{C}\|_{\cal{F}}.
$$
\end{lemma}
\begin{proof}
See \aref{app:appendix2}.
\end{proof}
We can now prove the stability of GESL$_L$.

\begin{theorem}[Stability of $\mathbf{GESL_L}$]
\label{thm:stability}
Let $n_\mathcal{T}$ and $n_\mathcal{L}$ be respectively the number of training examples and landmark points. Assuming that $n_\mathcal{L}=\alpha n_\mathcal{T}$, $\alpha \in \left]0,1\right]$, and that the loss function used in GESL$_L$ is $k$-lipschitz, then GESL$_L$ has a uniform stability in $\frac{\kappa}{n_\mathcal{T}}$, where $\kappa = \frac{2(2+\alpha)k^2}{\beta \alpha}$.
\end{theorem}
\begin{proof}
Using $t=1/2$ on the left-hand side of \lref{lem:convexN2}, we get
$$
\|\mathbf{C}_\mathcal{T}\|^2_{\cal{F}}-\|\mathbf{C}_\mathcal{T} -\frac{1}{2}\Delta \mathbf{C}\|^2_{\cal{F}}+\|\mathbf{C}_{\mathcal{T}^{i,z}}\|^2_{\cal{F}}-\|\mathbf{C}_{\mathcal{T}^{i,z}} +\frac{1}{2}\Delta \mathbf{C}\|^2_{\cal{F}}=\frac{1}{2}\|\Delta \mathbf{C}\|^2_{\cal{F}}.
$$
Then, applying \lref{lem:convexN2}, we get
$$
\frac{1}{2}\|\Delta \mathbf{C}\|^2_{\cal{F}}\leq \frac{(2n_\mathcal{T}+n_\mathcal{L})k}{\beta n_\mathcal{T}n_\mathcal{L}}\|\Delta \mathbf{C}\|_{\cal{F}} \Rightarrow \|\Delta \mathbf{C}\|_{\cal{F}}\leq \frac{2(2n_\mathcal{T}+n_\mathcal{L})k}{\beta n_\mathcal{T} n_\mathcal{L}}.
$$
Now, from the $k$-lipschitz property of $\ell$, we have for any $z,z'$ 
$$
|\ell(\mathbf{C}_\mathcal{T},z,z')-\ell(\mathbf{C}_{\mathcal{T}^{i,z}},z,z')|\leq k\|\Delta \mathbf{C}\|_{\cal{F}} \leq \frac{2(2n_\mathcal{T}+n_\mathcal{L})k^2}{\beta n_\mathcal{T} n_\mathcal{L}}.
$$
Replacing $n_\mathcal{L}$ by $\alpha n_\mathcal{T}$ completes the proof.
\end{proof}
Now, using the property of stability, we can derive our generalization  bound over $R^\ell(\mathbf{C}_\mathcal{T})$. This is done by using the McDiarmid inequality \citep{McDiarmid1989}.

\begin{theorem}[McDiarmid inequality] \label{thm:McDiarmid}
Let $X_1, \ldots, X_n$ be $n$ independent random variables taking values in ${X}$ and let $Z=f(X_1, \ldots, X_n)$. If for each $1\leq i \leq n$, there exists a constant $c_i$ such that
\begin{eqnarray*}
\lefteqn{\sup_{x_1, \ldots, x_n, x'_i\in \mathcal{X}}|f(x_1, \ldots, x_n) - f(x_1, \ldots,
 x'_i, \ldots, x_n)|\leq c_i, \forall 1\leq i\leq n,}\\
&&\text{then for any } \epsilon>0,\quad\quad\quad \mathrm{Pr}[|Z-{\mathbb E}[Z]|\geq \epsilon]\leq
2\exp\left(\frac{-2\epsilon^2}{\sum_{i=1}^n c_i^2}\right).
\end{eqnarray*}
\end{theorem}
To derive our bound on $R^\ell(\mathbf{C}_\mathcal{T})$, we just need to replace $Z$ by $D_\mathcal{T}$ in \thref{thm:McDiarmid} and to bound $\mathbb{E}_\mathcal{T}[D_\mathcal{T}]$ and $|D_\mathcal{T}-D_{\mathcal{T}^{i,z}}|$, which is shown by the following lemmas.
\begin{lemma}\label{lem:espD}
For any learning method of estimation error $D_\mathcal{T}$ and satisfying a uniform stability in $\frac{\kappa}{n_\mathcal{T}}$, we have 
$
\mathbb{E}_\mathcal{T}[D_\mathcal{T}]\leq \frac{2\kappa}{n_\mathcal{T}}.
$
\end{lemma}
\begin{proof}
See \aref{app:appendixespD}.
\end{proof}

\begin{lemma}\label{lem:diffD}
For any edit cost matrix learned by GESL$_L$ using $n_\mathcal{T}$ training examples and $n_\mathcal{L}$ landmarks, and any loss function $\ell$ satisfying $(\sigma,m)$-admissibility, we have the following bound:
$$
\forall i,1\leq i\leq n_\mathcal{T},\quad\forall z,\quad |D_\mathcal{T} - D_{\mathcal{T}^{i,z}}|\leq \frac{2\kappa}{n_\mathcal{T}} +  \frac{(2n_\mathcal{T}+n_\mathcal{L})(2\sigma+m)}{n_\mathcal{T} n_\mathcal{L}}.
$$
\end{lemma}
\begin{proof}
See \aref{app:appendixdiffD}.
\end{proof}
We are now able to derive our generalization bound over $R^\ell(\mathbf{C}_\mathcal{T})$.

\begin{theorem}[Generalization bound for GESL$_L$]
\label{thm:bound}
Let $\mathcal{T}$ be a sample of $n_\mathcal{T}$ randomly selected training examples and
let $\mathbf{C}_\mathcal{T}$ be the edit cost matrix learned by GESL$_L$ with stability $\frac{\kappa}{n_\mathcal{T}}$. Assuming that $\ell(\mathbf{C}_\mathcal{T},z,z')$ is $k$-lipschitz and $(\sigma,m)$-admissible,  
and using $n_\mathcal{L}=\alpha n_\mathcal{T}$ landmark points, with probability $1-{\delta}$, we have
the following bound for $R^\ell(\mathbf{C}_\mathcal{T})$:
$$
R^\ell(\mathbf{C}_\mathcal{T})\leq R^\ell_\mathcal{T}(\mathbf{C}_\mathcal{T}) + 2\frac{\kappa}{n_\mathcal{T}}+ \left(2\kappa +\frac{2+\alpha}{\alpha}\left(2\sigma+m \right) \right)\sqrt{\frac{\ln(2/\delta)}{2n_\mathcal{T}}}
$$
with $\kappa=\frac{2(2+\alpha)k^2}{\alpha \beta }$.
\end{theorem}

\begin{proof}
Recall that $D_\mathcal{T}=R^\ell(\mathbf{C}_\mathcal{T})-R^\ell_\mathcal{T}(\mathbf{C}_\mathcal{T})$ and $n_\mathcal{L}=\alpha n_\mathcal{T}$. From \lref{lem:diffD}, we get
$$
|D_\mathcal{T} - D_{\mathcal{T}^{i,z}}|\leq \sup_{\mathcal{T},z'} |D_\mathcal{T} - D_{\mathcal{T}^{i,z'}}|\leq \frac{2\kappa+B}{n_\mathcal{T}}
\text{ with } B=\frac{(2+\alpha)}{\alpha}(2\sigma+m).
$$
Then by applying the McDiarmid inequality, we have 
\begin{equation}
\mathrm{Pr}[|D_\mathcal{T} -\mathbb{E}_T[D_\mathcal{T}]|\geq\epsilon]\leq2\exp\left(-\frac{2\epsilon^2}{\sum_{i=1}^{n_\mathcal{T}}\frac{(2\kappa+B)^2}{n_\mathcal{T}^2}}\right)\leq2\exp\left(-\frac{2\epsilon^2}{\frac{(2\kappa+B)^2}{n_\mathcal{T}}}\right).
\label{ineqdiarmid}
\end{equation}
By fixing
${\delta}=2\exp\left(-\frac{2\epsilon^2}{({2\kappa+B})^2/n_\mathcal{T}}\right)$, we get
  $\epsilon=({2\kappa+B})\sqrt{\frac{\ln(2/\delta)}{2n_\mathcal{T}}}$.
Finally, from \eqref{ineqdiarmid}, \lref{lem:espD} and the definition of $D_\mathcal{T}$, we have with probability at least $1-\delta$:
\begin{equation*}
D_\mathcal{T} < \mathbb{E}_T[D_\mathcal{T}]+\epsilon\Rightarrow 
R^\ell(\mathbf{C}_\mathcal{T}) < R^\ell_\mathcal{T}(\mathbf{C}_\mathcal{T}) + 2\frac{\kappa}{n_\mathcal{T}}+(2\kappa+B)\sqrt{\frac{\ln(2/\delta)}{2n_\mathcal{T}}},
\end{equation*}
which gives the theorem.
\label{general_bound}
\end{proof}

\subsection{Generalization Bound for the Hinge Loss}
\label{sec:spec_hinge}

\thref{thm:bound} holds for any loss function $\ell(\mathbf{C}_\mathcal{T},z,z')$ that is $k$-lipschitz and ($\sigma,m$)-admissible with respect to $\mathbf{C}_\mathcal{T}$. Let us now rewrite this bound when $\ell$ is the hinge loss-based function $\ell_{HL}$ used in GESL$_{HL}$. We first have to prove that $\ell_{HL}$ is $k$-lipschitz (\lref{lem:k-lips-V}) and $(\sigma,m)$-admissible (\lref{lem:s-m-adm}). Then, we derive the generalization bound for GESL$_{HL}$.

In order to fulfill the $k$-lipschitz and $(\sigma,m)$-admissibility properties, we suppose every string length bounded by a constant $W > 0$. Since the Levenshtein script between two strings $\mathsf{x}$ and $\mathsf{x'}$ contains at most $\operatorname{max}(|\mathsf{x}|,|\mathsf{x'}|)$ operations, we have
$$\|\boldsymbol{\#}(\mathsf{x},\mathsf{x'})\|_{\cal{F}}=\sqrt{\sum_{l,c}\boldsymbol{\#}_{l,c}(\mathsf{x},\mathsf{x'})^2}\leq \sqrt{\left(\sum_{l,c}\boldsymbol{\#}_{l,c}(\mathsf{x},\mathsf{x'})\right)^2}\ \leq W.$$  
When dealing with labeled instances, we will sometimes denote $\|\boldsymbol{\#}(\mathsf{x},\mathsf{x'})\|_{\cal{F}}\leq W$ by  $\|\boldsymbol{\#}(z_1,z_2)\|_{\cal{F}}\leq W$ for the sake of convenience.



\begin{lemma}\label{lem:k-lips-V}
The function $\ell_{HL}$ is $k$-lipschitz with $k=W$.
\end{lemma}
\begin{proof}
See \aref{app:appendix1}.
\end{proof}

We will now prove that $\ell_{HL}$ is $(\sigma,m)$-admissible for any optimal solution $\mathbf{C}_\mathcal{T}$ learned by GESL$_{HL}$ (\lref{lem:s-m-adm}). To be able to do this, we must show that the norm of $\mathbf{C}_\mathcal{T}$ is bounded (\lref{lem:boundC}).

\begin{lemma}\label{lem:boundC}
Let $(\mathbf{C}_\mathcal{T},B_1,B_2)$ an optimal solution learned by GESL$_{HL}$ from a training sample $\mathcal{T}$, and let $B_\gamma=max(\eta_\gamma,-log(1/2))$. Then 
$
\|\mathbf{C}_\mathcal{T}\|_{\cal{F}}\leq \sqrt{\frac{B_\gamma}{\beta}}.
$
\end{lemma}
\begin{proof}
See \aref{app:appendixboundC}.
\end{proof}


\begin{lemma}\label{lem:s-m-adm}
For any optimal solution $(\mathbf{C}_\mathcal{T},B_1,B_2)$, $\ell_{HL}$ is  $(\sigma,m)$-admissible with $\sigma=\frac{\sqrt{\frac{B_\gamma}{\beta}}W+3B_\gamma}{2}$ and $m=\sqrt{\frac{B_\gamma}{\beta}}W$, with $B_\gamma=\max(\eta_\gamma,-\log(1/2))$.
\end{lemma}


\begin{proof}
  Let $\mathbf{C}_\mathcal{T}$ be an optimal solution learned by  GESL$_{HL}$ from a training sample T and let $ z_1, z_2, z_3, z_4$ be four labeled examples.
We study two cases:

\begin{enumerate}

\item If $y_1y_2=y_3y_4$, regardless of the label values, using the 1-lispschitz property of the hinge loss, $B_1$ (when $y_1y_2=y_3y_4=-1$) or $B_2$ ($y_1y_2=y_3y_4=1$) cancels out (in a similar way as in \aref{lem:k-lips-V})  
and thus :
\begin{eqnarray*}
|\ell_{HL}(\mathbf{C}_\mathcal{T},z_1,z_2) - \ell_{HL}(\mathbf{C}_\mathcal{T},z_3,z_4)|&\leq &\|\mathbf{C}_\mathcal{T}\|_{\cal{F}}\|\boldsymbol{\#}(z_1,z_2)-\boldsymbol{\#}(z_3,z_4)\|_{\cal{F}}\\
&\leq &\sqrt{\frac{B_\gamma}{\beta}}W \ \hspace*{2cm}\textrm{from \lref{lem:boundC}}.
\end{eqnarray*}

\item Otherwise, if $y_1y_2\neq y_3y_4$, note that $|B_1+B_2|=\eta_\gamma+2B_2\leq 3B_\gamma$ and $|y_1y_2-y_3y_4|=2$. Hence, whatever the labels of the examples compatible with this case, by using the 1-lipschitz property of hinge loss and application of the triangular inequality, we get
\begin{eqnarray*}
|\ell_{HL}(\mathbf{C}_\mathcal{T},z_1,z_2) - \ell_{HL}(\mathbf{C}_\mathcal{T},z_3,z_4)|&\leq&|\sum_{l,c}C_{\mathcal{T},l,c}(\boldsymbol{\#}_{l,c}(z_1,z_2)+\boldsymbol{\#}_{l,c}(z_3,z_4)) |+\\
&&\hspace*{.25cm}|B_1+B_2|\\
&\leq&\|\mathbf{C}_\mathcal{T}\|_{\cal{F}}\|\boldsymbol{\#}(z_1,z_2)+\boldsymbol{\#}(z_3,z_4)\|_{\cal{F}} +3B_\gamma\\
&\leq&\sqrt{\frac{B_\gamma}{\beta}}2W +3B_\gamma\\
&\leq& \frac{\sqrt{\frac{B_\gamma}{\beta}}W+3B_\gamma}{2}|y_1y_2-y_3y_4|+ \sqrt{\frac{B_\gamma}{\beta}}W.
\end{eqnarray*} 

\end{enumerate}
Then, by choosing $\sigma=\frac{\sqrt{\frac{B_\gamma}{\beta}}W+3B_\gamma}{2}$ and $m=\sqrt{\frac{B_\gamma}{\beta}}W$, we have that $\ell_{HL}$ is $(\sigma,m)$-admissible.
\end{proof}

We can now give the convergence bound for GESL$_{HL}$.

\begin{theorem}[Generalization bound for GESL$_{HL}$]
\label{thm:bound2}
Let $\mathcal{T}$ be a sample of $n_\mathcal{T}$ randomly selected training examples and
let $\mathbf{C}_\mathcal{T}$ be the edit cost matrix learned by GESL$_{HL}$ with stability $\frac{\kappa}{n_\mathcal{T}}$ 
using $n_\mathcal{L}=\alpha n_\mathcal{T}$ landmark points. With probability $1-{\delta}$, we have
the following bound for $R^\ell(\mathbf{C}_\mathcal{T})$:
$$
R^\ell(\mathbf{C}_\mathcal{T})\leq R^\ell_\mathcal{T}(\mathbf{C}_\mathcal{T}) + 2\frac{\kappa}{n_\mathcal{T}}+ \left(2\kappa +\frac{2+\alpha}{\alpha}\left(\frac{2W}{\sqrt{\beta B_\gamma}}+3\right)B_\gamma \right)\sqrt{\frac{\ln(2/\delta)}{2n_\mathcal{T}}}
$$
with $\kappa=\frac{2(2+\alpha)W^2}{\alpha \beta }$ and $B_\gamma=\operatorname{max}(\eta_\gamma,-log(1/2))$.
\end{theorem}

\begin{proof}
It directly follows from \thref{thm:bound}, \lref{lem:k-lips-V} and \lref{lem:s-m-adm} by noting that $2\sigma+m=\left(\frac{2W}{\sqrt{\beta B_\gamma}}+3\right)B_\gamma$.
\end{proof}

\subsection{Discussion}
\label{sec:stoch_lang}

The generalization bounds presented in \thref{thm:bound} and \thref{thm:bound2} outline three important features of our approach. To begin with, it has a classic $O(\sqrt{1/n_\mathcal{T}})$ convergence rate. Second, this rate of convergence is independent of the alphabet size, which means that our method should scale well to problems with large alphabets. We will see in \sref{sec:mljexperiments} that it is actually the case in practice. Finally, thanks to the relation between the optimized criterion and the definition of $(\epsilon,\gamma,\tau)$-goodness that we established earlier, these bounds also ensure the goodness in generalization of the learned similarity function. Therefore, they guarantee that the similarity will induce classifiers with small true risk for the classification task at hand.

Note that to derive \thref{thm:bound2}, we assumed the size of the strings was bounded by a constant $W$. Even though this is not a strong restriction, it would be interesting to get rid of this assumption and derive a bound that is independent of $W$. This is possible when the marginal distribution of $P$ over the set of strings follows a generative model ensuring that the probability of a string decreases exponentially fast with its length. In this case, we can use the fact that very long strings have a very small probability to occur. Then with high probability, we can bound the maximum string length in a sample and remove $W$ from the generalization bound. Indeed, one can show that for any string stochastic language $p$ defined by a probabilistic automaton \citep{Denis2006} or a stochastic context-free grammar \citep{Etessami2009}, there exist some constants $U>0$ and $0<\rho<1$ such that the sum of the probabilities of strings of length at least $k$ is bounded:
\begin{equation}\label{eq:lang_sto}
\sum_{x, |x|>=k} p(x)< U \rho^k. 
\end{equation}
To take into account this result in our framework, we need an estimation of the length of the examples used to derive the generalization bound, that is, a sample of $n_\mathcal{T}$ examples with two additional examples $z$ and $z'$. 
For any sample of $n_\mathcal{T}+2$ strings identically and independently drawn from $p$, we can bound the length of any string $x$ of this sample. With a confidence greater than $1-\delta/2(n_\mathcal{T}+2)$, we have:
$$
|x|<\frac{\log(U2(n_\mathcal{T}+2)/\delta)}{\log(1/\rho)},
$$
by fixing $\delta/2(n_\mathcal{T}+2)= U \rho^k$.

Applying this result to every string of the sample, we get that with probability at least $1-\delta/2$, any sample of $n_\mathcal{T}+2$ elements has only strings of size at most $\frac{\log(U2(n_\mathcal{T}+2)/\delta)}{\log(1/\rho)}$. 
Then, by using \thref{thm:bound2} with a confidence $\delta/2$ and replacing $W$ by $\frac{\log(2(n_\mathcal{T}+2)U/\delta)}{\log(1/\rho)}$, we obtain the following bound.

\begin{theorem}
\label{theor_bound3}
Let $\mathcal{T}$ be a sample of $n_\mathcal{T}$ randomly selected training examples drawn from a stochastic language $p$ and
let $\mathbf{C}_\mathcal{T}$ be the edit costs learned by GESL$_{HL}$ with stability $\frac{\kappa}{n_\mathcal{T}}$ 
using $n_\mathcal{L}=\alpha n_\mathcal{T}$ landmark points. Then there exists constants $U>0$ and $0<\rho<1$ such that with probability at least $1-{\delta}$, we have:
\begin{small}
 $$
 R^\ell(\mathbf{C}_\mathcal{T})\leq R^\ell_\mathcal{T}(\mathbf{C}_\mathcal{T}) + 2\frac{\kappa}{n_\mathcal{T}}+ \left(2\kappa +\frac{2+\alpha}{\alpha}\left(\frac{2\log(2(n_\mathcal{T}+2)U/\delta)}{\sqrt{\beta B_\gamma}\log(1/\rho)}+3\right)B_\gamma \right)\sqrt{\frac{\ln(4/\delta)}{2n_\mathcal{T}}}
 $$
\end{small}
 with $\kappa=\frac{2(2+\alpha)\log^2(2(n_\mathcal{T}+2)U/\delta)}{\alpha \beta\log^2(1/\rho) }$ and $B_\gamma=\operatorname{max}(\eta_\gamma,-\log(1/2))$.
 \end{theorem}

Finally, let us conclude this section by discussing the adaptation of the entire theoretical analysis to tree edit similarity learning. The generalization bound for GESL$_L$ (\thref{thm:bound}) holds for trees since the arguments used in \sref{sec:mljguarantees} are not specific to strings. Regarding the bound for GESL$_{HL}$ (\thref{thm:bound2}), we used the assumption that the length of the strings is bounded by a constant $W$. This can be easily adapted to trees: if we assume that the size of each tree (in its number of nodes) is bounded by $W$, \thref{thm:bound2} also holds. Finally, the arguments for deriving a bound independent of the constant $W$ hold for trees since the property \eqref{eq:lang_sto} is also valid for rational stochastic tree languages \citep{Denis2008}.

\section{Experimental Validation}
\label{sec:mljexperiments}

In this section, we provide an experimental evaluation of GESL$_{HL}$.\footnote{An open-source implementation of our method is available at:\\\url{http://labh-curien.univ-st-etienne.fr/~bellet/}.} We are interested in evaluating the performance of different (standard or learned) edit similarities directly plugged into linear classifiers, as suggested by the theory of $(\epsilon,\gamma,\tau)$-goodness presented in \sref{sec:balcan}. Linear classifiers are learned using Balcan's learning rule \eqref{eq:lp1}.
We compare three edit similarity functions: (i) $K_\mathbf{C}$, learned by GESL$_{HL}$,\footnote{In this series of experiments, we constrained the cost matrices to be symmetric to be independent from the order in which the instances are paired.} (ii) the Levenshtein distance $d_{lev}$, which constitutes the baseline, and (iii) an edit similarity function $p_e$ learned with an EM-like algorithm \citep{Oncina2006}.
We show results on the same datasets as in the preliminary study (\sref{sec:ictai}): English and French words (\sref{sec:wiktionary}) and handwritten digits (\sref{sec:digits}).

\subsection{English and French Words}
\label{sec:wiktionary}

Recall that the task is to learn a model to classify words as either English or French. We use the 2,000 top words lists from Wiktionary.\footnote{These lists are available at \url{http://en.wiktionary.org/wiki/Wiktionary:Frequency_lists}. We only considered unique words (i.e., not appearing in both lists) of length at least 4, and we also got rid of accent and punctuation marks. We ended up with about 2,600 words over an alphabet of 26 symbols.}

\subsubsection{Convergence rate}
\label{sec:conv}

We first assess the convergence rate of the two considered edit cost learning methods (i and iii).
We keep aside 600 words as a validation set to tune the parameters, using 5-fold cross-validation and selecting the value offering the best classification accuracy.
We then build bootstrap samples $\mathcal{T}$ from the remaining 2,000 words to learn the edit costs (5 runs for each size $n_\mathcal{T}$), as well as 600 words to train the separator $\boldsymbol{\alpha}$ and 400 words to test its performance.

\fref{fig:figvarcost} shows the accuracy and sparsity results of each method with respect to $n_\mathcal{T}$, averaged over 5 runs.
We see that $K_\mathbf{C}$ leads to more accurate classifiers than $d_{lev}$ and $p_e$ for $n_\mathcal{T} > 20$. The difference is statistically significant: the Student's $t$-test yields a $p$-value $< 0.01$. At the same time, $K_\mathbf{C}$ requires 3 to 4 times less reasonable points, thus increasing classification speed by just as much.
The exact figures are as follows: $d_{lev}$ achieves 69.55\% accuracy with a model size of 197, $p_e$ achieves at best 74.80\% with a model size of 155, and $K_\mathbf{C}$ achieves at best 78.65\% with a model size of only 45. 
This clearly indicates that GESL$_{HL}$ leads to a better similarity than (ii) and (iii).
Moreover, the convergence rate of GESL$_{HL}$ is very fast, considering that $(26+1)^2 = 729$ costs must be learned: it needs very few examples (about 20) to outperform the Levenshtein distance, and about 200 examples to reach convergence.
This provides experimental evidence that our method scales well with the size of the alphabet, as suggested by the generalization bound derived in \sref{sec:spec_hinge}. On the other hand, (iii) seems to suffer from the large number of costs to estimate: it needs a lot more examples to outperform Levenshtein (about 200) and convergence seems to be only reached at 1,000.

\begin{figure}[t]
\begin{center}
\psfrag{Size of cost training set}[][][0.8]{\textbf{Size of the cost training sample}}
\psfrag{Classification accuracy}[][][0.8]{\textbf{Classification accuracy}}
\psfrag{Model size}[][][0.8]{\textbf{Model size}}
\psfrag{sim1}[][][0.8]{$d_{lev}$}
\psfrag{sim2}[][][0.8]{\hspace{0.2cm}$p_e$}
\psfrag{sim3}[][][0.8]{$K_\mathbf{C}$}
\includegraphics[width=0.49\columnwidth]{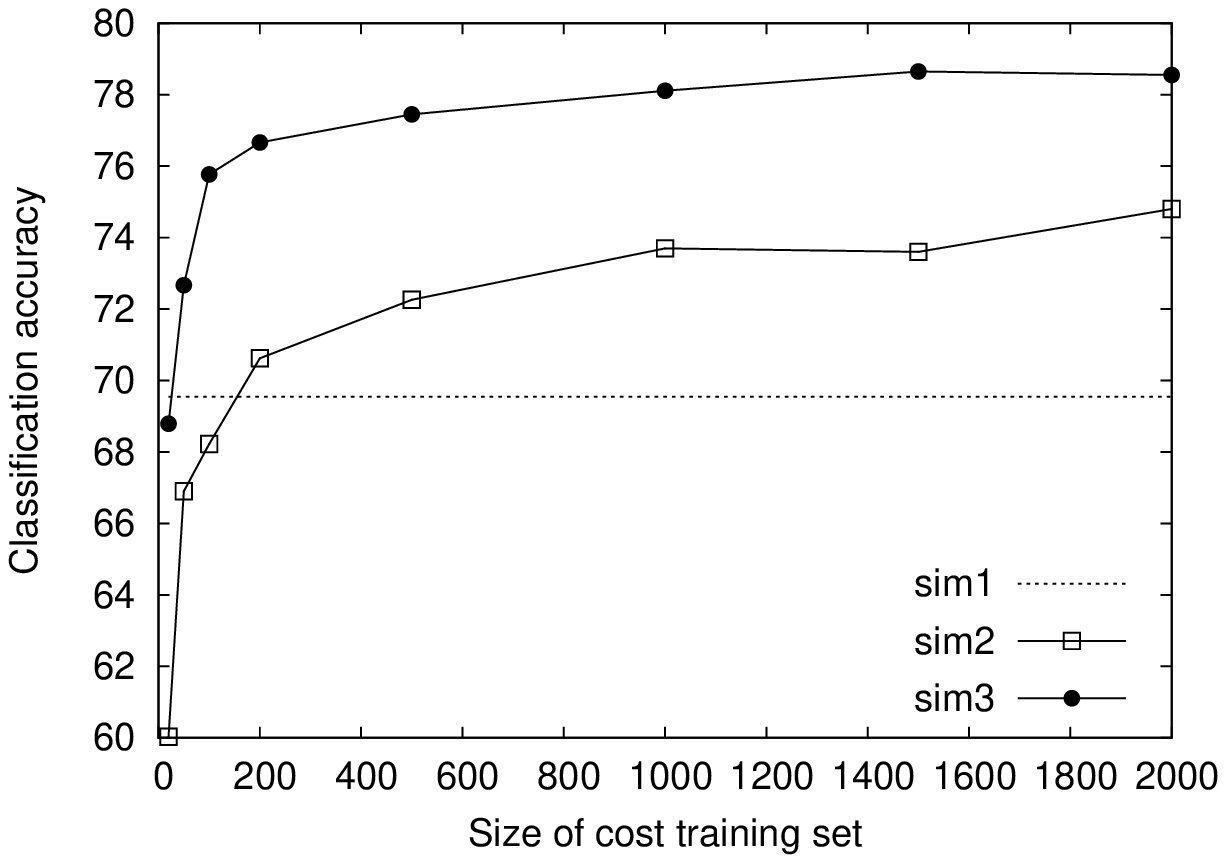}
\includegraphics[width=0.49\columnwidth]{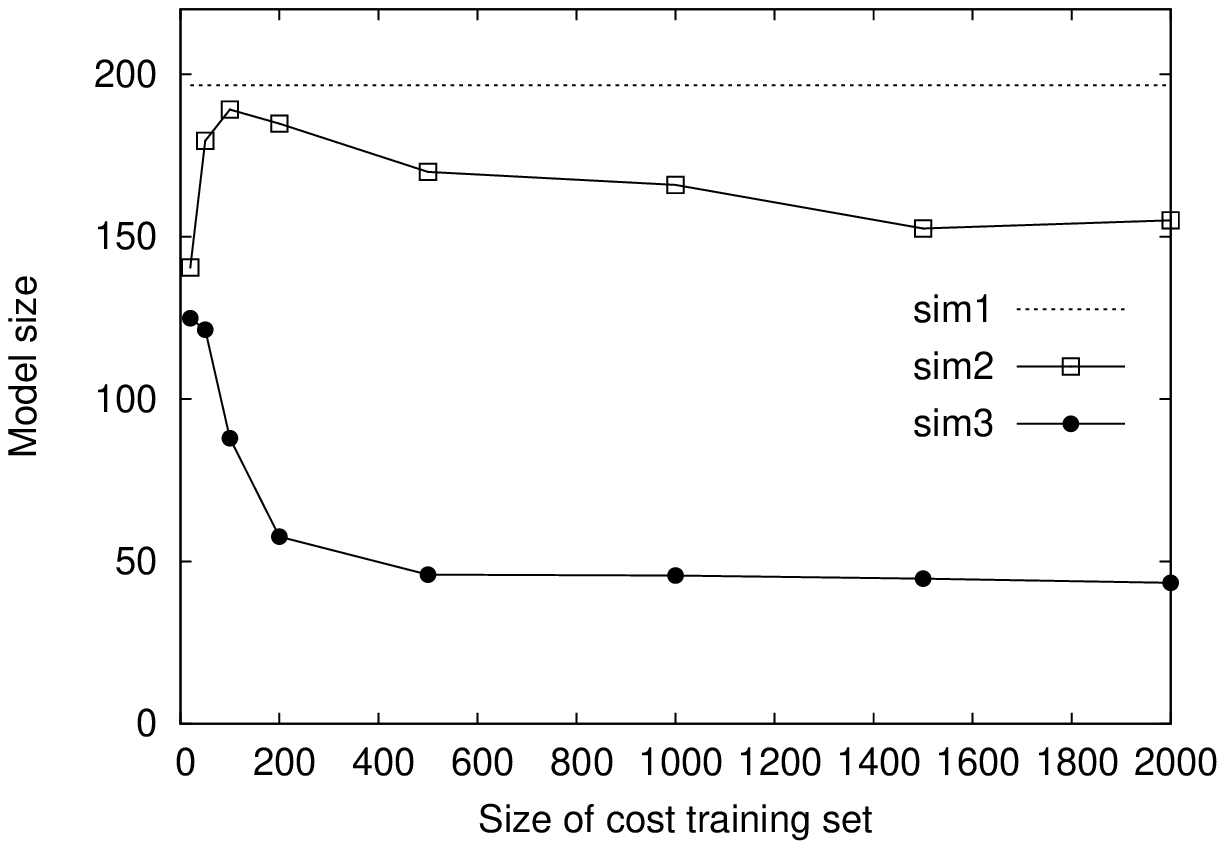}
\caption[Learning the edit costs: rate of convergence (Word dataset)]{Learning the edit costs: accuracy and sparsity results (Word dataset).}
\label{fig:figvarcost}
\end{center}
\end{figure}

\subsubsection{Pairing strategy and influence of $\alpha$}

In the previous experiment, the pairing strategy and the value of $\alpha$ was set by cross-validation. In this section, we compare the two pairing strategies (random pairing and Levenshtein pairing) presented in \sref{sec:mljmatching} as well as the influence of $\alpha$ (the proportion of landmarks associated with each training example).
\fref{fig:figalphawords} shows the accuracy and sparsity results obtained for $n_\mathcal{T}=1,500$ with respect to $\alpha$ and the pairing strategies.\footnote{We do not evaluate the pairing strategies on the whole data ($n_\mathcal{T}=2,000$) so that we can build 5 bootstrap samples and average the results over these.} The accuracy for $d_{lev}$ and $p_e$ is carried over from \fref{fig:figvarcost} for comparison (model sizes for $d_{lev}$ and $p_e$, which are not shown for scale reasons, are 197 and 152 respectively).

These results are very informative. Regardless of the pairing strategy, $K_\mathbf{C}$ outperforms $d_{lev}$ and $p_e$ even when making use of a very small proportion of the available pairs (1\%), which tremendously reduces the complexity of the similarity learning phase. Random pairing gives better results than Levenshtein pairing for $\alpha \leq 0.4$. When $\alpha \geq 0.6$, this trend is reversed. This means that for a small proportion of pairs, we learn better from pairing random landmarks than from pairing landmarks that are already good representatives of the training examples. On the other hand, when the proportion increases, Levenshtein pairing allows us to avoid pairing examples with the ``worst'' landmarks: best results are obtained with Levenshtein pairing and $\alpha=0.8$.

\begin{figure}[t]
\begin{center}
\psfrag{Value of alpha}[][][0.8]{\textbf{Value of} $\alpha$}
\psfrag{Classification accuracy}[][][0.8]{\textbf{Classification accuracy}}
\psfrag{Model size}[][][0.8]{\textbf{Model size}}
\psfrag{KG lev}[][][0.8]{\hspace{-2.95cm}$K_\mathbf{C}$ with Lev. pairing}
\psfrag{KG rand}[][][0.8]{\hspace{-2.90cm}$K_\mathbf{C}$ with rand. pairing}
\psfrag{lev}[][][0.8]{$d_{lev}$\hspace{0.3cm}}
\psfrag{proba}[][][0.8]{\hspace{0.30cm}$p_e$}
\includegraphics[width=0.49\columnwidth]{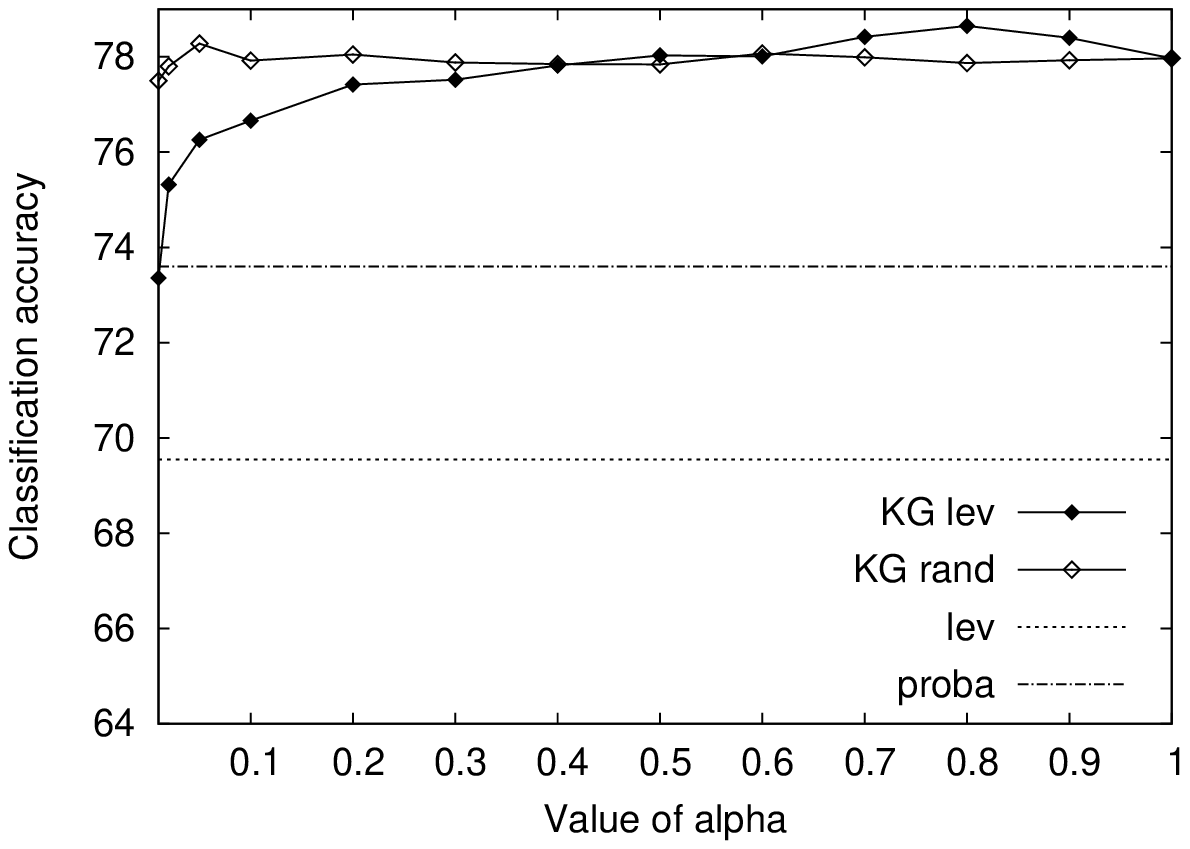}
\includegraphics[width=0.49\columnwidth]{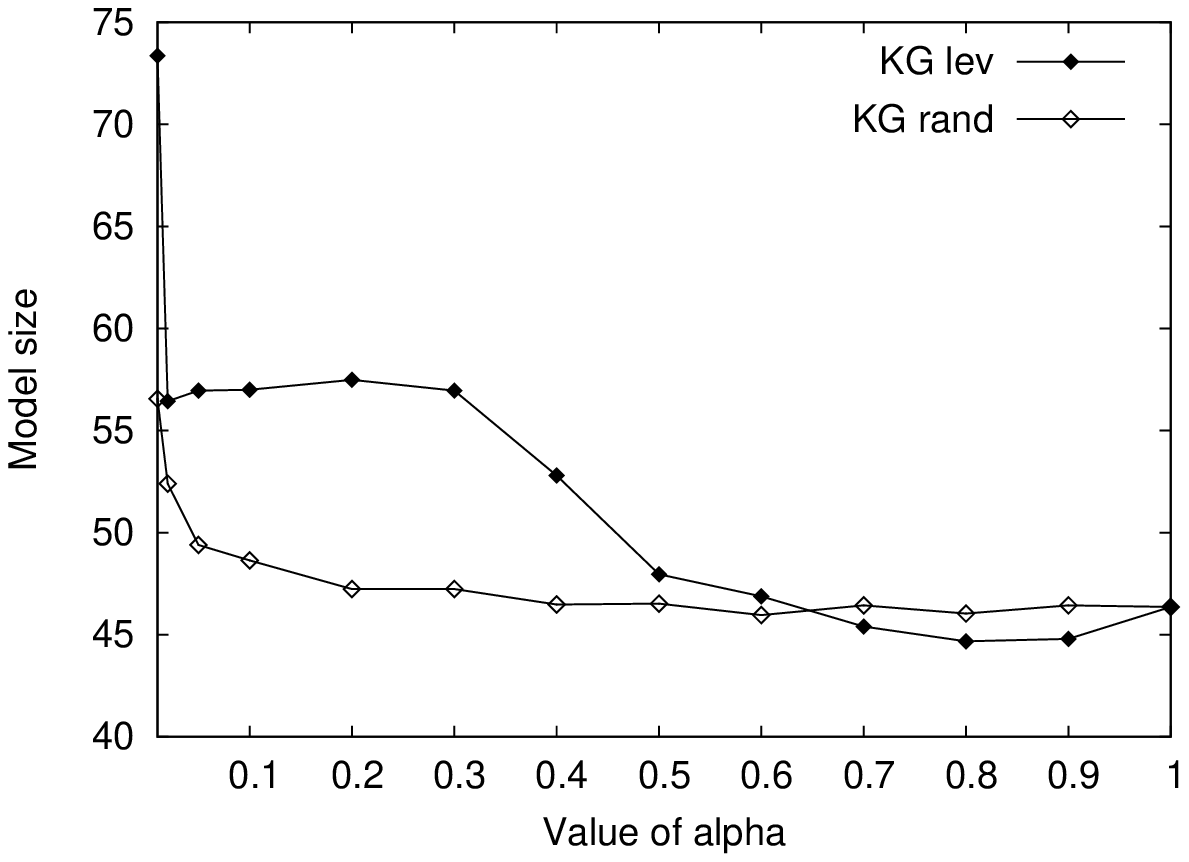}
\caption[Influence of the pairing strategies (Word dataset)]{Pairing strategies: accuracy and sparsity results w.r.t. $\alpha$ (Word dataset).}
\label{fig:figalphawords}
\end{center}
\end{figure}

\subsubsection{Learning the separator}

We now assess the performance of the three edit similarities with respect to the number of examples $n$ used to learn the separator $\boldsymbol{\alpha}$. For $K_\mathbf{C}$ and $p_e$, we use the edit cost matrix that performed best in \sref{sec:conv}. Taking our set of 2,000 words, we keep aside 400 examples to test the models and build bootstrap samples from the remaining 1,600 words to learn $\boldsymbol{\alpha}$.
\fref{fig:figvarsep} shows the accuracy and sparsity results of each method with respect to $n$, averaged over 5 runs.
Again, $K_\mathbf{C}$ outperforms $d_{lev}$ and $p_e$ for every size $n$ (the difference is statistically significant with a $p$-value $< 0.01$ using a Student's $t$-test) while always leading to (up to 5 times) sparser models.
Moreover, the size of the models induced by $K_\mathbf{C}$ stabilizes for $n \geq 400$ while the accuracy still increases. This is not the case for the models induced by $d_{lev}$ and $p_e$, whose size keeps growing. To sum up, the best similarity learned by GESL$_{HL}$ outperforms the best similarity learned with the method of \citet{Oncina2006}, which had been proven to outperform other state-of-the-art methods.

\begin{figure}[t]
\begin{center}
\psfrag{Size of separator training set}[][][0.8]{\textbf{Size of the separator training sample}}
\psfrag{Classification accuracy}[][][0.8]{\textbf{Classification accuracy}}
\psfrag{Model size}[][][0.8]{\textbf{Model size}}
\psfrag{sim1}[][][0.8]{$d_{lev}$}
\psfrag{sim2}[][][0.8]{\hspace{0.2cm}$p_e$}
\psfrag{sim3}[][][0.8]{$K_\mathbf{C}$}
\includegraphics[width=0.49\columnwidth]{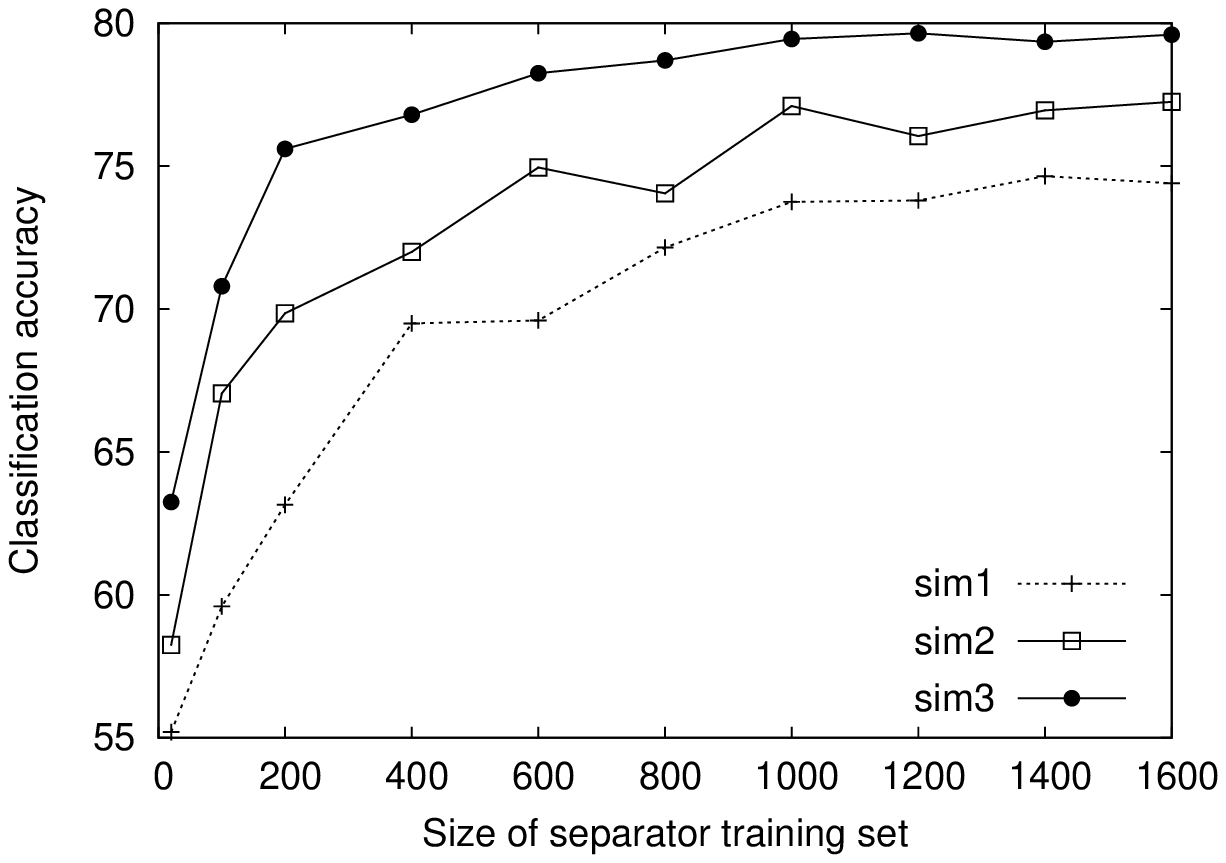}
\includegraphics[width=0.49\columnwidth]{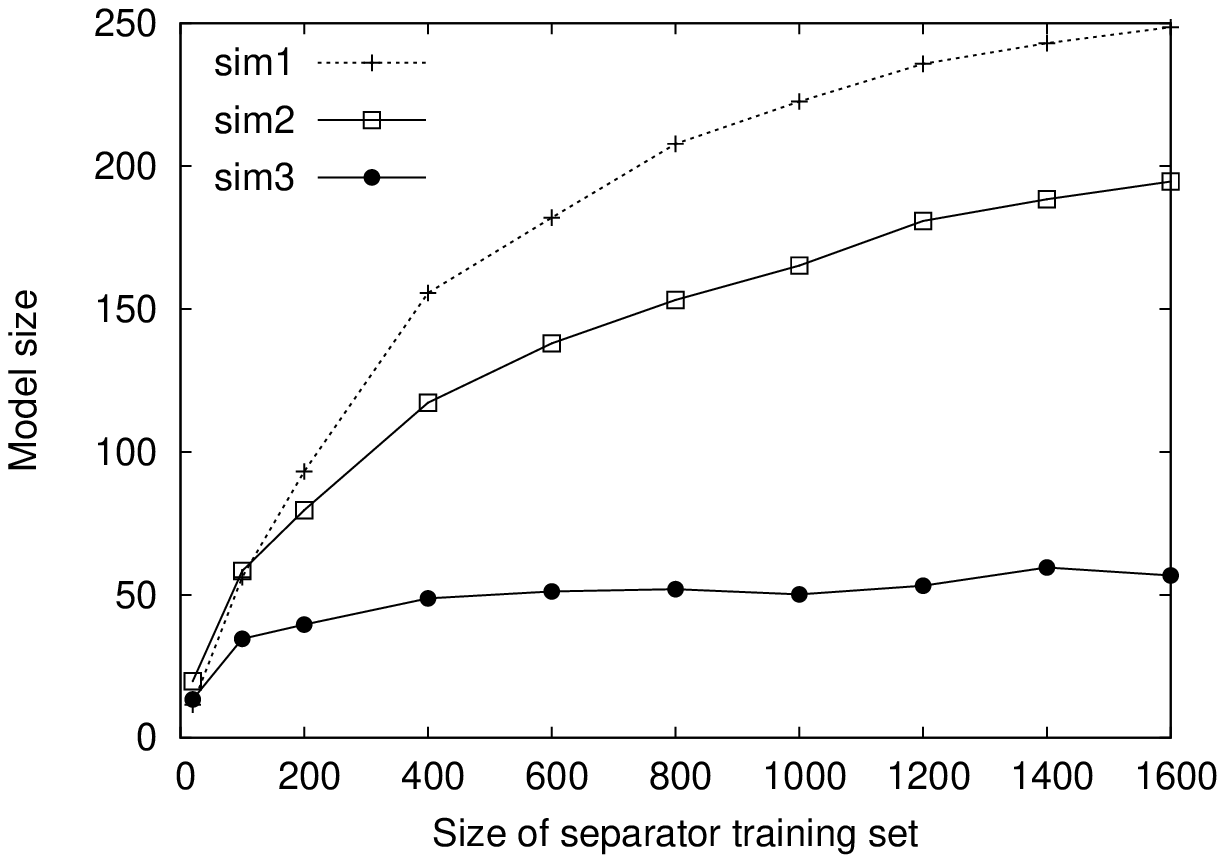}
\caption[Learning the separator: accuracy and sparsity results (Word dataset)]{Learning the separator: accuracy and sparsity results (Word dataset).}
\label{fig:figvarsep}
\end{center}
\end{figure}

\subsubsection{Reasonable points analysis}
\label{sec:reaswords}

Finally, one may wonder what kind of words are selected as reasonable points in
the models. The intuition is that they should be some sort of
``discriminative prototypes'' the classifier is based on.
To investigate this, using $K_\mathbf{C}$ and a training set of 1,200 examples, we learned a classifier $\boldsymbol{\alpha}$ with a high value of $\lambda$ to enforce a very sparse model, thus making the analysis easier. The set of 11 reasonable points automatically selected during the learning process is shown in \tref{tab:tabwords}. Our interpretation of why these particular words were chosen is that this small set actually carries a lot of discriminative patterns. \tref{tab:patterns} shows some of these patterns (extracted by hand from the reasonable points of \tref{tab:tabwords}) along with their number of occurrences in each class over the entire dataset. For
example, words ending with \textit{ly} correspond  to
English words, while those ending with \textit{que} characterize
French words.
Note that \tref{tab:tabwords} also reflects the fact that English words are shorter on average
(6.99) than French words (8.26) in the dataset, but the
English (resp. French) reasonable points are significantly shorter (resp. longer)
than the average (mean of 5.00 and 10.83 resp.), which allows
better discrimination.  
Note that we generated other sets of reasonable points from several training sets and observed the same patterns.

\begin{table}[t]
  \begin{center}
  \begin{small}
  \begin{tabular}{llllll}
    \toprule
    \multicolumn{3}{c}{\textbf{English}} & \multicolumn{3}{c}{\textbf{French}}\\
    \midrule
    \quad\texttt{high} &\texttt{showed} &\texttt{holy\quad} & \quad\texttt{economiques} & \texttt{americaines} & \texttt{decouverte\quad}\quad\\
    \quad\texttt{liked}& \texttt{hardly} & & \quad\texttt{britannique} & \texttt{informatique} & \texttt{couverture}\\
    \bottomrule
  \end{tabular}
  \end{small}
  \caption[Example of a set of reasonable points (Word dataset)]{Example of a set of 11 reasonable points (Word dataset).}
  \label{tab:tabwords}
  \end{center}
\end{table}

\begin{table}[t]
  \begin{center}
  \begin{scriptsize}
  \begin{tabular}{ccccccccccccc}
    \toprule
    \textbf{Patterns} & \verb+w+ & \verb+y+  & \verb+k+  & \verb+q+ & \verb+nn+ & \verb+gh+ & \verb+ai+ & \verb+ed$+ & \ \verb+ly$+ \ & \ \verb+es?$+ \ & \ \verb+ques?$+\ & \verb+^h+\\
    \midrule
    \textbf{English} & 146 & 144 & 83 & 14 & 5 & 34 & 39 & 151 & 51 & 265 & 0 & 62\\
    \textbf{French} & 7 & 19 & 5 & 72 & 35 & 0 & 114 & 51 & 0 & 630 & 43 & 14 \\
    \bottomrule
  \end{tabular}
  \end{scriptsize}
  \caption[Discriminative patterns extracted from the reasonable points of Table~\ref*{tab:tabwords}]{Some discriminative patterns extracted from the reasonable points of \tref{tab:tabwords} (\texttt{\^}: start of word, \texttt{\$}: end of word, \texttt{?}: 0 or 1 occurrence of preceding letter).}
  \label{tab:patterns}
  \end{center}
\end{table}

\subsection{Handwritten Digits}
\label{sec:digits}

We use the same NIST Special Database 3 as earlier in this manuscript. We have seen that classifying digits using a Freeman code representation and edit similarities yields close-to-perfect accuracy, even in the multi-class setting.
In order to make the comparison between the edit similarities (i-iii) easier, we evaluate them on the binary task of discriminating between even and odd digits.
This task is harder due to extreme within-class variability: each class is in fact a ``meta-class'' containing instances of 5 basic classes of digits. Therefore, every example is highly dissimilar to about 80\% of the examples of its own class (e.g., 1's are dissimilar to 5's and 0's are dissimilar to 4's, although they belong to the same class).

\subsubsection{Convergence rate}

Once again, we assess the convergence of the cost learning methods (i and iii).
We keep aside 2,000 words as a validation set to tune the parameters (using 5-fold cross-validation and selecting the value offering the best classification accuracy) as well as 2,000 words for testing the models. We build bootstrap samples $\mathcal{T}$ from the remaining 6,000 words to learn the edit costs (5 runs for each size $n_\mathcal{T}$), as well as 400 words to train the separator $\boldsymbol{\alpha}$.

\fref{fig:digitsres} shows the accuracy and sparsity results of each method with respect to $n_\mathcal{T}$, averaged over 5 runs.
First of all, we notice that the Levenshtein distance $d_{lev}$ performs nicely on this task (95.19\% with a model size of 70) and that $p_e$ is never able to match $d_{lev}$'s accuracy level (94.94\% at best with a model size of 78). In our opinion, this poor performance comes from the fact that $p_e$ does not take advantage of negative pairs. In a context of extreme within-class variability, moving closer examples of the same class without making sure that examples of different class are kept far from each others does not yield an appropriate similarity.
On the other hand, our method shows the same general behavior on this task as on the previous one. Indeed, convergence is fast despite the richness of the two classes (only 100 examples to match Levenshtein's accuracy and about 1,000 to reach convergence). Moreover, $K_\mathbf{C}$ achieves significantly better performance (95.63\% at best with a model size of 57) than both $d_{lev}$ ($p$-value $< 0.05$ for $n_\mathcal{T}\geq 250$ using a Student's $t$-test) and $p_e$ ($p$-value $<0.01$ for $n_\mathcal{T}>20$).

\begin{figure}[t]
\begin{center}
\psfrag{Size of cost training set}[][][0.8]{\textbf{Size of the cost training sample}}
\psfrag{Classification accuracy}[][][0.8]{\textbf{Classification accuracy}}
\psfrag{Model size}[][][0.8]{\textbf{Model size}}
\psfrag{sim1}[][][0.8]{$d_{lev}$}
\psfrag{sim2}[][][0.8]{\hspace{0.2cm}$p_e$}
\psfrag{sim3}[][][0.8]{$K_\mathbf{C}$}
\includegraphics[width=0.49\columnwidth]{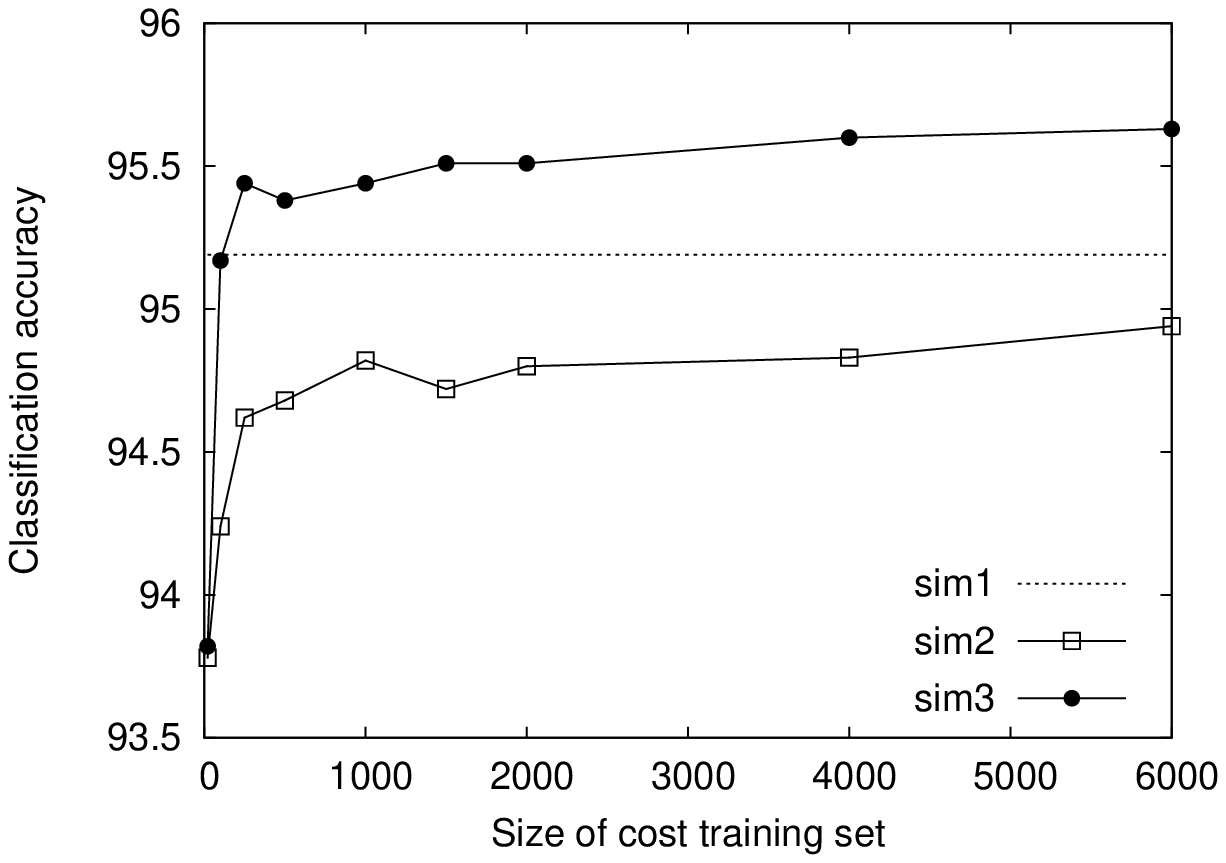}
\includegraphics[width=0.49\columnwidth]{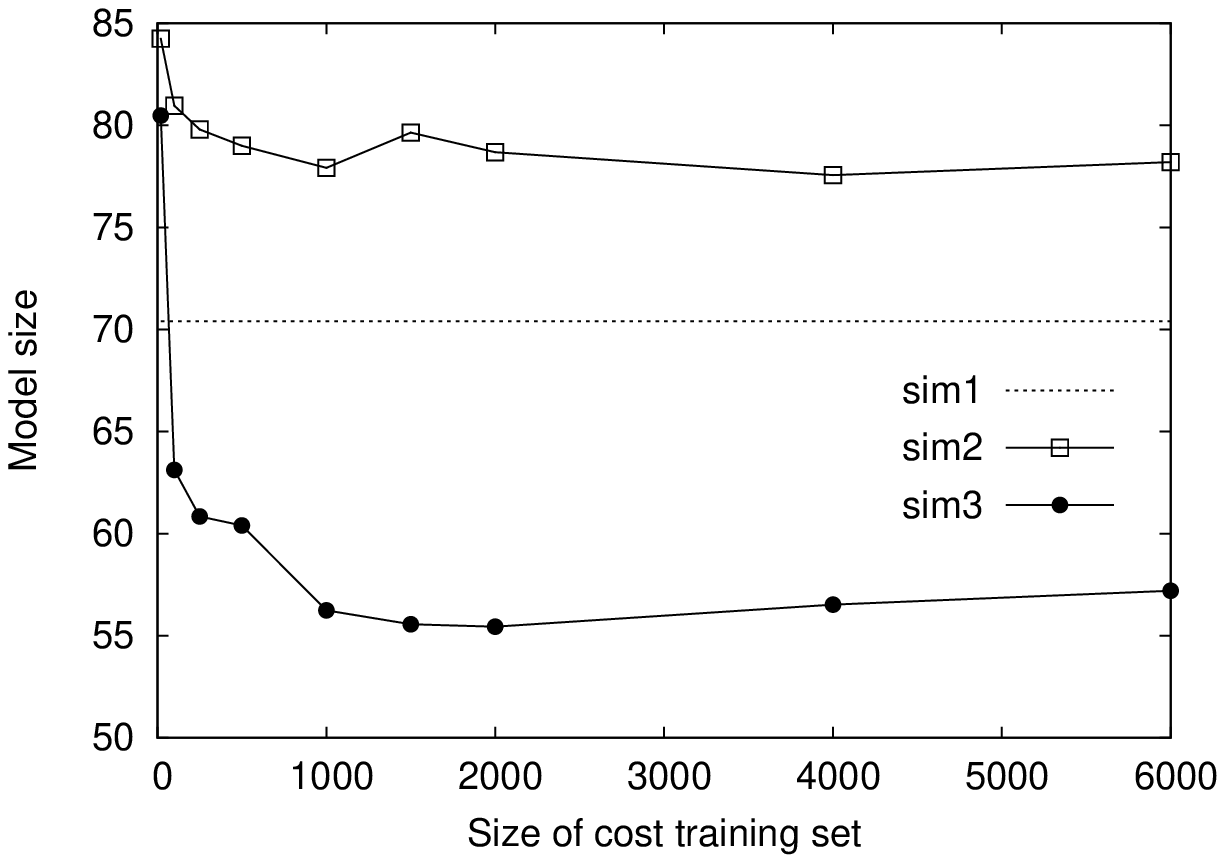}
\caption[Learning the edit costs: rate of convergence (Digit dataset)]{Learning the edit costs: accuracy and sparsity results (Digit dataset).}
\label{fig:digitsres}
\end{center}
\end{figure}

\subsubsection{Pairing strategy and influence of $\alpha$}

\fref{fig:figalphadigits} shows the accuracy and sparsity results obtained for $n_\mathcal{T}=2,000$ with respect to $\alpha$ and the pairing strategies. The performance for $d_{lev}$ and $p_e$ is carried over from \fref{fig:digitsres} for comparison.
Results are very different from those obtained on the previous dataset. Here, $K_\mathbf{C}$ with random pairing fails: it is always largely outperformed by both $d_{lev}$ and $p_e$. On the other hand, $K_\mathbf{C}$ with Levenshtein pairing performs better than every other approaches for $0.05 \leq \alpha \leq 0.4$.
This behavior can be explained by the meta-class structure of the dataset. When using random pairing, many training examples are paired with landmarks of the same class but yet very different (for instance, a 1 paired with a 5, or a 0 paired with a 4), and trying to ``move them closer'' is a fruitless effort. On the other hand, we have seen earlier that the Levenshtein distance is an appropriate measure to discriminate between handwritten digits. Therefore, when using Levenshtein pairing with $\alpha \leq 0.2$, the problematic situation explained above rarely occurs. When $\alpha > 0.2$, since each meta-class is made out of 5 basic classes in even proportions, more and more examples are paired with ``wrong'' landmarks and the performance drops dramatically.

This result yields a valuable conclusion: similarity learning should not always focus on optimizing over all possible pairs (although it is often the case in the literature), since it may lead to poor classification performance. In some situations, such as the presence of high within-class variability, it may be a better strategy to improve the similarity according to a few carefully selected pairs.

\begin{figure}[t]
\begin{center}
\psfrag{Value of alpha}[][][0.8]{\textbf{Value of} $\alpha$}
\psfrag{Classification accuracy}[][][0.8]{\textbf{Classification accuracy}}
\psfrag{Model size}[][][0.8]{\textbf{Model size}}
\psfrag{KG lev}[][][0.8]{\hspace{-2.95cm}$K_\mathbf{C}$ with Lev. pairing}
\psfrag{KG rand}[][][0.8]{\hspace{-2.90cm}$K_\mathbf{C}$ with rand. pairing}
\psfrag{lev}[][][0.8]{$d_{lev}$\hspace{0.3cm}}
\psfrag{proba}[][][0.8]{\hspace{0.30cm}$p_e$}
\includegraphics[width=0.49\columnwidth]{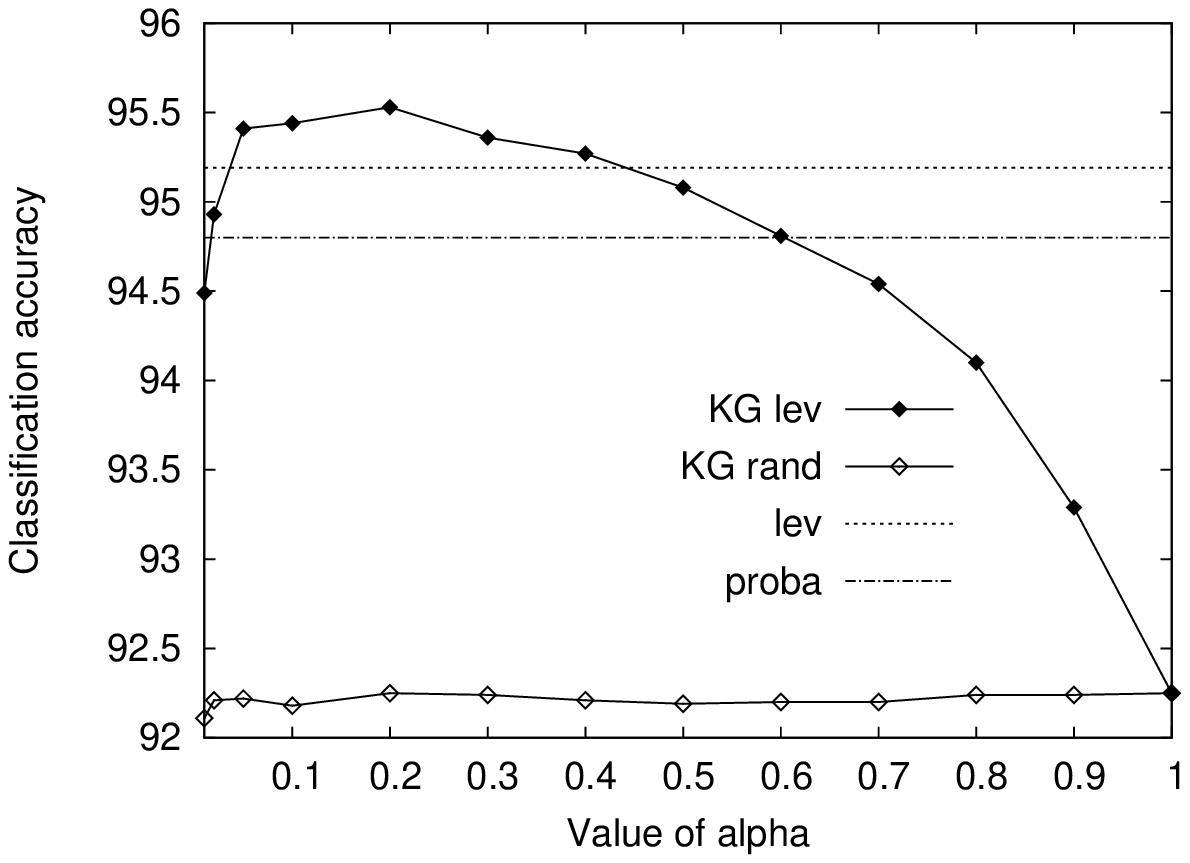}
\includegraphics[width=0.49\columnwidth]{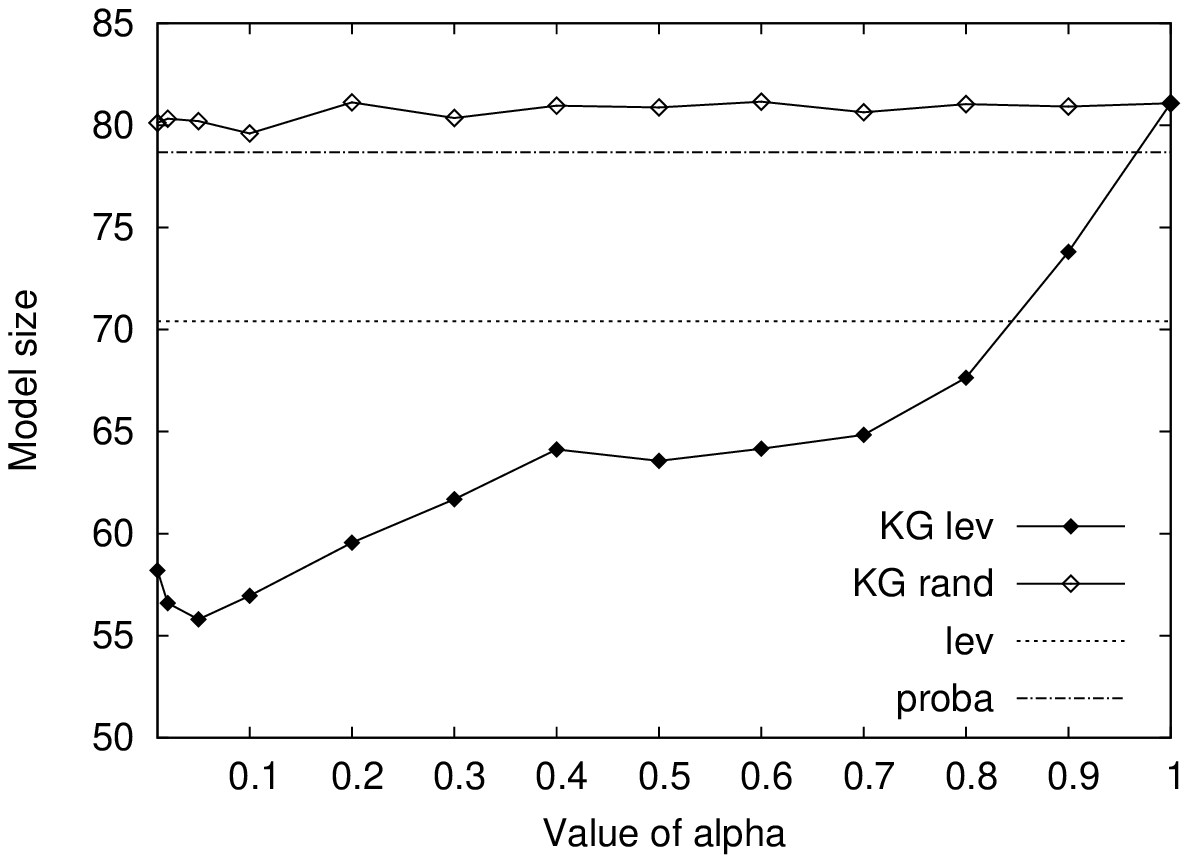}
\caption[Influence of the pairing strategies (Digit dataset)]{Pairing strategies: accuracy and sparsity results with respect to $\alpha$ (Digit dataset).}
\label{fig:figalphadigits}
\end{center}
\end{figure}

\subsubsection{Reasonable points analysis}

To provide an insight into the sort of digits that are selected as reasonable points, we follow the same procedure as in \sref{sec:reaswords} using a training set of 2,000 examples.
We end up with a set of 13 reasonable points. The corresponding digit contours are drawn in \fref{fig:reasdigit}, allowing a graphical interpretation of why these particular examples were chosen. Note that this set is representative of a general tendency: we experimented with several training sets and obtained similar results.
The most striking thing about this set is that 7's are over-represented (4 out of the 6 reasonable points of the odd class). This is explained by the fact that 7's (i) account for 1's and 9's (their contour is very similar), which also gives a reason for the absence of 1's and 9's in the set, and (ii) are not similar to any even digits. The same kind of reasoning applies to 6's (the lower part of 6's is shared by 0's and 8's, but not by any odd number) and 3's (lower part is the same as 5's). We can also notice the presence of 4's: they have a contour mostly made of straight lines, which is unique in the even class. There is also a 2 whose contour is somewhat similar to a 1. Lastly, another explanation for having several occurrences of the same digit may be to account for variations of size (the two 4's), shape or orientations (the three 6's).

\begin{figure}[t]
\begin{center}
\includegraphics[width=0.7\columnwidth]{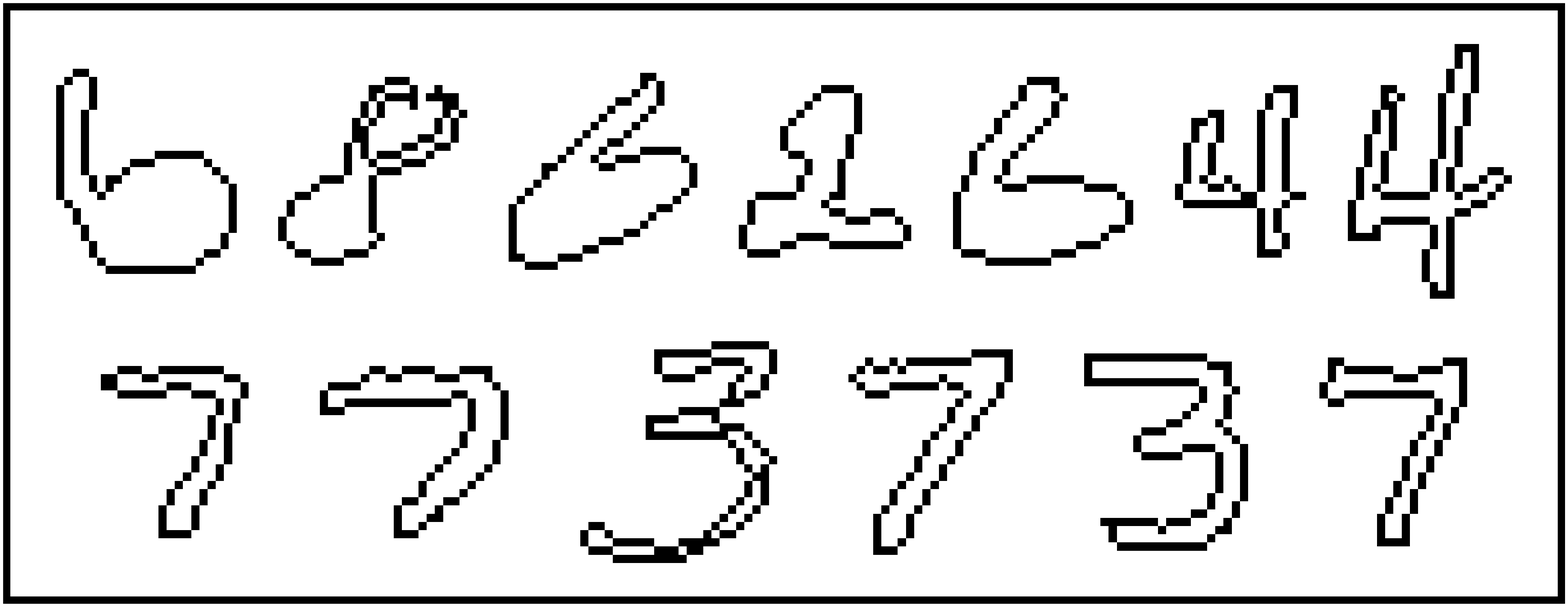}
\caption[Example of a set of reasonable points (Digit dataset)]{Example of a set of 13 reasonable points (Digit dataset).}
\label{fig:reasdigit}
\end{center}
\end{figure}

\section{Conclusion}
\label{sec:mljconclu}

In this chapter, we made use of the theory of $(\epsilon,\gamma,\tau)$-good similarity functions in the context of edit similarities. We first conducted a preliminary experimental study confirming that this framework is well-suited to edit similarities, leading to classification performance competitive with standard SVM but with a number of additional advantages, among which the absence of PSD constraint and the sparsity of the models.

We then went a step further and proposed a novel approach to the problem of learning edit similarities from data, called GESL, driven by the notion of $(\epsilon,\gamma,\tau)$-goodness. As opposed to most state-of-the-art approaches, GESL is not based on a costly iterative procedure but on solving an efficient convex program, and can accommodate both positive and negative training pairs. Furthermore, it is also a promising way to learn tree edit similarities, even though we did not perform any series of experiments in this case.
We provided a theoretical analysis of GESL, which holds for a large class of loss functions. A generalization bound in $O(\sqrt{1/n_\mathcal{T}})$ was derived using the notion of uniform stability. This bound is (i) related to the goodness of the resulting similarity, which gives guarantees that the similarity will induce accurate classifiers for the task at hand, and (ii) independent from the size of the alphabet, making GESL suitable for problems involving large vocabularies.
We conducted experiments on two string datasets that show that GESL has fast convergence and that the learned similarities perform very well in practice, inducing more accurate and sparser models than other (standard or learned) edit similarities. We also studied two pairing strategies and observed that Levenshtein pairing is more stable to high within-class variability, and that considering all possible pairs is not always a good approach. \tref{tab:geslsum} summarizes the main features of GESL using the same format as in the survey of \cref{chap:metriclearning} (\tref{tab:mlstructsum}).
\begin{table}[t]
\begin{center}
\begin{scriptsize}
\begin{tabular}{cccccccc}
\toprule
\textbf{Method} & \textbf{Data} & \textbf{Model} & \textbf{Scripts} & \textbf{Opt.} & \textbf{Global sol.} & \textbf{Neg. pairs} & \textbf{Gen.}\\
\midrule
GESL & Strings/Trees & --- & Optimal & CO & \tickYes & \tickYes & \tickYes\\
\bottomrule
\end{tabular}
\end{scriptsize}
\caption[Summary of the main features of GESL]{Summary of the main features of GESL (``Opt.'', ``Global sol.'', ``Neg. pairs'', ``Gen.'' and ``CO'' respectively stand for ``Optimization'', ``Global solution'', ``Negative pairs'', ``Generalization guarantees'' and ``Convex optimization'').}
\label{tab:geslsum}
\end{center}
\end{table}

An extension of this work would be to consider sparsity-inducing regularizers on the edit cost matrix. For instance, using an $L_1$ regularization would lead to more interpretable matrices: an edit cost set to zero during learning would suggest that the corresponding edit operation is not relevant to the task, which can be a valuable information in many real-world applications. This would however prevent the derivation of generalization guarantees using uniform stability, but the theoretical framework presented later in this thesis (\cref{chap:nips}) could be used instead.

Another interesting perspective is to assess the relevance of similarities learned with GESL when used in $k$-Nearest Neighbors classifiers. Indeed, when using Levenshtein pairing, GESL's objective is somewhat related to the $k$-NN prediction rule and to the objective of the metric learning method LMNN \citep{Weinberger2009}. This intuition is confirmed by preliminary results using a $1$-NN classifier (see \fref{fig:knn} and \tref{tab:knn}), where $K_\mathbf{C}$ outperforms $e_L$ and $p_e$ on both datasets. These first results open the door to a further theoretical analysis and might lead to $k$-NN generalization guarantees for GESL.

\begin{figure}[t]
\begin{center}
\psfrag{Size of training set}[][][0.8]{\textbf{Size of $k$-NN training sample}}
\psfrag{Classification accuracy}[][][0.8]{\textbf{Classification accuracy}}
\psfrag{sim1}[][][0.8]{$d_{lev}$}
\psfrag{sim2}[][][0.8]{\hspace{0.2cm}$p_e$}
\psfrag{sim3}[][][0.8]{$K_\mathbf{C}$}
\includegraphics[width=0.6\columnwidth]{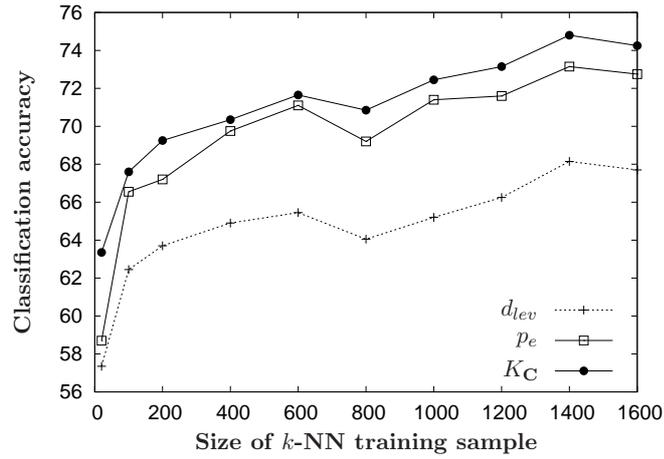}
\caption[1-Nearest Neighbor results (Word dataset)]{1-Nearest Neighbor accuracy results (Word dataset).}
\label{fig:knn}
\end{center}
\end{figure}

\begin{table}[t]
\begin{center}
\begin{small}
\begin{tabular}{cc}
  \toprule
  \textbf{Similarity} & \textbf{Accuracy}\\
  \midrule
  $e_L$ & 97.08\%\\
  $p_e$ & 96.44\%\\
  $K_\mathbf{C}$ & \textbf{97.50\%}\\
  \bottomrule
\end{tabular}
\end{small}
\caption[1-Nearest Neighbor results on the Digit dataset]{1-Nearest Neighbor accuracy results on the Digit dataset.}
\label{tab:knn}
\end{center}
\end{table}


After having dealt with structured data, in the next part of this thesis we will focus on data consisting of feature vectors. While in the context of strings or trees, we could only optimize a pair-based objective that is a loose bound on the empirical $(\epsilon,\gamma,\tau)$-goodness (see \eref{eq:newgoodsim}) due to the form of the edit similarity, we will see in the next chapter that using a simple bilinear similarity allows us to optimize the actual $(\epsilon,\gamma,\tau)$-goodness, relying on global constraints instead of pairs.

\part{Contributions in Metric Learning from Feature Vectors}\label{part:vect}
\chapter{Learning Good Bilinear Similarities from Global Constraints}
\label{chap:icml}

\begin{chapabstract}
In this chapter, we build upon GESL (proposed in \cref{chap:ecml}) to learn good similarities between feature vectors. We focus on the bilinear similarity, which is not PSD-constrained. Thanks to this simple form of similarity, we are able to efficiently optimize its empirical $(\epsilon,\gamma,\tau)$-goodness (instead of an upper bound as done in \cref{chap:ecml} for structured data) in a nonlinear feature space by formulating the approach as a convex minimization problem. Unlike other metric learning methods, this results in the similarity being optimized with respect to global constraints instead of local pairs or triplets. Then, relying on uniform stability arguments similar to those used in the previous chapter, we derive generalization guarantees directly in terms of the goodness in generalization of the learned similarity. As compared to GESL, our method minimizes a tighter bound on the true risk of the linear classifier built from the similarity. Experiments performed on various datasets confirm the effectiveness of our approach compared to state-of-the-art methods and provide evidence that (i) it is fast, (ii) robust to overfitting and (iii) produces very sparse classifiers.\\

The material of this chapter is based on the following international publication:\\

\bibentry{Bellet2012a}.
\end{chapabstract}

\section{Introduction}

In the previous chapter, we used a relaxation of the notion of $(\epsilon,\gamma,\tau)$-goodness to propose a pair-based edit similarity learning method and showed that, in this context, we could establish the consistency of the learned metric with respect to unseen pairs of examples, and a relation to the goodness in generalization of the metric.
In this chapter, we focus on metric learning from feature vectors and aim at optimizing the \emph{exact} criterion of $(\epsilon,\gamma,\tau)$-goodness. Thanks to the simple form of the bilinear similarity, we are able to do this in an efficient way, leading to a similarity optimized with respect to global constraints (rather than local pairs) and used to build a global linear classifier.
Our approach, called SLLC (Similarity Learning for Linear Classification), has several advantages: (i) it is tailored to linear classifiers, (ii) theoretically well-founded, (iii) does not require positive semi-definiteness, and (iv) is in a sense less restrictive than pair or triplet-based settings. We formulate the problem of learning a good similarity function as a convex minimization problem that can be efficiently solved in a batch or online way. Furthermore, by using the Kernel Principal Component Analysis (KPCA) trick \citep{Chatpatanasiri2010}, we are able to kernelize our algorithm and thereby learn more powerful similarity functions and classifiers in the nonlinear feature space induced by a kernel. From the theoretical standpoint, we show that our approach has uniform stability, which leads to generalization guarantees directly in terms of the $(\epsilon,\gamma,\tau)$-goodness in generalization of the learned similarity. In other words, our approach minimizes an upper bound on the true risk of the linear classifier built from the similarity, and this bound is tighter than that obtained for GESL in \cref{chap:ecml}. Lastly, we provide an experimental study on seven datasets of various domains and compare SLLC with two widely-used metric learning approaches: LMNN \citep{Weinberger2009} and ITML \citep{Davis2007}.
This study demonstrates the practical effectiveness of our method and shows that it is fast, robust to overfitting and induces very sparse classifiers, making it suitable for dealing with high-dimensional data.

The rest of the chapter is organized as follows. \sref{sec:icmlsimlearning} presents our approach, SLLC, and the KPCA trick used to kernelize it. In \sref{sec:icmltheo}, we provide a theoretical analysis of SLLC, leading to the derivation of generalization guarantees both in terms of the consistency of the learned similarity and the error of the linear classifier. Finally, \sref{sec:icmlexpes} features an experimental study on various datasets and we conclude in \sref{sec:icmlconclu}.

\section{Learning $(\epsilon,\gamma,\tau)$-Good Bilinear Similarity Functions}
\label{sec:icmlsimlearning}

We consider the bilinear similarity $K_\mathbf{M}$ defined by
$$K_\mathbf{M}(\mathbf{x},\mathbf{x'}) = \mathbf{x}^T\mathbf{M}\mathbf{x'}.$$
In order to satisfy $K_\mathbf{M} \in [-1,1]$, we assume that inputs are normalized such that $||\mathbf{x}||_2 \leq 1$, and we require $||\mathbf{M}||_{\mathcal{F}} \leq 1$.

\subsection{Similarity Learning Formulation}

Our goal is to directly optimize the empirical $(\epsilon,\gamma,\tau)$-goodness of $K_\mathbf{M}$. To this end, we are given a training sample of $n_\mathcal{T}$ labeled points $\mathcal{T} = \{z_i=(\mathbf{x_i},y_i)\}_{i=1}^{n_\mathcal{T}}$ and a sample of $n_\mathcal{R}$ labeled reasonable points $\mathcal{R}=\{z_k=(\mathbf{x_k},y_k)\}_{k=1}^{n_\mathcal{R}}$. In practice, $\mathcal{R}$ is a subset of $\mathcal{T}$ with $n_\mathcal{R}=\hat{\tau}n_\mathcal{T}$ ($\hat{\tau} \in~]0,1]$). In the lack of background knowledge, it can be drawn randomly or according to some criterion, e.g., diversity \citep{Kar2011}.

Based on the definition of $(\epsilon,\gamma,\tau)$-goodness in hinge loss (\defref{def:defgoodsim2}), given $\mathcal{R}$ and a margin $\gamma$, we want to optimize the amount of margin violation $\epsilon$ on the training sample (the empirical goodness). Thus, let
$$\ell(\mathbf{M},z_i,\mathcal{R})= {[} 1 - y_i\frac{1}{\gamma n_\mathcal{R}}\sum_{k=1}^{n_\mathcal{R}}y_kK_\mathbf{M}(\mathbf{x_i},\mathbf{x_k}) {]}_{+}$$
denote the empirical goodness of $K_\mathbf{M}$ with respect to a single training point $z_i$. The empirical goodness over the sample $\mathcal{T}$ is denoted by
$$\epsilon_\mathcal{T} = \frac{1}{n_\mathcal{T}}\sum_{i=1}^{n_\mathcal{T}}\ell(\mathbf{M},z_i,\mathcal{R}).$$

We want to learn the matrix $\mathbf{M}$ that minimizes $\epsilon_\mathcal{T}$. This can be done by solving the following regularized problem, referred to as SLLC (Similarity Learning for Linear Classification):
\begin{eqnarray}
\label{eq:sllc1}
\displaystyle\min_{\mathbf{M}\in \mathbb{R}^{d\times d}} & \epsilon_\mathcal{T}\quad+\quad\beta\|\mathbf{M}\|_{\mathcal{F}}^2
\end{eqnarray}
where $\beta$ is a regularization parameter.

Note that SLLC can be cast as a convex quadratic program (QP) by rewriting the sum of $n_\mathcal{T}$ hinge losses in the objective function as $n_\mathcal{T}$ margin constraints and introducing slack variables $\boldsymbol{\xi}\in\mathbb{R}_+^{n_\mathcal{T}}$ in the objective:
\begin{equation}
\label{eq:sllcqp}
\begin{aligned}
 \displaystyle\min_{\mathbf{M}\in \mathbb{R}^{d\times d},\boldsymbol{\xi}\in\mathbb{R}_+^{n_\mathcal{T}}} &&& \frac{1}{n_\mathcal{T}}\sum_{i=1}^{n_\mathcal{T}}\xi_i\quad+\quad\beta\|\mathbf{M}\|_{\mathcal{F}}^2\\
 \text{s.t.} &&& 1 - y_i\frac{1}{\gamma n_\mathcal{R}}\sum_{k=1}^{n_\mathcal{R}}y_kK_\mathbf{M}(\mathbf{x_i},\mathbf{x_k}) \leq \xi_i, && 1 \leq i \leq n_\mathcal{T}.
\end{aligned}
\end{equation}

SLLC is radically different from classic metric and similarity learning algorithms presented in \cref{chap:metriclearning}, which are based on pair or triplet-based constraints.
It learns a global similarity rather than a local one, since $\mathcal{R}$ is the same for each training example. 
Moreover, the constraints are easier to satisfy since they are defined over an average of similarity scores to the points in $\mathcal{R}$ instead of over a single pair or triplet. This means that one can fulfill a constraint without satisfying the margin for each point in $\mathcal{R}$ individually (unlike what we did with GESL in \cref{chap:ecml}).  
SLLC also has a number of desirable properties:
\begin{enumerate}
\item No costly semi-definite programming is required, as opposed to many Mahalanobis distance learning methods. In its convex QP form \eqref{eq:sllcqp}, SLLC can be solved efficiently using standard convex minimization solvers. Moreover, it has only one constraint per training example (instead of one for each pair or triplet), i.e., a total of only $n_\mathcal{T}$ constraints and $n_\mathcal{T}+d^2$ variables. In its unconstrained form \eqref{eq:sllc1}, it is convex but not differentiable everywhere due to the hinge function in the loss. It can be solved in a stochastic or online setting using composite objective mirror descent \citep{Duchi2010a} or dual averaging methods \citep{Xiao2010} and thereby scales to very large problems.
\item The size of $\mathcal{R}$ does not affect the complexity of Problem~\ref{eq:sllcqp}, since each constraint is simply a linear combination of entries of $\mathbf{M}$.
\item If $\mathbf{x_i}$ is sparse, then the associated constraint is sparse as well: some variables of the problem (corresponding to entries of $\mathbf{M}$) have a zero coefficient in the constraint. This makes the problem easier to solve when data have a sparse representation.
\end{enumerate}

We now explain how SLLC can be kernelized to deal with nonlinear problems.

\subsection{Kernelization of SLLC}
\label{sec:nonlinear}

The framework presented in the previous section is theoretically well-founded with respect to Balcan et al.'s theory and has some generalization guarantees, as we will see in the next section. Moreover, it has the advantage of being very simple: we learn a global linear similarity and use it to build a global linear classifier. In order to learn more powerful similarities (and therefore classifiers), we propose to kernelize the approach by learning them in the nonlinear feature space induced by a kernel.

As discussed in \sref{sec:nonlinearml}, kernelizing a particular metric learning algorithm is difficult in general and may lead to intractable problems unless dimensionality reduction is applied.
For these reasons, we instead use the KPCA trick, recently proposed by Chatpatanasiri et al. \citeyearpar{Chatpatanasiri2010}. It provides a straightforward way to kernelize a metric learning algorithm while performing dimensionality reduction at no additional cost, and is based on Kernel Principal Component Analysis \citep{Scholkopf1998}, a nonlinear extension of PCA \citep{Pearson1901}.

PCA provides a way of representing the data by a small number $k$ of linearly uncorrelated variables (called the principal components) that account for most of the variance in the data. Assuming zero-centered data, let $\mathbf{C}$ denote the data covariance matrix:
$$\mathbf{C} = \frac{1}{n}\sum_{i=1}^n\mathbf{x_i}\mathbf{x_j}^T.$$
The new representation $\mathbf{x'_i}\in\mathbb{R}^k$ ($k\leq d$) of a data point $\mathbf{x_i}\in\mathbb{R}^d$ is given by $\mathbf{x'_i} = \mathbf{x_i}^T\mathbf{V}$, where $\mathbf{V}$ is a matrix whose columns are the top $k$ eigenvectors of $\mathbf{C}$.

The basic idea of KPCA is to use a kernel function to implicitly perform PCA in the (possibly infinite-dimensional) nonlinear feature space induced by the kernel, in the spirit of what is done in SVM. Let $K$ be a kernel such that $K(x,y) = \innerp{\phi(x),\phi(x')}$. The data covariance matrix in the new feature space is given by
$$\mathbf{C} = \frac{1}{n}\sum_{i=1}^n\phi(x_i)\phi(x_j)^T.$$
It can be shown that the projection of a point $\phi(x_i)$ onto the $j^{th}$ principal component only depends on inner products and therefore can be computed implicitly through the kernel function. The solution can actually be obtained through an eigendecomposition of the kernel matrix $\mathbf{K}$ whose entries are defined as $K_{i,j} = K(x_i,x_j)$.

Therefore, KPCA allows us to project the data into a new feature space of dimension $k\leq n$. The (unchanged) metric learning algorithm can then be used to learn a metric in that nonlinear space. Chatpatanasiri et al. \citeyearpar{Chatpatanasiri2010} showed that the KPCA trick is theoretically sound for unconstrained metric learning algorithms (they proved representer theorems), which includes SLLC. Throughout the rest of this chapter, we will only consider the kernelized version of SLLC.

Generally speaking, kernelizing a metric learning algorithm may cause or increase overfitting, especially when data are scarce and/or high-dimensional. However, since our framework is entirely linear and global, we expect our method to be quite robust to this undesirable effect. This will be doubly confirmed in the rest of this chapter: experimentally in \sref{sec:icmlexpes}, but also theoretically with the derivation in the following section of generalization guarantees independent from the size of the projection space.

\section{Theoretical Analysis}
\label{sec:icmltheo}

In this section, we present a theoretical analysis of our approach. 
Our main result is the derivation of a generalization bound (\thref{thm:icmlguarantee}) guaranteeing the consistency of SLLC and thus the $(\epsilon,\gamma,\tau)$-goodness in generalization for the considered task.

\subsection{Notations}

For convenience, given a bilinear model $K_\mathbf{M}$, we denote by $\mathbf{M}_\mathcal{R}$ both the similarity defined by the matrix $\mathbf{M}$
and its associated set of reasonable points $\mathcal{R}$  (when it is clear from the context we may omit the subscript $\mathcal{R}$). 
Given a similarity $\mathbf{M}_\mathcal{R}$, $\ell(\mathbf{M}_\mathcal{R},z,\mathcal{R})$ is the loss function over one example $z$. The empirical risk of $\mathbf{M}_\mathcal{R}$ over the sample $\mathcal{T}$ is thus given by
$$R^\ell_\mathcal{T}(\mathbf{M}_\mathcal{R}) = \epsilon_\mathcal{T}(\mathbf{M}_\mathcal{R})=\frac{1}{n_\mathcal{T}} \sum_{i=1}^{n_\mathcal{T}} \ell(\mathbf{M}_\mathcal{R},z_i,\mathcal{R})$$
and corresponds to the empirical goodness, while the true risk is given by
$$R^\ell(\mathbf{M}_\mathcal{R}) = \epsilon(\mathbf{M}_\mathcal{R})= \mathbb{E}_{z\sim P}[\ell(\mathbf{M}_\mathcal{R},z,\mathcal{R})]$$
and corresponds to the ``true'' goodness (or goodness in generalization). In the following, we will rather use $\epsilon_\mathcal{T}(\mathbf{M}_\mathcal{R})$ and $\epsilon(\mathbf{M}_\mathcal{R})$ to denote respectively the empirical and true risks to highlight the equivalence between risk and $(\epsilon,\gamma,\tau)$-goodness in SLLC. When it is clear from the context, we may simply use $\epsilon_\mathcal{T}$ and $\epsilon$.

The similarity is optimized according to a fixed set $\mathcal{R}$ of reasonable points coming from the training sample. Therefore, these reasonable points may not follow the distribution from which the training sample has been generated. Once again, the framework of uniform stability allows us to cope with this situation. Note that the empirical and true risks are defined with respect to a single example and not with respect to pairs. Therefore, we use the standard uniform stability setting (presented in \sref{sec:stability}) instead of the adaptation to the pair-based case introduced by \citet{Jin2009} and  used in \cref{chap:ecml}. 

\subsection{Generalization Bound}

In our case, to prove the uniform stability property we need to show that
\begin{equation}
\label{eq:icmlstab}
\forall \mathcal{T},\forall i, \sup_{z}|\ell(\mathbf{M},z,\mathcal{R})-\ell(\mathbf{M}^i,z,\mathcal{R}^i)|\leq\frac{\kappa}{n_\mathcal{T}},
\end{equation}
where $\mathbf{M}$ is learned from $\mathcal{T}$ and $\mathcal{R}\subseteq \mathcal{T}$, $\mathbf{M}^i$ is the matrix learned from $\mathcal{T}^i$ and $\mathcal{R}^i\subseteq \mathcal{T}^i$ is the set of reasonable points associated to $\mathcal{T}^i$. $\mathcal{T}^i$ is obtained from $\mathcal{T}$ by replacing the $i^{th}$ example $z_i\in \mathcal{T}$ by another example $z'_i$ independent from $\mathcal{T}$ and drawn from $P$.
Note that $\mathcal{R}$ and $\mathcal{R}^i$ are of equal size and can differ in at most one example, depending on whether $z_i$ or $z'_i$ belong to their corresponding set of reasonable points.
For the sake of simplicity, we assume that $\ell$ is bounded by 1.\footnote{Since we assume $\|\mathbf{x}\|_2\leq 1$ and $\|\mathbf{M}\|_{\mathcal{F}}\leq 1$, this can be obtained by dividing $\ell$ by the constant $1+\frac{1}{\gamma}$.}
To show \eqref{eq:icmlstab}, we need the following results.
\begin{lemma}\label{lem:res1}
For any labeled examples $z=(\mathbf{x},y)$, $z'=(\mathbf{x}',y')$ and any models $\mathbf{M}_\mathcal{R}$,
$\mathbf{M}'_{\mathcal{R}'}$, the following properties hold:
\begin{itemize}
\item[P1:] $|K_\mathbf{M}(\mathbf{x},\mathbf{x'})|\leq 1$,  
\item[P2:] $|K_\mathbf{M}(\mathbf{x},\mathbf{x'})-K_\mathbf{M'}(\mathbf{x},\mathbf{x'})|\leq
\|\mathbf{M}-\mathbf{M}'\|_{\mathcal{F}}$,
\item[P3:] 1-admissibility property of $\ell$:
$$|\ell(\mathbf{M},z,{\mathcal{R}})-\ell(\mathbf{M}',z,{\mathcal{R}}')|\leq 1|\frac{\sum_{k=1}^{n_\mathcal{R}}y_{k}
    K_{\mathbf{M}}(\mathbf{x},\mathbf{x_k})}{\gamma n_\mathcal{R}}-\frac{\sum_{j=1}^{n_{\mathcal{R}'}}
    y_{k}'K_{\mathbf{M}'}(\mathbf{x},\mathbf{x_k'})}{\gamma n_{\mathcal{R}'}}|.$$
\end{itemize}
\end{lemma}
\begin{proof}
$P1$ comes from 
$|K_\mathbf{M}(\mathbf{x},\mathbf{x'})|\leq
\|\mathbf{x}\|_2\|\mathbf{M}\|_{\mathcal{F}}\|\mathbf{x}'\|_2$,
the  normalization on  examples ($\|\mathbf{x}\|_2\leq 1$) and the  requirement on matrices  ($\|\mathbf{M}\|_{\mathcal{F}}\leq 1$). 

For $P2$, we observe that 
$$|K_\mathbf{M}(\mathbf{x},\mathbf{x'})-K_\mathbf{M'}(\mathbf{x},\mathbf{x'})|=|K_{\mathbf{M}-\mathbf{M}'}(\mathbf{x},\mathbf{x}')|,$$
and we use the normalization $\|\mathbf{x}\|_2\leq 1$.

$P3$ follows directly from $|y|=1$ and the 1-lipschitz property of the hinge loss:
$$|[U]_+-[V]_+|\leq |U-V|.$$
\end{proof}

Let $F_\mathcal{T}=\epsilon_\mathcal{T}(\mathbf{M})+\beta\|\mathbf{M}\|_{\mathcal{F}}^2$ be the objective function of SLLC with respect to a sample $\mathcal{T}$ and a
 set of reasonable points $R\subseteq T$.  The following lemma bounds the deviation between $\mathbf{M}$ and $\mathbf{M}^i$.
\begin{lemma}\label{lem:diffA}
For any models $\mathbf{M}$ and $\mathbf{M^i}$ that are minimizers of
$F_\mathcal{T}$ and $F_{\mathcal{T}^i}$ respectively, we have:
$$\|\mathbf{M}-\mathbf{M^i}\|_{\mathcal{F}}\leq \frac{1}{\beta n_\mathcal{T}\gamma}.$$
\end{lemma}
\begin{proof}
We follow closely the proof of Lemma~20
of \cite{Bousquet2002} and omit some details for the sake of readability (similar ideas are used in the first part of the more detailed proof of \lref{lem:convexN2}). Let $\Delta
\mathbf{M}=\mathbf{M^i}-\mathbf{M}$, $0\leq t\leq 1$ and
\begin{eqnarray*}
M_1 & = & \|\mathbf{M}\|_{\mathcal{F}}^2-\|\mathbf{M}+t\Delta
\mathbf{M}\|_{\mathcal{F}}^2+\|\mathbf{M}^i\|_{\mathcal{F}}^2-\|\mathbf{M}^i-t\Delta
\mathbf{M}\|_{\mathcal{F}}^2\\
M_2 & = &\frac{1}{\beta n_\mathcal{T}}(\epsilon_\mathcal{T}(\mathbf{M}_\mathcal{R})-\epsilon_\mathcal{T}((\mathbf{M}+t\Delta
\mathbf{M})_\mathcal{R}) + \epsilon_{\mathcal{T}^{i}}((\mathbf{M}+t\Delta
\mathbf{M})_\mathcal{R})-\epsilon_{\mathcal{T}^i}(\mathbf{M}_\mathcal{R})).
\end{eqnarray*}  
Using the fact that $F_\mathcal{T}$ and $F_{\mathcal{T}^{i}}$ are convex functions, that $\mathbf{M}$ and $\mathbf{M}^i$ are their respective minimizers and
property P3,  we have $M_1 \leq M_2$. 
Fixing $t=1/2$, we obtain $M_1=\|\mathbf{M}-\mathbf{M}^i\|_{\mathcal{F}}^2$, and using  property $P3$
 and the normalization $\|\mathbf{x}\|_2\leq 1$, we get:
$$M_2\leq\frac{1}{\beta n_\mathcal{T}\gamma}(\|\frac{1}{2}\Delta \mathbf{M}\|_{\mathcal{F}}+\|-\frac{1}{2}\Delta
\mathbf{M}\|_{\mathcal{F}})=\frac{\|\mathbf{M}-\mathbf{M}^i\|_{\mathcal{F}}}{\beta
  n_\mathcal{T}\gamma}.$$ 
This leads to  the inequality 
$\|\mathbf{M}-\mathbf{M}^i\|_{\mathcal{F}}^2\leq \frac{\|\mathbf{M}-\mathbf{M}^i\|_{\mathcal{F}}}{\beta n_\mathcal{T}\gamma}$ from which \lref{lem:diffA} is directly derived.
\end{proof}
We now have all the material needed to prove the stability property of our algorithm.
\begin{lemma}
Let $n_\mathcal{T}$ and $n_\mathcal{R}$ be the number of training examples and reasonable points respectively, $n_\mathcal{R}=\hat{\tau} n_\mathcal{T}$ with $\hat{\tau} \in \left]0,1\right]$.
SLLC has a uniform stability in $\frac{\kappa}{n_\mathcal{T}}$ with $\kappa=\frac{1}{\gamma}(\frac{1}{\beta\gamma}+\frac{2}{\hat{\tau}})=\frac{\hat{\tau}+2\beta\gamma}{\hat{\tau}\beta\gamma^2}$, where $\beta$ is the regularization parameter  and $\gamma$ the margin.
\end{lemma}
\begin{proof}
For any sample $\mathcal{T}$ of size $n_\mathcal{T}$, any $1\leq i\leq n_\mathcal{T}$, any labeled examples $z=(\mathbf{x},y)$ and  $z'_i=(\mathbf{x_i'},y'_i)\sim P$:
\begin{eqnarray*}
\lefteqn{|\ell(\mathbf{M},z,\mathcal{R}) - \ell(\mathbf{M}^i,z,\mathcal{R}^i)| }\\
&\leq&\left|\frac{1}{\gamma  n_{\mathcal{R}}}\displaystyle\sum_{k=1}^{n_{\mathcal{R}}}y_kK_{\mathbf{M}}(\mathbf{x},\mathbf{x_k})-\frac{1}{\gamma  n_{\mathcal{R}^i}}\displaystyle\sum_{k=1}^{n_{\mathcal{R}^i}}y_kK_{\mathbf{M}^i}(\mathbf{x},\mathbf{x_k})\right|\\
&= &\left|\frac{1}{\gamma n_\mathcal{R}} \left(\left(\sum_{k=1,k\neq i}^{n_\mathcal{R}}
y_k(K_{\mathbf{M}}(\mathbf{x},\mathbf{x_k})-K_{\mathbf{M}^i}(\mathbf{x},\mathbf{x_k}))\right)+\right.\right.\\
&&\left.\left. \phantom{\sum_{k=1,k\neq i}^{n_\mathcal{R}}}
y_iK_{\mathbf{M}}(\mathbf{x},\mathbf{x_i})-y_i'K_{\mathbf{M}^i}(\mathbf{x},\mathbf{x_i}')\right)\right|\\
\end{eqnarray*}
\begin{eqnarray*}
&\leq&\frac{1}{\gamma n_\mathcal{R}}\left(\left(\sum_{k=1,k\neq i}^{n_\mathcal{R}} (|y_k|\|\mathbf{M}-\mathbf{M}^i\|_{\mathcal{F}})\right)+\right.\\
&&\left.\phantom{\sum_{k=1,k\neq i}^{n_\mathcal{R}}} |y_iK_{\mathbf{M}^i}(\mathbf{x},\mathbf{x_i})|+|y'_iK_{\mathbf{M}}(\mathbf{x},\mathbf{x_i'})|\right)\\
&\leq&\frac{1}{\gamma n_\mathcal{R}}\left(\frac{n_\mathcal{R}-1}{\beta n_\mathcal{T}\gamma}+2\right)\leq\frac{1}{\gamma n_\mathcal{R}}\left(\frac{n_\mathcal{R}}{\beta n_\mathcal{T}\gamma}+2\right).
\end{eqnarray*}
The first inequality follows from $P3$. The second comes from the fact that $\mathcal{R}$ and $\mathcal{R}^i$ differ in at most one element, corresponding to the  example $z_i$ in $\mathcal{R}$ and the example $z'_i$ replacing $z_i$ in $\mathcal{R}^i$. The last inequalities are obtained by the use of the triangle inequality, $P1$, $P2$, \lref{lem:diffA}, and the fact that the labels belong to $\{-1,1\}$.   
Since $n_\mathcal{R}=\hat{\tau} n_\mathcal{T}$, we get 
$
|\ell(\mathbf{M},z,\mathcal{R}) - \ell(\mathbf{M}^i,z,\mathcal{R}^i)| \leq \frac{1}{\gamma n_\mathcal{T}}(\frac{1}{\beta\gamma}+\frac{2}{\hat{\tau}}).
$
\end{proof}

Applying \thref{thm:stability} with \lref{lem:diffA} gives our main result.
\begin{theorem}\label{thm:icmlguarantee}
Let $\gamma>0$, $\delta>0$ and $n_\mathcal{T}>1$. With probability at least $1-\delta$, for any model $\mathbf{M}_R$ learned with SLLC, we have:
$$\epsilon\leq \epsilon_\mathcal{T}+\frac{1}{ n_\mathcal{T}}\left(\frac{\hat{\tau}+2\beta\gamma}{\hat{\tau}\beta\gamma^2}\right)+\left(\frac{2(\hat{\tau}+2\beta\gamma)}{\hat{\tau}\beta\gamma^2}+1\right)\sqrt{\frac{\ln 1/\delta}{2 n_\mathcal{T}}}.$$
\end{theorem}

\thref{thm:icmlguarantee} highlights three important properties of SLLC. First, it has a reasonable $O(1/\sqrt{n_\mathcal{T}})$ convergence rate. Second, it is independent from the dimensionality of the data. This is due to the fact that $\|\mathbf{M}\|_{\mathcal{F}}$ is bounded by a constant.  Third, \thref{thm:icmlguarantee} bounds the true goodness of the learned similarity
function. By minimizing $\epsilon_\mathcal{T}$ with SLLC, we minimize $\epsilon$ and thus an upper bound on the true risk of the resulting linear classifier, as stated by \thref{thm:thmsim}. Note that this is a much tighter bound on the goodness than that derived in \cref{chap:ecml}, where only a loose bound on the empirical goodness was optimized.

\section{Experimental Validation}
\label{sec:icmlexpes}

We propose a comparative study of our method against two widely-used Mahalanobis distance learning algorithms: Large Margin Nearest Neighbor\footnote{Code download from: \url{http://www.cse.wustl.edu/~kilian/code/lmnn/lmnn.html}} (LMNN) from \citet{Weinberger2009} and Information-Theoretic Metric Learning\footnote{Code download from: \url{http://www.cs.utexas.edu/~pjain/itml/}} (ITML) from \citet{Davis2007}. Recall that LMNN essentially optimizes the $k$-NN error on the training set (with a safety margin), whereas ITML aims at best satisfying pair-based constraints while minimizing the LogDet divergence between the learned matrix $\mathbf{M}$ and the identity matrix (refer to \sref{sec:mahalearning} for more details on these methods). We conduct this experimental study on seven classic binary classification datasets of varying domain, size and difficulty, mostly taken from the UCI Machine Learning Repository\footnote{\url{http://archive.ics.uci.edu/ml/}}. Their properties are summarized in \tref{tab:datasets}. Some of them, such as Breast, Ionosphere or Pima, have already been extensively used to evaluate metric learning methods.

\subsection{Setup}

We compare the following methods: (i) the cosine similarity $K_\mathbf{I}$ in KPCA space, as a baseline, (ii) SLLC, (iii) LMNN in the original space, (iv) LMNN in KPCA space, (v) ITML in the original space, and (vi) ITML in KPCA space.\footnote{$K_\mathbf{I}$, LMNN and ITML are normalized to ensure their values belong to $[-1,1]$.} All attributes are scaled to $[-1/d;1/d]$ to ensure $\|\mathbf{x}\|_2 \leq 1$.

To generate a new feature space using KPCA, we use the Gaussian kernel with parameter $\sigma$ equal to the mean of all pairwise training data Euclidean
distances \citep[a standard heuristic, used for instance by][]{Kar2011}. Ideally, we would like to project the data to the feature space of maximum size (equal to the number of training examples), but to keep the computations tractable we only retain three times the number of features of the original data (four times for the low-dimensional datasets), as shown in \tref{tab:datasets}.\footnote{Note that the amount of variance captured thereby was greater than 90\% for all datasets.} On Cod-RNA, KPCA was run on a randomly drawn subsample of 10\% of the training data.

Unless predefined training and test sets are available (as for Splice, Svmguide1 and Cod-RNA), we randomly generate 70/30 splits of the data, and average the results over 100 runs. Training sets are further partitioned 70/30 for validation purposes.

We tune the following parameters by cross-validation: $\beta,\gamma\in\{10^{-7},\dots,10^{-2}\}$ for SLLC, $\lambda_{ITML}\in\{10^{-4},\dots,10^{4}\}$ for ITML, and $\lambda\in\{10^{-3},\dots,10^{2}\}$ for learning the linear classifiers, choosing the value offering the best accuracy.
We choose $\mathcal{R}$ to be the entire training set, i.e., $\hat{\tau}=1$ (interestingly, cross-validation of $\hat{\tau}$ did not improve the results significantly).
We take $k=3$ and $\mu=0.5$ for LMNN, as suggested by \citet{Weinberger2009}. For ITML, we generate $n_\mathcal{T}$ random constraints for a fair comparison with SLLC.

\begin{table}[t]
\begin{center}
\begin{scriptsize}
\begin{tabular}{rccccccc}
\toprule
\textbf{Dataset} & \textbf{Breast} & \textbf{Iono.} & \textbf{Rings} & \textbf{Pima} & \textbf{Splice} & \textbf{Svmguide1} & \textbf{Cod-RNA}\\
\midrule
\# training examples & 488 & 245 & 700 & 537 & 1,000 & 3,089 & 59,535\\
\# test examples & 211 & 106 & 300 & 231 & 2,175 & 4,000 & 271,617\\
\# dimensions & 9 & 34 & 2 & 8 & 60 & 4 & 8\\
\# dim. after KPCA & 27 & 102 & 8 & 24 & 180 & 16 & 24\\
\# runs & 100 & 100 & 100 & 100 & 1 & 1 & 1\\
\bottomrule
\end{tabular}
\end{scriptsize}
\caption[Properties of the datasets used in the experimental study]{Properties of the seven datasets used in the experimental study.}
\label{tab:datasets}
\end{center}
\end{table}

\subsection{Results}

\paragraph{Linear classification} We first report the results obtained in linear classification using Balcan's learning rule (\tref{tab:resultsbalcan}).
SLLC achieves the highest accuracy on 5 out of 7 datasets and competitive performance on the remaining 2. At the same time, on all datasets, SLLC leads to extremely sparse classifiers. The sparsity of the classifier corresponds to the number of training examples that are involved in classifying a new example. Therefore, SLLC leads to much simpler and yet often more accurate classifiers than those built from other similarities. Furthermore, sparsity allows faster predictions, especially when data are plentiful and/or high-dimensional (e.g., Cod-RNA or Splice). Often enough, the learned linear classifier has sparsity 1, which means that classifying a new example boils down to computing its similarity score to a single training example and compare the value with a threshold. Note that we tried large values of $\lambda$ to obtain sparser classifiers from $K_\mathbf{I}$, LMNN and ITML, but this yielded dramatic drops in accuracy. The extreme sparsity brought by SLLC comes from the fact that the constraints are based on an average of similarity scores over the same set of points for all training examples. This brings to the fore the relevance of optimizing the similarity with respect to global constraints.

\begin{table}[t]
\begin{center}
\begin{scriptsize}
\begin{tabular}{rccccccc}
\toprule
\textbf{Dataset} & \textbf{Breast} & \textbf{Iono.} & \textbf{Rings} & \textbf{Pima} & \textbf{Splice} & \textbf{Svmguide1} & \textbf{Cod-RNA}\\
\midrule
\multirow{2}{*}{$K_\mathbf{I}$} & 96.57 & 89.81 & 100.00 & 75.62 & 83.86 & \textbf{96.95} & \textbf{95.91}\\
 & \emph{20.39} & \emph{52.93} & \emph{18.20} & \emph{25.93} & \emph{362} & \textbf{\emph{64}} & \textbf{\emph{557}}\\
\midrule
\multirow{2}{*}{SLLC} & \textbf{96.90} & \textbf{93.25} & \textbf{100.00} & \textbf{75.94} & \textbf{87.36} & 96.55 & 94.08 \\
 & \textbf{\emph{1.00}} & \textbf{\emph{1.00}} & \textbf{\emph{1.00}} & \textbf{\emph{1.00}} & \textbf{\emph{1}} & \emph{8} & \emph{1} \\
\midrule
\multirow{2}{*}{LMNN} & 96.81 & 90.21 & 100.00 & 75.15 & 85.61 & 95.80 & 88.40\\
 & \emph{9.98} & \emph{13.30} & \emph{18.04} & \emph{69.71} & \emph{315} & \emph{157} & \emph{61}\\
\midrule
\multirow{2}{*}{LMNN KPCA} & 96.01 & 86.12 & 100.00 & 74.92 & 86.85 & 96.53 & 95.15\\
 & \emph{8.46} & \emph{9.96} & \emph{8.73} & \emph{22.20} & \emph{156} & \emph{82} & \emph{591}\\
\midrule
\multirow{2}{*}{ITML} & 96.80 & 92.09 & 100.00 & 75.25 & 81.47 & 96.70 & 95.06\\
 & \emph{9.79} & \emph{9.51} & \emph{17.85} & \emph{56.22} & \emph{377} & \emph{49} & \emph{164}\\
\midrule
\multirow{2}{*}{ITML KPCA} & 96.23 & 93.05 & 100.00 & 75.25 & 85.29 & 96.55 & 95.14\\
 & \emph{17.17} & \emph{18.01} & \emph{15.21} & \emph{16.40} & \emph{287} & \emph{89} & \emph{206}\\
\bottomrule
\end{tabular}
\end{scriptsize}
\caption[Accuracy of the linear classifiers built from the studied similarities]{Average accuracy (normal type) and sparsity (italic type) of the linear classifiers built from the studied similarity functions. For each dataset, boldface indicates the most accurate method (sparsity is used to break the ties).}
\label{tab:resultsbalcan}
\end{center}
\end{table}

\paragraph{Nearest neighbor classification} Since LMNN and ITML are designed for $k$-NN use, we also give the results obtained in 3-NN classification (\tref{tab:results3nn}).
Surprinsingly (because it is not designed for $k$-NN), SLLC achieves the best results on 4 datasets (a possible reason for this is given in the next paragraph). It is, however, outperformed by LMNN or ITML on the 3 biggest problems. For most tasks, the accuracy obtained in linear classification is better or similar to that of 3-NN (highlighting the fact that metric learning for linear classification is of interest) while prediction is many orders of magnitude faster due to the sparsity of the linear separators.
Also note that an accurate similarity for $k$-NN classification can achieve poor results in linear classification (LMNN on Cod-RNA), and vice versa (SLLC on Svmguide1).

\begin{table}[t]
\begin{center}
\begin{scriptsize}
\begin{tabular}{rccccccc}
\toprule
\textbf{Dataset} & \textbf{Breast} & \textbf{Iono.} & \textbf{Rings} & \textbf{Pima} & \textbf{Splice} & \textbf{Svmguide1} & \textbf{Cod-RNA}\\
\midrule
$K_\mathbf{I}$ & 96.71 & 83.57 & \textbf{100.00} & 72.78 & 77.52 & 93.93 & 90.07\\
SLLC & \textbf{96.90} & \textbf{93.25} & \textbf{100.00} & \textbf{75.94} & 87.36 & 93.82 & 94.08\\
LMNN & 96.46 & 88.68 & \textbf{100.00} & 72.84 & 83.49 & 96.23 & 94.98\\
LMNN KPCA & 96.23 & 87.13 & \textbf{100.00} & 73.50 & \textbf{87.59} & 95.85 & 94.43\\
ITML & 92.67 & 88.29 & \textbf{100.00} & 72.07 & 77.43 & 95.97 & \textbf{95.42}\\
ITML KPCA & 96.38 & 87.56 & \textbf{100.00} & 72.80 & 84.41 & \textbf{96.80} & 95.32\\
\bottomrule
\end{tabular}
\end{scriptsize}
\caption[Accuracy of 3-NN classifiers using the studied similarities]{Average accuracy of 3-NN classifiers using the studied similarity functions. For each dataset, boldface indicates the most accurate method.}
\label{tab:results3nn}
\end{center}
\end{table}

\begin{figure}[t]
\begin{center}
\psfrag{Dimension}[][][0.8]{\textbf{Dimension}}
\psfrag{Classification accuracy}[][][0.8]{\textbf{Classification accuracy}}
\psfrag{SLLC}[][][0.8]{SLLC}
\psfrag{LMNN}[][][0.8]{LMNN}
\psfrag{ITML}[][][0.8]{ITML\hspace{0.1cm}}
\includegraphics[width=0.7\textwidth]{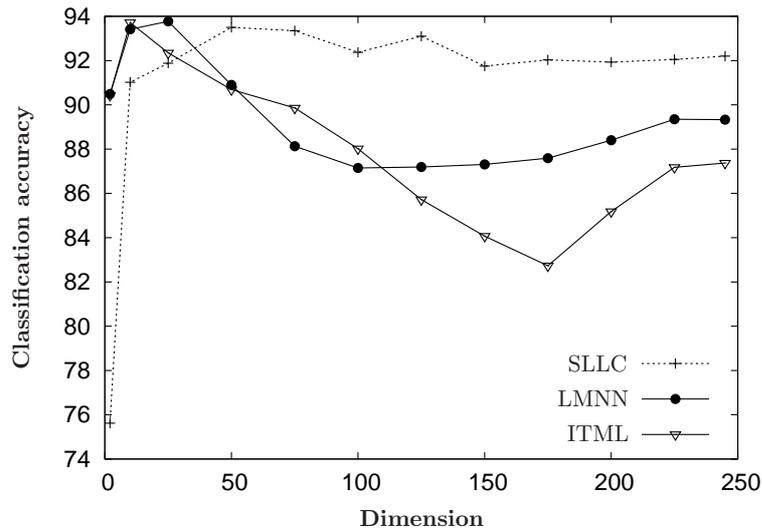}
\caption[Accuracy of the methods with respect to KPCA dimension]{Accuracy of the methods with respect to the dimensionality of the KPCA space on Ionosphere.}
\label{fig:fig-ion}
\end{center}
\end{figure}

\paragraph{Robustness to overfitting} SLLC's good performance on small datasets can be credited to its robustness to overfitting. Indeed, LMNN and ITML are optimized with respect to local constraints, which tend to get easier to satisfy simultaneously as dimensionality grows. On the other hand, SLLC is optimized with respect to global constraints and can thus be seen as more robust. This is confirmed by \fref{fig:fig-ion}, which shows the accuracy of SLLC, LMNN and ITML on the Ionosphere dataset with respect to the number of dimensions retained in KPCA. As expected, LMNN and ITML, tend to overfit as the dimensionality grows while SLLC suffers from very limited overfitting.

\paragraph{Visualization of the projection space} Recall that in Balcan's learning rule, the similarity is used to build a similarity map: data are projected into a new feature space where each coordinate corresponds to the similarity score to a training example, and a linear classifier is learned in that space. \fref{fig:PCA_space_lunes} and \fref{fig:PCA_space_svmguide1} show a low-dimensional embedding of the feature space induced by each similarity for the Rings and Svmguide1 datasets respectively. On both datasets, the space induced by SLLC is the most appropriate to linear classification: the data is well-separated even in this 2D representation of the space. On the Rings dataset, the data is actually perfectly separated in 1D, which explains why we achieve perfect classification accuracy relying on 1 training instance only. This highlights the fact that SLLC optimizes a criterion which is designed for linear classification, and its potential for dimensionality reduction. Conversely, the feature spaces induced by $K_\mathbf{I}$, LMNN and ITML do not offer such quality of linear separability --- for instance and unsurprisingly, LMNN tends to induce spaces that are better suited to nearest neighbor classification.

\begin{figure}[t]
\begin{center}
\psfrag{SLLC (1)}[][][0.8]{SLLC}
\psfrag{KI (0.50196)}[][][0.8]{KI}
\psfrag{LMNN (0.85804)}[][][0.8]{LMNN}
\psfrag{ITML (0.50002)}[][][0.8]{ITML}
\includegraphics[width=0.7\textwidth]{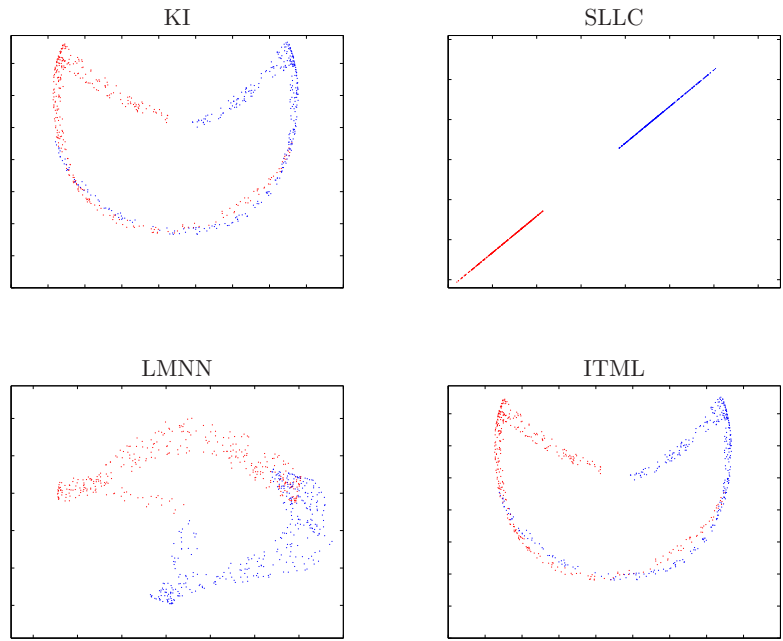}
\caption[Feature space induced by the similarity (Rings dataset)]{Feature space induced by the similarity in which the linear classifier is learned (Rings dataset). Dimension was reduced to 2 for visualization purposes using Principal Component Analysis.}
\label{fig:PCA_space_lunes}
\end{center}
\end{figure}

\begin{figure}[t]
\begin{center}
\psfrag{SLLC (0.93035)}[][][0.8]{SLLC}
\psfrag{KI (0.69754)}[][][0.8]{KI}
\psfrag{LMNN (0.80975)}[][][0.8]{LMNN}
\psfrag{ITML (0.97105)}[][][0.8]{ITML}
\includegraphics[width=0.7\textwidth]{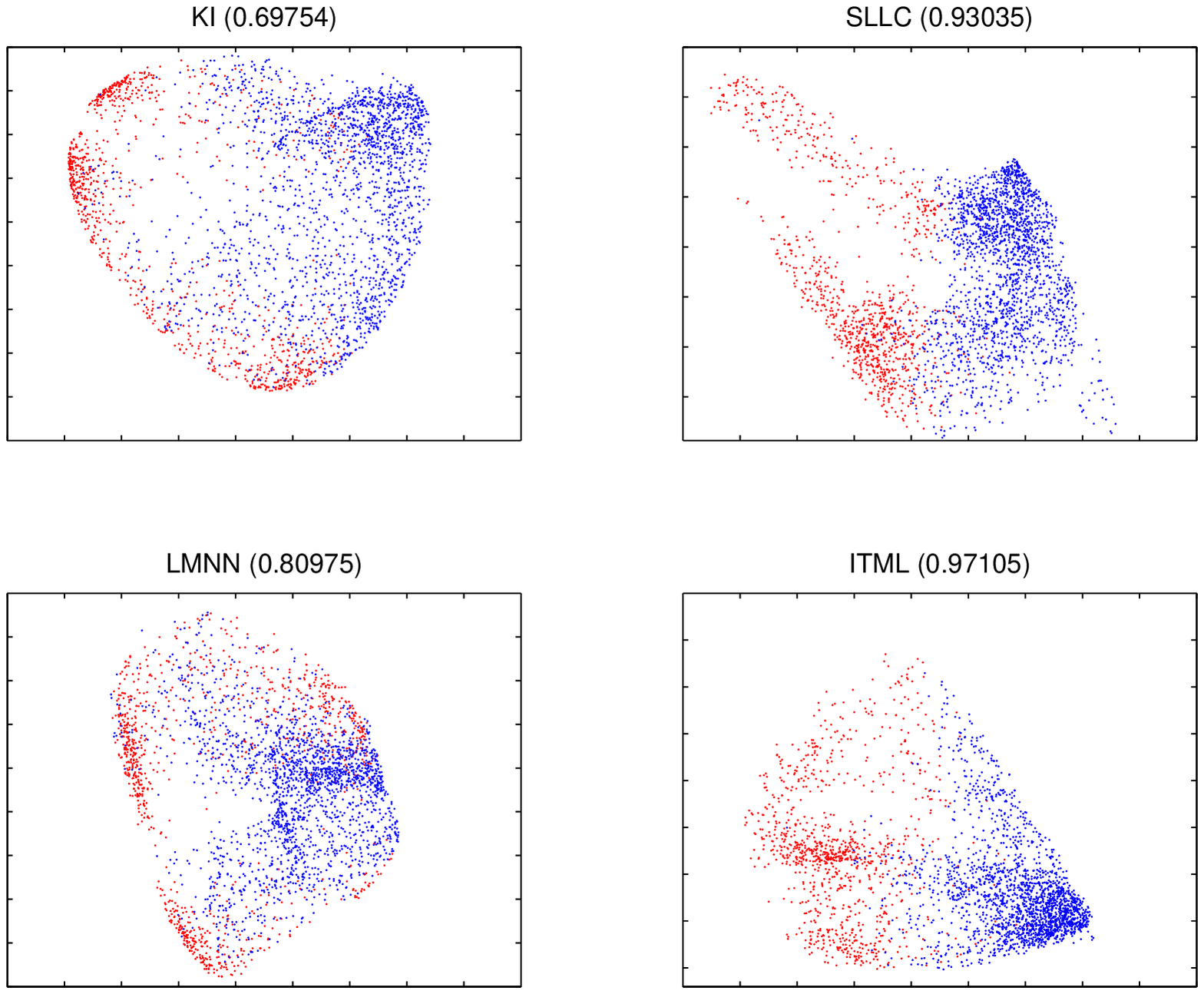}
\caption[Feature space induced by the similarity (Svmguide1 dataset)]{Feature space induced by the similarity in which the linear classifier is learned (Svmguide1 dataset). Dimension was reduced to 2 for visualization purposes using Principal Component Analysis.}
\label{fig:PCA_space_svmguide1}
\end{center}
\end{figure}

\paragraph{Runtime comparison} In this series of experiments, SLLC was solved in its QP form using the standard convex minimization solver \textsc{Mosek}\footnote{\url{http://www.mosek.com/}} while LMNN and ITML have their own specific and sophisticated solver.
Despite this fact, SLLC is several orders of magnitude faster than LMNN (see \tref{tab:resultstime}) because its number of constraints is much smaller. However, it remains slower than ITML.

\begin{table}[t]
\begin{center}
\begin{scriptsize}
\begin{tabular}{rccccccc}
\toprule
\textbf{Dataset} & \textbf{Breast} & \textbf{Iono.} & \textbf{Rings} & \textbf{Pima} & \textbf{Splice} & \textbf{Svmguide1} & \textbf{Cod-RNA}\\
\midrule
SLLC & 4.76 & 5.36 & 0.05 & 4.01 & 158.38 & 185.53 & 2471.25\\
LMNN & 25.99 & 16.27 & 37.95 & 32.14 & 309.36 & 331.28 & 10418.73\\
LMNN KPCA & 41.06 & 34.57 & 84.86 & 48.28 & 1122.60 & 369.31 & 24296.41\\
ITML & 2.09 & 3.09 & 0.19 & 2.96 & 3.41 & 0.83 & 5.98\\
ITML KPCA & 1.68 & 5.77 & 0.20 & 2.74 & 56.14 & 5.30 & 25.25\\
\bottomrule
\end{tabular}
\end{scriptsize}
\caption[Runtime of the studied metric learning methods]{Average time per run (in seconds) required for learning the similarity.}
\label{tab:resultstime}
\end{center}
\end{table}

\section{Conclusion}
\label{sec:icmlconclu}

In this chapter, we presented SLLC, a novel approach to bilinear similarity learning that makes use of both the theory of $(\epsilon,\gamma,\tau)$-goodness and the KPCA trick. It is formulated as a convex minimization problem that can be solved efficiently using standard techniques.
We derived a generalization bound based on the notion of uniform stability that is independent from the size of the input space, and thus from the number of dimensions selected by KPCA. 
It guarantees the true goodness of the learned similarity, and therefore our method can be seen as minimizing an upper bound on the true risk of the linear classifier built from the learned similarity.
We experimentally demonstrated the effectiveness of SLLC and also showed that the learned similarities induce extremely sparse classifiers. Combined with the independence from dimensionality and the robustness to overfitting, it makes the approach very efficient and suitable for high-dimensional data.  \tref{tab:sllcsum} summarizes the main features of SLLC using the same format as in the survey of \cref{chap:metriclearning} (\tref{tab:mlvectsum}).

\begin{table}[t]
\begin{center}
\begin{scriptsize}
\begin{tabular}{cccccccc}
\toprule
\textbf{Method} & \textbf{Convex} & \textbf{Scalable} & \textbf{Competitive} & \textbf{Reg.} & \textbf{Low-rank} & \textbf{Online} & \textbf{Gen.}\\
\midrule
SLLC & \tickYes & \tickYes\tickYes\tickYes & \tickYes & \tickYes & \tickNo & \tickYes & \tickYes\\
\bottomrule
\end{tabular}
\end{scriptsize}
\caption[Summary of the main features of SLLC]{Summary of the main features of SLLC (``Reg.'' and ``Gen.'' respectively stand for ``Regularized'' and ``Generalization guarantees'').}
\label{tab:sllcsum}
\end{center}
\end{table}

It would be interesting to investigate the performance of SLLC when solved in its unconstrained form, either in a stochastic or online way. This would dramatically improve its runtime on large-scale problems and hopefully not significantly reduce the classification performance.

As shown in \tref{tab:sllcsum}, SLLC is not a low-rank approach, since Frobenius norm regularization does not favor low-rank matrices. Another promising perspective would be to study the influence of other regularizers on $\mathbf{A}$, in particular the trace norm or the $L_{2,1}$ norm that tend to induce such matrices. Recent advances in stochastic and online optimization of problems regularized with these norms \citep{Duchi2010a,Xiao2010,Yang2010} could be used to derive an efficient algorithm. The use of such norms would add sparsity at the metric level in addition to the sparsity already obtained at the classifier level.

However, recall that the generalization of such formulations cannot be studied using stability-based arguments, since sparse algorithms are known not to be stable. On the other hand, algorithmic robustness can deal with such algorithms more easily. In the next chapter, we propose an adaptation of robustness to the metric learning setting.

\chapter{Robustness and Generalization for Metric Learning}
\label{chap:nips}

\begin{chapabstract}
Throughout this thesis, we have argued that little work has been done about the generalization ability of metric learning algorithms. We made use in \cref{chap:ecml} and \cref{chap:icml} of uniform stability arguments to derive generalization guarantees for our metric learning methods. Unfortunately, these arguments are somewhat limited to the use of Frobenius regularizarion and thus cannot be applied to many existing metric learning algorithms, in particular those using a sparse or low-rank regularizer on the metric. In this chapter, we address this theoretical issue by proposing an adaptation of the notion of algorithmic robustness (previously introduced by Xu and Mannor) to the classic metric learning setting, where training data consist of pairs or triplets. We show that if a metric learning algorithm is robust in our sense, then it has generalization guarantees. We further show that a weak notion of robustness is a necessary and sufficient condition for an algorithm to generalize, justifying that it is fundamental to metric learning. Lastly, we illustrate how our framework can be used to derive generalization bounds for a large class of metric learning algorithms, some of which could not be studied using previous approaches.\\

The material of this chapter is based on the following technical report:\\

\bibentry{Bellet2012b}.
\end{chapabstract}

\section{Introduction}

Most of the research effort in metric learning has gone into formulating the problem as tractable optimization procedures, but very little has been done on the generalization ability of learned metrics on unseen data, due to the fact that the training pairs/triplets are not i.i.d.
As we have seen in \sref{sec:onlineml}, online metric learning methods \citep[e.g.,][]{Shalev-Shwartz2004,Jain2008,Chechik2009} offer some guarantees, but only in the form of regret bounds assuming that the algorithm is provided with i.i.d. pairs/triplets, and say nothing about generalization to unseen data. Conversion of regret bounds into batch generalization bounds is possible \citep[see for instance][]{Cesa-Bianchi2001,Cesa-Bianchi2004} but as a consequence these bounds also require the i.i.d. assumption.

Putting aside our contributions in \cref{chap:ecml} and \cref{chap:icml}, the question of the generalization ability of batch metric learning has only been addressed in two recent papers, described in \sref{sec:genml}. For the sake of readability, we recall here their main features. The approach of Bian \& Tao \citeyearpar{Bian2011,Bian2012} uses a statistical analysis to give generalization guarantees for loss minimization methods, but their results rely on some hypotheses on the distribution of the examples and do not take into account any regularization on the metric. The most general contribution was proposed by \citet{Jin2009} who adapted the framework of uniform stability to regularized metric learning. However, their approach is based on Frobenius norm regularization and cannot be applied to many types of regularization, in particular sparsity-inducing norms \citep{Xu2012}.   

In this last contribution, we propose to address the lack of theoretical framework by studying the generalization ability of metric learning algorithms according to a notion of algorithmic robustness. Recall that algorithmic robustness, introduced by \citet{Xu2010,Xu2012a} and described in \sref{sec:robustness}, allows one to derive generalization bounds when, given two ``close'' training and testing examples, the variation between their associated loss is bounded. This notion of closeness of examples relies on a partition of the input space into different regions such that two examples in the same region are seen as close.
We propose here to adapt this notion of algorithmic robustness to metric learning, where training data is made of pairs (or triplets).
We show that, in the context of robustness, the problem of training pairs not being i.i.d. can be worked around by simply assuming that the pairs are built from an i.i.d. sample of labeled examples. 
Moreover, following the work of \citet{Xu2010,Xu2012a}, we establish that a weaker notion robustness is actually necessary and sufficient for metric learning algorithms to generalize, highlighting that robustness is a fundamental property. 
Lastly, we illustrate the applicability of our framework by deriving generalization bounds for a larger class of problems than \citet{Jin2009}, using very few algorithm-specific arguments. In particular, it can accommodate a vast choice of regularizers and unlike the approach of Bian \& Tao \citeyearpar{Bian2011,Bian2012}, requires no assumption on the distribution of the examples. 

The rest of the chapter is organized as follows. Our notion of algorithmic robustness for metric learning is presented in \sref{sec:nipsrobustsec}. The necessity and sufficiency of weak robustness is shown in \sref{sec:nipsnessec}. \sref{sec:nipsexsec} is devoted to the application of the proposed framework: we show that a large class of metric learning algorithms are robust. Finally, we conclude in \sref{sec:nipsconclu}.

\section{Robustness and Generalization for Metric Learning}
\label{sec:nipsrobustsec}

After introducing some notations and assumptions, we present our definition of robustness for metric learning and show that if a metric learning algorithm is robust, then it has generalization guarantees.

\subsection{Preliminaries}

We assume that the instance space $\mathcal{X}$ is a compact convex metric space with respect to a norm $\|\cdot\|$ such that $\mathcal{X}\subset\mathbb{R}^d$, thus there exists a constant $R$ such that $\forall \mathbf{x}\in \mathcal{X}$, $\|\mathbf{x}\|\leq R$.
A metric is a function $f:\mathcal{X}\times \mathcal{X}\rightarrow \mathbb{R}$.  
Recall that we use the generic term metric to refer to a distance or a (dis)similarity function.

Given a training sample $\mathcal{T} = \{z_i=(\mathbf{x_i},y_i)\}_{i=1}^n$ drawn i.i.d. from an unknown joint distribution $P$ over the space $\mathcal{Z} = \mathcal{X}\times\mathcal{Y}$, we denote by $\mathcal{P}_\mathcal{T}$ the set of all possible pairs built from $\mathcal{T}$:
$$\mathcal{P}_\mathcal{T}=\{(z_1,z_1),\cdots,(z_1,z_n),\cdots,(z_n,z_n)\}.$$

We generally assume that a metric learning algorithm $\mathcal{A}$ takes as input a finite set of pairs from $(\mathcal{Z}\times\mathcal{Z})^n$ and outputs a metric. We denote by $\mathcal{A}_\mathcal{P}$ the metric learned by an algorithm $\mathcal{A}$ from a sample $\mathcal{P}$ of pairs. 
With any pair of labeled examples $(z,z')$ and any metric $f$, we associate a loss function $\ell(f,z,z')$ that depends on the examples and their labels. This loss is assumed to be nonnegative and uniformly bounded by a constant $B$.  
We define the true risk of $f$ by
$$R^\ell(f)=\mathbb{E}_{z,z'\sim P}[\ell(f,z,z')].$$
We denote the empirical risk of $f$ over the sample of pairs $\mathcal{P}$ by
$$R^\ell_\mathcal{P}(f)=\frac{1}{|\mathcal{P}|}\sum_{(z_i,z_j)\in\mathcal{P}}\ell(f,z_i,z_j).$$

On a few occasions, we discuss the extension of our framework to triplet-based metric learning, where an algorithm $\mathcal{A}$ takes as input a finite set of triplets from $(\mathcal{Z}\times\mathcal{Z}\times\mathcal{Z})^n$. Instead of considering all pairs $\mathcal{P}_\mathcal{T}$ built from $\mathcal{T}$, we consider the sample of \emph{admissible} triplets $\mathcal{R}_{\mathcal{T}}$ built from $\mathcal{T}$ such that for any $(z_1,z_2,z_3)\in \mathcal{R}_{\mathcal{T}}$, $z_1$ and $z_2$ share the same label while $z_3$ does not, with the interpretation that $z_1$ must be more similar to $z_2$ than to $z_3$. In this context, the loss function $\ell$ is defined with respect to triplets of examples and the true risk of a metric $f$ is given by
$$R^\ell(f)=\mathbb{E}_{\substack{z,z',z''\sim P\\y=y'\neq y''}}[\ell(f,z,z',z'')]$$
and the empirical risk of $f$ over the sample of admissible triplets $\mathcal{R}$ by
$$R^\ell_\mathcal{R}(f)=\frac{1}{|\mathcal{R}|}\sum_{(z_i,z_j,z_k)\in\mathcal{R}}\ell(f,z_i,z_j,z_k).$$

\subsection{Robustness for Metric Learning}

We present here our adaptation of the definition of robustness to metric learning.

In \sref{sec:robustness}, we have seen that robustness relies on a partition of the space $\mathcal{Z}$ into $K$ disjoint subsets such that for every training and testing instances belonging to the same region of the partition, the deviation between their respective losses is bounded by a term $\epsilon(\mathcal{T})$.\footnote{Recall from \sref{sec:robustness} that $\mathcal{Z}$ is partitioned such that if two examples fall into the same region, then they share the same label.}
In order to adapt this notion to metric learning, the idea is to use the partition of $\mathcal{Z}$ at the pair level: if a new test pair of examples is close to a training pair, then the respective losses of the two pairs must be close. Two pairs are close when each instance of the first pair falls into the same subset of the partition of $\mathcal{Z}$ as the corresponding instance of the other pair, as shown in \fref{fig:robustness}.
A metric learning algorithm with this property is called robust. This notion is formalized in the following definition.

\begin{figure}[t]
\begin{center}
\psfrag{Classic robustness}[][][0.8]{Classic robustness}
\psfrag{Robustness for metric learning}[][][0.8]{Robustness for metric learning}
\psfrag{z}[][][0.8]{$z$}
\psfrag{z'}[][][0.8]{$z'$}
\psfrag{z1}[][][0.8]{$z_1$}
\psfrag{z2}[][][0.8]{$z_2$}
\psfrag{Ci}[][][0.6]{$C_i$}
\psfrag{Cj}[][][0.6]{$C_j$}
\includegraphics[width=0.7\columnwidth]{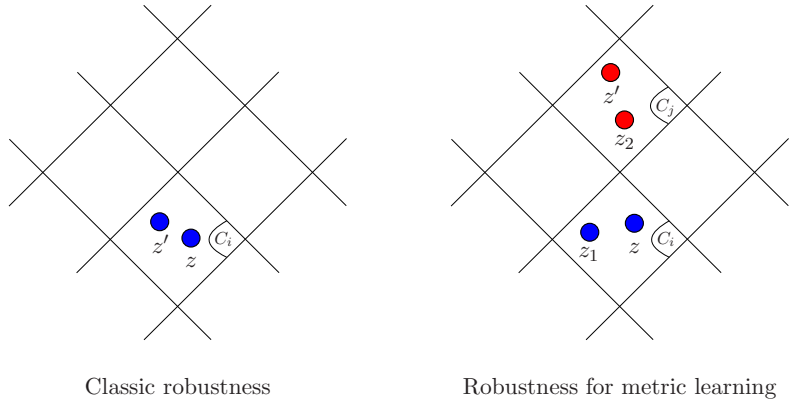}
\caption[Illustration of robustness in the classic and metric learning settings]{Illustration of the property of robustness in the classic and metric learning settings. In this example, we use a cover based on the $L_1$ norm. In the classic definition, if any example $z'$ falls in the same region $C_i$ as a training example $z$, then the deviation between their loss must be bounded. In the metric learning definition proposed in this work, for any pair $(z,z')$ and a training pair $(z_1,z_2)$, if $z,z_1$ belong to some region $C_i$ and $z',z_2$ to some region $C_j$, then the deviation between the loss of these two pairs must be bounded.}
\label{fig:robustness}
\end{center}
\end{figure}

\begin{definition}[Robustness for metric learning]
\label{def:robuml}
An algorithm $\mathcal{A}$ is $(K,\epsilon(\cdot))$ robust for
$K\in\mathbb{N}$ and $\epsilon(\cdot): (\mathcal{Z}\times\mathcal{Z})^n \rightarrow
\mathbb{R}$ if $\mathcal{Z}$ can be partitioned into $K$ disjoints sets, denoted
by $\{C_i\}_{i=1}^K$, such that the following holds for all $\mathcal{T}\in\mathcal{Z}^n$:\\
$\forall (z_1,z_2) \in \mathcal{P}_\mathcal{T}, \forall z,z' \in
\mathcal{Z}, \forall i,j\in[K]:$  if 
$z_1,z\in C_i$ and $z_2,z'\in C_j$ then
$$|\ell(\mathcal{A}_{\mathcal{P}_\mathcal{T}},z_1,z_2)-\ell(\mathcal{A}_{\mathcal{P}_\mathcal{T}},z,z')|\leq \epsilon(\mathcal{P}_\mathcal{T}).$$
\end{definition}

$K$ and $\epsilon(\cdot)$ quantify the robustness of the algorithm which depends on the training sample. Note that the property of robustness is required for every training pair of the sample --- we will later see that this property can be relaxed.

Note that this definition of robustness can be easily extended to triplet-based metric learning. In this context, the robustness property can then be expressed by:\\
$\forall (z_1,z_2,z_3) \in \mathcal{R}_{\mathcal{T}}, \forall z,z',z'' \in \mathcal{Z}, \forall i,j\in[K]:$  if 
$z_1,z\in C_i$,  $z_2,z'\in C_j$, $z_3,z''\in C_k$ then 
\begin{equation}\label{eq:robu_trip} |\ell(\mathcal{A}_{\mathcal{R}_{\mathcal{T}}},z_1,z_2,z_3)-\ell(\mathcal{A}_{\mathcal{R}_{\mathcal{T}}},z,z',z'')|\leq \epsilon(\mathcal{R}_{\mathcal{T}}).
\end{equation}

\subsection{Generalization of Robust Metric Learning Algorithms}

We now give a PAC generalization bound for metric learning algorithms satisfying the property of robustness (\defref{def:robuml}).
We first give the following concentration inequality that we will use in the derivation of the bound.
\begin{proposition}[\citeauthor{Vaart2000}, \citeyear{Vaart2000}]\label{prop:BHC}
Let $(|N_1|,\dots,|N_K|)$ an i.i.d. multinomial random variable with
parameters $n$ and $(\mu(C_1),\dots,\mu(C_K))$. 
By the Breteganolle-Huber-Carol inequality we have:
$
Pr\left\{\sum_{i=1}^K \left|\frac{|N_i|}{n}-\mu(C_i)\right| \geq
  \lambda \right\}\leq 2^K \exp\left(\frac{-n\lambda^2}{2}\right)
$, 
hence with probability at least $1-\delta$,
\begin{equation}
\sum_{i=1}^K \left|\frac{N_i}{n}-\mu(C_i)\right|\leq \sqrt{\frac{2K\ln
  2 + 2 \ln(1/\delta)}{n}}.
\end{equation}
\end{proposition}

We now give our first result on the generalization of metric learning algorithms.
\begin{theorem}\label{thm:robu}
If a learning algorithm $\mathcal{A}$ is $(K,\epsilon(\cdot))$-robust
and the training sample consists of the pairs $\mathcal{P}_\mathcal{T}$ obtained from a sample $\mathcal{T}$ generated by $n$ i.i.d. draws from $P$, then
for any $\delta>0$, with probability at least $1-\delta$ we have:
$$
|R^\ell(\mathcal{A}_{\mathcal{P}_\mathcal{T}})-R^\ell_{\mathcal{P}_\mathcal{T}}(\mathcal{A}_{\mathcal{P}_\mathcal{T}})|\leq \epsilon(\mathcal{P}_\mathcal{T})+2B\sqrt{\frac{2K \ln 2 + 2\ln (1/\delta)}{n}}.
$$
\end{theorem}
\begin{proof}
Let $N_i$ be the set of index of points of  $\mathcal{T}$ that fall into the $C_i$. $(|N_1|,\dots,|N_K|)$ is an i.i.d. random variable with parameters $n$ and $(\mu(C_1),\dots,\mu(C_K))$. 
We have:
\begin{small}
\begin{eqnarray*}
\lefteqn{|R^\ell(\mathcal{A}_{\mathcal{P}_\mathcal{T}})-R^\ell_{\mathcal{P}_\mathcal{T}}(\mathcal{A}_{\mathcal{P}_\mathcal{T}})|}\\
&=&\left|\sum_{i=1}^K\sum_{j=1}^K\mathbb{E}_{z,z'\sim
    P}(\ell(\mathcal{A}_{\mathcal{P}_\mathcal{T}},z,z')|z\in C_i,z'\in  C_j)\mu(C_i)\mu(C_j)-\frac{1}{n^2}\sum_{i=1}^n\sum_{j=1}^n\ell(\mathcal{A}_{\mathcal{P}_\mathcal{T}},z_i,z_j)\right|\\
&\stackrel{(a)}{\leq}&\left|\sum_{i=1}^K\sum_{j=1}^K\mathbb{E}_{z,z'\sim
    P}(\ell(\mathcal{A}_{\mathcal{P}_\mathcal{T}},z,z')|z\in C_i,z'\in
  C_j)\mu(C_i)\mu(C_j)-\right.\\
&&\hspace{2cm}\left.\sum_{i=1}^K\sum_{j=1}^K\mathbb{E}_{z,z'\sim
    P}(\ell(\mathcal{A}_{\mathcal{P}_\mathcal{T}},z,z')|z\in C_i,z'\in
  C_j)\mu(C_i)\frac{|N_j|}{n}\right|+\\
&&\left|\sum_{i=1}^K\sum_{j=1}^K\mathbb{E}_{z,z'\sim
    P}(\ell(\mathcal{A}_{\mathcal{P}_\mathcal{T}},z,z')|z\in C_i,z'\in  C_j)\mu(C_i)\frac{|N_j|}{n}-\frac{1}{n^2}\sum_{i=1}^n\sum_{j=1}^n\ell(\mathcal{A}_{\mathcal{P}_\mathcal{T}},z_i,z_j)\right|\\
\end{eqnarray*}
\begin{eqnarray*}
&\stackrel{(b)}{\leq}&\left|\sum_{i=1}^K\sum_{j=1}^K\mathbb{E}_{z,z'\sim
    P}(\ell(\mathcal{A}_{\mathcal{P}_\mathcal{T}},z,z')|z\in C_i,z'\in
  C_j)\mu(C_i)(\mu(C_j)-\frac{|N_j|}{n})\right|+\\
&&\left|\sum_{i=1}^K\sum_{j=1}^K\mathbb{E}_{z,z'\sim
    P}(\ell(\mathcal{A}_{\mathcal{P}_\mathcal{T}},z,z')|z\in C_i,z'\in
  C_j)\mu(C_i)\frac{|N_j|}{n}-\right.\\
&&\hspace{2cm}\left.\sum_{i=1}^K\sum_{j=1}^K\mathbb{E}_{z,z'\sim
    P}(\ell(\mathcal{A}_{\mathcal{P}_\mathcal{T}},z,z')|z\in C_i,z'\in
  C_j)\frac{|N_i|}{n}\frac{|N_j|}{n}\right|+\\
&&\left|\sum_{i=1}^K\sum_{j=1}^K\mathbb{E}_{z,z'\sim
    P}(\ell(\mathcal{A}_{\mathcal{P}_\mathcal{T}},z,z')|z\in C_i,z'\in
  C_j)\frac{|N_i|}{n}\frac{|N_j|}{n}
-\frac{1}{n^2}\sum_{i=1}^n\sum_{j=1}^n\ell(\mathcal{A}_{\mathcal{P}_\mathcal{T}},z_i,z_j)\right|\\
&\stackrel{(c)}{\leq}&B\left(\left|\sum_{j=1}^K\mu(C_j)-\frac{|N_j|}{n}\right|
+\left|\sum_{i=1}^K\mu(C_i)-\frac{|N_i|}{n}\right|\right)+\\
&&\left|\frac{1}{n^2}\sum_{i=1}^K\sum_{j=1}^K\sum_{z_o\in N_i}\sum_{z_l\in
  N_j}\max_{z\in C_i}\max_{z'\in
  C_j}|\ell(\mathcal{A}_{\mathcal{P}_\mathcal{T}},z,z')-\ell(\mathcal{A}_{\mathcal{P}_\mathcal{T}},z_o,z_l)|\right|\\
&\stackrel{(d)}{\leq}&\epsilon(\mathcal{P}_\mathcal{T})+2B\sum_{i=1}^K\left|\frac{|N_i|}{n}-\mu(C_i)\right|\\
&\stackrel{(e)}{\leq}&\epsilon(\mathcal{P}_\mathcal{T})+2B\sqrt{\frac{2K \ln 2 + 2\ln (1/\delta)}{n}}.
\end{eqnarray*}
\end{small}Inequalities $(a)$ and $(b)$ are due to the triangle inequality, $(c)$ uses the fact that $\ell$ is bounded by $B$, that $\sum_{i=1}^K\mu(C_i)=1$  by definition of a multinomial random variable and that $\sum_{j=1}^K \frac{|N_j|}{n}=1$ by definition of the $N_j$. Lastly, $(d)$ comes from the definition of robustness (\defref{def:robuml}) and $(e)$ from the application of \propref{prop:BHC}.
\end{proof}

The previous bound depends on $K$ which is given by the cover chosen for $\mathcal{Z}$. If for any $K$, the associated $\epsilon(\cdot)$ is constant with respect to $\mathcal{T}$ (i.e., $\epsilon_K(\mathcal{T})=\epsilon_K$), we can prove a bound holding uniformly for all $K$:
$$|R^\ell(\mathcal{A}_{\mathcal{P}_\mathcal{T}})-R^\ell_{\mathcal{P}_\mathcal{T}}(\mathcal{A}_{\mathcal{P}_\mathcal{T}})|\leq \inf_{K\geq 1}\left[\epsilon_K+2B\sqrt{\frac{2K \ln 2 + 2\ln(1/\delta)}{n}}\right].$$
The bound also gives an insight into what should be the objective of a robust metric learning algorithm: according to a partition of the labeled input space, given two regions, minimize the maximum loss over pairs of examples belonging to each region. 

For triplet-based metric learning algorithms, by following the definition of robustness given by \eqref{eq:robu_trip} and adapting straightforwardly the losses to triplets such that they output zero for non-admissible triplets, \thref{thm:robu} can be easily extended to obtain the following generalization bound:
\begin{equation}\label{eq:bound_trip}
|R^\ell(\mathcal{A}_{\mathcal{R}_\mathcal{T}})-R^\ell_{\mathcal{R}_\mathcal{T}}(\mathcal{A}_{\mathcal{R}_\mathcal{T}})|\leq \epsilon(\mathcal{R}_\mathcal{T})+3B\sqrt{\frac{2K \ln 2 + 2\ln (1/\delta)}{n}}.
\end{equation}

\subsection{Pseudo-robustness}

The previous study requires the robustness property to be satisfied for every training pair. We show, with the following definition, that it is possible to relax the robustness to be fulfilled for only a subpart of the training sample and yet be able to derive generalization guarantees. 

\begin{definition}
An algorithm $\mathcal{A}$ is $(K,\epsilon(\cdot),\hat{p}_n(\cdot))$
pseudo-robust for
$K\in\mathbb{N}$, $\epsilon(\cdot): (\mathcal{Z}\times\mathcal{Z})^n \rightarrow
\mathbb{R}$ and $\hat{p}_n(\cdot): (\mathcal{Z}\times\mathcal{Z})^n \rightarrow
\{1,\dots,n^2\}$, if $\mathcal{Z}$ can be partitioned into $K$ disjoints sets,
denoted 
by $\{C_i\}_{i=1}^K$, such that  for all
$\mathcal{T}\in\mathcal{Z}^n$ i.i.d. from $P$, there exists a subset of training pairs samples $\hat{\mathcal{P}}_{\mathcal{T}} \subseteq \mathcal{P}_\mathcal{T}$, with $|\hat{\mathcal{P}}_{\mathcal{T}}|=\hat{p}_n(\mathcal{P}_\mathcal{T})$,  
 such that the following holds:\\
$\forall (z_1,z_2)\in \hat{\mathcal{P}}_{\mathcal{T}}, \forall z,z' \in
\mathcal{Z}, \forall i,j\in[K]$:  if $z_1,z\in C_i$  and  $z_2,z'\in C_j$ then
\begin{equation}
|\ell(\mathcal{A}_{\mathcal{P}_\mathcal{T}},z_1,z_2)-l(\mathcal{A}_{\mathcal{P}_\mathcal{T}},z,z')|\leq \epsilon(\mathcal{P}_\mathcal{T}).
\end{equation}
\end{definition}
We can easily observe that $(K,\epsilon(\cdot))$-robust is equivalent to $(K,\epsilon(\cdot),n^2)$ pseudo-robust. The following theorem illustrates the generalization guarantees associated to the pseudo-robustness property.

\begin{theorem}\label{thm:pseudorobustess}
If a learning algorithm $\mathcal{A}$ is
$(K,\epsilon(\cdot),\hat{p}_n(\cdot))$ pseudo-robust and the training pairs $\mathcal{P}_\mathcal{T}$ come from a sample generated by $n$ i.i.d. draws from $P$, then
for any $\delta>0$, with probability at least $1-\delta$ we have:
$$
|R^\ell(\mathcal{A}_{\mathcal{P}_\mathcal{T}})-R^\ell_{\mathcal{P}_\mathcal{T}}(\mathcal{A}_{\mathcal{P}_\mathcal{T}})|\leq \frac{\hat{p}_n(\mathcal{P}_\mathcal{T})}{n^2}\epsilon(\mathcal{P}_\mathcal{T})+B\left(\frac{n^2-\hat{p}_n(\mathcal{P}_\mathcal{T})}{n^2}+2\sqrt{\frac{2K \ln 2 + 2\ln (1/\delta)}{n}}\right).
$$
\end{theorem}
\begin{proof}
The proof is similar to that of \thref{thm:robu} and is given in \aref{app:proofpseudo}.
\end{proof}


The notion of pseudo-robustness characterizes a situation that often occurs
in metric learning: it is difficult to satisfy pair-based constraints for all possible pairs. \thref{thm:pseudorobustess} shows that it is sufficient to satisfy a property of robustness over only a subset of the pairs to have generalization guarantees. 
Moreover, it also gives  an insight into the behavior of metric learning approaches aiming at learning a distance to be plugged in a $k$-NN classifier such as LMNN \citep{Weinberger2009}. These methods do not optimize the distance according to all possible pairs, but only according to the nearest neighbors of the same class and some pairs of different class. 
According to the previous theorem, this strategy is well-founded provided that the robustness property is fulfilled for some of the pairs used to optimize the metric. 
Finally, note that this notion of pseudo-robustness can be also easily adapted to triplet based metric learning. 

\section{Necessity of Robustness}
\label{sec:nipsnessec}

We prove here that a notion of weak robustness is actually necessary and sufficient to generalize in a metric learning setup. 
This result is based on  an asymptotic analysis following the work of \citet{Xu2012a}. 
We consider pairs of instances coming from an increasing sample of training instances
$\mathcal{T}=(z_1,z_2,\dots)$ and from a sample of test instances
$\mathcal{U}=(z_1',z_2',\dots)$ such that both samples are assumed to be
drawn i.i.d. from some distribution $P$. We use $\mathcal{T}(n)$ and
$\mathcal{U}(n)$ to denote the first $n$ examples of $\mathcal{T}$ and $\mathcal{U}$ respectively, while 
$\mathcal{T}^*$ denotes a fixed sequence of training examples. 

We first define a notion of generalizability for metric learning.
\begin{definition}[Generalizability for metric learning]
Given a training pair set $\mathcal{P}_{\mathcal{T}^*}$ built from a sequence of examples $\mathcal{T}^*$, a metric learning method
$\mathcal{A}$ generalizes with respect to $\mathcal{P}_{\mathcal{T}^*}$ if
$$\lim_n\left|R^\ell(\mathcal{A}_{\mathcal{P}_{\mathcal{T}^*(n)}})-R^\ell_{\mathcal{P}_{\mathcal{T}^*(n)}}(\mathcal{A}_{\mathcal{P}_{\mathcal{T}^*(n)}})\right|=0.$$ 

A learning method $\mathcal{A}$ generalizes with probability 1 if it generalizes
with respect to the pairs $\mathcal{P}_\mathcal{T}$ of almost all samples $\mathcal{T}$  i.i.d. from $P$.
\end{definition}

Note that this notion of generalizability implies convergence in mean. We then introduce the notion of weak robustness for metric learning.
\begin{definition}[Weak robustness for metric learning]
Given a set of training pairs $\mathcal{P}_{\mathcal{T}^*}$ built from a sequence of examples ${\mathcal{T}^*}$, a metric learning
  method $\mathcal{A}$ is weakly robust with respect to $\mathcal{P}_{\mathcal{T}^*}$ if there
  exists a sequence of $\{\mathcal{D}_n\subseteq \mathcal{Z}^n\}$ such
  that $Pr(\mathcal{U}(n)\in \mathcal{D}_n)\rightarrow 1$ and
$$
\lim_n\left\{\max_{\hat{\mathcal{T}}(n)\in\mathcal{D}_n}\left|R^\ell_{\mathcal{P}_{\hat{\mathcal{T}}(n)}}(\mathcal{A}_{\mathcal{P}_{\mathcal{T}^*(n)}})-R^\ell_{\mathcal{P}_{\mathcal{T}^*(n)}}\mathcal{A}_{\mathcal{P}_{\mathcal{T}^*(n)}})\right|\right\}=0.
$$

A learning method $\mathcal{A}$ is almost surely weakly robust if it is robust with respect to almost all $\mathcal{T}$.
\end{definition}

The definition of robustness requires the labeled sample space to be
partitioned into disjoint subsets such that if some instances of pairs of
train/test examples belong to the same partition, then they have
similar loss. Weak robustness is a generalization of this notion where
we consider the average loss of testing and training pairs: if for a
large (in the probabilistic sense) subset of data, the
testing loss is close to the training loss, then the algorithm is
weakly robust. From \propref{prop:BHC}, we can see that if for any fixed 
$\epsilon>0$ there exists $K$ such that an algorithm $\mathcal{A}$ is $(K,\epsilon(\cdot))$ robust, then $\mathcal{A}$ is weakly robust. We now give the main result of this section about the necessity of robustness.

\begin{theorem}\label{thm:weak}
Given a fixed sequence of training examples $\mathcal{T}^*$, a metric learning
method $\mathcal{A}$ generalizes with respect to $\mathcal{P}_{\mathcal{T}^*}$ if and only if
it is weakly robust with respect to $\mathcal{P}_{\mathcal{T}^*}$.
\end{theorem}
\begin{proof2}
Following \citet{Xu2012a}, the sufficiency is obtained by the fact that the testing pairs are built from a sample $\mathcal{U}(n)$ made of $n$ i.i.d. instances. We give the proof in \aref{app:proofsuff}.

For the necessity, we need the following lemma which is a direct adaptation of Lemma~2 from \citet{Xu2012a}. We provide the proof in \aref{app:prooflem1} for the sake of completeness. 
\end{proof2}
\begin{lemma}\label{lem:div}
Given $\mathcal{T}^*$, if a learning method is not weakly robust
with respect to $\mathcal{P}_{\mathcal{T}^*}$, there exist $\epsilon^*,\delta^*>0$ such that
the following holds for infinitely many $n$:
\begin{equation}\label{eq:div}
Pr(|R^\ell_{\mathcal{P}_{\mathcal{U}(n)}}(\mathcal{A}_{\mathcal{P}_{\mathcal{T}^*(n)}})-R^\ell_{\mathcal{P}_{\mathcal{T}^*(n)}}(\mathcal{A}_{\mathcal{P}_{\mathcal{T}^*(n)}})|\geq\epsilon^*)\geq \delta^*.
\end{equation}
\end{lemma}
\begin{proof3}

Now, recall that $\ell$ is nonnegative and uniformly bounded by $B$, thus by the
McDiarmid inequality (\thref{thm:McDiarmid}) we have that for any $\epsilon,\delta>0$ there
exists an index $n^*$ such that for any $n>n^*$, with probability at least
$1-\delta$, we have:
$$\left|\frac{1}{n^2}\sum_{(z_i',z_j')\in \mathcal{P}_{\mathcal{U}(n)}}\ell(\mathcal{A}_{\mathcal{P}_{\mathcal{T}^*(n)}},z_i',z_j')-R^\ell(\mathcal{A}_{\mathcal{P}_{\mathcal{T}^*(n)}})\right|\leq
\epsilon.$$
This implies the convergence 
$$R^\ell_{\mathcal{P}_{\mathcal{U}(n)}}(\mathcal{A}_{\mathcal{P}_{\mathcal{T}^*(n)}})-R^\ell(\mathcal{A}_{\mathcal{P}_{\mathcal{T}^*(n)}})\stackrel{Pr}{\rightarrow}0,$$
and thus from a given index:
\begin{equation}\label{eq:lim}
|R^\ell_{\mathcal{P}_{\mathcal{U}(n)}}(\mathcal{A}_{\mathcal{P}_{\mathcal{T}^*(n)}})-R^\ell(\mathcal{A}_{\mathcal{P}_{\mathcal{T}^*(n)}})|\leq \frac{\epsilon^*}{2}.
\end{equation}

Now, by contradiction, suppose algorithm $\mathcal{A}$ is not weakly robust, \lref{lem:div} implies \eref{eq:div} holds for infinitely many $n$. This combined with \eref{eq:lim} implies that for infinitely many $n$:
$$
|R^\ell_{\mathcal{P}_{\mathcal{U}(n)}}(\mathcal{A}_{\mathcal{P}_{\mathcal{T}^*(n)}})-R^\ell_{\mathcal{P}_{\mathcal{T}^*(n)}}(\mathcal{A}_{\mathcal{P}_{\mathcal{T}^*(n)}})|\geq \frac{\epsilon^*}{2}
$$
which means $\mathcal{A}$ does not generalize, thus the necessity of
weak robustness is established.
\end{proof3}

The following corollary follows immediately from \thref{thm:weak}.
\begin{corollary}
A metric learning method $\mathcal{A}$ generalizes with probability 1 if and
only if it is almost surely weakly robust.
\end{corollary}

This corollary establishes a strong link between generalization in metric learning and the notion of weak robustness.
In the next section, we illustrate the applicability of our framework by showing that many existing metric learning algorithms are robust in our sense.

\section{Examples of Robust Metric Learning Algorithms}
\label{sec:nipsexsec}

We first restrict our attention to Mahalanobis distance learning algorithms of the form:
\begin{eqnarray}
\label{eq:generalform}
\displaystyle\min_{\mathbf{M} \succeq 0} & \frac{1}{n^2}\displaystyle\sum_{(z_i,z_j)\in \mathcal{P}_\mathcal{T}}\ell(d_\mathbf{M}^2,z_i,z_j) \quad+\quad C\|\mathbf{M}\|,
\end{eqnarray}
where $\|\cdot\|$ is some matrix norm and $C>0$ a regularization parameter. The loss function $\ell$ is assumed to be of the form
$$\ell(d_\mathbf{M}^2,z_i,z_j) = g(y_iy_j[1-d_\mathbf{M}^2(\mathbf{x_i},\mathbf{x_j})]),$$
where $g$ is nonnegative and Lipschitz continuous with Lipschitz constant $U$. It typically outputs a small value when its input is large positive and a large value when it is large negative.
Lastly, $g_0 = \sup_{z,z'}g(y_iy_j[1-d_\mathbf{M}^2(\mathbf{0},\mathbf{x},\mathbf{x'})])$ is the largest loss when $\mathbf{M}$ is the zero matrix $\mathbf{0}$.

Recall that showing that a metric learning algorithm is robust (\defref{def:robuml}) implies that the algorithm has generalization guarantees (\thref{thm:robu}). To prove the robustness of \eqref{eq:generalform}, we will use the following theorem, which essentially says that if a metric learning algorithm achieves approximately the same testing loss for pairs that are close to each other, then it is robust.
\begin{theorem}
Fix $\gamma>0$ and a metric $\rho$ of $\mathcal{Z}$. Suppose that $\forall z_1,z_2,z,z' : (z_1,z_2)\in \mathcal{P}_\mathcal{T}, \rho(z_1,z)\leq \gamma, \rho(z_2,z')\leq \gamma$, $\mathcal{A}$ satisfies
$$|\ell(\mathcal{A}_{\mathcal{P}_\mathcal{T}},z_1,z_2)-\ell(\mathcal{A}_{\mathcal{P}_\mathcal{T}},z,z')|\leq \epsilon(\mathcal{P}_\mathcal{T}),$$
and $\mathcal{N}(\gamma/2,\mathcal{Z},\rho) < \infty$. Then $\mathcal{A}$ is $(\mathcal{N}(\gamma/2,\mathcal{Z},\rho),\epsilon(\mathcal{P}_\mathcal{T}))$-robust.
\label{thm:testtheorem}
\end{theorem}
\begin{proof}
By definition of covering number, we can partition $\mathcal{X}$ in $\mathcal{N}(\gamma/2,\mathcal{X},\rho)$ subsets such that each subset has a diameter less or equal to $\gamma$. Furthermore, since $\mathcal{Y}$ is a finite set, we can partition $\mathcal{Z}$ into $|\mathcal{Y}|\mathcal{N}(\gamma/2,\mathcal{X},\rho)$ subsets $\{C_i\}$ such that $z_1,z\in C_i \Rightarrow \rho(z_1,z)\leq \gamma$.
Therefore,
{\small
$$|\ell(\mathcal{A}_{p_\mathcal{T}},z_1,z_2)-\ell(\mathcal{A}_{p_\mathcal{T}},z,z')|\leq \epsilon(\mathcal{P}_\mathcal{T}),\quad \forall z_1,z_2,z,z' : (z_1,z_2)\in \mathcal{P}_\mathcal{T}, \rho(z_1,z)\leq \gamma, \rho(z_2,z')\leq \gamma$$
}implies
$z_1,z_2\in \mathcal{P}_\mathcal{T}, z_1,z\in C_i,z_2,z'\in C_j \Rightarrow |\ell(\mathcal{A}_{\mathcal{P}_\mathcal{T}},z_1,z_2)-\ell(\mathcal{A}_{\mathcal{P}_\mathcal{T}},z,z')|\leq \epsilon(\mathcal{P}_\mathcal{T}),$
which establishes the theorem.
\end{proof}


We now prove the robustness of \eqref{eq:generalform} when $\|\mathbf{M}\|$ is the Frobenius norm, which corresponds to the formulation \eqref{eq:jinform} addressed by \citet{Jin2009}.
\begin{example}[Frobenius norm]
Algorithm \eqref{eq:generalform} with $\|\mathbf{M}\| = \|\mathbf{M}\|_{\mathcal{F}}$ is $(|\mathcal{Y}|\mathcal{N}(\gamma/2,\mathcal{X},\|\cdot\|_2),\frac{8UR\gamma g_0}{C})$-robust.
\label{ex:ex1}
\end{example}
\begin{proof}
Let $\mathbf{M^*}$ be the solution given training data $\mathcal{P}_\mathcal{T}$. Due to optimality of $\mathbf{M^*}$, we have
{\footnotesize
$$\frac{1}{n^2}\displaystyle\sum_{(z_i,z_j)\in \mathcal{P}_\mathcal{T}}g(y_iy_j[1-d_\mathbf{M^*}^2(\mathbf{x_i},\mathbf{x_j})]) + C\|\mathbf{M^*}\|_\mathcal{F}\leq \frac{1}{n^2}\displaystyle\sum_{(z_i,z_j)\in \mathcal{P}_\mathcal{T}}g(y_iy_j[1-d_\mathbf{0}^2(\mathbf{x_i},\mathbf{x_j})]) + C\|\mathbf{0}\|_\mathcal{F} = g_0$$
}and thus $\|\mathbf{M^*}\|_{\mathcal{F}} \leq g_0/C$.

We can partition $\mathcal{Z}$ as $|\mathcal{Y}|\mathcal{N}(\gamma/2,\mathcal{X},\|\cdot\|_2)$ sets, such that if $z$ and $z'$ belong to the same set, then $y=y'$ and $\|\mathbf{x}-\mathbf{x}'\|_2 \leq \gamma$. Now, for $z_1,z_2,z_1',z_2'\in\mathcal{Z}$, if $y_1=y_1'$, $\|\mathbf{x_1}-\mathbf{x_1'}\|_2 \leq \gamma$, $y_2=y_2'$ and $\|\mathbf{x_2}-\mathbf{x_2'}\|_2 \leq \gamma$, then:
\begin{eqnarray*}
\lefteqn{|g(y_1y_2[1-d_\mathbf{M^*}^2(\mathbf{x_1},\mathbf{x_2})]) - g(y_1'y_2'[1-d_\mathbf{M^*}^2(\mathbf{x_1'},\mathbf{x_2'})])|}\\
& \leq & U|(\mathbf{x_1}-\mathbf{x_2})^T\mathbf{M^*}(\mathbf{x_1}-\mathbf{x_2})-(\mathbf{x_1'}-\mathbf{x_2'})^T\mathbf{M^*}(\mathbf{x_1'}-\mathbf{x_2'})|\\
& = & U|(\mathbf{x_1}-\mathbf{x_2})^T\mathbf{M^*}(\mathbf{x_1}-\mathbf{x_2})-(\mathbf{x_1}-\mathbf{x_2})^T\mathbf{M^*}(\mathbf{x_1'}-\mathbf{x_2'})\\
& & + ~(\mathbf{x_1}-\mathbf{x_2})^T\mathbf{M^*}(\mathbf{x_1'}-\mathbf{x_2'})|-(\mathbf{x_1'}-\mathbf{x_2'})^T\mathbf{M^*}(\mathbf{x_1'}-\mathbf{x_2'})|\\
& = & U|(\mathbf{x_1}-\mathbf{x_2})^T\mathbf{M^*}(\mathbf{x_1}-\mathbf{x_2}-(\mathbf{x_1'}+\mathbf{x_2'})) + (\mathbf{x_1}-\mathbf{x_2}-(\mathbf{x_1'}+\mathbf{x_2'}))^T\mathbf{M^*}(\mathbf{x_1'}+\mathbf{x_2'})|\\
& \leq & U (|(\mathbf{x_1}-\mathbf{x_2})^T\mathbf{M^*}(\mathbf{x_1}-\mathbf{x_1'})| + |(\mathbf{x_1}-\mathbf{x_2})^T\mathbf{M^*}(\mathbf{x_2'}-\mathbf{x_2})|\\
& & + ~|(\mathbf{x_1}-\mathbf{x_1'})^T\mathbf{M^*}(\mathbf{x_1'}+\mathbf{x_2'})| + |(\mathbf{x_2'}-\mathbf{x_2})^T\mathbf{M^*}(\mathbf{x_1'}+\mathbf{x_2'})|)\\
& \leq & U(\|\mathbf{x_1}-\mathbf{x_2}\|_2\|\mathbf{M^*}\|_{\mathcal{F}}\|\mathbf{x_1}-\mathbf{x_1'}\|_2 + \|\mathbf{x_1}-\mathbf{x_2}\|_2\|\mathbf{M^*}\|_{\mathcal{F}}\|\mathbf{x_2'}-\mathbf{x_2}\|_2\\
& & + ~\|\mathbf{x_1}-\mathbf{x_1'}\|_2\|\mathbf{M^*}\|_{\mathcal{F}}\|\mathbf{x_1'}-\mathbf{x_2'}\|_2 + \|\mathbf{x_2'}-\mathbf{x_2}\|_2\|\mathbf{M^*}\|_{\mathcal{F}}\|\mathbf{x_1'}-\mathbf{x_2'}\|_2)\\
& \leq & \frac{8UR\gamma g_0}{C}.
\end{eqnarray*}
Hence, the example holds by \thref{thm:testtheorem}.
\end{proof}

The generalization bound for \exref{ex:ex1} derived by \citet{Jin2009} using uniform stability arguments has the same order of convergence.
However, their framework cannot be used to establish generalization bounds for recent sparse metric learning approaches \citep{Rosales2006,Qi2009,Ying2009,Kunapuli2012} because sparse algorithms are known not to be stable \citep{Xu2012}.
The key advantage of robustness over stability is that it can accommodate arbitrary $p$-norms (or even any regularizer which is bounded below by some $p$-norm), thanks to the equivalence of norms. To illustrate this, we show the robustness when $\|\mathbf{M}\|$ is the $L_1$ norm \citep[used in][]{Rosales2006,Qi2009} which promotes sparsity at the component level, the $L_{2,1}$ norm \citep[used in][]{Ying2009} which induces group sparsity at the column/row level, and the trace norm \citep[used in][]{Kunapuli2012} which induces low-rank matrices.

\begin{example}[$L_1$ norm]
Algorithm \eqref{eq:generalform} with $\|\mathbf{M}\| = \|\mathbf{M}\|_1$ is $(|\mathcal{Y}|\mathcal{N}(\gamma,\mathcal{X},\|\cdot\|_1),$ $\frac{8UR\gamma g_0}{C})$-robust.
\label{ex:ex2}
\end{example}
\begin{proof}
See \aref{app:proofex2}.
\end{proof}

\begin{example}[$L_{2,1}$ norm and trace norm]
Algorithm \eqref{eq:generalform} with $\|\mathbf{M}\| = \|\mathbf{M}\|_{2,1}$ or $\|\mathbf{M}\| = \|\mathbf{M}\|_*$ is $(|\mathcal{Y}|\mathcal{N}(\gamma,\mathcal{X},\|\cdot\|_2),\frac{8UR\gamma g_0}{C})$-robust.
\label{ex:ex3}
\end{example}
\begin{proof}
See \aref{app:proofex3}.
\end{proof}

We have seen that kernelization is a convenient way to learn a nonlinear metric. In the following example, we show robustness for a kernelized formulation.
\begin{example}[Kernelization]\label{ex:kernel} Consider the kernelized version of Algorithm~\eqref{eq:generalform}:
\begin{eqnarray}
\displaystyle\min_{\mathbf{M} \succeq 0} & \frac{1}{n^2}\displaystyle\sum_{(z_i,z_j)\in \mathcal{P}_\mathcal{T}}g(y_iy_j[1-d_\mathbf{M}^2(\phi(\mathbf{x_i}),\phi(\mathbf{x_j}))]) \quad+\quad C\|\mathbf{M}\|_\mathbb{H},
\label{eq:generalform_kernel}
\end{eqnarray}
where $\phi(\cdot)$ is a feature mapping to a kernel space $\mathbb{H}$,
$\|\cdot\|_{\mathbb{H}}$ the norm function of $\mathbb{H}$ and
$k(\cdot,\cdot)$ the kernel function. 
Consider a cover of $\mathcal{X}$ by $\|\cdot\|_2$ ($\mathcal{X}$ being compact) and let
$$f_{\mathbb{H}}(\gamma)=\max_{\mathbf{a},\mathbf{b}\in \mathcal{X},
  \|\mathbf{a}-\mathbf{b}\|_2\leq \gamma}K(\mathbf{a},\mathbf{a})+K(\mathbf{b},\mathbf{b})-2K(\mathbf{a},\mathbf{b})\quad\text{and}\quad B_\gamma=\max_{x\in \mathcal{X}}\sqrt{K(\mathbf{x},\mathbf{x})}.$$
If the kernel function is
continuous, $B_\gamma$ and $f_{\mathbb{H}}$ are finite for any
$\gamma>0$ and thus Algorithm~\eqref{eq:generalform_kernel} is
$(|\mathcal{Y}|\mathcal{N}(\gamma,\mathcal{X},\|\cdot\|_2),\frac{8 U B_\gamma
  \sqrt{f_{\mathbb{H}}} g_0}{C})$-robust.
\end{example}
\begin{proof}
See \aref{app:proofexkernel}.
\end{proof}

Using triplet-based robustness \eqref{eq:robu_trip}, we can for instance show the robustness of two popular triplet-based metric learning approaches \citep{Schultz2003,Ying2009} for which no generalization guarantees were known (to the best of our knowledge). Recall that these algorithms have the following form:
\begin{eqnarray*}
\displaystyle\min_{\mathbf{M} \succeq 0} & \frac{1}{|\mathcal{R}_\mathcal{T}|}\displaystyle\sum_{(z_i,z_j,z_k)\in \mathcal{R}_\mathcal{T}} [1 - d_\mathbf{M}^2(\mathbf{x_i},\mathbf{x_k})+d_\mathbf{M}^2(\mathbf{x_i},\mathbf{x_j})]_+ \quad+\quad C\|\mathbf{M}\|,
\end{eqnarray*}
where \citet{Schultz2003} use $\|\mathbf{M}\|$ = $\|\mathbf{M}\|_{\mathcal{F}}$ and \citet{Ying2009} use $\|\mathbf{M}\| = \|\mathbf{M}\|_{1,2}$. These methods are $(\mathcal{N}(\gamma,\mathcal{Z},\|\cdot\|_2),\frac{16UR\gamma g_0}{C})$-robust (by using the same proof technique as in \exref{ex:ex1} and \exref{ex:ex3}). The additional factor 2 comes from the use of triplets instead of pairs.

Furthermore, we can easily prove similar results for other forms of metrics using the same technique. For instance, when the function is the bilinear similarity $\mathbf{x_i}^T\mathbf{M}\mathbf{x_j}$ where $\mathbf{M}$ is not constrained to be PSD \citep[see for instance][]{Chechik2009,Qamar2008,Bellet2012a}, we can improve the robustness to $2UR\gamma g_0/C$.



\section{Conclusion}
\label{sec:nipsconclu}

In this chapter, we proposed a new theoretical framework for establishing generalization bounds for metric learning algorithms, based on the notion of algorithmic robustness originally introduced by \citep{Xu2010,Xu2012a}. We showed that robustness can be adapted to pair and triplet-based metric learning and can be used to derive generalization guarantees without assuming that the pairs or triplets are drawn i.i.d. Furthermore, we showed that a weak notion of robustness characterizes the generalizability of metric learning algorithms, justifying that robustness is fundamental for such algorithms.
The proposed framework is used to derive generalization bounds for a large class of metric learning algorithms with different regularizations, such as sparsity-inducing norms, making the analysis more powerful and general than the (few) existing frameworks. 
Moreover, almost no algorithm-specific argument is needed to derive these bounds.

It is worth noting that our adaptation of robustness to metric learning is relatively straightforward: in most cases, the proof techniques of \citet{Xu2010,Xu2012a} could be reused with only slight modification. Nevertheless, this adaptation is promising since it leads to generalization bounds for many metric learning methods that could not be studied through the prism of previous frameworks. Note that it could be used to make the link between the generalization ability of metric learning methods and their $(\epsilon,\gamma,\tau)$-goodness, in a similar fashion to what we did in \cref{chap:ecml} with uniform stability. An obvious drawback of the proposed framework is that the resulting bounds are loose and often similar from one method to another due to the use of covering numbers and equivalence of norms.

A natural perspective is to consider different, harder settings. Besides extending our framework to more general loss functions \citep[for example those that use both pairs and triplets, such as][]{Weinberger2009} and regularizers \citep[e.g., the LogDet divergence used in][]{Davis2007,Jain2008}, studying other paradigms for metric learning (such as unsupervised, semi-supervised or domain adaptation methods) would be of great interest.

Lastly, another interesting avenue is to design a metric learning algorithm that would maximize the robustness of the resulting metric.

\addtocontents{toc}{\protect\vspace{14pt}}

\chapter{Conclusion \& Perspectives}
\label{chap:conclu}

In this thesis, we have addressed some important limitations of existing supervised metric learning methods by proposing new approaches for feature vectors and structured data.
We paid particular attention to the desirable properties and justifications of each contribution presented in this document. We studied both theoretical frameworks and algorithmic issues, but also the applicability of the different approaches. Overall, it constitutes a wide range of research.

Our first contribution (which was actually not a metric learning algorithm) was to propose a new string kernel built from learned edit similarities. This kernel combines powerful learned edit similarities with the classification performance of support vector machines: it is more adaptable than classic string kernels (such as the spectrum, subsequence or mismatch kernels) while being guaranteed to be PSD, unlike other kernels based on the edit distance. We provided a tractable way to compute it, although the proposed solution can remain computationally expensive.

In order to avoid the cost of transforming learned edit similarities into kernels, we then proposed to use them directly to build a linear classifier, following the framework of learning with $(\epsilon,\gamma,\tau)$-good similarity functions \citep{Balcan2006,Balcan2008,Balcan2008a}. We observed that this yields competitive results in practice. We went one step further by introducing our main second contribution with GESL, a string and tree edit similarity learning method driven by a relaxation of $(\epsilon,\gamma,\tau)$-goodness. The problem is formulated as an efficient convex quadratic program and solved by convex optimization tools, thereby avoiding the use of expensive and locally optimal EM-based algorithms. Unlike many other edit metric learning methods, we were able to use the information brought by both positive and negative pairs, and to derive generalization guarantees for the learned similarity using uniform stability arguments. These guarantees give an upper bound on the true risk of the classifier built from the learned similarity (although a rather loose one). Furthermore, experimental evaluation showed the accuracy of the method but also its ability to output sparse models, which is a valuable property from a practical point of view. Note that the source code for GESL is available and distributed under GNU/GPL 3 license.\footnote{Download from: \url{http://labh-curien.univ-st-etienne.fr/~bellet/}}

To provide a wider range of applicability, our third contribution was an extension of the ideas of GESL to metric learning from feature vectors. The proposed approach, called SLLC, takes advantage of the simple form of the bilinear similarity to efficiently optimize the actual $(\epsilon,\gamma,\tau)$-goodness, instead of only a loose upper bound in GESL. In this context, the similarity is not learned from local pairs or triplets but according to a global criterion. We also kernelized SLLC to be able to learn linear similarities in a nonlinear feature space induced by a kernel. Generalization guarantees based on uniform stability are established for SLLC and give a tighter bound on the true risk of the linear classifier. To the best of our knowledge, GESL and SLLC are the first metric learning methods for which the link between the quality of the learned metric and the error of the classifier using it is formally established.

Finally, purely on the theoretical side, our last contribution overcame the limitations of the previous frameworks studying the generalization of metric learning algorithms. It is based on a relatively straightforward adaptation of algorithmic robustness \citep{Xu2010,Xu2012a} but provides an easy way to derive nontrivial results. We illustrated this by showing how it can be used to prove the robustness of a large class of metric learning algorithms, thereby establishing generalization guarantees for methods that could not be handled with previous arguments.

Staying in the scope of the proposed methods, the adaptation to other metrics or other regularizers are possible future directions. In particular, extending the methods to sparsity-inducing regularizers (in order to obtain more interpretable results as well as additional properties such as low-rank solutions and dimensionality reduction) can be done without giving up generalization guarantees, thanks to the theoretical contribution of \cref{chap:nips}.
To improve the scalability of the approaches, an interesting avenue would be to develop online versions of the algorithms. 
Another promising idea for future work is to explore the field of information geometry, in particular to study the problem of metric learning in the context of Bregman divergences \citep{Bregman1967}. Such divergences are known to generalize many metrics for vectors and matrices, and have interesting properties for solving tasks such as clustering \citep[see e.g.,][]{Banerjee2005,Fischer2010}. To the best of our knowledge, learning Bregman divergences has only been addressed by \citet{Wu2009,Wu2012}.

From a more high-level perspective, many questions remain open as to the theoretical understanding of metric learning. Some of our contributions make the link between the learned metric and its performance in classification, but our results are so far restricted to the context of linear classification, relying on $(\epsilon,\gamma,\tau)$-goodness. A promising avenue would be to derive methods or analytical frameworks capable of making that link for other classifiers. In particular, since most learned metrics are used in $k$-NN, tying the generalization ability of the learned metric to the true risk of the $k$-NN classifier would constitute a beautiful result. One could also derive theoretically sound metric learning methods for other supervised learning tasks, such as regression or ranking, using the recently proposed generalization of the notion of similarity goodness to these settings \citep{Kar2012}.

Another interesting perspective would be to study the generalization ability of learned metrics in other settings, such as domain adaptation \citep{Mansour2009,Ben-David2010}. Domain adaptation (DA) studies the generalization ability of a hypothesis learned from labeled \emph{source} data and used to predict the labels of \emph{target} data, where the distributions generating the source and target data are different. It was shown that successful adaptation is possible when the two distributions are not too different --- a common example of such situation is covariate shift, where only the data distributions are different, while the conditional distribution of labels given a data point remains the same \citep[see for instance][and references therein]{Bickel2009}. Although a few DA metric learning methods already exist \citep{Cao2011,Geng2011}, insights provided by DA generalization bounds \citep{Mansour2009,Ben-David2010} could be used to derive theoretically well-founded approaches.

Finally, one could also focus on clustering, since metrics are essential to many clustering algorithms (such as the prominent $K$-Means). We identify two promising directions for future research. First, one could use the fact that algorithmic robustness is based on a partition of the input space. This geometric interpretation seems particularly relevant to clustering, and a metric learning algorithm that maximizes a notion of robustness could be appropriate to deal with clustering tasks. Another avenue could consist in formally determining which properties of a metric are important to induce quality clusterings. The work of \citet{Balcan2008b} is a first attempt towards a better understanding of this question.

\addtocontents{toc}{\protect\vspace{14pt}}

\begin{listpublis}
\begin{small}

\subsubsection*{International Journals}

\bibentry{Bellet2012}.

\bibentry{Bellet2010}.

\subsubsection*{International Conferences}

\bibentry{Bellet2012a}.

\bibentry{Bellet2011}.

\bibentry{Bellet2011a}.

\subsubsection*{French Conferences}

\bibentry{french3}.

\bibentry{french2}.

\bibentry{french1}.

\end{small}
\end{listpublis}

\appendix
\chapter{Learning Conditional Edit Probabilities}
\label{app:pr}

Our string edit kernel introduced in \cref{chap:pr} is based on edit probabilities learned from a generative or discriminative probabilistic model. In the experimental section, we build the kernel from the method of \citet{Oncina2006}, which is based on estimating the parameters of a conditional memoryless transducer. This appendix gives the technical details of their approach.

Recall that $\mathcal{S}$ denotes the set of positive pairs. For the sake of simplicity, we assume that the input and the output alphabets are the same, denoted by $\Sigma$.
In the following, unless stated otherwise, symbols are denoted by $\mathsf{a}, \mathsf{b}, \dots$, and pairs of input and output strings by $(\mathsf{x},\mathsf{x'})$ or $(\mathsf{w},\mathsf{w'})$ when needed.
Let $f$ be a function such that $[f(\mathsf{x})]_{\pi(\mathsf{x},\dots)}$ is equal to $f(\mathsf{x})$ if the predicate $\pi(x,\dots)$ holds and 0 otherwise, where $x$ is a (set of) dummy variable(s).
In this appendix, for notational convenience, we will see the edit probability matrix as a function. Let $c$ be the conditional probability function that returns for any edit operation $(\mathsf{b}|\mathsf{a})$ the probability to output the symbol $\mathsf{b}$ given an input symbol $\mathsf{a}$.

The aim of this appendix is to show how one can automatically learn the function $c$ from the training pairs $\mathcal{S}$. The values $c(\mathsf{b}|\mathsf{a}), \forall \mathsf{a} \in \Sigma \cup \{\$\}, \mathsf{b} \in \Sigma \cup \{\$\}$ represent the parameters of the memoryless machine $T$. These parameters are trained using an EM-based algorithm that relies on the so-called forward and backward functions.

The conditional edit probability $p_e(\mathsf{x'}|\mathsf{x})$ of the string $\mathsf{x'}$ given an input string $\mathsf{x}$ can be recursively computed using the forward function $\alpha: \Sigma^* \times \Sigma^* \to \mathbb{R}_+$ defined as follows:
\begin{eqnarray*}
  \alpha(\mathsf{x'}|\mathsf{x}) & = & [1]_{\mathsf{x} = \$ \wedge \mathsf{x'} = \$} \\ 
    && +~[c(\mathsf{b}|\mathsf{a}) \cdot \alpha(\mathsf{w'}|\mathsf{w})]_{\mathsf{x} = \mathsf{w}\mathsf{a} \wedge \mathsf{x'} = \mathsf{w'}\mathsf{b}} \\
    && +~[c(\$|\mathsf{a}) \cdot \alpha(\mathsf{x'}|\mathsf{w}]_{\mathsf{x} = \mathsf{w}\mathsf{a}} \\
    && +~[c(\mathsf{b}|\$) \cdot \alpha(\mathsf{w'}|\mathsf{x})]_{\mathsf{x'} = \mathsf{w'}\mathsf{b}}.
\end{eqnarray*}
Using $\alpha(\mathsf{x'}|\mathsf{x})$, we get 
$$p_e(\mathsf{x'}|\mathsf{x}) = c(\$|\$)\cdot \alpha(\mathsf{x'}|\mathsf{x}),$$
where $c(\$|\$)$ is the probability of the termination symbol of a string.

In a symmetric way, $p_e(\mathsf{x'}|\mathsf{x})$ can be recursively computed using the backward function $\beta: \Sigma^* \times \Sigma^* \to \mathbb{R}_+$ defined as follows:
\begin{eqnarray*}
  \beta(\mathsf{x'}|\mathsf{x}) & = & [1]_{\mathsf{x} = \$ \wedge \mathsf{x'} = \$}\\ 
    && +~[c(\mathsf{b}|\mathsf{a}) \cdot \beta(\mathsf{w'}|\mathsf{w})]_{\mathsf{x} = \mathsf{a}\mathsf{w} \wedge \mathsf{x'} = \mathsf{b}\mathsf{w'}} \\
    && +~[c(\$|\mathsf{a}) \cdot \beta(\mathsf{x'}|\mathsf{w})]_{\mathsf{x} = \mathsf{a}\mathsf{w}} \\
    && +~[c(\mathsf{b}|\$) \cdot \beta(\mathsf{w'}|\mathsf{x})]_{\mathsf{x'} = \mathsf{b}\mathsf{w'}}.
\end{eqnarray*}
And we get that
$$p_e(\mathsf{x'}|\mathsf{x}) = c(\$|\$)\cdot \beta(\mathsf{x'}|\mathsf{x}).$$

Both functions can be computed in $O(|\mathsf{x}||\mathsf{x'}|)$ time using a dynamic programming technique and will be used in the following to learn the function $c$. 

In the considered model, a probability distribution is assigned conditionally to each input string, i.e.,
\begin{equation*}
\begin{aligned}
  \sum_{\mathsf{x'} \in \Sigma^*} p(\mathsf{x'}|\mathsf{x}) \in \{0,1\} &&& \forall \mathsf{x} \in \Sigma^*.
\end{aligned}
\end{equation*}
This is equal to $0$ when the input string $\mathsf{x}$ is not in the domain 
of the function.\footnote{If $p_e(\mathsf{x}) = 0$ then $p_e(\mathsf{x},\mathsf{x'}) = 0$ and as $p_e(\mathsf{x'}|\mathsf{x}) = \frac{p_e(\mathsf{x},\mathsf{x'})}{p(\mathsf{x})}$ we have a $\frac{0}{0}$ indeterminate. We choose to avoid it by taking $\frac{0}{0} = 0$, in order to keep $\sum_{\mathsf{x'} \in \Sigma^*} p(\mathsf{x'}|\mathsf{x})$ finite.}

It can be shown \citep[see][for the proof]{Oncina2006} that correct normalization of each conditional distribution is obtained when the following conditions over the function $c$ are fulfilled:
\begin{equation*}
\begin{aligned}
  c(\$|\$) > 0,\\
  c(\mathsf{b}|\mathsf{a}), c(\mathsf{b}|\$), c(\$|\mathsf{a}) \geq 0, &&& \forall \mathsf{a} \in \Sigma, \mathsf{b} \in \Sigma,\\
  \sum_{\mathsf{b} \in \Sigma} c(\mathsf{b}|\$) + \sum_{\mathsf{b} \in \Sigma} c(\mathsf{b}|\mathsf{a}) + c(\$|\mathsf{a}) = 1, &&& \forall \mathsf{a} \in \Sigma,\\
  \sum_{\mathsf{b} \in \Sigma} c(\mathsf{b}|\$) + c(\$|\$) = 1.
\end{aligned}
\end{equation*}

The EM algorithm \citep{Dempster1977} can be used in order the find the optimal parameters of the function $c$ by alternating between an E-step and an M-step.
Given an auxiliary $(|\Sigma|+1)\times(|\Sigma|+1)$ matrix $\delta$, the E-step aims at computing the values of $\delta$ as follows: $\forall \mathsf{a} \in \Sigma, \mathsf{b} \in \Sigma$,

\begin{eqnarray*}
  \delta(\mathsf{b}|\mathsf{a}) & = & \sum_{(\mathsf{x}\mathsf{a}\mathsf{w},\mathsf{x'}\mathsf{b}\mathsf{w'}) \in \mathcal{S}} 
      \frac{\alpha(\mathsf{x'}|\mathsf{x}) \cdot c(\mathsf{b}|\mathsf{a}) \cdot \beta(\mathsf{w'}|\mathsf{w}) \cdot c(\$|\$)}{p_e(\mathsf{x'}\mathsf{b}\mathsf{w'}|\mathsf{x}\mathsf{a}\mathsf{w})}, \\
  \delta(\mathsf{b}|\$) & = & \sum_{(\mathsf{x}\mathsf{w},\mathsf{x'}\mathsf{b}\mathsf{w'}) \in \mathcal{S}} 
      \frac{\alpha(\mathsf{x'}|\mathsf{x}) \cdot c(\mathsf{b}|\$) \cdot \beta(\mathsf{w'}|\mathsf{w}) \cdot c(\$|\$)}{p_e(\mathsf{x'}\mathsf{b}\mathsf{w'}|\mathsf{x}\mathsf{w})},\\
  \delta(\$|\mathsf{a}) & = & \sum_{(\mathsf{x}\mathsf{a}\mathsf{w},\mathsf{x'}\mathsf{w'}) \in \mathcal{S}} 
      \frac{\alpha(\mathsf{x'}|\mathsf{x}) c(\$|\mathsf{a}) \cdot \beta(\mathsf{w'}|\mathsf{w}) \cdot c(\$|\$)}{p_e(\mathsf{x'}\mathsf{w'}|\mathsf{x}\mathsf{a}\mathsf{w})}, \\
  \delta(\$|\$) & = & \sum_{(\mathsf{x},\mathsf{x'}) \in \mathcal{S}} 
      \frac{\alpha(\mathsf{x'}|\mathsf{x}) \cdot c(\$|\$)}{p_e(\mathsf{x'}|\mathsf{x})} = |\mathcal{S}|.
\end{eqnarray*}

The M-step allows us to get the current edit costs:

\begin{equation*}
\begin{aligned}
 & c(\mathsf{b}|\$)  & = &&& \frac{\delta(\mathsf{b}|\$)}{N}, && \text{(insertion)} \\
 & c(\$|\$) & = &&& \frac{N - N(\$)}{N}, && \text{(termination symbol)} \\
 & c(\mathsf{b}|\mathsf{a})       & = &&& \frac{\delta(\mathsf{b}|\mathsf{a})}{N(\mathsf{a})} \cdot \frac{N - N(\$)}{N}, && \text{(substitution)} \\
 & c(\$|\mathsf{a}) & = &&& \frac{\delta(\$|\mathsf{a})}{N(\mathsf{a})} \cdot \frac{N - N(\$)}{N}, && \text{(deletion)}
\end{aligned}
\end{equation*}

where
\begin{align*}
  N = \sum_{\substack{ a \in \Sigma \cup \{\$\} \\
                       b \in \Sigma \cup \{\$\} }} \delta(b|a),&&
  N(\$) = \sum_{b \in \Sigma} \delta(b|\$),&&
  N(a) = \sum_{b \in \Sigma \cup \{\$\}} \delta(b|a).
\end{align*}

\chapter{Proofs}
\label{app:proofs}

\section[Proofs of Chapter~\ref*{chap:ecml}]{Proofs of \cref{chap:ecml}}

\subsection{Proof of \lref{lem:convexN2}}
\label{app:appendix2}

\paragraph{Lemma} Let $F_\mathcal{T}$ and $F_{\mathcal{T}^{i,z}}$ be the functions to optimize, $\mathbf{C}_\mathcal{T}$ and $\mathbf{C}_{\mathcal{T}^{i,z}}$ their corresponding minimizers, and $\beta$ the regularization parameter used in $GESL_L$. Let $\Delta \mathbf{C}=(\mathbf{C}_\mathcal{T}-\mathbf{C}_{\mathcal{T}^{i,z}})$. For any $t\in[0,1]$:
$$
 \|\mathbf{C}_\mathcal{T}\|^2_{\cal{F}}-\|\mathbf{C}_\mathcal{T} -t\Delta \mathbf{C}\|^2_{\cal{F}}+\|\mathbf{C}_{\mathcal{T}^{i,z}}\|^2_{\cal{F}}-\|\mathbf{C}_{\mathcal{T}^{i,z}} +t\Delta \mathbf{C}\|^2_{\cal{F}}  \leq \frac{(2n_\mathcal{T}+n_\mathcal{L})t2k}{\beta n_\mathcal{T}n_\mathcal{L}}\|\Delta \mathbf{C}\|_{\cal{F}}.
$$

\begin{proof}
The first steps of this proof are similar to the proof of Lemma 20 in \citep{Bousquet2002} which we recall for the sake of completeness.
Recall that any convex function $g$ verifies
$$
\forall x,y, \forall t\in[0,1], g(x+t(y-x))-g(x)\leq t(g(y)-g(x)).
$$
$R^\ell_{\mathcal{T}^{i,z}}$ is convex and thus for any $t\in [0,1]$,
\begin{equation}\label{eq:c1}
R^\ell_{\mathcal{T}^{i,z}}(\mathbf{C}_\mathcal{T}-t\Delta \mathbf{C}) -R^\ell_{\mathcal{T}^{i,z}}(\mathbf{C}_\mathcal{T})\leq t(R^\ell_{\mathcal{T}^{i,z}}(\mathbf{C}_{\mathcal{T}^{i,z}})- R^\ell_{\mathcal{T}^{i,z}}(\mathbf{C}_\mathcal{T})).
\end{equation}
Switching the role of $\mathbf{C}_\mathcal{T}$ and $\mathbf{C}_{\mathcal{T}^{i,z}}$, we get:
\begin{equation}\label{eq:c2}
R^\ell_{\mathcal{T}^{i,z}}(\mathbf{C}_{\mathcal{T}^{i,z}}+t\Delta \mathbf{C}) -R^\ell_{\mathcal{T}^{i,z}}(\mathbf{C}_{\mathcal{T}^{i,z}})\leq t(  R^\ell_{\mathcal{T}^{i,z}}(\mathbf{C}_\mathcal{T})-R^\ell_{\mathcal{T}^{i,z}}(\mathbf{C}_{\mathcal{T}^{i,z}})).
\end{equation}
Summing up  inequalities \eqref{eq:c1} and \eqref{eq:c2} yields
\begin{eqnarray}
R^\ell_{\mathcal{T}^{i,z}}(\mathbf{C}_\mathcal{T}-t\Delta \mathbf{C}) -R^\ell_{\mathcal{T}^{i,z}}(\mathbf{C}_\mathcal{T})+R^\ell_{\mathcal{T}^{i,z}}(\mathbf{C}_{\mathcal{T}^{i,z}}+t\Delta \mathbf{C}) -R^\ell_{\mathcal{T}^{i,z}}(\mathbf{C}_{\mathcal{T}^{i,z}})&\leq& 0.\label{eq:c1-2}
\end{eqnarray}
Now, since $\mathbf{C}_\mathcal{T}$ and $\mathbf{C}_{\mathcal{T}^{i,z}}$ are minimizers of $F_\mathcal{T}$ and $F_{\mathcal{T}^{i,z}}$ respectively, we have:
\begin{eqnarray}
F_\mathcal{T}(\mathbf{C}_\mathcal{T})-F_\mathcal{T}(\mathbf{C}_\mathcal{T}-t\Delta \mathbf{C})&\leq& 0 \label{eq:c3}\\ F_{\mathcal{T}^{i,z}}(\mathbf{C}_{\mathcal{T}^{i,z}})-F_{\mathcal{T}^{i,z}}(\mathbf{C}_{\mathcal{T}^{i,z}}+t\Delta \mathbf{C})&\leq& 0.\label{eq:c4}
\end{eqnarray}

By summing up  \eqref{eq:c3} and \eqref{eq:c4} we get:
\begin{eqnarray*}
R^\ell_\mathcal{T}(\mathbf{C}_\mathcal{T})+\beta \|\mathbf{C}_\mathcal{T}\|_{\cal{F}} -\left(R^\ell_\mathcal{T}(\mathbf{C}_\mathcal{T}-t\Delta \mathbf{C})+\beta \|\mathbf{C}_\mathcal{T}-t\Delta \mathbf{C}\|_{\cal{F}}\right) + \hspace*{2cm} &&\\
R^\ell_{\mathcal{T}^{i,z}}(\mathbf{C}_{\mathcal{T}^{i,z}})+\beta \|\mathbf{C}_{\mathcal{T}^{i,z}}\|_{\cal{F}}-( R^\ell_{\mathcal{T}^{i,z}}(\mathbf{C}_{\mathcal{T}^{i,z}}+t\Delta \mathbf{C})+\beta \|\mathbf{C}_{\mathcal{T}^{i,z}}+t\Delta \mathbf{C}\|_{\cal{F}}) &\leq&0.
\end{eqnarray*}
By  summing this last inequality with \eqref{eq:c1-2}, we obtain
\begin{eqnarray*}
R^\ell_\mathcal{T}(\mathbf{C}_\mathcal{T})+\beta \|\mathbf{C}_\mathcal{T}\|_{\cal{F}} -\left(R^\ell_\mathcal{T}(\mathbf{C}_\mathcal{T}-t\Delta \mathbf{C})+\beta \|\mathbf{C}_\mathcal{T}-t\Delta \mathbf{C}\|_{\cal{F}}\right) + \hspace*{2cm}&&\\
\beta \|\mathbf{C}_{\mathcal{T}^{i,z}}\|_{\cal{F}}-(\beta \|\mathbf{C}_{\mathcal{T}^{i,z}}+t\Delta \mathbf{C}\|_{\cal{F}}) + R^\ell_{\mathcal{T}^{i,z}}(\mathbf{C}_\mathcal{T}-t\Delta \mathbf{C}) -R^\ell_{\mathcal{T}^{i,z}}(\mathbf{C}_\mathcal{T})&\leq&0.\\
\end{eqnarray*}

\noindent Let $B=R^\ell_\mathcal{T}(\mathbf{C}_\mathcal{T}-t\Delta \mathbf{C})-R^\ell_{\mathcal{T}^{i,z}}(\mathbf{C}_\mathcal{T}-t\Delta \mathbf{C})-(R^\ell_\mathcal{T}(\mathbf{C}_\mathcal{T})-R^\ell_{\mathcal{T}^{i,z}}(\mathbf{C}_\mathcal{T}))$, we have then
\begin{eqnarray}\label{eq:c5}
\beta(\|\mathbf{C}_\mathcal{T}\|_{\cal{F}}-\|\mathbf{C}_\mathcal{T} -t(\Delta \mathbf{C})\|_{\cal{F}}+\|\mathbf{C}_{\mathcal{T}^{i,z}}\|_{\cal{F}}-\|\mathbf{C}_{\mathcal{T}^{i,z}} +t(\Delta \mathbf{C})\|_{\cal{F}})\leq  B.
\end{eqnarray}

\noindent We now derive a bound for $B$. In the following, $z'_{k_j}\in \mathcal{T}$ denotes the $j^{th}$ landmark associated to $z_k\in \mathcal{T}$ such that $f_{land_\mathcal{T}}(z_k,z'_{k_j})=1$ in $\mathcal{T}$, and $z'^i_{k_j}\in \mathcal{T}^{i,z}$ the $j^{th}$ landmark associated to $z^i_k\in \mathcal{T}^{i,z}$ such that $f_{land_{\mathcal{T}^{i,z}}}(z^i_k,z'^i_{k_j})=1$ in $\mathcal{T}^{i,z}$. 
\begin{eqnarray*}
 B &\leq& |R^\ell_\mathcal{T}(\mathbf{C}_\mathcal{T}-t\Delta \mathbf{C})-R^\ell_{\mathcal{T}^{i,z}}(\mathbf{C}_\mathcal{T}-t\Delta \mathbf{C})-(R^\ell_\mathcal{T}(\mathbf{C}_\mathcal{T})-R^\ell_{\mathcal{T}^{i,z}}(\mathbf{C}_\mathcal{T}))|\\
&\leq&\frac{1}{n_\mathcal{T} n_\mathcal{L}}\left| \sum_{k=1}^{n_\mathcal{T}}\sum_{j=1}^{n_\mathcal{L}} \ell(\mathbf{C}_\mathcal{T}-t\Delta \mathbf{C},z_k,z'_{k_j})-  \sum_{k=1}^{n_\mathcal{T}}\sum_{j=1}^{n_\mathcal{L}} \ell(\mathbf{C}_\mathcal{T}-t\Delta \mathbf{C},z^i_k,z'^i_{k_j}) \right.\\
&&\hspace*{2.5cm}\left.-\left(  \sum_{k=1}^{n_\mathcal{T}}\sum_{j=1}^{n_\mathcal{L}} \ell(\mathbf{C}_\mathcal{T},z_k,z'_{k_j}) -  \sum_{k=1}^{n_\mathcal{T}}\sum_{j=1}^{n_\mathcal{L}} \ell(\mathbf{C}_\mathcal{T},z^i_k,z'^i_{k_j})  \right)\right| \\
&\leq&\frac{1}{n_\mathcal{T} n_\mathcal{L}}\left| \sum_{j=1}^{n_\mathcal{L}} \left(\ell(\mathbf{C}_\mathcal{T}-t\Delta \mathbf{C},z_i,z'_{i_j})-   \ell(\mathbf{C}_\mathcal{T}-t\Delta \mathbf{C},z,z'^i_{j})\right) \right.+\\
&&\hspace*{1.5cm}\sum_{\substack{k=1\\k\neq i}}^{n_\mathcal{T}}\sum_{j=1}^{n_\mathcal{L}}
 \left(\ell(\mathbf{C}_\mathcal{T}-t\Delta \mathbf{C},z_k,z'_{k_j}) - \ell(\mathbf{C}_\mathcal{T}-t\Delta \mathbf{C},z^i_k,z'^i_{k_j})\right)\\
&&\hspace*{2.5cm}\left.-\left(  \sum_{k=1}^{n_\mathcal{T}}\sum_{j=1}^{n_\mathcal{L}} \ell(\mathbf{C}_\mathcal{T},z_k,z'_{k_j}) -  \sum_{k=1}^{n_\mathcal{T}}\sum_{j=1}^{n_\mathcal{L}} \ell(\mathbf{C}_\mathcal{T},z^i_k,z'^i_{k_j})  \right)\right| \\
\end{eqnarray*}
This inequality  is obtained by developing the sum of the first two terms of the second line. The examples $z_i$ in $\mathcal{T}$ and $z$ in $\mathcal{T}^{i,z}$ have $n_\mathcal{L}$ landmarks defined by $f_{land_\mathcal{T}}$ and $f_{land_{\mathcal{T}^{i,z}}}$ respectively.\\
Note that the samples of $n_\mathcal{T}-1$ elements  $\mathcal{T}\backslash\{z_i\}$ and  $\mathcal{T}^{i,z}\backslash\{z\}$ are the same and thus $z_k=z^i_k$ when $k\neq i$. Therefore, for any $z_k\in \mathcal{T}\backslash\{z_i\}$, the sets of landmarks $\mathcal{L}_{\mathcal{T}}^{z_k}=\{z'_{k_j}\in \mathcal{T}|f_{land_\mathcal{T}}(z_k,z'_{k_j})=1\}$ and  $\mathcal{L}_{\mathcal{T}^{i,z}}^{z_k}=\{z'^i_{k_j}\in \mathcal{T}^{i,z}|f_{land_{\mathcal{T}^{i,z}}}(z_k,z'^i_{k_j})=1\}$ differ on at most two elements, say $z_i,z'_{k_{j_2}}\in \mathcal{L}_{\mathcal{T}}^{z_k}\backslash \mathcal{L}_{\mathcal{T}^{i,z}}^{z_k}$ and $z,z'^i_{k_{j_1}}\in \mathcal{L}_{\mathcal{T}^{i,z}}^{z_k}\backslash  \mathcal{L}_{\mathcal{T}}^{z_k}$. Thus, some terms cancel out and we have:
\begin{small}
\begin{eqnarray*}
B&\leq&\frac{1}{n_\mathcal{T} n_\mathcal{L}}\left| \sum_{j=1}^{n_\mathcal{L}} \left(\ell(\mathbf{C}_\mathcal{T}-t\Delta \mathbf{C},z_i,z'_{i_j})-\ell(\mathbf{C}_\mathcal{T}-t\Delta \mathbf{C},z,z'^i_{j})\right) \right.+\sum_{\substack{k=1\\k\neq i}}^{n_\mathcal{T}}
 \left(\ell(\mathbf{C}_\mathcal{T}-t\Delta \mathbf{C},z_k,z_i)\phantom{ \ell(\mathbf{C}_\mathcal{T}-t\Delta \mathbf{C},z_k,z'^i_{k_{j_1}})}\right.\\
&&\hspace*{1.8cm} \left.- \ell(\mathbf{C}_\mathcal{T}-t\Delta \mathbf{C},z_k,z'^i_{k_{j_1}}) + \ell(\mathbf{C}_\mathcal{T}-t\Delta \mathbf{C},z_k,z'_{k_{j_2}}) - \ell(\mathbf{C}_\mathcal{T}-t\Delta \mathbf{C},z_k,z) \right)\\
&&\hspace*{4.4cm}\left.   -\left(  \sum_{k=1}^{n_\mathcal{T}}\sum_{j=1}^{n_\mathcal{L}} \ell(\mathbf{C}_\mathcal{T},z_k,z'_{k_j}) -  \sum_{k=1}^{n_\mathcal{T}}\sum_{j=1}^{n_\mathcal{L}} \ell(\mathbf{C}_\mathcal{T},z_k,z'^i_{k_j})  \right)\right| \\
\end{eqnarray*}
\end{small}The first two lines of the absolute value can be bounded by: $$(2(n_\mathcal{T}-1)+n_\mathcal{L})\sup_{\substack{z_1,z_2\in T\\z_3,z_4\in  \mathcal{T}^{i,z}}}|\ell(\mathbf{C}_\mathcal{T}-t\Delta \mathbf{C},z_1,z_2)-\ell(\mathbf{C}_\mathcal{T}-t\Delta \mathbf{C},z_3,z_4)|.$$ 
The same analysis can be done for the part in parentheses of the last line of the absolute value and we can take the pair of examples in $\mathcal{T}$ and in $\mathcal{T}^{i,z}$ maximizing the whole absolute value to obtain the next inequality:

\begin{eqnarray*}
B&\leq&\frac{2(n_\mathcal{T}-1)+n_\mathcal{L}}{n_\mathcal{T} n_\mathcal{L}} \sup_{\substack{z_1,z_2\in T\\z_3,z_4\in  \mathcal{T}^{i,z}}} \left|\ell(\mathbf{C}_\mathcal{T}-t\Delta \mathbf{C},z_1,z_2)-\ell(\mathbf{C}_\mathcal{T}-t\Delta \mathbf{C},z_3,z_4)\right.\\
&&\hspace*{5.75cm}\left.-\left(\ell(\mathbf{C}_\mathcal{T},z_1,z_2)-\ell(\mathbf{C}_\mathcal{T},z_3,z_4)\right)\right|.\\
\end{eqnarray*}
We continue by applying a reordering of the terms and  the triangular inequality to get the next result:\\
\begin{eqnarray*}
B&\leq&\frac{2(n_\mathcal{T}-1)+n_\mathcal{L}}{n_\mathcal{T} n_\mathcal{L}}\left(\sup_{z_1,z_2\in T}|\ell(\mathbf{C}_\mathcal{T}-t\Delta \mathbf{C},z_1,z_2)-\ell(\mathbf{C}_\mathcal{T},z_1,z_2)|+\right.\\
&&\left.\hspace*{4.5cm}\sup_{z_3,z_4\in  \mathcal{T}^{i,z}}|\ell(\mathbf{C}_\mathcal{T}-t\Delta \mathbf{C},z_3,z_4)-\ell(\mathbf{C}_\mathcal{T} ,z_3,z_4)|\right).\\
\end{eqnarray*}

\noindent We then use twice the k-lipschitz property of $\ell$ which leads to:\\
\begin{eqnarray*}
B&\leq&
\frac{(2n_\mathcal{T}+n_\mathcal{L})}{n_\mathcal{T} n_\mathcal{L}}2k\|-t\Delta \mathbf{C}\|_{\cal{F}}\\
&\leq&
\frac{(2n_\mathcal{T}+n_\mathcal{L})}{n_\mathcal{T} n_\mathcal{L}}t2k\|\Delta \mathbf{C}\|_{\cal{F}}.\\
\end{eqnarray*}
  Then, by applying this bound on $B$ from inequality \eqref{eq:c5}, we get the lemma.
\end{proof}

\subsection{Proof of \lref{lem:espD}}
\label{app:appendixespD}

\paragraph{Lemma} For any learning method of estimation error $D_\mathcal{T}$ and satisfying a uniform stability in $\frac{\kappa}{n_\mathcal{T}}$, we have 
$
\mathbb{E}_\mathcal{T}[D_\mathcal{T}]\leq \frac{2\kappa}{n_\mathcal{T}}.
$

\begin{proof}
First recall that for any $T, z, z'$,  by hypothesis of uniform stability we have:
$$|\ell(\mathbf{C}_\mathcal{T},z,z')   -\ell(\mathbf{C}_{\mathcal{T}^{k,z}},z,z')|\leq \sup_{z_1,z_2} |\ell(\mathbf{C}_\mathcal{T},z_1,z_2)   -\ell(\mathbf{C}_{\mathcal{T}^{k,z}},z_1,z_2)|\leq \frac{\kappa}{n_\mathcal{T}}.$$
Now, we can derive a bound for $\mathbb{E}_\mathcal{T}[D_\mathcal{T}]$.
\begin{eqnarray*}
 \mathbb{E}_\mathcal{T}[D_\mathcal{T}] &\leq&\mathbb{E}_\mathcal{T}[\mathbb{E}_{z,z'}[\ell(\mathbf{C}_\mathcal{T},z,z')]-R^\ell_\mathcal{T}(\mathbf{C}_\mathcal{T})]\\
&\leq& \displaystyle\mathbb{E}_{\substack{\mathcal{T},z,z'}}[|\ell(\mathbf{C}_\mathcal{T},z,z') - \frac{1}{n_\mathcal{T}}\sum_{k=1}^{n_\mathcal{T}}\frac{1}{n_\mathcal{L}}\sum_{j=1}^{n_\mathcal{L}} \ell(\mathbf{C}_\mathcal{T},z_k,z'_{k_j}) |]\\
&\leq& \displaystyle\mathbb{E}_{\substack{\mathcal{T},z,z'}}[|\frac{1}{n_\mathcal{T}}\sum_{k=1}^{n_\mathcal{T}}\frac{1}{n_\mathcal{L}}\sum_{j=1}^{n_\mathcal{L}}(\ell(\mathbf{C}_\mathcal{T},z,z')   -\ell(\mathbf{C}_{\mathcal{T}^{k,z}},z_k,z'_{k_j}) 
+\\
&&\hspace*{3.25cm}\ell(\mathbf{C}_{\mathcal{T}^{k,z}},z_k,z'_{k_j}) -  \ell(\mathbf{C}_\mathcal{T},z_k,z'_{k_j}) )|]\\
&\leq& \displaystyle\mathbb{E}_{\substack{\mathcal{T},z,z'}}[|\frac{1}{n_\mathcal{T}}\sum_{k=1}^{n_\mathcal{T}}\frac{1}{n_\mathcal{L}}\sum_{j=1}^{n_\mathcal{L}}(\ell(\mathbf{C}_\mathcal{T},z,z')   -\ell(\mathbf{C}_{\mathcal{T}^{k,z}},z_k,z'_{k_j}))|] +\\
&& \frac{1}{n_\mathcal{T}}\sum_{k=1}^{n_\mathcal{T}}\frac{1}{n_\mathcal{L}}\sum_{j=1}^{n_\mathcal{L}}\mathbb{E}_{\substack{\mathcal{T},z,z'}}[|
\ell(\mathbf{C}_{\mathcal{T}^{k,z}},z_k,z'_{k_j}) -  \ell(\mathbf{C}_\mathcal{T},z_k,z'_{k_j}) |]\\
&\leq& \displaystyle\mathbb{E}_{\substack{\mathcal{T},z,z'}}[|\frac{1}{n_\mathcal{T}}\sum_{k=1}^{n_\mathcal{T}}\frac{1}{n_\mathcal{L}}\sum_{j=1}^{n_\mathcal{L}}(\ell(\mathbf{C}_\mathcal{T},z,z')   -\ell(\mathbf{C}_{\mathcal{T}^{k,z}},z_k,z'_{k_j}))|] + \frac{\kappa}{n_\mathcal{T}}.\\
\end{eqnarray*}
The last inequality is obtained by applying the hypothesis of uniform stability to the second part of the sum.
Now, since $\mathcal{T}$, $z$ and $z'$ are drawn i.i.d. from distribution $P$, we do not change the expected value by replacing one point with another and thus: 
$$
\mathbb{E}_{\substack{\mathcal{T},z,z'}}[|\ell(\mathbf{C}_\mathcal{T},z,z')-   \ell(\mathbf{C}_\mathcal{T},z_k,z')|]=\mathbb{E}_{\substack{\mathcal{T},z,z'}}[|\ell(\mathbf{C}_\mathcal{T}^{z,k},z_k,z')-   \ell(\mathbf{C}_\mathcal{T},z_k,z')|].
$$
Then, by applying this trick twice on the first element of the sum:
\begin{eqnarray*}
\mathbb{E}_\mathcal{T}[D_\mathcal{T}]&\leq&  \displaystyle\mathbb{E}_{\substack{\mathcal{T},z,z'}}[|\frac{1}{n_\mathcal{T}}\sum_{k=1}^{n_\mathcal{T}}\frac{1}{n_\mathcal{L}}\sum_{j=1}^{n_\mathcal{L}}(\ell(\mathbf{C}_{\mathcal{T}^{k,z}},z_k,z')   -\ell(\mathbf{C}_{\mathcal{T}^{k,z}},z_k,z'_{k_j}))|] + \frac{\kappa}{n_\mathcal{T}}\\
&\leq&  \displaystyle\mathbb{E}_{\substack{\mathcal{T},z,z'}}[|\frac{1}{n_\mathcal{T}}\sum_{k=1}^{n_\mathcal{T}}\frac{1}{n_\mathcal{L}}\sum_{j=1}^{n_\mathcal{L}}(\ell(\mathbf{C}_{\{T^{k,z}\}^{k_j,z'}},z_k,z'_{k_j})   -\ell(\mathbf{C}_{\mathcal{T}^{k,z}},z_k,z'_{k_j}))|] + \frac{\kappa}{n_\mathcal{T}}\\
&\leq& \frac{\kappa}{n_\mathcal{T}}+\frac{\kappa}{n_\mathcal{T}},
\end{eqnarray*}
which gives the lemma.
\end{proof}

\subsection{Proof of \lref{lem:diffD}}
\label{app:appendixdiffD}

\paragraph{Lemma} For any edit cost matrix learned by $GESL_L$ using $n_\mathcal{T}$ training examples and $n_\mathcal{L}$ landmarks, and any loss function $\ell$ satisfying $(\sigma,m)$-admissibility, we have the following bound:
$$
\forall i,1\leq i\leq n_\mathcal{T},\quad\forall z,\quad |D_\mathcal{T} - D_{\mathcal{T}^{i,z}}|\leq \frac{2\kappa}{n_\mathcal{T}} +  \frac{(2n_\mathcal{T}+n_\mathcal{L})(2\sigma+m)}{n_\mathcal{T} n_\mathcal{L}}.
$$

\begin{proof}
First, we derive a bound on $|D_\mathcal{T} - D_{\mathcal{T}^{i,z}}|$.
\begin{eqnarray*}
\lefteqn{|D_\mathcal{T} - D_{\mathcal{T}^{i,z}}|}\\
&&= | R^\ell(\mathbf{C}_\mathcal{T}) -R^\ell_\mathcal{T}(\mathbf{C}_\mathcal{T}) - (R^\ell(\mathbf{C}_{\mathcal{T}^{i,z}}) - R^\ell_{\mathcal{T}^{i,z}}(\mathbf{C}_{\mathcal{T}^{i,z}}))|\\
&&= | R^\ell(\mathbf{C}_\mathcal{T}) -R^\ell_\mathcal{T}(\mathbf{C}_\mathcal{T}) - R^\ell(\mathbf{C}_{\mathcal{T}^{i,z}}) + R^\ell_{\mathcal{T}^{i,z}}(\mathbf{C}_{\mathcal{T}^{i,z}}) + R^\ell_\mathcal{T}(\mathbf{C}_{\mathcal{T}^{i,z}}) - R^\ell_\mathcal{T}(\mathbf{C}_{\mathcal{T}^{i,z}})|\\
&&= | R^\ell(\mathbf{C}_\mathcal{T}) - R^\ell(\mathbf{C}_{\mathcal{T}^{i,z}})+R^\ell_\mathcal{T}(\mathbf{C}_{\mathcal{T}^{i,z}}) -R^\ell_\mathcal{T}(\mathbf{C}_\mathcal{T})  + R^\ell_{\mathcal{T}^{i,z}}(\mathbf{C}_{\mathcal{T}^{i,z}})  - R^\ell_\mathcal{T}(\mathbf{C}_{\mathcal{T}^{i,z}})|\\
&&\leq|R^\ell(\mathbf{C}_\mathcal{T})-R^\ell(\mathbf{C}_{\mathcal{T}^{i,z}})|+|R^\ell_\mathcal{T}(\mathbf{C}_{\mathcal{T}^{i,z}})-R^\ell_\mathcal{T}(\mathbf{C}_\mathcal{T})|+
|R^\ell_{\mathcal{T}^{i,z}}(\mathbf{C}_{\mathcal{T}^{i,z}})-R^\ell_{T}(\mathbf{C}_{\mathcal{T}^{i,z}})|\\
&&\leq \mathbb{E}_{z_1,z_2}[|\ell(\mathbf{C}_\mathcal{T},z_1,z_2)-\ell(\mathbf{C}_{\mathcal{T}^{i,z}},z_1,z_2)|] + \\
&&\frac{1}{n_\mathcal{T}}\sum_{k=1}^{n_\mathcal{T}} \frac{1}{n_\mathcal{L}}\sum_{j=1}^{n_\mathcal{L}} |\ell(\mathbf{C}_{\mathcal{T}^{i,z}},z_k,z'_{k_j})-\ell(\mathbf{C}_\mathcal{T},z_k,z'_{k_j})| +|R^\ell_{\mathcal{T}^{i,z}}(\mathbf{C}_{\mathcal{T}^{i,z}})-R^\ell_{T}(\mathbf{C}_{\mathcal{T}^{i,z}})|\\
&&\leq 2\frac{\kappa}{n_\mathcal{T}}+|R^\ell_{\mathcal{T}^{i,z}}(\mathbf{C}_{\mathcal{T}^{i,z}})-R^\ell_{T}(\mathbf{C}_{\mathcal{T}^{i,z}})| \textrm{ by using the hypothesis of stability twice.}
\end{eqnarray*}
Now, proving \lref{lem:diffD} boils down to bounding the last term above. Using arguments similar to those used in the second part of the proof of \lref{lem:convexN2}, we get
$$
|R^\ell_{\mathcal{T}^{i,z}}(\mathbf{C}_{\mathcal{T}^{i,z}})-R^\ell_{T}(\mathbf{C}_{\mathcal{T}^{i,z}})|\leq\frac{(2n_\mathcal{T}+n_\mathcal{L})}{n_\mathcal{T} n_\mathcal{L}}\sup_{\substack{z_1,z_2\in T\\z_3,z_4\in  \mathcal{T}^{i,z}}}|\ell(\mathbf{C}_{\mathcal{T}^{i,z}},z_1,z_2) - \ell(\mathbf{C}_{\mathcal{T}^{i,z}},z_3,z_4)|.
$$
\noindent Now by the $(\sigma,m)$-admissibility of $\ell$, we have that: 

\begin{equation*}
|\ell(\mathbf{C}_{\mathcal{T}^{i,z}},z_1,z_2) - \ell(\mathbf{C}_{\mathcal{T}^{i,z}},z_3,z_4)|\leq
\sigma|y_1y_2-y_3y_4|+m\leq 2\sigma+m, 
\end{equation*} 
since whatever the labels, $|y_1y_2-y_3y_4|\leq 2$. This leads us to the desired result.
\end{proof}

\subsection{Proof of \lref{lem:k-lips-V}}
\label{app:appendix1}

\paragraph{Lemma} The function $\ell_{HL}$ is $k$-lipschitz with $k=W$.

\begin{proof}
We need to bound $|\ell_{HL}(\mathbf{C},z,z')-\ell_{HL}(\mathbf{C}',z,z')|$ which implies to consider two cases:  when z and z' have the same labels and when they have different labels. We consider here the first case, the second one can be easily derived from the first one ($B_1$ playing the same role as $B_2$).
\begin{small}
\begin{eqnarray*}
|\ell_{HL}(\mathbf{C},z,z')-\ell_{HL}(\mathbf{C}',z,z')|&\leq& |[\sum_{l,c} C_{l,c}\#_{l,c}(\mathsf{x},\mathsf{x'})-B_2]_+-[\sum_{l,c} C'_{l,c}\#_{l,c}(\mathsf{x},\mathsf{x'})-B_2]_+| \\
&\leq&|\sum_{l,c} C_{l,c}\#_{l,c}(\mathsf{x},\mathsf{x'})-B_2- (\sum_{l,c} C'_{l,c}\#_{l,c}(\mathsf{x},\mathsf{x'})-B_2)|\\
&\leq&|\sum_{l,c} (C_{l,c} - C'_{l,c})\#_{l,c}(\mathsf{x},\mathsf{x'})|\\
&\leq& \|\mathbf{C}-\mathbf{C'}\|_{\cal{F}}\|\boldsymbol{\#}(\mathsf{x},\mathsf{x'})\|_{\cal{F}}\\
&\leq& W \|\mathbf{C}-\mathbf{C'}\|_{\cal{F}}.
\end{eqnarray*}
\end{small}The second line is obtained by the 1-lipschitz property of the hinge loss:
$$|[U]_+-[V]_+|\leq |U-V|.$$

\noindent The fourth one comes from the Cauchy-Schwartz inequality:
$$|\sum_{i=1}^n\sum_{j=1}^m A_{i,j} B_{i,j}| \leq \|\mathbf{A}\|_{\cal{F}}\|\mathbf{B}\|_{\cal{F}}.$$

\noindent Finally, since by hypothesis $\|\boldsymbol{\#}(z,z')\|_{\cal{F}}\leq  W$, the lemma holds. 
\end{proof}

\subsection{Proof of \lref{lem:boundC}}
\label{app:appendixboundC}

\paragraph{Lemma} Let $(\mathbf{C}_\mathcal{T},B_1,B_2)$ an optimal solution learned by $GESL_{HL}$ from a training sample $\mathcal{T}$, and let $B_\gamma=max(\eta_\gamma,-log(1/2))$. Then 
$
\|\mathbf{C}_\mathcal{T}\|_{\cal{F}}\leq \sqrt{\frac{B_\gamma}{\beta}}.
$

\begin{proof}
Since $(\mathbf{C}_\mathcal{T},B_1,B_2)$ is an optimal solution, the value reached by the objective function is lower than the one obtained with  $(\boldsymbol{0},B_\gamma,0)$, where $\boldsymbol{0}$ denotes the matrix of zeros:
$$ 
\displaystyle\frac{1}{n_\mathcal{T}}\sum_{k=1}^{n_\mathcal{T}} \frac{1}{n_\mathcal{L}}\sum_{j=1}^{n_\mathcal{L}}\ell_{HL}(\mathbf{C},z_k,z'_{k_j})+\beta\| \mathbf{C}_\mathcal{T}\|^2_{\cal{F}} \leq \frac{1}{n_\mathcal{T}} \sum_{k=1}^{n_\mathcal{T}}\frac{1}{n_\mathcal{L}} \sum_{j=1}^{n_\mathcal{L}} \ell_{HL}(\boldsymbol{0},z_k,z'_{k_j})+\beta\| \boldsymbol{0}\|^2_{\cal{F}}\leq B_\gamma.
$$
For the last inequality, note that regardless of the labels of $z_k$ and $z'_{k_j}$, $\ell_{HL}(\boldsymbol{0},z_k,z'_{k_j})$ is bounded either by $B_\gamma$ or $0$.
Since
$$\frac{1}{n_\mathcal{T}}\sum_{k=1}^{n_\mathcal{T}}\frac{1}{n_\mathcal{L}}\sum_{j=1}^{n_\mathcal{L}} \ell_{HL}(\mathbf{C},z_k,z'_{k_j})\geq 0,$$
we get $\beta\|\mathbf{C}_\mathcal{T}\|^2_{\cal{F}}\leq B_\gamma$.
\end{proof}

\section[Proofs of Chapter~\ref*{chap:nips}]{Proofs of \cref{chap:nips}}

\subsection{Proof of \thref{thm:pseudorobustess} (pseudo-robustness)}
\label{app:proofpseudo}

\paragraph{Theorem} If a learning algorithm $\mathcal{A}$ is
$(K,\epsilon(\cdot),\hat{p}_n(\cdot))$ pseudo-robust and the training pairs $\mathcal{P}_\mathcal{T}$ come from a sample generated by $n$ i.i.d. draws from $P$, then
for any $\delta>0$, with probability at least $1-\delta$ we have:
$$
|R^\ell(\mathcal{A}_{\mathcal{P}_\mathcal{T}})-R^\ell_{\mathcal{P}_\mathcal{T}}(\mathcal{A}_{\mathcal{P}_\mathcal{T}})|\leq \frac{\hat{p}_n(\mathcal{P}_\mathcal{T})}{n^2}\epsilon(\mathcal{P}_\mathcal{T})+B\left(\frac{n^2-\hat{p}_n(\mathcal{P}_\mathcal{T})}{n^2}+2\sqrt{\frac{2K \ln 2 + 2\ln (1/\delta)}{n}}\right).
$$

\begin{proof}
From the proof of \thref{thm:robu}, we can easily deduce that:
\begin{eqnarray*}
\lefteqn{|R^\ell(\mathcal{A}_{\mathcal{P}_\mathcal{T}})-R^\ell_{\mathcal{P}_\mathcal{T}}(\mathcal{A}_{\mathcal{P}_\mathcal{T}})|\leq 2B\sum_{i=1}^K|\frac{|N_i|}{n}-\mu(C_i)|+}\\
&&\left|\sum_{i=1}^K\sum_{j=1}^K\mathbb{E}_{z,z'\sim
    P}[\ell(\mathcal{A}_{\mathcal{P}_\mathcal{T}},z,z')|z\in C_i,z'\in
  C_j]\frac{|N_i|}{n}\frac{|N_j|}{n}
-\frac{1}{n^2}\sum_{i=1}^n\sum_{j=1}^n\ell(\mathcal{A}_{\mathcal{P}_\mathcal{T}},z_i,z_j)\right|.
\end{eqnarray*}
Then, we have
\begin{eqnarray*}
&\leq&2B\sum_{i=1}^K|\frac{|N_i|}{n}-\mu(C_i)| +\\
&&\left|\frac{1}{n^2}\sum_{i=1}^K\sum_{j=1}^K\sum_{(z_o,z_l)\in\hat{\mathcal{P}}_\mathcal{T}}
\sum_{z_o\in N_i}\sum_{z_l\in  N_j}\max_{z\in C_i}\max_{z'\in C_j}|\ell(\mathcal{A}_{\mathcal{P}_\mathcal{T}},z,z')-\ell(\mathcal{A}_{\mathcal{P}_\mathcal{T}},z_o,z_l)|\right|+\\
&&\left|\frac{1}{n^2}\sum_{i=1}^K\sum_{j=1}^K\sum_{(z_o,z_l)\not\in\hat{\mathcal{P}}_\mathcal{T}}\sum_{z_o\in N_i}\sum_{z_l\in  N_j}\max_{z\in C_i}\max_{z'\in C_j}|\ell(\mathcal{A}_{\mathcal{P}_\mathcal{T}},z,z')- \ell(\mathcal{A}_{\mathcal{P}_\mathcal{T}},z_o,z_l)|\right|\\
&\leq&\frac{\hat{p}_n(\mathcal{P}_\mathcal{T})}{n^2}\epsilon(\mathcal{P}_\mathcal{T})+B\left(\frac{n^2-\hat{p}_n(\mathcal{P}_\mathcal{T})}{n^2}+2\sqrt{\frac{2K \ln 2 + 2\ln (1/\delta)}{n}}\right).
\end{eqnarray*}
The second inequality is obtained by the triangle inequality, the last one is obtained by the application of \propref{prop:BHC}, the hypothesis of pseudo-robustness and the fact that $\ell$ is nonnegative and bounded by $B$  and thus $|\ell(\mathcal{A}_{\mathcal{P}_\mathcal{T}},z,z')-\ell(\mathcal{A}_{\mathcal{P}_\mathcal{T}},z_o,z_l)|\leq B$.
\end{proof}

\subsection{Proof of sufficiency of \thref{thm:weak}}
\label{app:proofsuff}

\paragraph{Theorem} Given a fixed sequence of training examples $\mathcal{T}^*$, a metric learning
method $\mathcal{A}$ generalizes with respect to $\mathcal{P}_{\mathcal{T}^*}$ if and only if
it is weakly robust with respect to $\mathcal{P}_{\mathcal{T}^*}$.

\begin{proof} The proof of sufficiency corresponds to the first part of the proof of Theorem~8 of \citet{Xu2012a}. 
When  $\mathcal{A}$ is weakly robust there exists a sequence $\{\mathcal{D}_n\}$
such that for any $\delta,\epsilon>0$ there exists
$N(\delta,\epsilon)$ such that for all $n>N(\delta,\epsilon)$,
$Pr(\mathcal{U}(n)\in \mathcal{D}_n)>1-\delta$ and
\begin{equation}\label{eq:d_n}
\max_{\hat{\mathcal{T}}(n)\in \mathcal{D}_n}\left|R^\ell_{\mathcal{P}_{\hat{\mathcal{T}}(n)}}(\mathcal{A}_{\mathcal{P}_{\mathcal{T}^*(n)}})-R^\ell_{\mathcal{P}_{\mathcal{T}^*(n)}}(\mathcal{A}_{\mathcal{P}_{\mathcal{T}^*(n)}})\right|<\epsilon.
\end{equation} 
Therefore for any  $n>N(\delta,\epsilon)$,
\begin{eqnarray*}
\lefteqn{|R^\ell(\mathcal{A}_{\mathcal{P}_{\mathcal{T}^*(n)}})-R^\ell_{\mathcal{P}_{\mathcal{T}^*(n)}}(\mathcal{A}_{\mathcal{P}_{\mathcal{T}^*(n)}})|}\\
&=&|\mathbb{E}_{\mathcal{U}(n)}[R^\ell_{\mathcal{P}_{\mathcal{U}(n)}}(\mathcal{A}_{\mathcal{P}_{\mathcal{T}^*(n)}})]-R^\ell_{\mathcal{P}_{\mathcal{T}^*(n)}}(\mathcal{A}_{\mathcal{P}_{\mathcal{T}^*(n)}})|\\\
&=&|Pr(\mathcal{U}(n)\not\in \mathcal{D}_n)\mathbb{E}[R^\ell_{\mathcal{P}_{\mathcal{U}(n)}}(\mathcal{A}_{p_{\mathcal{T}^*(n)}})|\mathcal{U}(n)\not\in
\mathcal{D}_n]\\
&&+Pr(\mathcal{U}(n)\in \mathcal{D}_n)\mathbb{E}[R^\ell_{\mathcal{P}_{\mathcal{U}(n)}}(\mathcal{A}_{p_{\mathcal{T}^*(n)}})|\mathcal{U}(n)\in
\mathcal{D}_n]-
R^\ell_{\mathcal{P}_{\mathcal{T}^*(n)}}(\mathcal{A}_{p_{\mathcal{T}^*(n)}})|\\
&\leq&Pr(\mathcal{U}(n)\not\in \mathcal{D}_n)|\mathbb{E}[R^\ell_{\mathcal{P}_{\mathcal{U}(n)}}(\mathcal{A}_{p_{\mathcal{T}^*(n)}})|\mathcal{U}(n)\not\in
\mathcal{D}_n]-R^\ell_{\mathcal{P}_{\mathcal{T}^*(n)}}(\mathcal{A}_{p_{\mathcal{T}^*(n)}})|+\\
&&Pr(\mathcal{U}(n)\in \mathcal{D}_n)|\mathbb{E}[R^\ell_{\mathcal{P}_{\mathcal{U}(n)}}(\mathcal{A}_{p_{\mathcal{T}^*(n)}})|\mathcal{U}(n)\in
\mathcal{D}_n]-R^\ell_{\mathcal{P}_{\mathcal{T}^*(n)}}(\mathcal{A}_{p_{\mathcal{T}^*(n)}})|\\
&\leq&\delta
B+\max_{\hat{\mathcal{T}}(n)\in\mathcal{D}_n}|R^\ell_{\mathcal{P}_{\hat{\mathcal{T}}(n)}}(\mathcal{A}_{\mathcal{P}_{\mathcal{T}^*(n)}})-R^\ell_{\mathcal{P}_{\mathcal{T}^*(n)}}(\mathcal{A}_{\mathcal{P}_{\mathcal{T}^*(n)}})|\\ &\leq&
\delta B+\epsilon.
\end{eqnarray*}
The first inequality holds because the testing samples $\mathcal{U}(n)$
consist of $n$ instances IID from $P$. The second equality
is obtained by conditional expectation. The next inequality uses the fact that $\ell$ is nonnegative and upper bounded by $B$. Finally, we apply \eqref{eq:d_n}. 
We thus conclude that $\mathcal{A}$ generalizes for $\mathcal{P}_{\mathcal{T}^*}$
because $\epsilon$ and $\delta$ can be chosen arbitrarily.
\end{proof}

\subsection{Proof of \lref{lem:div}}
\label{app:prooflem1}

\paragraph{Lemma} Given $\mathcal{T}^*$, if a learning method is not weakly robust
with respect to $\mathcal{P}_{\mathcal{T}^*}$, there exist $\epsilon^*,\delta^*>0$ such that
the following holds for infinitely many $n$:
\begin{equation*}
Pr(|R^\ell_{\mathcal{P}_{\mathcal{U}(n)}}(\mathcal{A}_{\mathcal{P}_{\mathcal{T}^*(n)}})-R^\ell_{\mathcal{P}_{\mathcal{T}^*(n)}}(\mathcal{A}_{\mathcal{P}_{\mathcal{T}^*(n)}})|\geq\epsilon^*)\geq \delta^*.
\end{equation*}

\begin{proof}
This proof follows exactly the same principle as the proof of Lemma~2 from \citet{Xu2012a}.  
By contradiction, assume $\epsilon^*$ and $\delta^*$ do not exist. Let
$\epsilon_v=\delta_v=1/v$ for $v=1,2, ...$, then there exists a non
decreasing sequence $\{N(v)\}_{v=1}^\infty$ such that for all $v$, if
$n\geq N(v)$ then
$$Pr(|R^\ell_{\mathcal{P}_{\mathcal{U}(n)}}(\mathcal{A}_{\mathcal{P}_{\mathcal{T}^*(n)}})-R^\ell_{\mathcal{P}_{\mathcal{T}^*(n)}}(\mathcal{A}_{\mathcal{P}_{\mathcal{T}^*(n)}})|\geq
\epsilon_v)<\delta_v.$$ 
For each $n$ we define 
$$
\mathcal{D}_n^v\triangleq\{\hat{\mathcal{T}}(n)|R^\ell_{\mathcal{P}_{\hat{\mathcal{T}}(n)}}(\mathcal{A}_{\mathcal{P}_{\mathcal{T}^*(n)}})-R^\ell_{\mathcal{P}_{\mathcal{T}^*(n)}}(\mathcal{A}_{\mathcal{P}_{\mathcal{T}^*(n)}})|<\epsilon_v\}.
$$
For each $n\geq N(v)$ we have 
$$Pr(\mathcal{U}(n)\in
\mathcal{D}_n^v)=1-Pr(|R^\ell_{\mathcal{P}_{\mathcal{U}(n)}}(\mathcal{A}_{\mathcal{P}_{\mathcal{T}^*(n)}})-R^\ell_{\mathcal{P}_{\mathcal{T}^*(n)}}(\mathcal{A}_{\mathcal{P}_{\mathcal{T}^*(n)}})|\geq
\epsilon_v)>1-\delta_v.$$
For $n\geq N(1)$, define $\mathcal{D}_n\triangleq \mathcal{D}_n^{v(n)}$, where $v(n)=\max(v|N(v)\leq
n; v\leq n)$. Thus for all, $n\geq N(1)$ we have $Pr(\mathcal{U}(n)\in
\mathcal{D}_n)>1-\delta_{v(n)}$ and 
$$\sup_{\hat{\mathcal{T}}(n)\in \mathcal{D}_n}|R^\ell_{\mathcal{P}_{\hat{\mathcal{T}}(n)}}(\mathcal{A}_{\mathcal{P}_{\mathcal{T}^*(n)}})-R^\ell_{\mathcal{P}_{\mathcal{T}^*(n)}}(\mathcal{A}_{\mathcal{P}_{\mathcal{T}^*(n)}})|<\epsilon_{v(n)}.$$ 
Note that $v(n)$ tends to infinity, it follows that
$\delta_{v(n)}\rightarrow 0$ and $\epsilon_{v(n)}\rightarrow 0$. 
Therefore, $Pr(\mathcal{U}(n)\in \mathcal{D}_n)\rightarrow 1$ and 
$$
\lim_{n\rightarrow \infty}\{\sup_{\hat{\mathcal{T}}(n)\in \mathcal{D}_n}|R^\ell_{\mathcal{P}_{\hat{\mathcal{T}}(n)}}(\mathcal{A}_{\mathcal{P}_{\mathcal{T}^*(n)}})-R^\ell_{\mathcal{P}_{\mathcal{T}^*(n)}}(\mathcal{A}_{\mathcal{P}_{\mathcal{T}^*(n)}})|\}=0.
$$
That is $\mathcal{A}$ is weakly robust with respect to $\mathcal{P}_\mathcal{T}$, which is the
desired contradiction.
\end{proof}

\subsection{Proof of \exref{ex:ex2} ($L_1$ norm)}
\label{app:proofex2}

\paragraph{Example} Algorithm \eqref{eq:generalform} with $\|\mathbf{M}\| = \|\mathbf{M}\|_1$ is $(|\mathcal{Y}|\mathcal{N}(\gamma,\mathcal{X},\|\cdot\|_1),\frac{8UR\gamma g_0}{C})$-robust.

\begin{proof}
Let $\mathbf{M^*}$ be the solution given training data $\mathcal{P}_\mathcal{T}$. Due to optimality of $\mathbf{M^*}$, we have $\|\mathbf{M^*}\|_1 \leq g_0/C$.
We can partition $\mathcal{Z}$ as $|\mathcal{Y}|\mathcal{N}(\gamma/2,\mathcal{X},\|\cdot\|_1)$ sets, such that if $z$ and $z'$ belong to the same set, then $y=y'$ and $\|\mathbf{x}-\mathbf{x'}\|_1 \leq \gamma$. Now, for $z_1,z_2,z_1',z_2'\in\mathcal{Z}$, if $y_1=y'_1$, $\|\mathbf{x_1}-\mathbf{x_1'}\|_1 \leq \gamma$, $y_2=y'_2$ and $\|\mathbf{x_2}-\mathbf{x_2'}\|_1 \leq \gamma$, then:
\begin{eqnarray*}
\lefteqn{|g(y_1y_2[1-d_\mathbf{M^*}^2(\mathbf{x_1},\mathbf{x_2})]) - g(y_1'y_2'[1-d_\mathbf{M^*}^2(\mathbf{x_1'},\mathbf{x_2'})])|}\\
& \leq & U (|(\mathbf{x_1}-\mathbf{x_2})^T\mathbf{M^*}(\mathbf{x_1}-\mathbf{x_1'})| + |(\mathbf{x_1}-\mathbf{x_2})^T\mathbf{M^*}(\mathbf{x_2'}-\mathbf{x_2})|\\
& & + ~|(\mathbf{x_1}-\mathbf{x_1'})^T\mathbf{M^*}(\mathbf{x_1'}+\mathbf{x_2'})| + |(\mathbf{x_2'}-\mathbf{x_2})^T\mathbf{M^*}(\mathbf{x_1'}+\mathbf{x_2'})|)\\
& \leq & U(\|\mathbf{x_1}-\mathbf{x_2}\|_\infty\|\mathbf{M^*}\|_1\|\mathbf{x_1}-\mathbf{x_1'}\|_1 + \|\mathbf{x_1}-\mathbf{x_2}\|_\infty\|\mathbf{M^*}\|_1\|\mathbf{x_2'}-\mathbf{x_2}\|_1\\
& & + ~\|\mathbf{x_1}-\mathbf{x_1'}\|_1\|\mathbf{M^*}\|_1\|\mathbf{x_1'}-\mathbf{x_2'}\|_\infty + \|\mathbf{x_2'}-\mathbf{x_2}\|_1\|\mathbf{M^*}\|_1\|\mathbf{x_1'}-\mathbf{x_2'}\|_\infty)\\
& \leq & \frac{8UR\gamma g_0}{C}.
\end{eqnarray*}
\end{proof}

\subsection{Proof of \exref{ex:ex3} ($L_{2,1}$ norm and trace norm)}
\label{app:proofex3}

\paragraph{Example} Algorithm \eqref{eq:generalform} with $\|\mathbf{M}\| = \|\mathbf{M}\|_{2,1}$ or $\|\mathbf{M}\| = \|\mathbf{M}\|_*$ is $(|\mathcal{Y}|\mathcal{N}(\gamma,\mathcal{X},\|\cdot\|_2),\frac{8UR\gamma g_0}{C})$-robust.

\begin{proof}
We can prove the robustness for the $L_{2,1}$ norm and the trace norm in the same way. Let $\|M\|$ be either the $L_{2,1}$ norm or the trace norm and $\mathbf{M^*}$ be the solution given training data $\mathcal{P}_\mathcal{T}$. Due to optimality of $\mathbf{M^*}$, we have $\|\mathbf{M^*}\| \leq g_0/C$. We can partition $\mathcal{Z}$ in the same way as in the proof of \exref{ex:ex1} and use the inequality $\|\mathbf{M^*}\|_{\mathcal{F}} \leq \|\mathbf{M^*}\|_{2,1}$ \citep[from Theorem~3 of][]{Feng2003} for the $L_{2,1}$ norm or the well-known inequality $\|\mathbf{M^*}\|_{\mathcal{F}} \leq \|\mathbf{M^*}\|_*$ for the trace norm to derive the same bound:
\begin{eqnarray*}
\lefteqn{|g(y_1y_2[1-d_\mathbf{M^*}^2(\mathbf{x_1},\mathbf{x_2})]) - g(y_1'y_2'[1-d_\mathbf{M^*}^2(\mathbf{x_1'},\mathbf{x_2'})])|}\\
& \leq & U(\|\mathbf{x_1}-\mathbf{x_2}\|_2\|\mathbf{M^*}\|_{\mathcal{F}}\|\mathbf{x_1}-\mathbf{x_1'}\|_2 + \|\mathbf{x_1}-\mathbf{x_2}\|_2\|\mathbf{M^*}\|_{\mathcal{F}}\|\mathbf{x_2'}-\mathbf{x_2}\|_2\\
& & + ~\|\mathbf{x_1}-\mathbf{x_1'}\|_2\|\mathbf{M^*}\|_{\mathcal{F}}\|\mathbf{x_1'}-\mathbf{x_2'}\|_2 + \|\mathbf{x_2'}-\mathbf{x_2}\|_2\|\mathbf{M^*}\|_{\mathcal{F}}\|\mathbf{x_1'}-\mathbf{x_2'}\|_2)\\
& \leq & U(\|\mathbf{x_1}-\mathbf{x_2}\|_2\|\mathbf{M^*}\|\|\mathbf{x_1}-\mathbf{x_1'}\|_2 + \|\mathbf{x_1}-\mathbf{x_2}\|_2\|\mathbf{M^*}\|\|\mathbf{x_2'}-\mathbf{x_2}\|_2\\
& & + ~\|\mathbf{x_1}-\mathbf{x_1'}\|_2\|\mathbf{M^*}\|\|\mathbf{x_1'}-\mathbf{x_2'}\|_2 + \|\mathbf{x_2'}-\mathbf{x_2}\|_2\|\mathbf{M^*}\|\|\mathbf{x_1'}-\mathbf{x_2'}\|_2)\\
& \leq & \frac{8UR\gamma g_0}{C}.
\end{eqnarray*}
\end{proof}

\subsection{Proof of \exref{ex:kernel} (Kernelization)}
\label{app:proofexkernel}

\paragraph{Example}  Consider the kernelized version of Algorithm~\eqref{eq:generalform}:
\begin{eqnarray*}
\displaystyle\min_{\mathbf{M} \succeq 0} & \frac{1}{n^2}\displaystyle\sum_{(z_i,z_j)\in \mathcal{P}_\mathcal{T}}g(y_iy_j[1-d_\mathbf{M}^2(\phi(\mathbf{x_i}),\phi(\mathbf{x_j}))]) \quad+\quad C\|\mathbf{M}\|_\mathbb{H},
\end{eqnarray*}
where $\phi(\cdot)$ is a feature mapping to a kernel space $\mathbb{H}$,
$\|\cdot\|_{\mathbb{H}}$ the norm function of $\mathbb{H}$ and
$k(\cdot,\cdot)$ the kernel function. 
Consider a cover of $\mathcal{X}$ by $\|\cdot\|_2$ ($\mathcal{X}$ being compact) and let
$$f_{\mathbb{H}}(\gamma)=\max_{\mathbf{a},\mathbf{b}\in \mathcal{X},
  \|\mathbf{a}-\mathbf{b}\|_2\leq \gamma}K(\mathbf{a},\mathbf{a})+K(\mathbf{b},\mathbf{b})-2K(\mathbf{a},\mathbf{b})\quad\text{and}\quad B_\gamma=\max_{x\in \mathcal{X}}\sqrt{K(\mathbf{x},\mathbf{x})}.$$
If the kernel function is
continuous, $B_\gamma$ and $f_{\mathbb{H}}$ are finite for any
$\gamma>0$ and thus the algorithm is
$(|\mathcal{Y}|\mathcal{N}(\gamma,\mathcal{X},\|\cdot\|_2),\frac{8 U B_\gamma
  \sqrt{f_{\mathbb{H}}} g_0}{C})$-robust.

\begin{proof}
We assume $\mathbb{H}$ to be an Hilbert space with an inner product operator $\innerp{\cdot,\cdot}$. The mapping $\phi$ is continuous from $\mathcal{X}$ to $\mathbb{H}$. The norm $\|\cdot\|_{\mathbb{H}}:\mathbb{H}\rightarrow \mathbb{R}$ is defined as $\|\mathbf{x}\|_{\mathbb{H}}=\sqrt{\innerp{\mathbf{x},\mathbf{x}}}$ for all $\mathbf{x}\in \mathbb{H}$, for matrices $\|\mathbf{M}\|_{\mathbb{H}}$ we take the Frobenius norm. The kernel function is defined as $K(\mathbf{x_1},\mathbf{x_2})=\innerp{\phi(\mathbf{x_1}),\phi(\mathbf{x_2})}$. 

$B_\gamma$ and $f_{\mathbb{H}}(\gamma)$ are finite by the compactness of $\mathcal{X}$ and continuity of $K(\cdot,\cdot)$. Let $\mathbf{M^*}$ be the solution given training data $\mathcal{P}_\mathcal{T}$, by the optimality of $\mathbf{M^*}$ and using the same trick as for the previous example proofs we have $\|\mathbf{M^*}\|_{\mathbb{H}} \leq g_0/c$. 
Then, by considering a partition of $\mathcal{Z}$ into $|\mathcal{Y}|\mathcal{N}(\gamma/2,\mathcal{X},\|\cdot\|_2)$ disjoint subsets such that if $(\mathbf{x_1},y_1)$ and $(\mathbf{x_2},y_2)$ belong to the same set then $y_1=y_2$ and $\|\mathbf{x_1}-\mathbf{x_2}\|_2\leq \gamma$. 

We have:
\begin{eqnarray}
\lefteqn{|g(y_1y_2[1-d_\mathbf{M^*}^2(\phi(\mathbf{x_1}),\phi(\mathbf{x_2}))]) - g(y_1'y_2'[1-d_\mathbf{M^*}^2(\phi(\mathbf{x_1'}),\phi(\mathbf{x_2'}))])|}\nonumber\\
& \leq & U( |(\phi(\mathbf{x_1})-\phi(\mathbf{x_2}))^T\mathbf{M^*}(\phi(\mathbf{x_1})-\phi(\mathbf{x_1'}))|
+ |(\phi(\mathbf{x_1})-\phi(\mathbf{x_2}))^T\mathbf{M^*}(\phi(\mathbf{x_2'})-\phi(\mathbf{x_2}))|\nonumber\\
& & + ~|(\phi(\mathbf{x_1})-\phi(\mathbf{x_1'}))^T\mathbf{M^*}(\phi(\mathbf{x_1'})+\phi(\mathbf{x_2'}))| +
|(\phi(\mathbf{x_2'})-\phi(\mathbf{x_2}))^T\mathbf{M^*}(\phi(\mathbf{x_1'})+\phi(\mathbf{x_2'}))|)\nonumber\\
&\leq&U(|\phi(\mathbf{x_1})^T\mathbf{M^*}(\phi(\mathbf{x_1})-\phi(\mathbf{x_1'}))| +|\phi(\mathbf{x_2})^T\mathbf{M^*}(\phi(\mathbf{x_1})-\phi(\mathbf{x_1'}))|+\label{eq:l}\\
&&|\phi(\mathbf{x_1})^T\mathbf{M^*}(\phi(\mathbf{x_2'})\phi(\mathbf{x_2}))| +|\phi(\mathbf{x_2})^T\mathbf{M^*}(\phi(\mathbf{x_2'})-\phi(\mathbf{x_2}))|+\nonumber\\
&&|(\phi(\mathbf{x_1})-\phi(\mathbf{x_1'}))^T\mathbf{M^*}\phi(\mathbf{x_1'})| +|(\phi(\mathbf{x_1})-\phi(\mathbf{x_1'}))^T\mathbf{M^*}\phi(\mathbf{x_2'})|+\nonumber\\
&&|(\phi(\mathbf{x_2'})-\phi(\mathbf{x_2}))^T\mathbf{M^*}\phi(\mathbf{x_1'})| +|(\phi(\mathbf{x_2'})-\phi(\mathbf{x_2}))^T\mathbf{M^*}\phi(\mathbf{x_2'})|).\nonumber
\end{eqnarray}

Then, note that 
\begin{small}
\begin{eqnarray*}
|\phi(\mathbf{x_1})^T\mathbf{M^*}(\phi(\mathbf{x_1})-\phi(\mathbf{x_1'}))|&\leq &
\sqrt{\langle\phi(\mathbf{x_1}),\phi(\mathbf{x_1})\rangle} \|\mathbf{M}^*\|_{\mathbb{H}}\sqrt{\langle\phi(\mathbf{x_1'})-\phi(\mathbf{x_2'}),\phi(\mathbf{x_1'})-\phi(\mathbf{x_2'})\rangle}\\
&\leq& B_\gamma\frac{g_o}{C}\sqrt{f_{\mathbb{H}}(\gamma)}.
\end{eqnarray*}
\end{small}Thus, by applying the same principle to all the terms in the right part of inequality \eqref{eq:l}, we obtain:
\begin{eqnarray*}
|g(y_1y_2[1-d_\mathbf{M^*}^2(\phi(\mathbf{x_1}),\phi(\mathbf{x_2}))]) -
g(y_{ij}[1-d_\mathbf{M^*}^2(\phi(\mathbf{x_1'}),\phi(\mathbf{x_2'}))])|& \leq & \frac{8U B_\gamma \sqrt{f_{\mathbb{H}}(\gamma)}  g_0}{C}.
\end{eqnarray*}
\end{proof}

\backmatter
\bibliographystyle{ecs}
\bibliography{Thesis}
\begin{abstract}
\begin{scriptsize}

\paragraph{Abstract} In recent years, the crucial importance of metrics in machine learning
algorithms has led to an increasing interest in optimizing distance
and similarity functions using knowledge from training data to make them suitable for the problem at hand.
This area of research is known as \emph{metric learning}. Existing methods typically aim at optimizing the parameters of a given metric with respect to some local constraints over the training sample. The learned metrics are generally used in nearest-neighbor and clustering algorithms.
When data consist of feature vectors, a large body of work has focused on learning a Mahalanobis distance, which is parameterized by a positive semi-definite matrix. Recent methods offer good scalability to large datasets.
Less work has been devoted to metric learning from structured objects (such as strings or trees), because it often involves complex procedures. Most of the work has focused on optimizing a notion of edit distance, which measures (in terms of number of operations) the cost of turning an object into another.
We identify two important limitations of current supervised metric learning approaches. First, they allow to improve the performance of \emph{local} algorithms such as $k$-nearest neighbors, but metric learning for \emph{global} algorithms (such as linear classifiers) has not really been studied so far. Second, and perhaps more importantly, the question of the generalization ability of metric learning methods has been largely ignored.
In this thesis, we propose theoretical and algorithmic contributions that address these limitations. Our first contribution is the derivation of a new kernel function built from learned edit probabilities. Unlike other string kernels, it is guaranteed to be valid and parameter-free. Our second contribution is a novel framework for learning string and tree edit similarities inspired by the recent theory of $(\epsilon,\gamma,\tau)$-good similarity functions and formulated as a convex optimization problem. Using uniform stability arguments, we establish theoretical guarantees for the learned similarity that give a bound on the generalization error of a linear classifier built from that similarity. In our third contribution, we extend the same ideas to metric learning from feature vectors by proposing a bilinear similarity learning method that efficiently optimizes the $(\epsilon,\gamma,\tau)$-goodness. The similarity is learned based on global constraints that are more appropriate to linear classification. Generalization guarantees are derived for our approach, highlighting that our method minimizes a tighter bound on the generalization error of the classifier. Our last contribution is a framework for establishing generalization bounds for a large class of existing metric learning algorithms. It is based on a simple adaptation of the notion of algorithmic robustness and allows the derivation of bounds for various loss functions and regularizers.

\paragraph{R\'esum\'e} Ces derni\`eres ann\'ees, l'importance cruciale des m\'etriques en apprentissage automatique a men\'e \`a un int\'er\^et grandissant pour l'optimisation de distances et de similarit\'es en utilisant l'information contenue dans des donn\'ees d'apprentissage pour les rendre adapt\'ees au probl\`eme trait\'e. Ce domaine de recherche est souvent appel\'e \emph{apprentissage de m\'etriques}. En g\'en\'eral, les m\'ethodes existantes optimisent les param\`etres d'une m\'etrique devant respecter des contraintes locales sur les donn\'ees d'apprentissage. Les m\'etriques ainsi apprises sont g\'en\'eralement utilis\'ees dans des algorithms de plus proches voisins ou de clustering.
Concernant les donn\'ees num\'eriques, beaucoup de travaux ont port\'e sur l'apprentissage de distance de Mahalanobis, param\'etris\'ee par une matrice positive semi-d\'efinie. Les m\'ethodes r\'ecentes sont capables de traiter des jeux de donn\'ees de grande taille.
Moins de travaux ont \'et\'e d\'edi\'es \`a l'apprentissage de m\'etriques pour les donn\'ees structur\'ees (comme les cha\^ines ou les arbres), car cela implique souvent des proc\'edures plus complexes. La plupart des travaux portent sur l'optimisation d'une notion de distance d'\'edition, qui mesure (en termes de nombre d'op\'erations) le co\^ut de transformer un objet en un autre.
Au regard de l'\'etat de l'art, nous avons identifi\'e deux limites importantes des approches actuelles. Premi\`erement, elles permettent d'am\'eliorer la performance d'algorithmes \emph{locaux} comme les $k$ plus proches voisins, mais l'apprentissage de m\'etriques pour des algorithmes \emph{globaux} (comme les classifieurs lin\'eaires) n'a pour l'instant pas \'et\'e beaucoup \'etudi\'e. Le deuxi\`eme point, sans doute le plus important, est que la question de la capacit\'e de g\'en\'eralisation des m\'ethodes d'apprentissage de m\'etriques a \'et\'e largement ignor\'ee.
Dans cette th\`ese, nous proposons des contributions th\'eoriques et algorithmiques qui r\'epondent \`a ces limites. Notre premi\`ere contribution est la construction d'un nouveau noyau construit \`a partir de probabilit\'es d'\'edition apprises. A l'inverse d'autres noyaux entre cha\^ines, sa validit\'e est garantie et il ne comporte aucun param\`etre. Notre deuxi\`eme contribution est une nouvelle approche d'apprentissage de similarit\'es d'\'edition pour les cha\^ines et les arbres inspir\'ee par la th\'eorie des $(\epsilon,\gamma,\tau)$-bonnes fonctions de similarit\'e et formul\'ee comme un probl\`eme d'optimisation convexe. En utilisant la notion de stabilit\'e uniforme, nous \'etablissons des garanties th\'eoriques pour la similarit\'e apprise qui donne une borne sur l'erreur en g\'en\'eralisation d'un classifieur lin\'eaire construit \`a partir de cette similarit\'e. Dans notre troisi\`eme contribution, nous \'etendons ces principes \`a l'apprentissage de m\'etriques pour les donn\'ees num\'eriques en proposant une m\'ethode d'apprentissage de similarit\'e bilin\'eaire qui optimise efficacement l'$(\epsilon,\gamma,\tau)$-goodness. La similarit\'e est apprise sous contraintes globales, plus appropri\'ees \`a la classification lin\'eaire. Nous d\'erivons des garanties th\'eoriques pour notre approche, qui donnent de meilleurs bornes en g\'en\'eralisation pour le classifieur que dans le cas des donn\'ees structur\'ees. Notre derni\`ere contribution est un cadre th\'eorique permettant d'\'etablir des bornes en g\'en\'eralisation pour de nombreuses m\'ethodes existantes d'apprentissage de m\'etriques. Ce cadre est bas\'e sur la notion de robustesse algorithmique et permet la d\'erivation de bornes pour des fonctions de perte et des r\'egulariseurs vari\'es.

\end{scriptsize}
\end{abstract}

\end{document}